\documentclass{article}

\usepackage{transparent}
\usepackage{etoolbox}
\usepackage{inconsolata}
\usepackage[scaled=.90]{helvet}
\newcommand{\arxiv}[1]{\iftoggle{neurips}{}{#1}}
\newcommand{\neurips}[1]{\iftoggle{neurips}{#1}{}}
\newtoggle{neurips}
\global\togglefalse{neurips}
\newcommand{\loose}{\looseness=-1}

\neurips{
  \usepackage[nonatbib]{neurips_2024}
} 

\usepackage[style=alphabetic,backend=bibtex,maxcitenames=6,maxalphanames=6,maxbibnames=99]{biblatex}
\newcommand{\citep}{\cite}
\newcommand{\citet}[1]{\citeauthor*{#1} \cite{#1}}

  \addtocontents{toc}{\protect\setcounter{tocdepth}{0}}

\usepackage{mathalfa}
\usepackage{upgreek}

\usepackage{pifont}
\usepackage{caption}

\makeatletter
\DeclareFontFamily{OMX}{MnSymbolE}{}
\DeclareSymbolFont{MnLargeSymbols}{OMX}{MnSymbolE}{m}{n}
\SetSymbolFont{MnLargeSymbols}{bold}{OMX}{MnSymbolE}{b}{n}
\DeclareFontShape{OMX}{MnSymbolE}{m}{n}{
    <-6>  MnSymbolE5
   <6-7>  MnSymbolE6
   <7-8>  MnSymbolE7
   <8-9>  MnSymbolE8
   <9-10> MnSymbolE9
  <10-12> MnSymbolE10
  <12->   MnSymbolE12
}{}
\DeclareFontShape{OMX}{MnSymbolE}{b}{n}{
    <-6>  MnSymbolE-Bold5
   <6-7>  MnSymbolE-Bold6
   <7-8>  MnSymbolE-Bold7
   <8-9>  MnSymbolE-Bold8
   <9-10> MnSymbolE-Bold9
  <10-12> MnSymbolE-Bold10
  <12->   MnSymbolE-Bold12
}{}

\let\llangle\@undefined
\let\rrangle\@undefined
\DeclareMathDelimiter{\llangle}{\mathopen}%
                     {MnLargeSymbols}{'164}{MnLargeSymbols}{'164}
\DeclareMathDelimiter{\rrangle}{\mathclose}%
                     {MnLargeSymbols}{'171}{MnLargeSymbols}{'171}
\makeatother
\neurips{
\title{Reinforcement Learning under Latent Dynamics: Toward Statistical and Algorithmic Modularity}
}
\arxiv{
\title{Reinforcement Learning under Latent Dynamics:  \\ Toward Statistical and Algorithmic Modularity}
}

   \author{Philip Amortila%
\\
\normalsize
\href{mailto:philipa4@illinois.edu}{\texttt{philipa4@illinois.edu}}
\and
Dylan J. Foster
\\
\normalsize
\href{mailto:dylanfoster@microsoft.com}{\texttt{dylanfoster@microsoft.com}}
\and
Nan Jiang 
\\
\normalsize
\href{mailto:nanjiang@illinois.edu}{\texttt{nanjiang@illinois.edu}}
\and
Akshay Krishnamurthy
\\
\normalsize
\href{mailto:akshaykr@microsoft.com}{\texttt{akshaykr@microsoft.com}}
\and 
Zakaria Mhammedi
\\
\normalsize
\href{mailto:mhammedi@google.com}{\texttt{mhammedi@google.com}}
}
\date{}

\newcommand{\hatphi}{\widehat{\phi}}
\newcommand{\barphi}{\bar{\phi}}
\newcommand{\unif}{\texttt{Unif}}
\newcommand{\approxleq}{\lesssim}
\usepackage{algorithm}
\usepackage{algpseudocode}
\usepackage{longtable}%
\usepackage{tabularx}
\usepackage{pbox}
\usepackage{verbatim}
\usepackage{mathtools}
\usepackage[T1]{fontenc}
\usepackage[utf8]{inputenc}
\usepackage{amssymb}
\usepackage[normalem]{ulem}
\usepackage{soul}
\usepackage{bbm}
\usepackage{framed}
\usepackage{color}
\usepackage{amsmath}
\usepackage{amsthm}
\usepackage{thmtools, thm-restate}
\usepackage{amsmath}
\usepackage{amssymb}
\usepackage{cancel}
\usepackage{changepage}
\usepackage{bbm}
\usepackage{mathtools}
\usepackage{mathrsfs}
\usepackage{enumerate}
\usepackage{multicol}
\usepackage{upgreek}
\usepackage{bigints}
\usepackage{physics}
\usepackage{accents}
\usepackage[toc,page,title]{appendix}
\usepackage{bm}
\usepackage{enumitem}
\usepackage{booktabs}
\usepackage{graphics}
\usepackage{blindtext}
\usepackage{physics}
\usepackage{nicefrac}

\newcommand{\cA}{\mathcal{A}}

\newcommand{\cD}{\mathcal{D}}
\newcommand{\cE}{\mathcal{E}}
\newcommand{\cF}{\mathcal{F}}
\newcommand{\cG}{\mathcal{G}}

\newcommand{\cI}{\mathcal{I}}

\newcommand{\cL}{\mathcal{L}}
\newcommand{\cM}{\mathcal{M}}

\newcommand{\cO}{\mathcal{O}}

\newcommand{\cS}{\mathcal{S}}
\newcommand{\cT}{\mathcal{T}}

\newcommand{\cV}{\mathcal{V}}

\newcommand{\cX}{\mathcal{X}}
\newcommand{\cY}{\mathcal{Y}}
\newcommand{\cZ}{\mathcal{Z}}

\newcommand{\fF}{\mathfrak{F}}

\usepackage{empheq}
\usepackage[most]{tcolorbox}
\newtcbox{\mymath}[1][]{%
    nobeforeafter, math upper, tcbox raise base,
    enhanced, colframe=blue!30!black,
    colback=blue!30, boxrule=1pt,
    #1}

\DeclareMathOperator{\supp}{supp}

\newcommand{\tv}{\texttt{tv}}

\DeclareMathOperator{\poly}{\texttt{poly}}
       
   \newcommand{\clip}{\texttt{clip}_{[0,2]}}   
\newcommand{\Reg}{\texttt{Reg}}

\newcommand{\Regbase}{\texttt{Reg}_\texttt{base}}
\newcommand{\Riskbase}{\texttt{Risk}_\texttt{base}}

\newcommand{\Risk}{\texttt{Risk}}

\newcommand{\Riskstar}{\texttt{Risk}_{\star}}
\newcommand{\Riskobs}{\texttt{Risk}_{\obs}}

\newcommand{\class}{\texttt{class}}
\newcommand{\Regrep}{\texttt{Reg}_\class}
\newcommand{\Estrep}{\texttt{Est}_\class}

\newcommand{\Regsim}{\texttt{Reg}_\self}
\newcommand{\Regsimopt}{\texttt{Reg}_{\self;\opt}}

\newcommand{\Estsim}{\texttt{Est}_\self}
\newcommand{\Estsimopt}{\texttt{Est}_{\self;\opt}}

\definecolor{emerald}{rgb}{0.31, 0.78, 0.47}

\let\abs\undefined
\DeclarePairedDelimiter{\abs}{\lvert}{\rvert} %
\DeclarePairedDelimiter{\brk}{[}{]}
\DeclarePairedDelimiter{\crl}{\{}{\}}
\DeclarePairedDelimiter{\prn}{(}{)}
\DeclarePairedDelimiter{\nrm}{\|}{\|}
\DeclarePairedDelimiter{\tri}{\langle}{\rangle}
\DeclarePairedDelimiter{\dtri}{\llangle}{\rrangle}

\DeclareMathOperator{\En}{\mathbb{E}}

\DeclareMathOperator*{\argmin}{arg\,min} %
\DeclareMathOperator*{\argmax}{arg\,max}

\newcommand{\wt}[1]{\widetilde{#1}}
\newcommand{\wh}[1]{\widehat{#1}}
\newcommand{\wb}[1]{\widebar{#1}}

\def\ddefloop#1{\ifx\ddefloop#1\else\ddef{#1}\expandafter\ddefloop\fi}
\def\ddef#1{\expandafter\def\csname bb#1\endcsname{\ensuremath{\mathbb{#1}}}}
\ddefloop ABCDEFGHIJKLMNOPQRSTUVWXYZ\ddefloop
\def\ddefloop#1{\ifx\ddefloop#1\else\ddef{#1}\expandafter\ddefloop\fi}
\def\ddef#1{\expandafter\def\csname b#1\endcsname{\ensuremath{\mathbf{#1}}}}
\ddefloop ABCDEFGHIJKLMNOPQRSTUVWXYZ\ddefloop
\def\ddef#1{\expandafter\def\csname sf#1\endcsname{\ensuremath{\mathsf{#1}}}}
\ddefloop ABCDEFGHIJKLMNOPQRSTUVWXYZ\ddefloop
\def\ddef#1{\expandafter\def\csname c#1\endcsname{\ensuremath{\mathcal{#1}}}}
\ddefloop ABCDEFGHIJKLMNOPQRSTUVWXYZ\ddefloop
\def\ddef#1{\expandafter\def\csname h#1\endcsname{\ensuremath{\widehat{#1}}}}
\ddefloop ABCDEFGHIJKLMNOPQRSTUVWXYZ\ddefloop
\def\ddef#1{\expandafter\def\csname hc#1\endcsname{\ensuremath{\widehat{\mathcal{#1}}}}}
\ddefloop ABCDEFGHIJKLMNOPQRSTUVWXYZ\ddefloop
\def\ddef#1{\expandafter\def\csname t#1\endcsname{\ensuremath{\widetilde{#1}}}}
\ddefloop ABCDEFGHIJKLMNOPQRSTUVWXYZ\ddefloop
\def\ddef#1{\expandafter\def\csname tc#1\endcsname{\ensuremath{\widetilde{\mathcal{#1}}}}}
\ddefloop ABCDEFGHIJKLMNOPQRSTUVWXYZ\ddefloop
\def\ddefloop#1{\ifx\ddefloop#1\else\ddef{#1}\expandafter\ddefloop\fi}
\def\ddef#1{\expandafter\def\csname scr#1\endcsname{\ensuremath{\mathscr{#1}}}}
\ddefloop ABCDEFGHIJKLMNOPQRSTUVWXYZ\ddefloop

\newcommand{\pmo}{\crl*{\pm{}1}}

\newcommand{\veps}{\varepsilon}
\newcommand{\vepsinv}{\varepsilon^{-1}}
\newcommand{\logdelinv}{\log(\delta^{-1})}
\newcommand{\vphi}{\varphi}

\newcommand{\ldef}{\vcentcolon=}

\arxiv{
\usepackage[letterpaper, left=1in, right=1in, top=1in, bottom=1in]{geometry}
\PassOptionsToPackage{hypertexnames=false}{hyperref}  %
\usepackage{parskip}
\usepackage[colorlinks=true, linkcolor=blue!70!black, citecolor=blue!70!black,urlcolor=black,breaklinks=true]{hyperref}
}
\neurips{
\PassOptionsToPackage{hypertexnames=false}{hyperref}  %
\usepackage{parskip}
\usepackage[colorlinks=true, linkcolor=blue!70!black, citecolor=blue!70!black,urlcolor=black,breaklinks=true]{hyperref}
\PassOptionsToPackage{dvipsnames}{xcolor} 
}

\usepackage{microtype}
\usepackage{hhline}

\makeatletter
\newcommand{\neutralize}[1]{\expandafter\let\csname c@#1\endcsname\count@}
\makeatother

\usepackage{algorithm}

\usepackage{amsthm}
\usepackage{mathtools}
\usepackage{amsmath}
\usepackage{bbm}
\usepackage{amsfonts}
\usepackage{amssymb}

\usepackage{xpatch}

\usepackage{thmtools}
\usepackage{thm-restate}
\declaretheorem[name=Theorem,parent=section]{theorem}
\declaretheorem[name=Lemma,parent=section]{lemma}
\declaretheorem[name=Assumption, parent=section]{assumption}
\declaretheorem[name=Condition, parent=section]{condition}
\declaretheorem[name=Corollary, parent=section]{corollary}

\declaretheorem[name=Remark,style=definition, parent=section]{remark}
\declaretheorem[name=Proposition, parent=section]{proposition}

\makeatletter
  \renewenvironment{proof}[1][Proof]%
  {%
   \par\noindent{\bfseries\upshape {#1.}\ }%
  }%
  {\qed\newline}
  \makeatother

\theoremstyle{definition}  %

\theoremstyle{plain}
\newtheorem{definition}{Definition}[section]

\xpatchcmd{\proof}{\itshape}{\normalfont\proofnameformat}{}{}
\newcommand{\proofnameformat}{\bfseries}

\usepackage[nameinlink,capitalize]{cleveref}

\newcommand{\pref}[1]{\cref{#1}}
\newcommand{\pfref}[1]{Proof of \pref{#1}}

\renewcommand{\eqref}[1]{\texorpdfstring{\hyperref[#1]{Eq. (\ref*{#1})}}{Eq. (\ref*{#1})}}

\crefformat{equation}{#2(#1)#3}
\Crefformat{equation}{#2(#1)#3}

\Crefformat{figure}{#2Figure #1#3}
\Crefname{assumption}{Assumption}{Assumptions}
\Crefformat{assumption}{#2Assumption #1#3}
\Crefname{subsubsection}{Section}{Sections}
\crefformat{subsubsection}{#2Section #1#3}
\Crefformat{subsubsection}{#2Section #1#3}

\usepackage{crossreftools}
\usepackage{xparse}

\ExplSyntaxOn
\DeclareDocumentCommand{\XDeclarePairedDelimiter}{mm}
 {
  \__egreg_delimiter_clear_keys: %
  \keys_set:nn { egreg/delimiters } { #2 }
  \use:x %
   {
    \exp_not:n {\NewDocumentCommand{#1}{sO{}m} }
     {
      \exp_not:n { \IfBooleanTF{##1} }
       {
        \exp_not:N \egreg_paired_delimiter_expand:nnnn
         { \exp_not:V \l_egreg_delimiter_left_tl }
         { \exp_not:V \l_egreg_delimiter_right_tl }
         { \exp_not:n { ##3 } }
         { \exp_not:V \l_egreg_delimiter_subscript_tl }
       }
       {
        \exp_not:N \egreg_paired_delimiter_fixed:nnnnn 
         { \exp_not:n { ##2 } }
         { \exp_not:V \l_egreg_delimiter_left_tl }
         { \exp_not:V \l_egreg_delimiter_right_tl }
         { \exp_not:n { ##3 } }
         { \exp_not:V \l_egreg_delimiter_subscript_tl }
       }
     }
   }
 }

\keys_define:nn { egreg/delimiters }
 {
  left      .tl_set:N = \l_egreg_delimiter_left_tl,
  right     .tl_set:N = \l_egreg_delimiter_right_tl,
  subscript .tl_set:N = \l_egreg_delimiter_subscript_tl,
 }

\cs_new_protected:Npn \__egreg_delimiter_clear_keys:
 {
  \keys_set:nn { egreg/delimiters } { left=.,right=.,subscript={} }
 }

\cs_new_protected:Npn \egreg_paired_delimiter_expand:nnnn #1 #2 #3 #4
 {%
  \mathopen{}
  \mathclose\c_group_begin_token
   \left#1
   #3
   \group_insert_after:N \c_group_end_token
   \right#2
   \tl_if_empty:nF {#4} { \c_math_subscript_token {#4} }
 }
\cs_new_protected:Npn \egreg_paired_delimiter_fixed:nnnnn #1 #2 #3 #4 #5
 {
  \mathopen{#1#2}#4\mathclose{#1#3}
  \tl_if_empty:nF {#5} { \c_math_subscript_token {#5} }
 }
\ExplSyntaxOff

\XDeclarePairedDelimiter{\supnorm}{
  left=\lVert,
  right=\rVert,
  subscript=\infty
  }

\input{widebar.tex}

\usepackage{enumitem}
\neurips{
\setlist[enumerate]{leftmargin=*}
\setlist[itemize]{leftmargin=*}
}

\usepackage[suppress]{color-edits}
\addauthor{pa}{cyan}

\definecolor{ForestGreen}{RGB}{34,139,34}
\addauthor{df}{ForestGreen}

\addauthor{ak}{orange}

\addauthor{zm}{red}

\addauthor{cut}{purple}

\renewcommand{\cutedit}[1]{}

\definecolor{bananayellow}{rgb}{1.0, 0.65, 0.0}

\newcommand{\pamodel}{pushforward model\xspace}

\newcommand{\rhs}{right-hand side\xspace}
\newcommand{\lhs}{left-hand side\xspace}

\newcommand{\filt}{\mathscr{F}}

\newcommand{\bigoh}{O}
\newcommand{\bigoht}{\wt{O}}

\newcommand{\Ccov}{C_\texttt{cov}}
\newcommand{\Ccovs}{C_{\texttt{cov},\texttt{st}}}

\newcommand{\mustar}{\mu^\star} %

\newcommand{\mustarh}{\mu^\star_h}

\newcommand{\PiRNS}{\Pi_{\texttt{rns}}}

\newcommand{\Mstar}{M^{\star}}
\newcommand{\Pstar}{P^\star}
\newcommand{\Rstar}{R^\star}
\newcommand{\Pstarh}{P^\star_h}

\newcommand{\phistar}{\phi^\star}
\newcommand{\phistarh}{\phi^\star_h}
\newcommand{\phih}{\phi_h}
\newcommand{\psihpo}{\psi_{h+1}}
\newcommand{\psistar}{\psi^\star}
\newcommand{\psistarh}{\psi^\star_h}
\newcommand{\psistarhpo}{\psi^\star_{h+1}}
\newcommand{\psiinv}{\psi^{-1}}
\newcommand{\psiinvh}{\psi^{-1}_h}
\newcommand{\gammainv}{\gamma^{-1}}
\newcommand{\deltainv}{\delta^{-1}}

\newcommand{\obs}{\texttt{obs}}
\newcommand{\lat}{\texttt{lat}}

\newcommand{\Mlath}{M_{\lat,h}}

\newcommand{\Tlat}{\cT_{\lat}}
\newcommand{\Tlath}{\cT_{\lat,h}}

\newcommand{\cFalg}{\cF_{\texttt{alg}}}
\newcommand{\cGalg}{\cG_{\texttt{alg}}}

\newcommand{\dlath}{d_{\lat,h}}

\newcommand{\dlathmo}{d_{\lat,h-1}}

\newcommand{\Mstarobs}{M^\star_{\obs}}
\newcommand{\Pstarobs}{P^\star_{\obs}}
\newcommand{\Pstarobsh}{P^\star_{\obs,h}}
\newcommand{\Mstarobsh}{M^\star_{\obs,h}}

\newcommand{\rstarobs}{r^\star_{\obs}}

\newcommand{\Rstarobsh}{R^\star_{\obs,h}}
\newcommand{\Mstarlat}{M^\star_{\lat}}

\newcommand{\Pstarlath}{P^\star_{\lat,h}}
\newcommand{\Mstarlath}{M^\star_{\lat,h}}

\newcommand{\rstarlat}{r^\star_{\lat}}

\newcommand{\Qstarobsh}{Q^\star_{\obs,h}}

\newcommand{\Qstarlath}{Q^\star_{\lat,h}}
\newcommand{\Qstarlathpo}{Q^\star_{\lat,h+1}}

\newcommand{\QMstarlat}{Q^{\sMlat,\star}}

\newcommand{\Qstarlat}{Q^\star_{\lat}}
\newcommand{\Vstarlat}{V^\star_{\lat}}
\newcommand{\Qstar}{Q^\star}
\newcommand{\Vstar}{V^\star}

\newcommand{\Mobs}{M_{\obs}}
\newcommand{\Pobs}{P_{\obs}}
\newcommand{\Pobsh}{P_{\obs,h}}

\newcommand{\Robsh}{R_{\obs,h}}
\newcommand{\Mlat}{M_{\lat}}

\newcommand{\sMlat}{\sss{\Mlat}}

\newcommand{\sMstarobs}{\sss{M^\star_{\obs}}}
\newcommand{\Mobsstar}{M^\star_{\obs}}
\newcommand{\Plat}{P_{\lat}}
\newcommand{\Plath}{P_{\lat,h}}
\newcommand{\Rlat}{R_{\lat}}
\newcommand{\rlat}{r_{\lat}}
\newcommand{\rlath}{r_{\lat,h}}
\newcommand{\Rlath}{R_{\lat,h}}

\newcommand{\piobs}{{\pi_{\obs}}}
\newcommand{\pistarobs}{{\pi^\star_{\obs}}}

\newcommand{\pilat}{{\pi_{\lat}}}

\newcommand{\Var}{\texttt{Var}}

\newcommand{\Pilat}{\Pi_{\lat}} %

\newcommand{\sinit}{s_{\mathsf{root}}}%
\newcommand{\sfin}{s_{\mathsf{leaf}}}%
\newcommand{\comp}{\texttt{comp}}
\newcommand{\Mbar}{\wb{M}}
\newcommand{\Mbarlat}{\Mbar_{\lat}}%
\newcommand{\Mbarlatent}{\Mbar_{\lat}}%
\newcommand{\dec}{\mathsf{dec}_{\veps}}%
\newcommand{\Jm}{J^{\sM}}%
\newcommand{\Dhels}[2]{D^{2}_{\mathsf{H}}\prn*{#1,#2}}

\newcommand{\pim}[1][M]{\pi_{\sss{#1}}}
\newcommand{\id}{\texttt{id}}%

\newcommand{\cMlatphi}{\dtri*{\cMlat,\Phi}}

\newcommand{\Cpush}{C_{\texttt{push}}}
\newcommand{\vepsapx}{\veps_{\mathsf{apx}}}
\newcommand{\vepsrep}{\veps_{\mathsf{rep}}}

\newcommand{\hphi}{\Gamma_\phi}
\newcommand{\hPhi}{\Gamma_\Phi}

\newcommand{\hphih}{\Gamma_{\phi,h}}
\newcommand{\hphihpo}{\Gamma_{\phi,h+1}}

\newcommand{\phihat}{\wh{\phi}}
\renewcommand{\phistar}{\phi^{\star}}

\newcommand{\alg}{{\normalfont \textsc{Alg}\xspace}}

\newcommand{\self}{\texttt{self}}

\newcommand{\opt}{\texttt{opt}}

\newcommand{\obslatreduction}{\textsc{O2L}\xspace}
\newcommand{\olr}{\obslatreduction}
\newcommand{\GenReplearn}{\textsc{RepLearn}\xspace}
\newcommand{\OptReplearn}{\textnormal{\textsc{Rep$_{\self;\opt}$}}\xspace}
\newcommand{\Hsightlearn}{\textnormal{\textsc{Rep$_\class$}}\xspace}
\newcommand{\Alglat}{\textnormal{\textsc{Alg$_\lat$}}\xspace}

\newcommand{\Algsim}{\OptReplearn}
\newcommand{\alglat}{\Alglat}

\newcommand{\Golf}{\textsc{Golf}\xspace}

\newcommand{\ExpWeights}{\textnormal{\textsc{ExpWeights.Dr}}\xspace}
\newcommand{\SelfPred}{\textnormal{\textsc{SelfPredict.Opt}}\xspace}

\newcommand{\algcommentlight}[1]{\textcolor{blue!70!black}{\transparent{0.5}\scriptsize{\texttt{\textbf{//\hspace{2pt}#1}}}}}

\newcommand{\cMlat}{\cM_\lat}
\newcommand{\cLlat}{\cL_\lat}

\newcommand{\cFlat}{\cF_\lat}

\newcommand{\wtM}{\widetilde{M}}
\newcommand{\wtP}{\widetilde{P}}
\newcommand{\wtE}{\widetilde{\En}}
\newcommand{\wtr}{\tilde{r}}

\newcommand{\wtd}{\tilde{d}}
\newcommand{\wtT}{\widetilde{\cT}}
\newcommand{\Mphi}{\wt{M}^\star_\phi}
\newcommand{\Mphih}{\wt{M}^\star_{\phi,h}}

\newcommand{\Ephipi}[1]{\wt{\En}^{\pi\ind{#1}}_\phi}
\newcommand{\Ephipilat}{\wt{\En}^{\pilat}_\phi}
\newcommand{\Tphipilath}{\widetilde{\cT}^{\pilat}_{\phi,h}}
\newcommand{\Tphipih}[1]{\widetilde{\cT}^{\pi\ind{#1}}_{\phi,h}}
\newcommand{\rphipih}{\tilde{r}^{\pilat}_{\phi,h}}
\newcommand{\Pphipih}{\wt{P}^{\pi_\lat}_{\phi,h}}

\newcommand{\Pphipilat}{\wt{P}^{\pilat}_{\phi}}
\newcommand{\Tphipikh}{\widetilde{\cT}^{\pi\ind{k}}_{\phi,h}}

\newcommand{\Pphipikh}{\wt{P}^{\pi\ind{k}}_{\phi,h}}

\newcommand{\Mphipih}{\wt{M}^{\star,\pilat}_{\phi,h}}

\newcommand{\Crobust}{\texttt{CorruptionRobust}\xspace}
\newcommand{\Crobustness}{\texttt{CorruptionRobustness}\xspace}
\newcommand{\compressedPOMDP}{$\phi$-compressed POMDP}

\newcommand{\ind}[1]{^{\scriptscriptstyle{(#1)}}}
\newcommand{\sub}[1]{_{\scriptscriptstyle{#1}}}

\newcommand{\indic}{\mathbb{I}}

\newcommand{\pistar}{\pi^\star}
\newcommand{\pihat}{\wh{\pi}}
\newcommand{\hatpi}{\wh{\pi}}
\newcommand{\hatpilat}{\wh{\pi}_\lat}

\newcommand{\sss}[1]{{\scriptscriptstyle#1}}
\newcommand{\sM}{\sss{M}}
\newcommand{\sMstar}{\sss{\Mstar}}
\newcommand{\sMbar}{\sss{\Mbar}}

\begin{document}

\maketitle

\begin{abstract}
  Real-world applications of reinforcement learning often involve
  environments where agents operate on complex, high-dimensional observations, but the underlying (``latent'')
  dynamics are comparatively simple. However, outside of restrictive settings such as small latent spaces, %
  the fundamental statistical requirements and algorithmic principles
  for reinforcement learning under latent dynamics are poorly
  understood.\loose

  This paper addresses the question of reinforcement learning under
  \emph{general} latent dynamics from a
  statistical and algorithmic perspective.
  On the statistical side, our main negative
result shows that \emph{most} well-studied settings for reinforcement learning \paedit{with function approximation}
become intractable when composed with rich observations; we complement
this with a positive result, identifying \emph{latent pushforward coverability} as a
general condition that enables statistical tractability. Algorithmically,
we develop provably efficient \emph{observable-to-latent} reductions\dfedit{---that is, reductions that transform an arbitrary algorithm for the
  latent MDP into an algorithm that can operate on rich observations---}in two settings: one where the agent has access to hindsight
observations of the latent dynamics \citep{lee2023learning}, and one
where the agent can %
estimate
\paedit{\emph{self-predictive} latent models}
\citep{schwarzer2020data}. Together, our results \dfedit{serve as a
  first step toward}
a unified statistical and algorithmic theory for
reinforcement learning under latent dynamics.

\end{abstract}

\section{Introduction}\label{sec:intro}

Many application domains for reinforcement learning (RL) require the agent to operate on rich, high-dimensional observations of the environment, such as images or text \citep{wahlstrom2015pixels,levine2016end,kumar2021rma,nair2022r3m,baker2022video,brohan2022rt}. However, the environment itself can often be summarized by \emph{latent dynamics} for a low-dimensional or otherwise simple latent state space. The decoupling of latent dynamics from the complex observation process naturally suggests a \emph{modular} framework for algorithm design: first learn a representation that \emph{decodes} the latent state from observations, then apply a reinforcement learning algorithm \paedit{for the latent dynamics} on top of the learned representation. This paper investigates the algorithmic and statistical foundations of this framework. \dfedit{We ask: \emph{Can we take existing algorithms and sample complexity guarantees for reinforcement learning in the latent state space and lift them to the observation space in a modular fashion?}}\loose

There is a growing body of theoretical and empirical work
developing algorithms that combine representation learning and
reinforcement learning to develop scalable algorithms.
On the empirical side, a plethora of representation learning objectives have been deployed to varying degrees of success \citep{pathak2017curiosity,tang2017exploration,zhang2020learning,laskin2020curl,yarats2021image,lamb2022guaranteed,guo2022byol,hafner2023mastering}, but we lack a mathematical framework to systematically compare these objectives and understand when one might
be preferred to another. On the theoretical side, all existing approaches suffer
from three primary drawbacks: (a) they are tailored to restricted
classes of latent dynamics models (tabular MDPs \citep{krishnamurthy2016pac,du2019latent,misra2020kinematic,zhang2022efficient,mhammedi2023representation}, LQR \citep{dean2020certainty,mhammedi2020learning}, or factored MDPs \citep{misra2021provable}),
limiting generality; (b) the analyses, despite focusing on restrictive settings,
are unwieldy, limiting progress in algorithm development; and (c) they
are not \emph{modular}, in the sense that the representation learning
procedures are specialized to specific choices of latent reinforcement
learning algorithm, limiting ease of use.

\subsection{Contributions}

We address the aforementioned limitations by introducing a new framework, \emph{reinforcement learning under general latent dynamics}.

\paragraph{Reinforcement learning under general latent dynamics (\cref{sec:background}).}
In our framework, the agent performs control based on high-dimensional observations, but the dynamics of the environment are governed by an unobserved latent state space. Following prior work (particularly the so-called Block MDP formulation \citep{du2019latent}), we assume that the latent states can be \emph{uniquely decoded from observations}, but that the true decoder is unknown and must be learned. To aid in the decoding process, we supply the learner with a class of representations that is \emph{realizable} in the sense that it is powerful enough to represent the true decoder. Our point of departure from prior theoretical works is that we do not assume specific structure (e.g., tabular or linear dynamics) on the Markov decision process (MDP) that governs the latent dynamics. Instead, we make the minimal assumption that \paedit{the latent dynamics belong to a \emph{base MDP class which is statistically tractable}, in the sense that when the latent states are directly observed there exists \emph{some} reinforcement learning algorithm with low sample complexity that is capable of learning a near-optimal policy for every MDP in the class.} We take the first steps toward building a unified and modular theory for reinforcement learning in this setting.

\paragraph{Contributions: Statistical modularity (\cref{sec:statistical-results}).} %
A central consideration for reinforcement learning under latent dynamics is that representation learning and exploration must be intertwined: an accurate decoder is required to explore the latent state space, but exploration is required to learn an accurate decoder.
To develop provable sample complexity guarantees, one must prevent errors from compounding during this interleaving process, a challenging statistical problem which prior work addresses through strong structural assumptions on the base MDP \citep{krishnamurthy2016pac,du2019latent,misra2020kinematic,zhang2022efficient,mhammedi2023representation,dean2020certainty,mhammedi2020learning,misra2021provable}. For the general latent-dynamics setting we consider, it is unclear whether similar techniques can be applied, or whether the setting is even statistically tractable, ignoring computational considerations. Thus, our first contribution considers the question of \emph{statistical modularity}:\footnote{This question and associated definitions are restated formally in \cref{sec:statistical-modular-defn}.}
\begin{center}
\emph{If a \dfedit{base MDP} class is tractable when observed directly, is the corresponding latent-dynamics problem tractable?}
\end{center}

Statistical modularity adopts a minimax perspective by assuming that the base MDP lies in a given class, and demands that the sample complexity of the latent-dynamics setting is controlled by a natural bound on the sample complexity of the base MDP class. 
We show, perhaps surprisingly, that \emph{most} well-studied reinforcement learning settings involving function approximation \citep{li2009unifying,russo2013eluder,jiang2017contextual,sun2019model,modi2020sample,dong2019provably,ayoub2020model,wang2020provably,zhou2021nearly,weisz2021exponential, weisz2021query,du2021bilinear,jin2021bellman,foster2021statistical} do not admit statistical modularity (\cref{thm:stochastic-tree-lb}). In other words, \emph{statistical tractability of an MDP class does not extend to the latent-dynamics setting}.
We complement these negative findings with a positive result, identifying \emph{pushforward coverability} as a general structural condition on the latent dynamics that enables sample efficiency (\cref{thm:pushforwardgolf}).\loose
\akdelete{We adopt a minimax perspective by assuming that the base MDP lies in a given class and study two precise notions of statistical modularity.
First, under \emph{strong statistical modularity}, we demand that the sample complexity of the latent-dynamics setting is controlled by the \emph{minimax/optimal} sample complexity of the base MDP class.
We show, perhaps surprisingly, that \emph{most} well-studied RL settings involving function approximation \citep{russo2013eluder,jiang2017contextual,sun2019model,modi2020sample,ayoub2020model,li2009unifying,dong2019provably,wang2020provably,zhou2021nearly,du2021bilinear,jin2021bellman,foster2021statistical} do not admit strong statistical modularity (\cref{thm:stochastic-tree-lb}). In other words, \emph{statistical tractability of an MDP class does not extend to the latent-dynamics setting}.
Next, in \emph{weak statistical modularity}, we demand that sample complexity of the latent-dynamics setting is controlled by an \emph{upper bound} on that of the base MDP class.
Here, we obtain a positive result, identifying \emph{pushforward coverability} as a general structural parameter on the latent dynamics that enables sample efficiency (\cref{thm:pushforwardgolf}).
To our knowledge, this provides the most general known setting for which RL under latent dynamics is statistically tractable.}

\paragraph{Contributions: Algorithmic modularity (\cref{sec:algorithmic-results}).}
Beyond developing a modular understanding of the statistical landscape, we investigate \emph{modular algorithm design principles} for RL under general latent
dynamics. Specifically, we consider the question of \emph{observable-to-latent
reductions}, whereby RL under latent dynamics can be reduced to the simpler problem of RL with latent states directly observed: 
\begin{center}
\emph{Can we generically lift algorithms for a \dfedit{base MDP} class to solve the corresponding latent-dynamics problem?}
\end{center}

This property, which we refer to as \textit{algorithmic modularity}, enables modular, greatly simplified
algorithm design, allowing one to use an arbitrary base algorithm for the base MDP class to solve the corresponding latent-dynamics problem. Algorithmic modularity is a stronger property than mere statistical
modularity, and thus is subject to our statistical lower bound.
Accordingly, we consider two settings
that sidestep the lower bound through additional feedback and modeling assumptions.  Our first algorithmic result considers \emph{hindsight observability} \citep{lee2023learning},
where latent states are revealed during training, but not at
deployment (\cref{thm:hindsightreduction}). Our second considers stronger function
approximation conditions that enable the estimation of \emph{self-predictive latent models} \citep{schwarzer2020data} through representation learning (\cref{thm:online-reduction-main}). Both results are \emph{fully modular}: they
transform \emph{any} sample-efficient algorithm for the base MDP class into a sample-efficient algorithm for the latent-dynamics
setting. \akedit{Thus, they constitute the first \emph{general-purpose}
algorithms for RL under latent dynamics.}

Together, we believe our results can serve as a foundation for further development of practical, general-purpose algorithms for RL under latent dynamics. To this end, we highlight a number of fascinating and challenging open problems for future research (\cref{sec:discussion}).

\section{Reinforcement Learning under General Latent Dynamics}\label{sec:background}

In this section we formally introduce our framework, \emph{reinforcement
learning under general latent dynamics.}\loose

\arxiv{\subsection{MDP preliminaries}}
\neurips{\paragraph{MDP preliminaries.}}

 We consider an episodic finite-horizon online reinforcement
  learning setting. With $H$ denoting the horizon, a Markov decision process (MDP) $\Mstar=\crl*{\cX, \cA,
  \crl{\Pstar_h}_{h=0}^{H},\crl{\Rstar_h}_{h=1}^{H}, H}$ consists of a
  state space $\cX$, an action space $\cA$, a reward distribution $\Rstar_h:\cX\times\cA\to\Delta([0,1])$ (with expectation $r^\star_h(x,a)$), and a transition kernel
  $\Pstar_h:\cX\times\cA\to\Delta(\cX)$ (with the convention that
  $\Pstar_0(\cdot\mid{}\emptyset)$ is the initial state distribution).\neurips{\footnote{To simplify presentation, we assume that $\cX$ and $\cA$ are
countable; our results extend to handle continuous variables with an appropriate measure-theoretic
treatment.}} \arxiv{

}
  At the beginning of the episode, the learner selects a
  randomized, non-stationary \emph{policy}
  $\pi=(\pi_1,\ldots,\pi_H)$, where $\pi_h:\cX\to\Delta(\cA)$; we let
  $\PiRNS$ denote the set of all such policies.
  The episode evolves through the following process; beginning from
  $x_1\sim{}\Pstar_0(\cdot\mid{}\emptyset)$, the MDP
  generates a trajectory $(x_1,a_1,r_1),\ldots,(x_H,a_H,r_H)$ via
  $a_h\sim\pi_h(x_h)$,   $r_h\sim\Rstar_h(x_h,a_h)$, and
  $x_{h+1}\sim{}\Pstarh(\cdot\mid{}x_h,a_h)$. We let
  $\bbP^{\sMstar,\pi}$ denote the law under this process, and let
  $\En^{\sMstar,\pi}$ denote the corresponding expectation, and likewise let $\bbP^{\sM,\pi}$ and $\En^{\sM,\pi}$ denote the analogous laws and expectations in another MDP $M$.  \neurips{We assume that $\sum_{h=1}^H r_h \in [0,1]$ almost surely for any trajectory in $\Mstar$.}
  
  For a policy \(\pi\) and MDP $M$, the expected reward for $\pi$ is given by $J^{\sM}(\pi) \coloneqq \bbE^{\sM,\pi}\brk[\big]{ \sum_{h=1}^H r_h}$, and the value functions are given by 
\arxiv{
\[
V_h^{\sM,\pi}(x) \coloneqq \bbE^{\sM,\pi}\brk*{\sum_{h'=h}^H r_{h'} \mid x_h=x}, \quad \text{ and } \quad Q_h^{\sM,\pi}(x,a) \coloneqq \bbE^{\sM,\pi}\brk*{\sum_{h'=h}^H r_{h'} \mid x_h=x,a_h=a}.
\]
}
\neurips{
$
V_h^{\sM,\pi}(x) \coloneqq \bbE^{\sM,\pi}\brk[\big]{\sum_{h'=h}^H r_{h'} \mid x_h=x}$, and $Q_h^{\sM,\pi}(x,a) \coloneqq \bbE^{\sM,\pi}\brk[\big]{\sum_{h'=h}^H r_{h'} \mid x_h=x,a_h=a}.
$
}
We let $\pi_{\sM} = \{\pi_{\sM,h}\}_{h=1}^H$ denote an optimal deterministic policy of $M$, which maximizes $V^{\sM,\pi}$ (over $\pi$) at all states (and in particular, satisfies $\pi_{\sM}\in\argmax_{\pi\in\PiRNS}J^{\sM}(\pi)$), and write $Q^{\sM,\star} \coloneqq Q^{\sM,\pi_M}$. For $f: \cX \times \cA \rightarrow \bbR$, we write $\pi_f(x) \coloneqq \argmax_a f(x,a)$ as well as $V_f(x) = \max_a f(x,a)$. For MDP $M$, horizon $h \in [H]$, and $g: \cX \rightarrow \bbR$, we let $\cT^\sM_h$ denote the Bellman (optimality) operator defined via 
	\arxiv{
	\[
		[\cT^{\sM}_h g](x,a) = \En^\sM\brk*{r_h+g(x_{h+1})\mid{}x_h=x,a_h=a},
	\]
	}
		\neurips{
	$
		[\cT^{\sM}_h g](x,a) = \En^\sM\brk*{r_h+g(x_{h+1})\mid{}x_h=x,a_h=a},
	$
	}
	and we overload notation by letting $[\cT^{\sM}_h f](x,a) = \En^\sM\brk*{r_h+V_f(x_{h+1})\mid{}x_h=x,a_h=a}$. %
	We also let $\cT^{\sM,\pi}_h$ denote the Bellman \emph{evaluation} operator defined via
	\arxiv{
	\[
		[\cT^{\sM,\pi}_h f](x,a) = \En^\sM\brk*{r_h+ \En_{a' \sim \pi_{h+1}(\cdot \mid x_{h+1})}\brk{f(x_{h+1},a')}\mid{}x_h=x,a_h=a},
	\]
	}
	\neurips{
	$
		[\cT^{\sM,\pi}_h f](x,a) = \En^\sM\brk*{r_h+ \En_{a' \sim \pi_{h+1}(\cdot \mid x_{h+1})}\brk{f(x_{h+1},a')}\mid{}x_h=x,a_h=a},
	$
	}
	for any $\pi \in \PiRNS$.
        We define the \arxiv{induced}
\emph{occupancy measures} for layer $h$ via \arxiv{
\[
d_h^{\sM,\pi}(x) = \bbP^{\sM,\pi}\brk{x_h=x} \quad \& \quad d_h^{\sM,\pi}(x,a) = \bbP^{\sM,\pi}\brk{x_h=x,a_h=a}.
\]
}
\neurips{
$
d_h^{\sM,\pi}(x) = \bbP^{\sM,\pi}\brk{x_h=x}$ and $d_h^{\sM,\pi}(x,a) = \bbP^{\sM,\pi}\brk{x_h=x,a_h=a}
$.\loose
}

\paragraph{Online reinforcement learning.} In online reinforcement
learning, the learning algorithm $\alg$ repeatedly interacts with an unknown MDP $\Mstar$ by executing a
policy and observing the resulting trajectory. After $T$ rounds of
interaction, the algorithm outputs a final policy $\pihat$, with the
goal of minimizing their \emph{risk}, defined via \neurips{$\Risk(T,\alg,\Mstar) \coloneqq J^\sMstar(\pi_\sMstar) - J^\sMstar(\hatpi)$.} \arxiv{\begin{equation}\label{eq:risk-alg}
	\Risk(T,\alg,\Mstar) \coloneqq J^\sMstar(\pi_\sMstar) - J^\sMstar(\hatpi).
	\end{equation}
	}

\arxiv{\paragraph{Additional assumptions.} To simplify presentation, we assume that $\cX$ and $\cA$ are
countable; we expect that our results extend to continuous variables with an appropriate measure-theoretic
treatment. We assume that $\sum_{h=1}^H r_h \in [0,1]$ almost surely for any trajectory in $\Mstar$. }

\arxiv{\subsection{Framework: Reinforcement learning under general latent dynamics}\label{sec:romdp-definition}}
\neurips{\paragraph{Framework: Reinforcement learning under general latent dynamics.}}
In \emph{reinforcement learning under general latent dynamics}, we consider MDPs $\Mstar$ where the dynamics
are governed by the evolution of an unobserved \emph{latent state}
$s_h$, while the agent observes and acts on \emph{observations} $x_h$ generated
from these latent states. Formally, a \emph{latent-dynamics MDP}
consists of two ingredients: a \emph{base MDP} $\Mlat =
  \{\cS,\cA,\{\Plath\}_{h=0}^H,\{\Rlath\}_{h=1}^H,H\}$ defined over
a \emph{latent state space} $\cS$, and a \emph{decodable emission
  process} $\psi := \{ \psi_h : \cS \to \Delta(\cX) \}_{h=1}^H$,
which maps each latent state to a distribution over observations. The former is
an arbitrary MDP defined over $\cS$, while the latter is defined as follows.\loose
\begin{definition}[Emission process]
An \emph{emission process} is any function $\psi := \{ \psi_h : \cS
\to \Delta(\cX) \}_{h=1}^H$, and  is said to be \emph{decodable} if
\begin{align}
\label{eq:decodability}
     \forall h, \forall s' \neq s\in\cS: \quad    \supp \psi_h(s)\cap\supp \psi_h(s')=\emptyset. \quad.
\end{align}
When $\psi=\{\psi_h\}_{h=1}^H$ is decodable, we let $\psi^{-1} :=
\{\psi_h^{-1}: \cX \to \cS\}_{h=1}^H$ denote the associated
decoder.\loose %
\end{definition}
\akedit{With this, we can formally introduce the notion of a latent-dynamics MDP.}
\begin{definition}[Latent-dynamics MDP]
  For a \emph{base MDP} $\Mlat =
  \{\cS,\cA,\{\Plath\}_{h=0}^H,\{\Rlath\}_{h=1}^H,H\}$, and 
  a decodable emission process $\psi$, the \emph{latent-dynamics MDP}
  $\dtri*{\Mlat,\psi}\ldef\crl*{\cX,\cA,\crl{\Pobsh}_{h=0}^{H},\crl{\Robsh}_{h=1}^{H},H}$
  is defined as the MDP where %
  the latent dynamics evolve based on the agent's action $a_h \in \cA$ via the process %
  	$s_{h+1} \sim \Plath(s_h,a_h)$ and $r_h \sim \Rlath(s_h,a_h)$. The latent state is not observed directly, and instead the agent observes $x_h \in \cX$ generated by the emission process %
  $x_{h} \sim \psihpo(s_{h})$.\footnote{Equivalently the dynamics can be described via $\Robsh(x_h,a_h) = \Rlat(\psiinvh(x_h),a_h)$ and $\Pobsh(x_{h+1} \mid x_h, a_h) = \Plath(\psiinv_{h+1}(x_{h+1}) \mid \psiinvh(x_h), a_h) \cdot \psi_{h+1}(x_{h+1} \mid \psiinv_{h+1}(x_{h+1}))$.
    	}
  \end{definition}

\akedit{%
   Note that under these dynamics, the decoder $\psi^{-1}$ associated with
  $\psi$ ensures that $\psi_h^{-1}(x_h) = s_h$ almost surely for all
  $h \in [H]$. That is, the latent states can be uniquely decoded from
  the observations.} To emphasize the distinction between the
  latent-dynamics MDP $\dtri*{\Mlat,\psi}$ (which operates on the
  observable state space $\cX$) and the MDP $\Mlat$ (which operates on
  the latent state space $\cS$), we refer to the latter as a \emph{base MDP} rather than, for example, a ``latent MDP'', and apply a
  similar convention to other latent objects whenever possible.\footnote{For example, in \cref{sec:algorithmic-results} we will be concerned with
  reductions from observation-space algorithms to ``base algorithms'' that operate on the latent state space.}

Departing from prior work, we do not place any inherent restrictions on the base MDP, and in particular do not assume
that the latent space is small (i.e., tabular). Rather, we aim to
understand---in a unified fashion---what structural assumptions on the
base MDP $\Mlat$ are required to enable learnability under latent
dynamics. To this end, it will be useful to considers specific
\emph{classes} (i.e., subsets) of base MDPs $\cMlat$ and the classes of latent-dynamics MDPs they induce.\loose
    \begin{definition}[Latent-dynamics MDP class]
    Given a set of base MDPs $\cMlat$ and a set of decoders $\Phi \subset \crl*{\cX \rightarrow \cS}$, we let 
 \begin{align}
   \dtri*{\cMlat, \Phi} \coloneqq \{ \dtri*{\Mlat, \psi}: \Mlat
   \in \cMlat, \dfedit{\text{$\psi$ is decodable},}\ \psiinv \in \Phi\}
 \end{align}    
  denote the class of induced latent-dynamics MDPs.%
  \end{definition}
Stated another way, $\dtri*{\cMlat, \Phi}$ is the set of all latent-dynamics MDPs $\dtri*{\Mlat,\psi}$ where (i) the base MDP $\Mlat$ lies in $\cMlat$, and
(ii), the emission process $\psi$ is decodable, with the corresponding decoder
belonging to $\Phi$. The class $\cMlat$ represents our prior knowledge
about the underlying MDP $\Mlat$; concrete classes considered in prior
work include tabular MDPs
\citep{krishnamurthy2016pac,du2019latent,misra2020kinematic,zhang2022efficient,mhammedi2023representation}, linear dynamical systems
\citep{dean2019robust,dean2020certainty,mhammedi2020learning}, and
factored MDPs \citep{misra2021provable}. In particular, the class
$\cMlat$ may itself warrant using function approximation. At the same time, the class $\Phi$ represents our prior knowledge or inductive bias about the emission process, enabling representation learning. In what follows,
we investigate what conditions on $\cMlat$ make the induced class
$\dtri*{\cMlat, \Phi}$ tractable, both statistically (statistical modularity; \cref{sec:statistical-results}) and via reduction (algorithmic modularity; \cref{sec:algorithmic-results}).\loose %

\cutedit{\paragraph{Further latent notations} See appendix. }

\section{Statistical Modularity: Positive and Negative Results}\label{sec:statistical-results}

This section presents our main statistical results. We begin by
formally defining the notion of statistical modularity introduced in
\cref{sec:intro}, present our main impossibility result (lower bound) and its implications
(\cref{sec:statistical-lower}), then give positive results for the general class of \emph{pushforward-coverable} MDPs (\cref{sec:statistical-results-positive}).

\subsection{Statistical modularity: A formal definition}\label{sec:statistical-modular-defn}

We first define the
\emph{statistical complexity} for a MDP class (or, \emph{model class})
$\cM$.

\begin{definition}[Statistical complexity]%
\neurips{
We say that an MDP class $\cM$ \emph{can be learned up to $\varepsilon$-optimality using $\comp(\cM,\varepsilon,\delta)$ samples} if there exists an algorithm $\alg$ which attains $\Risk(T,\alg,\cM) \coloneqq \max_{M \in \cM}\Risk(T,\alg,M) \leq \varepsilon$ 
	with probability at least $1-\delta$, after at most $T = \comp(\cM,\varepsilon,\delta)$ rounds of online interaction.
}%
\arxiv{
We say that an MDP class $\cM$ \emph{can be learned up to $\varepsilon$-optimality using $\comp(\cM,\varepsilon,\delta)$ samples} if there exists an algorithm $\alg$ which, for every $M \in \cM$, attains \[
\Risk(T,\alg,M) \leq \varepsilon
\]
with probability at least $1-\delta$ after $T =
        \comp(\cM,\varepsilon,\delta)$ rounds of online interaction in $M$.
}
\end{definition}

We say that a \emph{base MDP class} $\cMlat$ admits
\emph{statistically modularity} if, for any decoder class $\Phi$, the
induced latent-dynamics MDP class $\dtri*{\cMlat,\Phi}$ can be learned
with statistical complexity that is polynomial in: (i) the statistical
complexity for the base class, and (ii) the capacity of the decoder class.\loose

\begin{definition}[Statistical modularity]
  \label{def:statistical-modularity}
	For a decoder class $\Phi$, we say the MDP class $\cMlat$ is
        \emph{statistically modular} under complexity $\comp(\cMlat,\veps,\delta)$ if
	\begin{equation}
          \label{eq:modular}
          \comp(\dtri*{\cMlat,\Phi}, \varepsilon, \delta) = {\normalfont\poly}(\comp(\cMlat,\varepsilon,\delta),\log\abs{\Phi}).
	\end{equation}
      We say that $\cMlat$ admits \emph{strong statistical modularity} if \eqref{eq:modular} holds when
  $\comp(\cMlat,\varepsilon,\delta)$ is the minimax sample complexity for $\cMlat$.\akdelete{\footnote{Formally, we define the minimax sample complexity for a class
    $\cM$ as $\textrm{minimax}(\cM,\veps,\delta)\ldef{}\inf\crl*{T:
      \inf_{\alg}\sup_{\Mstar\in\cM}\bbP\prn*{\Risk(T,\alg,\Mstar)\leq\veps}\geq{}1-\delta}$. }  }
\end{definition}

In the sequel, we examine well-studied MDP classes $\cMlat$ (e.g.,
those which admit low Bellman rank \citep{jiang2017contextual}) and
choose $\comp(\cMlat,\veps,\delta)$ based on natural upper bounds on their optimal sample complexity; in this case we will simply say they are (or are not) statistical modular, leaving the complexity upper bound $\comp$ implicit.\akdelete{Statistical modularity implies that if the base MDP class $\cMlat$ can be
learned by some algorithm with sample complexity
$\comp(\cMlat,\varepsilon,\delta)$ when the latent states $s_h$ are
observed directly\akdelete{ (perhaps with the aid of function approximation,
model-based, value-based, or otherwise)}, then the induced latent
dynamics MDP class $\dtri*{\cMlat,\Phi}$ can be learned up to
$\varepsilon$-accuracy with sample complexity polynomial in
$\comp(\cMlat,\varepsilon,\delta)$ and $\log\abs{\Phi}$ ; here,} Following prior work \citep{krishnamurthy2016pac,du2019latent,misra2020kinematic,zhang2022efficient,mhammedi2023representation,dean2020certainty,mhammedi2020learning,misra2021provable},
we use $\log\abs{\Phi}$ as a proxy for the statistical complexity of
supervised learning with the decoder class $\Phi$.\loose\footnote{Our main
results easily extend to infinite classes through standard arguments.}  \loose

The two most notable examples of statistical modularity covered by prior work are: (i) taking $\cMlat$ as the set of tabular MDPs admits strong statistical modularity~\citep{du2019latent,misra2020kinematic,mhammedi2023representation}, and (ii) taking $\cMlat$ as the set of linear MDPs admits statistical modularity with complexity $\poly(d,H,\abs{\cA},\vepsinv,\logdelinv)$~\citep{agarwal2020flambe,uehara2022representation,modi2021model,mhammedi2023efficient}\neurips{-- indeed, $\dtri*{\cMlat,\Phi}$ is a low-rank MDP with unknown features in this case }\arxiv{.\footnote{In the latter case, the latent-dynamics class $\dtri*{\cMlat,\Phi}$ may be seen to be a set of low-rank MDPs (that is, linear MDPs with unknown features), so that low-rank MDP algorithms may be applied directly on the observations (\cref{sec:filling-out-lb-table}).}} Interestingly, the latter does not admit strong statistical modularity, because the optimal rate for $\cMlat$ does not scale
with $\abs{\cA}$, but the rate for $\dtri*{\cMlat,\Phi}$ necessarily does~\citep{lattimore2020bandit,hao2021online}.
The results of \citet{mhammedi2020learning,misra2021provable,song2024rich} can be viewed as
instances of statistical modularity for other base MDP classes. %
\akdelete{Examples of statistical modularity covered by prior work include: (i) when
$\cMlat$ is a set of tabular MDPs, we can achieve
$\comp(\cMlat,\veps,\delta)=\poly(H,\abs{\cS},\abs{\cA},\veps^{-1},\log(\delta^{-1}))$
and $          \comp(\dtri*{\cMlat,\Phi}, \varepsilon, \delta) =
\poly(H,\abs{\cS},\abs{\cA},\veps^{-1},\log(\delta^{-1}),\log\abs{\Phi})$,
so strong statistical modularity is satisfied; and
(ii), when $\cMlat$ is a set of linear MDPs with dimension $d$ we
can achieve
$\comp(\cMlat,\veps,\delta)=\poly(H,d,\veps^{-1},\log(\delta^{-1}))$
and $\comp(\dtri*{\cMlat,\Phi}, \varepsilon, \delta) =
\poly(H,d\abs{\cA},\veps^{-1},\log(\delta^{-1}),\log\abs{\Phi})$ so
weak statistical modularity is satisfied.\footnote{$\dtri*{\cMlat,\Phi}$ is a low-rank MDP~\citep{agarwal2020flambe,modi2021model,mhammedi2023efficient}.}
We classify the latter case
as \emph{weak} because the minimax rate for $\cMlat$ does not scale
with $\abs{\cA}$, while the existing guarantees for
$\dtri*{\cMlat,\Phi}$ do. Other works such as
\citet{mhammedi2020learning,misra2021provable} can also be viewed as
instances of weak statistical modularity.}

\cutedit{
\paragraph{Examples}\,

\begin{itemize}
	\item Block MDPs
	\item Latent contextual bandits
\end{itemize}
}

\subsection{Lower bounds: Impossibility of statistical modularity}
\label{sec:statistical-lower}

Our main result in this section is to show that for most MDP classes $\cMlat$ considered in the
literature on sample-efficient reinforcement learning with function
approximation
\citep{russo2013eluder,jiang2017contextual,sun2019model,modi2020sample,ayoub2020model,li2009unifying,dong2019provably,wang2020provably,zhou2021nearly,du2021bilinear,jin2021bellman,foster2021statistical},
statistical modularity \akedit{(under the natural complexity upper
  bound for the class of interest)} is \emph{impossible}.
\akdelete{, in the sense that
the induced latent-dynamics MDP class $\dtri*{\cMlat,\Phi}$ can be
statistically intractable. }
Our central technical result is the
following lower bound, which shows that\akdelete{strong} statistical modularity
can be impossible \arxiv{even when  $\abs{\cMlat}=1$. That is, }\textit{even
  when the base MDP is known to the learner a-priori}. The lower bound
is \paedit{a significant generalization of the result from
  \citet{song2024rich}}; we first state the lower
bound, then discuss implications. 
\begin{restatable}[Impossibility of statistical modularity]{theorem}{mainlb}\label{thm:stochastic-tree-lb}
  For every $N \geq 4$, there exists a decoder class $\Phi$ with
  $|\Phi| = N$ and a family of base MDPs $\cMlat$ satisfying (i)
  $|\cMlat|=1$, (ii) $H \leq \cO(\log(N))$, (iii) $|\cS| = |\cX| \leq
  N^2$, (iv) $|\cA| = 2$, and such that
  \begin{enumerate}
  \item For all $\veps,\delta>0$, we have $\comp(\cMlat,\veps,\delta) = 0$.
  \item For an absolute constant $c>0$, $\comp(\dtri*{\cMlat,\Phi},c,c) \geq \Omega(N/\log(N))$.
  \end{enumerate}
\end{restatable}

In other words, even when the base dynamics are fully known, strong statistical modularity (in this case, $\poly(\log\abs{\Phi})$ complexity) is impossible;  %
	any algorithm\arxiv{ for the latent-dynamics
          setting}\akdelete{ which has knowledge of $\cMlat$ and
          $\Phi$} will require at least
        $\min\crl{\sqrt{S},\nicefrac{2^{\Omega(H)}}{H},\nicefrac{\abs{\Phi}}{\log\abs{\Phi}}}$
        episodes to learn a near-optimal policy for a\arxiv{ worst-case} latent-dynamics MDP $\dtri*{\Mlat,\psi} \in \dtri*{\cMlat,\Phi}$.\loose %

\arxiv{\paragraph{Intuition for lower bound.}}
\neurips{\textbf{Intuition for lower bound.}~~}
The intuition behind the lower bound in \cref{thm:stochastic-tree-lb}
is as follows: the unobserved latent state space consists of $N = \abs{\Phi}$
binary trees (indexed from $1$ to $N$), each with $N$ leaf nodes. The
starting distribution is uniform over the roots of the $N$ trees, and
the agent receives a reward of $1$ if and only if they navigate to the
leaf node that corresponds to the index of their current
  tree. The observed state space is identical to the latent state
space, but the emission process shifts the index of the tree by an amount which is unknown to the agent. Despite the base MDP
being known and the decoder class satisfying realizability, the agent
requires near-exhaustive search to identify the value of the shift and
recover a near-optimal policy. 
\arxiv{}

\arxiv{
\paragraph{A taxonomy of statistical modularity.}
}
\neurips{
\paragraph{A taxonomy of statistical modularity.}
}

\newcommand{\cmark}{\ding{51}}
\newcommand{\xmark}{\ding{55}}
\definecolor{cadmiumgreen}{rgb}{0.0, 0.65, 0.31}
\definecolor{richcarmine}{rgb}{0.83, 0.0, 0.0}
\newcommand{\greencheck}{{\color{cadmiumgreen}\cmark}}
\newcommand{\redx}{{\color{richcarmine}\xmark}}
\newcommand{\yellowcheck}{{\color{bananayellow}\cmark}}
\newcommand{\yellowq}{{\color{bananayellow}\textbf{?}}}

\begin{figure}

\captionsetup{font=small}
\begin{minipage}[c]{0.6\textwidth}
  \begin{small}
  	\hspace{0.4cm}
    \begin{tabular}{ |c|c| }
      \hline
      Base MDP class  $\cMlat$ &  \begin{tabular}{@{}c@{}}Statistical \\ Modularity?\end{tabular} \\
      \hline 
      \hline 
      Tabular &\greencheck \\ 
      \hline
      Contextual Bandits & \greencheck \\ 
      \hline
      Low-Rank MDP & \greencheck \\ 
      \hline
      Known Deterministic MDP ($|\cMlat|=1$) & \greencheck \\ 
      \hline
      Low State Occupancy ($\forall\,\pi:\cS \rightarrow \Delta(\cA))$ &  \greencheck \\ 
      \hline
      Model Class + Pushforward Coverability & {\greencheck} \\ 
      \hline
      Linear CB/MDP & \redx$^\star$\\%\redx$^\star$ \\ 
      \hline
      Model Class + Coverability ($\forall\, \pi_{\sss{M}}: M \in \cM$)  & {\redx} \\ 
      \hline
      Known Stochastic MDP ($|\cMlat|=1$) &  \redx \\ 
      \hline
      Bellman Rank ($Q$-type or $V$-type) &  \redx \\ 
      \hline
      Eluder Dimension + Bellman Completeness &  \redx \\ 
      \hline
      $Q^\star$-Irrelevant State Abstraction &  \redx \\ 
      \hline
      Linear Mixture MDP &  \redx \\ 
      \hline
      Linear $Q^\star/V^\star$ &  \redx \\ 
      \hline
      Low State/State-Action Occupancy ($\forall\,\pi_\sM: M \in \cM)$ &  {\redx} \\ 
      \hline
      Bisimulation &  \yellowq \\ 
      \hline 
      Low State-Action Occupancy ($\forall\,\pi: \cS \rightarrow \Delta(\cA)$) &  {\yellowq}$^\star$\\%{\yellowq$^{\star}$} \\ 
      \hline
      Model Class + Coverability ($\forall\, \pi: \cS \rightarrow \Delta(\cA)$) & {\yellowq} \\ 
      \hline
    \end{tabular}
  \end{small}
\end{minipage}\neurips{\hfill}
 \arxiv{\hspace{1cm}}
  \begin{minipage}[c]{0.3\textwidth}
  \arxiv{\vspace{1.25cm}}
  \caption{
  \small\arxiv{
Summary of statistical modularity (SM) results.\\ 
\underline{\greencheck:}~SM is possible for a natural choice of $\comp(\cdot)$ (e.g., $\poly(\abs{\cS},\abs{\cA},H,\vepsinv,\logdelinv)$~for tabular MDPs). \\ 
\underline{\redx:}~SM is not possible with natural choices of $\comp(\cdot)$.\loose\\
\underline{\yellowq:}~open.\\
\underline{$^\star$:}~SM is possible if willing to pay for (suboptimal) $\abs{\cA}$ complexity.\\
See \cref{sec:filling-out-lb-table} for precise descriptions of each setting and our choices for their complexities.}
\neurips{Summary of statistical modularity (SM) results.\\ 
\greencheck: Stat. modularity is possible for a natural choice of $\comp(\cdot)$ (e.g., $\poly(\abs{\cS},\abs{\cA},H,\vepsinv,\logdelinv)$ for tabular MDPs). \\ 
\redx: SM is not possible with natural choices of $\comp(\cdot)$.\loose\\
\yellowq: open.\\
$^\star$: SM is possible if willing to pay for (suboptimal) $\abs{\cA}$ complexity.\\
See \cref{sec:filling-out-lb-table} for precise descriptions of each setting and our choices for their complexities.}\label{fig:lb-table}
} \end{minipage}
\end{figure}

As a corollary, we prove that  many (but not all) well-studied function approximation
settings do not admit statistical modularity by embedding them into the lower
bound construction of \cref{thm:stochastic-tree-lb} (as well as a
variant of the result, \cref{thm:stochastic-cb-lb}). Our results are summarized in
\cref{fig:lb-table}. Our impossibility results highlight the
following phenomenon: many MDP classes $\cMlat$ that place structural
assumptions via the value functions  (e.g., MDPs with linear-$Q^{\star}/V^{\star}$ \citep{du2021bilinear}
or MDPs with a Bellman complete value function class of bounded eluder
dimension \citep{jin2021bellman,wang2020provably}) become intractable under latent
dynamics. Intuitively, this is because it is not possible to take advantage of structure in
value functions without learning a good representation\paedit{, and, simultaneously, these assumptions are too weak by themselves to enable learning such a representation.} Meanwhile,
MDP classes $\cMlat$ that place structural assumptions on the
transition distribution (e.g., MDPs with low state occupancy complexity \citep{du2021bilinear} or low-rank MDPs \citep{agarwal2020flambe})
are sometimes (but not always) tractable under latent
dynamics.\footnote{If one is willing to pay for suboptimal $\abs{\cA}$
  factors, then more (but not all) classes become statistically
  tractable (e.g., linear MDPs \citep{jin2020provably} and MDPs with
  low state-action occupancy \citep{du2021bilinear}).}

We point to \cref{sec:filling-out-lb-table} for background on all the
settings in \cref{fig:lb-table} and proofs that they are (or are not)
statistically modular. We remark that it is fairly straightforward to embed most of the MDP classes of \cref{fig:lb-table} into the construction of
\cref{thm:stochastic-tree-lb} since it only uses only a single base MDP $\Mlat$, and we expect that many other base MDP classes can similarly be shown to be intractable. However, proving the
\emph{positive} results in \cref{fig:lb-table} requires establishing
several new results showing that certain base classes are tractable
under latent dynamics; most notably, we next discuss the case of \emph{pushforward coverability}. %

\subsection{Upper bounds: Pushforward-coverable MDPs are statistically modular}\label{sec:statistical-results-positive}

\akdelete{In light of the impossibility results in the prequel, we now investigate whether one can make additional structural assumptions on the latent class $\cMlat$ to enable statistical modularity for the classes ruled out by \cref{fig:lb-table}. }
Our main \akedit{postive }result \akedit{concerning statistical modularity} is to highlight \emph{pushforward coverability} \citep{xie2021batch,amortila2024scalable,mhammedi2024power}---a strengthened version of the \emph{coverability} parameter introduced in \citet{xie2022role}---as a general structural parameter that enables sample-efficient reinforcement learning under latent dynamics. \neurips{Formally,  the pushforward coverability coefficient \citep{xie2021batch,amortila2024scalable,mhammedi2024power} for an MDP $\Mlat$ with transition kernel $\Plat$ is defined by \begin{equation}\label{eq:pushforward-coverability}
	\Cpush(\Mlat) = \max_{h \in [H]} \inf_{\mu_h \in \Delta(\cS)} \sup_{s_{h-1},a_h,s_h} \frac{P_{\lat,h-1}(s_{h}\mid s_{h-1},a_{h-1})}{\mu_h(s_h)}.
	\end{equation}}
	\arxiv{
\begin{definition}[Pushforward coverability]\label{def:pushforward_def_main} The pushforward coverability coefficient $\Cpush$ for an MDP $\Mlat$ with transition kernel $\Plat$ is defined by
	\begin{equation}\label{eq:pushforward-coverability}
	\Cpush(\Mlat) = \max_{h \in [H]} \inf_{\mu \in \Delta(\cS)} \sup_{(s,a,s') \in\cS \times \cA \times\cS} \frac{P_{\lat,h-1}(s' \mid s,a)}{\mu(s')}.
	\end{equation}
\end{definition}
}
Concrete examples  \citep{amortila2024scalable,mhammedi2024power} include: (i) tabular MDPs $\Mlat$ admit $\Cpush(\Mlat)\leq\abs{\cS}$; and (ii) Low-Rank MDPs $\Mlat$ (with or without known features) in dimension $d$ admit $\Cpush(\Mlat)\leq{}d$. Further examples include analytically sparse Low-Rank MDPs \citep{golowich2023exploring} and Exogenous Block MDPs with weakly correlated noise \citep{mhammedi2024power}. Our main result is as follows.
\begin{restatable}[Pushforward-coverable MDPs are statistically modular]{theorem}{pushforwardgolf}\label{thm:pushforwardgolf}
Let $\cMlat$ be a base MDP class such that %
	each $\Mlat \in \cMlat$ has pushforward coverability bounded by %
	$\Cpush(\Mlat)\leq{}\Cpush$. Then, for any decoder class $\Phi$, we have:
	\begin{enumerate}
		\item $\comp(\cMlat, \veps,\delta) \leq \poly(\Cpush,\abs{\cA},H,\log\abs{\cMlat},\vepsinv,\logdelinv)$, and
		\item $\comp(\dtri*{\cMlat,\Phi},\veps,\delta) \leq \poly(\Cpush,\abs{\cA},H,\log\abs{\cMlat}, \log\abs{\Phi},\vepsinv,\logdelinv, \log\log\abs{\cS})$.
	\end{enumerate}
\end{restatable}

\cref{thm:pushforwardgolf} shows that\paedit{, modulo a term that is doubly-logarithmic in $\abs{\cS}$,} latent pushforward coverability enables statistical modularity. That is, when the base (latent) dynamics satisfy pushforward coverability, there exists an algorithm for the latent-dynamics setting which scales with the statistical complexity of the base MDP class and $\log\abs{\Phi}$. \paedit{We suspect that the additional $\log\log\abs{\cS}$ factor is not essential and can be removed with a more sophisticated analysis.} We note that the complexity $\comp$ chosen above is not the minimax complexity for $\cMlat$, since every set of pushforward coverable MDPs is also a set of \emph{coverable} MDPs with a potentially smaller coverability parameter~\citep{amortila2024scalable}.  \dfdelete{\footnote{Here, we are using that the realizable model class can be used to construct a complete value function class \citep{chen2019information}, and that $\Ccov \leq \Cpush \abs{\cA}$. } 
  }

  Let us provide some intuition for this result. We firstly note that when $\Mstarlat$ has pushforward coverability parameter $\Cpush$, it holds that for any emission process $\psistar$, the observation-level MDP $\Mstarobs \coloneqq \dtri*{\Mstarlat,\psistar}$ also satisfies pushforward coverability with the same parameter $\Cpush$ (\cref{lem:pushforward-invariant}). Yet, despite access to realizable base MDP class $\cMlat$ and decoder class $\Phi$, it is unclear whether the latent-dynamics MDP $\Mstarobs$ %
  	satisfies any of the 
  \emph{observation-level} function approximation conditions required by existing approaches that provide sample complexity guarantees under pushforward coverability. In particular, known algorithms for this setting either require a Bellman-complete value function class \citep{xie2022role}, a class realizing certain density ratios  \citep{amortila2023harnessing,amortila2024scalable}, or a realizable model class \citep{amortila2024scalable}, and it is highly nontrivial to construct these \emph{for the latent-dynamics MDP $\Mstarobs = \dtri{\Mstarlat,\psistar}$} given only the base MDP class $\cMlat$ and the decoder class $\Phi$. Intuitively, this is because the former observation-level function approximation classes capture properties of the observation-level dynamics which cannot be obtained without some knowledge of the emission process.\loose%
        \paragraph{Technical overview: Low-dimensional embeddings for pushforward-coverable MDPs.}  The idea behind our positive result is to show that under the conditions of \cref{thm:pushforwardgolf}, it is possible to construct an (approximately) Bellman-complete value function class for the latent-dynamics MDP $\Mobsstar$, at which point we can apply the \textsc{Golf} algorithm of  \citet{jin2021bellman}. %
  	We achieve this via two technical contributions. The first is the introduction of the \emph{mismatch functions} $\hphi$, formally defined as follows.
  	
  	\begin{definition}[Mismatch functions]\label{def:h_phi}
	For a decodable emission process $\psistar$ and decoder $\phi \in \Phi$, the \emph{mismatch function} for $\phi$, $\hphi = \crl{\hphih: \cS \rightarrow \Delta(\cS)}_{h=1}^H$, is defined, for every $h \in [H]$, as the probability kernel 
	\[
		\hphih(s'_h \mid s_h) \coloneqq \bbP_{x_h \sim \psistarh(s_h)}\prn*{\phih(x_h) = s'_h}.
	\]
\end{definition}
  	
	The mismatch functions allow us to express functions of the decoders as latent objects, and we will revisit them in the context of self-predictive estimation (\cref{sec:online-rl}). %
  	  For the present result, we show (\cref{lem:hphi-bellman}) that the mismatch functions can capture the observation-level Bellman backups for any function of the decoders. That is, for any $x_h,a_h$, letting $s_h = (\psistar)^{-1}(x_h)$ denote the true latent state, we have that for any $f_\lat: \cS \times \cA \rightarrow \bbR$ and $\phi \in \Phi$:
   \begin{equation}\label{eq:hphi-explanation-golf}
  	\brk{\cT^{\Mobsstar}_h(f_\lat \circ \phi_{h+1})}(x_h,a_h) = \brk{\cT^{\Mstarlat}_h(\hphihpo \circ V_{f_\lat})}(s_h,a_h).
\end{equation} 	  	
 That is, the Bellman update of $f_\lat \circ \phi_{h+1}$ in the latent-dynamics MDP $\Mstarobs$ can be expressed as a Bellman update in the base MDP $\Mstarlat$ for a different (latent) function $\hphihpo \circ V_{f_\lat}(s_{h+1}) \coloneqq \sum_{s'_{h+1}} \hphihpo(s'_{h+1} \mid s_{h+1}) \max_{a'}f_\lat(s'_{h+1},a')$. %
 
 However, the mismatch functions $\hphi$ embed some knowledge of the emission process, and (with only decoder and base model realizability) are unknown to the learner. Our second technical contribution bypasses this by establishing a new structural property for pushforward-coverable MDPs (\cref{lem:jl_cov_push}): there exist low-dimensional linear embeddings of their transition kernels which can approximate Bellman backups for an arbitrary and \emph{potentially unknown} set of functions, as long as the set is not too large.
 
 	\begin{restatable}[MDPS with pushforward coverability admit low-dimensional embeddings]{lemma}{jlcovpush}
  \label{lem:jl_cov_push}
  Let $M $ be a known MDP with reward function $r$, transition kernel $P$, and pushforward coverability parameter $\Cpush$. Let $\mu = \crl{\mu_h}_{h \in [H]}$ denote its pushforward coverability distribution (i.e. the minimizer of \cref{def:pushforward_def_main}) and %
  $\cF\subseteq(\cS \times [H] \to[0,1])$ be an arbitrary class of functions. Suppose that we
  sample $W\in\pmo^{d\times\cS}$ as a matrix of independent Rademacher
  random variables, and define
  \begin{align*}
    \psi_h(s,a) = r_h(s,a) \oplus \frac{1}{\sqrt{d}}W\prn*{P_h(\cdot \mid s,a)/\mu^{1/2}_h(\cdot)}_{\cdot \in \cS} \,\,\text{ and }\,\,   w_{f,h} = 1 \oplus \frac{1}{\sqrt{d}}W\prn*{\mu^{1/2}_h(\cdot)f_{h+1}(\cdot)}_{\cdot \in \cS}\in\bbR^{d+1}.
  \end{align*}
  Then for any $\vepsapx \in(0,1)$, as long as we set
  \[
 	d \geq 2^9\frac{\Cpush\log(16\abs{\cF}H\deltainv/\vepsapx)}{\vepsapx},
 \]
 we have that for all $f\in\cF$ and $h \in [H]$, with probability at least $1-\delta$:
  \begin{align*}
	\En_{\mu_h \otimes \unif(\cA)}\brk*{\prn*{\clip\brk*{\tri*{w_{f,h},\psi_h(s,a)}}-\cT_h f_{h+1}(s,a)}^2}\leq\vepsapx,
  \end{align*}
  as well as $\max_{s,a,h} \nrm*{\psi_h(s,a)}_2^2 \leq \Cpush(16\log\prn{\abs{\cS}\abs{\cA}H} + 11)$ and $\max_{f,h} \nrm*{w_{f,h}}_2^2\leq{} 16\log\prn{\abs{\cF}H}+ 11$. We emphasize that the feature map $\psi = \crl*{\psi_h}_{h=1}^H$ is
  oblivious to $\cF$, in the sense that it can
  be computed directly from $M$ without any knowledge of $\cF$.
\end{restatable}

 We use this property, in conjunction with latent model realizability, to construct linear features that can approximate the %
 	right-hand-side of \eqref{eq:hphi-explanation-golf}, thus yielding an (approximately) Bellman-complete value function class for the latent-dynamics MDP $\Mstarobs$. A fascinating open question is whether a similar approach can be used to establish that standard (as opposed to pushforward) coverable MDPs are statistically modular, which would encompass all other known positive cases of statistical modularity (cf. \cref{fig:lb-table}).

  \akcomment{Above discussion needs to more gently introduce new concepts. "the mismatch functions" makes it sound like one should know what these are. What is $V_{f_\lat}$? Even the Bellman backup operators are a bit scary and I don't expect a reader will get it? Linear embeddings of models feels like it doesn't type-check, maybe we can write ``linear embeddings derived from random projections of the model transition operators'' or something.} 

  \,

\section{Algorithmic Modularity}\label{sec:algorithmic-results}

We now turn our attention to \emph{algorithmic modularity}. Specifically, we aim for \emph{observable-to-latent
reductions}, whereby---via representation learning---RL under latent dynamics can be efficiently reduced to the simpler problem of RL with latent states directly observed. Since algorithmic modularity is a stronger property than statistical modularity,
we sidestep the previous lower bounds in
\cref{sec:statistical-results} through additional feedback and
modeling assumptions. Our main result for this section is a new
meta-algorithm, \obslatreduction, which, under these assumptions
(and when equipped with an appropriately designed representation
learning oracle), acts as a \emph{universal} reduction in the sense
that, whenever the representation learning oracle has low risk, the
reduction transforms \emph{any} sample-efficient algorithm for
\emph{any} base MDP class into a sample-efficient algorithm for the
induced latent-dynamics MDP class.\loose  

\arxiv{\subsection{Setup and \obslatreduction meta-algorithm}\label{sec:obs-to-lat-reduce}}
\neurips{\paragraph{Setup and \obslatreduction meta-algorithm.}}

\dfedit{For the results in this section, we denote the (unknown)
  latent-dynamics MDP of interest by
  $\Mstarobs \ldef{} \dtri*{\Mlat^{\star},\psi^{\star}}$, and use
  $\phistar\ldef{}(\psi^{\star})^{-1}$ to denote the true decoder.}
The \olr meta-algorithm (\cref{alg:obs-to-lat}) learns a near-optimal
policy for $\Mstarobs$ by alternating between
performing representation learning and executing a black-box ``base'' RL
algorithm (designed for the base MDP) on the learned representation; this approach is inspired by
empirical methods that blend representation learning and RL in the latent space (e.g.,  \cite{gelada2019deepmdp,schwarzer2020data,ni2024bridging}).

Concretely, the algorithm takes as input a \emph{representation
  learning oracle} \GenReplearn and a \emph{base RL algorithm}
\Alglat that operates in the latent space. In each \emph{epoch} $t \in
[T]$, \GenReplearn produces a new
representation $\wh{\phi}\ind{t}:\cX\to\cS$ based on data observed
so far (potentially using additional side information, which we will
elaborate on in the sequel). Then, the reduction invokes \Alglat, using $\wh{\phi}\ind{t}$ to simulate access to the true latent states. In
particular, \Alglat runs for $K$ episodes, where at each
episode $k$: (i) \Alglat produces a latent policy
$\pilat\ind{t,k}:\cS \times [H] \to\Delta(\cA)$, (ii) the latent policy is transformed into
an observation-level policy via \akedit{composition with $\phihat\ind{t}$, i.e.} $\pilat\ind{t,k} \circ
\phihat\ind{t}$, which is then deployed to produce a trajectory $\{x\ind{t,k}_h,a\ind{t,k}_h,
          r\ind{t,k}_h\}_{h=1}^H$, and %
          (iii) the trajectory is 
          \emph{compressed through $\phihat\ind{t}$} and used to
          update \Alglat via $\{\wh{\phi}\ind{t}_h(x\ind{t,k}_h),
          a\ind{t,k}_h,
          r\ind{t,k}_h\}_{h=1}^H$
          (cf. \cref{line:phi-compressed-dynamics} of \cref{alg:obs-to-lat}).\footnote{Note that, %
            if $\wh{\phi}$ is inaccurate, the compressed trajectory cannot necessarily be viewed as
            being generated by a latent MDP, and must
            instead be viewed as coming from a \textit{Partially
              Observed} MDP (\cref{sec:confounded-pomdps}).} After the $K$ rounds conclude, $\Alglat$
            produces a final latent policy $\pihat\ind{t}_\lat:\cS \times [H] \to\Delta(\cA)$. The final policy $\pihat$ chosen by the \olr algorithm is a uniform mixture of $\pihat\ind{t}_\lat \circ \phihat\ind{t}$ over all the epochs. \loose

          The central assumption behind \olr is that the
            base algorithm $\Alglat$ can achieve low-risk in the underlying
            base MDP $\Mlat^{\star}$ \emph{if given access to the
              true latent states $s_h=\phistar(x_h)$.}
            Beyond this assumption, we require that the
            representation learning oracle \GenReplearn can learn a
            sufficiently high-quality representation. In our applications, this will be made possible by assuming access
            to a realizable decoder class $\Phi$ and two distinct assumptions: \emph{hindsight observability}
            (\cref{sec:hindsight-rl}) and conditions enabling \emph{self-predictive
              representation learning} (\neurips{\cref{sec:online-rl-teaser}}\arxiv{\cref{sec:online-rl}}).
            We will show that under these conditions, we can instantiate a representation learning
            oracle such that \olr inherits the sample complexity
            guarantee for $\Alglat$, thereby achieving algorithmic modularity.\loose

   \begin{algorithm}[t]
    \begin{algorithmic}[1]
      \State \textbf{input}: %
        Epochs $T\in \bbN$, episodes
        $K\in\bbN$, decoder set $\Phi$, \arxiv{rep. learning oracle}\neurips{rep. learning oracle} \GenReplearn, base \arxiv{alg.}\neurips{alg.} \Alglat.
  \For{$t=1, 2, \cdots, T$} 
  \State $\GenReplearn$ chooses a representation
  $\wh{\phi}\ind{t}:\cX\to\cS \in \Phi$ based on data collected so far.
  \State Initialize new instance of \Alglat.
	\For{$k=1, 2, \cdots, K$} \hfill\algcommentlight{\Alglat plays $K$ rounds in the ``$\wh{\phi}\ind{t}$-compressed dynamics.''}
          \State \Alglat chooses policy $\pi\ind{t,k}_\lat:\cS \times [H] \to\Delta(\cA)$.
          \State Deploy $\pi_{\lat} \circ \wh{\phi}\ind{t}$
          to collect trajectory $\{x\ind{t,k}_h,a\ind{t,k}_h,
          r\ind{t,k}_h\}_{h=1}^H$. \label{line:composed-policies}
         \State Update \Alglat with compressed trajectory $\{\wh{\phi}\ind{t}_h(x\ind{t,k}_h), a\ind{t,k}_h, r\ind{t,k}_h\}_{h=1}^H$. \label{line:phi-compressed-dynamics}
                 
    \EndFor
    \State \Alglat returns final policy $\hatpi\ind{t}:\cS \times [H] \to\Delta(\cA)$, \paedit{deploy $\hatpi\ind{t} \circ \hatphi\ind{t}$ to collect one more trajectory}.
\EndFor
\State  \textbf{return} $\hatpi =
\unif(\hatpi\ind{1} \circ \hatphi\ind{1},\ldots,\hatpi\ind{T}\circ\hatphi\ind{T})$. \label{line:pac-output}
\end{algorithmic}
\caption{\obslatreduction: Observable-to-Latent Reduction}
\label{alg:obs-to-lat}
\end{algorithm}

\subsection{Algorithmic modularity via hindsight observability}\label{sec:hindsight-rl} %

Our first algorithmic result bypasses the hardness in
\cref{sec:statistical-results} by considering the setting of
\emph{hindsight observability},
which has garnered recent interest \dfedit{in the context of
  POMDPs} \citep{lee2023learning,guo2023sample,shi2023theoretical,lanier2024learning}. Here, we assume that at training time (but not during deployment),
the algorithm has access to additional feedback in the form of the
true latent states, which are revealed at the end of each
episode.\neurips{\footnote{We emphasize that \arxiv{in the hindsight observability
  framework, }the learner must still execute \emph{observation-space
    policies} $\pi:\cX\to\Delta(\cA)$, as the latent states are only
  revealed \emph{at the end of each episode}.}} 
  \begin{assumption}[Hindsight Observability \citep{lee2023learning}]
  The latent states $(\phi^\star_1(x_1),\ldots,\phi^\star_H(x_H))$ are revealed to the learner
  after each episode $(x_1, a_1,r_1,\ldots,x_{H},a_H,r_H)$ concludes. 
\end{assumption}
\arxiv{ We emphasize that \arxiv{in the hindsight observability
  framework, }the learner must still execute \emph{observation-space
    policies} $\piobs:\cX\to\Delta(\cA)$, as the latent states are only
  revealed \emph{at the end of each episode}.} Under hindsight observability, we can instantiate the representation
learning oracle in \olr so that the reduction achieves low risk for
\emph{any choice of black-box base algorithm $\Alglat$.}
In particular, we make use of \emph{online} classification oracles, which use the revealed latent states
to achieve low classification loss with respect to $\phistar$ under
adaptively generated data. We first state a guarantee based on generic
classification oracles, then instantiate it to give a
concrete end-to-end sample complexity bound.\loose

Formally, at each step $t$, the online classification oracle, denoted %
	via $\Hsightlearn$, is given the states and hindsight
        observations collected so far and produces a \dfedit{deterministic}
        estimate $\hatphi\ind{t} =
        \Hsightlearn\prn{\crl{x\ind{i}_h,\phistarh(x\ind{i}_h)}_{i<t,h\leq
            H}}$ for the true decoder $\phistar$. We measure the regret of the oracle via the $0/1$ loss for classification: %
\[
	\Regrep(T) \coloneqq \sum_{t=1}^T \sum_{h=1}^H \En_{\pi\ind{t} \sim p\ind{t}} \En^{\pi\ind{t}}\brk*{\indic\crl[\big]{\hatphi\ind{t}_h(x_h) \neq \phistarh(x_h)}},
\]
where $p\ind{t}$ represents a randomization distribution over the
policy $\pi\ind{t}$. %
Our reduction %
succeeds under the assumption that the oracle has low expected regret. 
\begin{assumption}\label{ass:online-classifier}
For any (possibly adaptive) sequence $\pi\ind{t}$, with $\pi\ind{t} \sim p\ind{t}$, the %
	online classification oracle $\Hsightlearn$ has expected regret bounded by \neurips{$\En\brk*{\Regrep(T)} \leq \Estrep(T),$}
\arxiv{
\[
	\En\brk*{\Regrep(T)} \leq \Estrep(T),
\]
}
\dfedit{where $\Estrep(T)$ is a known upper bound.}%
\end{assumption}

We apply such an oracle within \olr as follows: at the end of each iteration
$t\in\brk{T}$ in \olr, we sample $k\sim{}\brk{K}$ uniformly, and
update the classification oracle with the trajectory
$(x_1\ind{t,k},a_1\ind{t,k},r_1\ind{t,k}), \ldots,
(x_H\ind{t,k}a_H\ind{t,k},r_H\ind{t,k})$; see the proof of
\cref{thm:hindsightreduction} for details. We let $\Riskobs(TK)$ denote the risk of
the \obslatreduction reduction when run for $T$ epochs of $K$ episodes, and we let $\Riskstar(K) \coloneqq \En\brk{\Risk(K,\Alglat,\Mstarlat)}$ denote the expected risk of $\Alglat$ when
            executed on $\Mstarlat$ with access to the true latent states $s_h = \phistar(x_h)$ for $K$
            episodes. %

\begin{restatable}[Risk bound for \olr under hindsight observability]{theorem}{hindsightreduction}\label{thm:hindsightreduction}
Let \Alglat be a base algorithm with base risk $\Riskstar(K)$, and $\Hsightlearn$ a representation learning oracle satisfying \cref{ass:online-classifier}. %
Then \cref{alg:obs-to-lat}, with inputs $T,K,\Phi$, $\Hsightlearn$, and \Alglat,\,%
has expected risk 
\[
	\En\brk*{\Risk_\obs(TK)} \leq \Riskstar(K) + \frac{2K}{T}\Estrep(T).
      \]
\end{restatable}
This result shows that we can achieve sublinear risk under latent
dynamics as long as (i) the
base algorithm achieves sublinear risk $\Riskstar(K)$ given access to the true latent states, and (ii) the classification oracle achieves sublinear regret
$\Estrep(T)$. Notably, the result is fully modular, meaning we
require no explicit conditions on \paedit{the latent dynamics or }the base algorithm, and is
computationally efficient whenever the base algorithm and
classification oracle are efficient.

To make \cref{thm:hindsightreduction} concrete, we next provide a
representation learning oracle (\ExpWeights; \cref{alg:derandomized-expweights} in \cref{app:implementing-hsightlearn}) based on a derandomization of the classical exponential weights
mechanism, which satisfies \cref{ass:online-classifier} with
$\Estrep \approxleq H\log|\Phi|$ whenever it has access to a class
$\Phi$ that satisfies decoder realizability.\neurips{; %
	 for the pseudocode, see \cref{alg:derandomized-expweights} in \cref{app:implementing-hsightlearn}.} %

\begin{restatable}[Online classification via \ExpWeights]{lemma}{detonlineclassifier}\label{lem:det-online-classifier}
Under decoder realizability $(\phistar \in \Phi)$, %
\ExpWeights (\cref{alg:derandomized-expweights}) satisfies \cref{ass:online-classifier} with\footnote{In this section, the notations $\wt\cO, \approx,$ and $\lesssim$ ignore only constants and logarithmic factors of $H$.}  \neurips{
$\Estrep(T) \leq \wt\cO\prn*{H\log\abs*{H\Phi}}$.\loose
} %
\arxiv{
\[
	\Estrep(T) = \wt\cO\prn*{H\log\abs*{\Phi}}.
\]
}
\end{restatable}

Instantiating \cref{thm:hindsightreduction} with the above representation learning oracle, we obtain the following algorithmic modularity result. 

\begin{corollary}[Algorithmic modularity under hindsight
  observability]
  \label{cor:hindsight}
  For any base algorithm $\Alglat$, under decoder realizability
  $(\phistar \in \Phi)$, \olr with inputs $T,K,\Phi,\ExpWeights,$ and $ \Alglat$ achieves %
  \[
	\En\brk*{\Risk_\obs(TK)} \lesssim \Riskstar(K) +
        \frac{HK\log\abs{\Phi}}{T}.
      \]
   Consequently, for any $\Alglat$, setting $T\approx KH\log\abs{\Phi}/\Riskstar(K)$ achieves $\En\brk*{\Riskobs(TK)} \lesssim \Riskstar(K)$ with a number of trajectories $TK = \wt\cO\prn{\nicefrac{K^2H\log\abs{\Phi}}{\Riskstar(K)}}$.
   \end{corollary}
    Beyond achieving algorithmic modularity, this result shows that under hindsight observability, we
    can achieve
strong statistical modularity (\paedit{modulo possible $H$ factors}) for \emph{every} base MDP class
$\cMlat$, an important result in its own right.\footnote{\paedit{Formally, while we have defined the statistical modularity condition in terms of \emph{high-probability} risk bounds, it is straightforward to extend it to instead consider \emph{expected} risk bounds.}} %
As an example, suppose that $\Riskstar(K) = \cO\prn{K^{-1/2}}$, which
is satisfied by many standard algorithms of interest \citep{jiang2017contextual,jin2020provably,jin2021bellman,foster2021statistical}. Then, setting $T$ according to \cref{cor:hindsight} obtains an expected risk bound of $\varepsilon$ using $\cO(\nicefrac{H\log\abs{\Phi}}{\veps^5})$ trajectories.

  \begin{remark}[Online versus offline oracles]
    \cref{thm:hindsightreduction} critically uses that
    assumption that $\Hsightlearn$ satisfies an \emph{online}
    classification error bound to handle the fact that data is
    generated adaptively based on the estimators
    $\phihat\ind{1},\ldots,\phihat\ind{T}$ it produces, which is by
    now a relatively standard technique in the design of interactive
    decision making algorithms
    \citep{foster2020beyond,foster2021statistical,foster2023foundations}. We
    note that under coverability and other exploration conditions,
    online oracles for classification can be directly obtained from
    \emph{offline} (i.e. supervised) classification oracles
    \citep{xie2022role,block2024performance,foster2024online}. %
  \end{remark}

\arxiv{}

\neurips{    
\subsection{Algorithmic modularity via self-predictive estimation}\label{sec:online-rl-teaser}

We complement the above results by studying the general online RL setting \emph{without} hindsight observations. 
To address this more challenging setting, we design an \emph{optimistic self-predictive estimation} objective (\eqref{eq:optim-self-prediction}), and prove that any representation learning oracle that attains low regret with respect to this objective can be used in \olr to obtain observable-to-latent reductions for any low-risk base algorithm \Alglat (for a formal statement, see \cref{thm:online-reduction-main}). 
This objective learns a representation by jointly fitting a decoder together with a latent model.
Accordingly, we provide a (computationally inefficient) estimator which minimizes the objective under certain statistical conditions (specifically, coverability of the base MDP and a function approximation condition enabling us to express the self-prediction target as a latent model), thereby obtaining an end-to-end reduction for the general online RL setting.
For lack of space, these results are deferred to \cref{sec:online-rl}.\loose
}

\subsection{Algorithmic modularity via self-predictive estimation}\label{sec:online-rl}

In this section, we remove the assumption of hindsight observability used in \cref{sec:hindsight-rl} and instantiate \olr in the general \emph{online RL} setting. Rather than assume access to additional side-information, we adopt a \emph{model-based representation learning} approach, and augment our ability to perform representation learning by equipping the representation learning algorithm with a \emph{set of base MDPs} $\cMlat$ %
	in addition to the decoder class $\Phi$. We will learn a representation by jointly fitting a decoder and the base (latent) dynamics, which is a common approach in practice \citep{gelada2019deepmdp,hafner2019dream,hafner2019learning,hafner2020mastering,schrittwieser2020mastering,schwarzer2020data,guo2022byol}.%
We firstly present in \cref{sec:self-pred-estimation} a new notion of \emph{optimistic self-predictive regret} %
	which combines self-predictive representation learning with a form of optimism over a learned latent model. We then show in \cref{sec:self-pred-main-result} that any representation learning oracle that attains low regret, when used within \obslatreduction (\cref{alg:obs-to-lat}), leads to observable-to-latent reductions that ensure low risk for \textit{any} base algorithm $\Alglat$, thereby achieving algorithmic modularity. Lastly, in \cref{sec:self-pred-example}, we instantiate this oracle under natural structural and function approximation conditions, yielding end-to-end modularity and sample complexity guarantees.

\neurips{\paragraph{Self-predictive estimation.}}
\arxiv{\subsubsection{Self-predictive estimation}\label{sec:self-pred-estimation}}

Our self-predictive representation learning oracles learn to fit a representation $\phi$ such that the \emph{induced latent transitions} ($\phi_h(x_h)$ to $\phi_{h+1}(x_{h+1})$) can be accurately modeled by some base (latent) MDP $\Mlat\in\cMlat$. To describe the objective, let us first introduce some notation. For a given MDP $M$ over either $\cS$ (resp. $\cX$), we write $M_h(r_h,s_{h+1} \mid s_h,a_h)$ (resp. $M_h(r_h,x_{h+1} \mid x_h,a_h)$) for the joint conditional distribution over rewards and next states. Next, for any $\phi \in \Phi$, we define the \emph{pushforward model} for $\Mstarobsh$ induced by $\phi$ via: %
	\begin{equation}\label{eq:pushforward-model}
\brk*{\phi_{h+1}\sharp\Mstarobsh}(r,s' \mid x,a) \coloneqq %
	\sum_{x':\phi_{h+1}(x')=s'} \Mstarobsh(r,x' \mid x,a). %
\end{equation}
The \emph{\pamodel} for $\phi$ captures the forward probability of the estimated latent state $\phi(x')$ given a current observation $x$.
To measure distance between models, we will use squared Hellinger distance (e.g, \citet
{foster2021statistical}), defined via $\Dhels{\bbP}{\bbQ} = \int \prn[\big]{\sqrt{\frac{\dd\bbP}{\dd\nu}} - \sqrt{\frac{\dd\bbQ}{\dd\nu}}}^2 \dd \nu$ for a common dominating measure $\nu$.
\arxiv{
Then, for a base model $\Mlat$ and a decoder $\phi$, the \emph{self-predictive error} of $(M_\lat,\phi)$, at state-action pair $x_h,a_h$, is given by
\[
	\brk*{\Delta_h(\Mlat, \phi)}(x_h,a_h) \coloneqq \Dhels{M\sub{\lat,h}(\phi\sub{h}(x_h),a_h)}{\brk[\big]{\phi\sub{h+1}\sharp M^\star_{\obs,h}}(x_h,a_h)}.
\]
This term captures the %
ability of $M\sub{\lat,h}(\phi\sub{h}(x_h),a_h)$ to predict the next latent state $\phi_{h+1}(x_{h+1})$ which is obtained by the pushforward model $\brk[\big]{\phi\sub{h+1}\sharp M^\star_{\obs,h}}(x_h,a_h)$. Formally, in our model-based representation learning setup, we consider oracles which, \dfedit{for each iteration $t$ within \olr}, take as input the trajectories collected so far and produce an estimate $(\wh{M}\ind{t}_\lat,\hatphi\ind{t})$ for the decoder and base model. The representation learning oracle's \emph{self-predictive regret}, for the sequence $(\wh{M}\ind{t}_\lat,\hatphi\ind{t})$, is then defined as
\[
\Regsim(T) = \sum_{t=1}^T	\sum_{h=0}^H \En_{\pi\ind{t} \sim p\ind{t}}\En^{\pi\ind{t}}\brk*{\brk{\Delta_h(\wh{M}\ind{t}_{\lat}, \hatphi\ind{t})}(x_h,a_h)},
\]
where $p\ind{t}$ represents a randomization distribution over the policy $\pi\ind{t}$. 

On its own, minimizing this regret may lead to degenerate solutions, a widely observed phenomenon in practice \citep{tang2023understanding}. For example, in a standard combination lock MDP (e.g., \citet{agarwal2019reinforcement,misra2020kinematic}), a degenerate decoder-model pair that maps all observations to a single latent state will have zero self-predictive loss until we reach the goal, which can take exponentially long.\footnote{This is similar to the observation that naive value function approximation methods, such as Fitted Q-Iteration, can fail to explore in online RL without optimism. We expect that given access to additional exploratory data (e.g., in the \emph{Hybrid RL} setting of \citet{song2022hybrid}), the latent optimism term can be removed.} We address this via the notion of \emph{optimistic estimation} used in \citet{zhang2021feel,foster2023model}, which biases the objective towards latent models with high return. This leads to the following \textit{optimistic self-predictive regret}, defined for a parameter $\gamma > 0$, via
\begin{align}
  \label{eq:regselfopt}
\Regsimopt(T,\gamma) = \sum_{t=1}^T	 \sum_{h=0}^H \En_{\pi\ind{t} \sim p\ind{t}}\En^{\pi\ind{t}}\brk*{\brk{\Delta_h(\wh{M}\ind{t}_{\lat}, \hatphi\ind{t})}(x_h,a_h)} + \gammainv \prn{J^{\Mstarlat}(\pi_{\Mstarlat}) - J^{\wh{M}\ind{t}_\lat}(\pi_{\wh{M}\ind{t}_\lat})}.
\end{align}
}

\neurips{
Formally, we consider an online representation learning oracle \OptReplearn which, \dfedit{for each iteration $t$ within \olr}, takes as input the trajectories collected so far and produce an estimate $(\wh{M}\ind{t}_\lat,\hatphi\ind{t})$ for the decoder and latent model. \dfedit{We note that only the decoder $\hatphi\ind{t}$ is used within \olr; the model $\wh{M}_\lat$ is only used for analysis (and possibly within the representation learner \OptReplearn).} We measure the algorithm's performance through regret according to the following \emph{optimistic self-prediction objective}, defined for a parameter $\gamma>0$, via \neurips{$\Regsim(T,\gamma) \coloneq$ 
{\small
\begin{equation}\label{eq:optim-self-prediction}
	 {\small \sum_{t=1}^T \sum_{h=0}^H\En^{\pi\ind{t}}\brk*{\Dhels{\wh{M}\ind{t}\sub{\lat,h}(\hatphi\ind{t}\sub{h}(x_h),a_h)}{\brk[\big]{\hatphi\ind{t}\sub{h+1}\sharp M^\star_{\obs,h}}(x_h,a_h)}} +\gamma^{-1}\prn{J^{\Mstarlat}(\pi_{\Mstarlat}) - J^{\wh{M}\ind{t}_\lat}(\pi_{\wh{M}\ind{t}_\lat})} }.
\end{equation}}
}
\arxiv{
\begin{equation}\label{eq:optim-self-prediction}
	 {\small \Regsimopt(T,\gamma) \coloneqq  \sum_{t=1}^T \sum_{h=0}^H\En^{\pi\ind{t}}\brk*{\Dhels{\wh{M}\ind{t}\sub{\lat,h}(\hatphi\ind{t}\sub{h}(x_h),a_h)}{\brk[\big]{\hatphi\ind{t}\sub{h+1}\sharp M^\star_{\obs,h}}(x_h,a_h)}} +\gamma^{-1}\prn{J^{\Mstarlat}(\pi_{\Mstarlat}) - J^{\wh{M}\ind{t}_\lat}(\pi_{\wh{M}\ind{t}_\lat})} }.
\end{equation}
}%
Here, the first term captures the \emph{self-predictive} error of $(\wh{M}\ind{t}_\lat,\hatphi\ind{t})$; that is, can $\wh{M}\ind{t}\sub{\lat,h}(\hatphi\ind{t}\sub{h}(x_h),a_h)$ predict the next  latent state $\hatphi\ind{t}$ which is sampled from the pushforward model $\brk[\big]{\hatphi\ind{t}\sub{h+1}\sharp M^\star_{\obs,h}}(x_h,a_h)$? %
	The second term, which is inspired by the notion of \emph{optimistic estimation} used in \citet{zhang2021feel,foster2023model}, biases the objective toward a latent model with high return. This is necessary to avoid degenerate solutions in the self-predictive objective, a widely observed phenomenon in practice \citep{tang2023understanding}.\footnote{For example, in a standard combination lock MDP (e.g., \citet{agarwal2019reinforcement,misra2020kinematic}), a degenerate decoder-model pair that maps observations to a single latent state will have zero self-predictive loss until we reach the goal, which can take exponentially long. This is similar to the observation that naive value function approximation methods, such as Fitted Q-Iteration, can fail to explore in online RL without optimism. We expect that given access to additional exploratory data (e.g., in the \emph{Hybrid RL} setting of \citet{song2022hybrid}), the latent optimism term can be removed.}
}

We assume going forward that \Algsim obtains low %
	optimistic self-predictive regret; in \cref{sec:self-pred-example} we provide a maximum-likelihood-type estimator and conditions under which this holds.
	
\begin{assumption}\label{ass:optim-replearn}
 For a parameter $\gamma > 0$ and any (possibly adaptive) sequence $\pi\ind{t}$, with $\pi\ind{t} \sim p\ind{t}$, the online representation learning oracle \Algsim is proper (i.e. outputs $\wh{M}\ind{t}_\lat \in \cMlat$ for all $t \in [T]$) and satisfies %
 	\neurips{$\En\brk*{\Regsimopt(T,\gamma)} \leq \Estsimopt(T,\gamma)$,}\arxiv{
 \begin{align*}
		\En\brk*{\Regsimopt(T,\gamma)} \leq \Estsimopt(T,\gamma),
	\end{align*}
 } where $\Estsimopt(T,\gamma)$ is a known upper bound. %
      \end{assumption}
      
\dfedit{We note that only the decoder $\hatphi\ind{t}$ is used within \olr; the model $\wh{M}\ind{t}_\lat$ is only used for analysis (and possibly within the representation learner \OptReplearn).}

     \neurips{\paragraph{Main result.}}\arxiv{\subsubsection{Main result}\label{sec:self-pred-main-result}} We now state the main guarantee for \olr with self-predictive representation learning. Recall that $\Riskobs(TK)$ denotes the risk of the \obslatreduction reduction. Compared to the hindsight-observable setting, we require a slightly stronger performance guarantee from the base algorithm $\Alglat$: our result scales with the \emph{worst-case} expected risk for $\Alglat$ over all $\Mlat\in\cMlat$, defined via $\Riskbase(K)\ldef{}\sup_{\Mlat\in\cMlat}\En\brk*{\Risk(K,\Alglat,\Mlat))}$. \neurips{\footnote{Intuitively, our result scales with $\Riskbase(K)$ instead of $\Riskstar(K)$ due to potential symmetries in the self-predictive objective. For example, there might be a representation $\hatphi$ (resp. latent model $\wh{M}_\lat$) that are identical to $\phistar$ (resp. $\Mstarlat$) \emph{up to} permutations of the latent state space. These cannot be distinguished without observing the latent states.}}
\begin{restatable}[Risk bound for \obslatreduction under self-predictive estimation]{theorem}{onlinereductionmain}\label{thm:online-reduction-main}
	Suppose \OptReplearn satisfies \cref{ass:optim-replearn} \dfedit{with parameter $\gamma>0$}. Then \cref{alg:obs-to-lat}, with inputs $T,K,\Phi$, $\OptReplearn$, and $\Alglat$\,%
has expected risk 	
		\[
		\En\brk*{\Riskobs(TK)} \leq c_1\cdot{}\Riskbase(K) + c_2\gamma \cdot \frac{K}{T}\Estsimopt(T,\gamma) + c_3\gammainv\cdot KH,
	\]
	for absolute constants $c_1,c_2,c_3>0$.
\end{restatable}

\cref{thm:online-reduction-main} achieves sublinear risk as long as (i) the
latent algorithm achieves sublinear risk $\Riskbase(K)$ given access to the true
states, and (ii) the self-predictive representation learning oracle achieves sublinear regret
$\Estsimopt(T,\gamma)$ \dfedit{for an appropriate choice of $\gamma$.}\footnote{For example, in our estimator of \cref{sec:self-pred-example}, we can first set $\gamma \approx KH / \Riskbase(K)$ so that the third term matches $\Riskbase(K)$, and then set $T$ so that the second term does.} %
Intuitively, our result scales with $\Riskbase(K)$ instead of $\Riskstar(K)$ due to potential symmetries in the self-predictive objective. For example, there might be a representation-model pair $(\wh{M}_\lat,\hatphi)$ that is identical to $(\Mstarlat,\phistar)$ \emph{up to} permutations of the latent state space; these cannot be distinguished by a representation learning oracle that does not observe the latent states directly, and thus the base algorithm may be tasked with solving either of these base MDPs. As with \cref{thm:hindsightreduction}, this result achieves algorithmic modularity (since \olr inherits the risk of the base algorithm), and is computationally efficient whenever the base algorithm and self-predictive representation learning oracle are efficient.%

Let us provide some intuition behind the proof of \cref{thm:online-reduction-main}. Recall that, within the inner loop of \olr, the latent algorithm $\alglat$ interacts with the $\hatphi\ind{t}$-compressed dynamics generated by compressing the observations $x_h,a_h$ through the current decoder $\hatphi\ind{t}_h$ (\cref{line:phi-compressed-dynamics}). The crux of the analysis is the following observation: by the self-predictive representation learning guarantee, these dynamics, %
	 despite being possibly non-Markovian and generated from a POMDP (\cref{def:phi-compressed-pomdp}), are well approximated in squared Hellinger distance by the base model $\wh{M}_\lat\ind{t}$ estimated by \OptReplearn (cf. \cref{lem:near-markov}). We can then show that \Alglat, when given data from the $\hatphi\ind{t}$-compressed dynamics, has risk (for solving $\wh{M}_\lat\ind{t}$) that is proportional to: i) its base risk if it were to observe states from $\wh{M}_\lat\ind{t}$, and ii) the Hellinger distance between $\wh{M}_\lat\ind{t}$ and the process induced by its $\hatphi\ind{t}$-compressed dynamics. The last ingredient is the use of latent optimism in \eqref{eq:regselfopt}, through which the risk on $\Mstarlat$ is upper bounded by the risk on $\wh{M}_\lat\ind{t}$.

In the above, showing that $\Alglat$ obtains low risk for $\wh{M}_\lat\ind{t}$ (despite given data from a different process) is done by establishing a certain form of \emph{corruption robustness} (\cref{def:alglat-robust}). Indeed, 
\cref{thm:online-reduction-main} is a special case of a more general theorem (\cref{thm:online-reduction-alpha}), which provides a bound that adapts to \Alglat's level of robustness. 
We obtain \cref{thm:online-reduction-main} by showing that \emph{any algorithm} satisfies the property we require (for a suitably slow rate), but we further show that tighter rates can be achieved by analyzing the specifics of various algorithms of interest (\cref{sec:crobustexamples}). %

\subsubsection{Instantiating the self-predictive estimation oracle}\label{sec:self-pred-example}

We now present an algorithm, \SelfPred (\cref{alg:debiased-mle-optimistic} in \cref{sec:implementing-replearn}), which satisfies \cref{ass:optim-replearn} under additional technical conditions, allowing us to instantiate \cref{thm:online-reduction-main} to give end-to-end guarantees. %
	Before stating the main guarantee, we highlight a few technical difficulties regarding obtaining finite-sample guarantees for (online) self-predictive estimation, and use them to motivate our statistical assumptions and algorithm design.
\neurips{
\begin{enumerate}
\item The first challenge is a realizability issue: when
  $\phi \neq \phistar$, we may not even be able to represent the
  objective $\phi\sharp\Mstarobs$ as a latent model using only
  decoder and latent model realizability. Since we can never guarantee that $\phi \neq \phistar$ in the presence of statistical errors, we must introduce a modelling assumption which lets us capture the pushforward models $\phi\sharp\Mstarobs$. We give a new function
  approximation condition (\emph{mismatch function completeness},
  \cref{ass:hphi-completeness}) which we show is sufficient for
  realizing the pushforward models of \eqref{eq:pushforward-model} 
  (\cref{lem:phi-compressed-model-realizable}). We view this assumption as a minimal way to realize the pushforward models $\phi\sharp\Mstarobs$ as latent models. 
  
\item The second challenge is a
  \emph{double-sampling} issue, which appears because the decoders are
  coupled at different horizons. We address this with a novel
  ``debiased'' maximum likelihood procedure that subtracts a form of
  excess risk to recover an unbiased estimator \cite{jiang2024note}. 
  \item The last issue is one
  related to online learning: the policies chosen by the latent
  algorithm are a function of the estimated decoders, which precludes
  the use of randomized estimators for online learning. We show that
  assuming \emph{coverability} (of the latent MDP, over a suitable set
  of latent policies) is sufficient for bypassing this, since we can
  use offline model estimation to obtain online model
  estimation. %
\end{enumerate}
}
\arxiv{

\paragraph{The statistics of (online) self-predictive estimation.} %
	The first challenge is a realizability issue: when
  $\phi \neq \phistar$, we may not even be able to represent the
  objective $\phi\sharp\Mstarobs$ as a latent model using only
  decoder and latent model realizability. Since we can never guarantee that $\phi =\phistar$ exactly in the presence of statistical errors, we must introduce a modelling assumption which lets us capture the pushforward models $\phi\sharp\Mstarobs$. To this end, we recall the mismatch functions  of \cref{def:h_phi}, which are defined via 
  \[
		\hphih(s'_h \mid s_h) \coloneqq \bbP_{x_h \sim \psistarh(s_h)}\prn*{\phih(x_h) = s'_h}.
		\]

In the context of self-prediction, we show that the following \emph{mismatch completeness} assumption suffices to capture the pushforward models $\phi\sharp\Mstarobs$.
\begin{assumption}[Mismatch completeness]\label{ass:hphi-completeness}
We have a model class $\cL$ such that, for each $\phi \in \Phi$, and $\Mlat \in \cMlat$, we have $\hphi \circ \Mlat \in \cL$, where 
\[
\brk*{\hphi \circ \Mlat}_h(r_h,s_{h+1} \mid s_h,a_h) \coloneqq \sum_{s'_{h+1} \in \cS} \Mlath(r_h,s'_{h+1} \mid s_h,a_h)
 \hphihpo(s_{h+1} \mid s'_{h+1}).
  \]
\end{assumption}

In particular, \cref{lem:phi-compressed-model-realizable} establishes that
\[
\brk{\phi_{h+1}\sharp\Mstarobsh}(\cdot \mid x,a) = \brk*{\hphi \circ \Mstarlat}_h(\cdot \mid \phistarh(x),a).
\]

Accordingly, we view this assumption as a minimal way to realize the pushforward models $\phi\sharp\Mstarobs$.

The second challenge is a
  \emph{double-sampling} issue, which appears because the decoders in \eqref{eq:regselfopt} are
  coupled at different horizons. We address this with a novel
  ``debiased'' maximum likelihood procedure that subtracts a form of
  excess risk (cf. \eqref{eq:debiasing}) to recover an unbiased estimator \cite{jiang2024note}. Our debiased estimator and the mismatch completeness assumption can be viewed as analogous to the techniques and assumptions that are required for squared Bellman error minimization in the context of value function approximation \cite{chen2019information,jin2021bellman}.  
  
The last issue stems from seeking an \emph{online} estimation guarantee:
  the policies chosen by the latent
  algorithm are a function of the estimated decoders, which precludes
  the use of randomized estimators (e.g. exponential weights). We bypass this issue by appealing to the structural condition of coverability \citep{xie2022role}, which allows us to restrict our attention to estimators that achieve low \emph{offline} estimation error (via \cref{lem:offline-to-online}).\footnote{\paedit{More generally, we expect that our results can be extended to any ``decoupling coefficient'' \cite{zhang2021feel,agarwal2022model}.}}
\begin{definition}[State Coverability ]\label{def:coverability} The state coverability coefficient for an MDP $M$ and a policy class $\Pi$ defined over a state space $\cZ$, $\Ccovs(M,\Pi)$, is given by
	\begin{equation}\label{eq:coverability}
		\Ccovs(M,\Pi) \coloneqq \max_{h \in [H]} \min_{\mu \in \Delta(\cZ)}\max_{\pi \in \Pi} \max_{z \in \cZ} \crl*{\frac{d^{\sM,\pi}_h(z)}{\mu(z)}}.
	\end{equation}
\end{definition}

 We require coverability in $\Mstarobs$ over the set of (observation-space) policies played by the \olr reduction (cf. \cref{line:composed-policies}). Again appealing to the mismatch functions, we can express this as an assumption about the base dynamics $\Mstarlat$; we show (\cref{lem:covinvariance}) that the latter is \emph{equivalent} to assuming coverability in $\Mstarlat$ over the set of stochastic policies 

\begin{equation}\label{eq:hPhi-Pilat-policies}
	\hPhi \circ \Pilat \coloneqq \crl*{[\hphi\circ \pilat]_h(a \mid s) = \sum_{s' \in \cS} \hphih(s' \mid s) \pi_{\lat,h}(a \mid s') \mid \phi \in \Phi, \pilat \in \Pilat}, %
\end{equation}
where $\Pilat$ denotes the set of policies that $\Alglat$ may execute. While this set may appear complicated, it is sufficient to assume coverability over the set of all \emph{deterministic} non-stationary policies on $\Mstarlat$.\footnote{This follows from \cref{lem:cov-state-reachability} by noting that each maximum on the right hand side of \eqref{eq:state-cov-equivalence} is attained by a deterministic non-stationary policy.}

}

\paragraph{Guarantee for our self-predictive estimation oracle.} With these prerequisites, the main guarantee for our estimator, \SelfPred (\cref{alg:debiased-mle-optimistic}), is as follows.
\begin{restatable}[Optimistic self-predictive estimation via \SelfPred]{lemma}{optreplearn}\label{lem:implementing-optim-replearn}
Let $\Pilat$ denote the set of policies played by $\Alglat$, and $\Ccovs = \Ccovs(\Mstarlat, \hPhi \circ \Pilat)$ be the state coverability parameter on $\Mstarlat$ over the set of stochastic policies $\hPhi \circ \Pilat$ (\eqref{eq:hPhi-Pilat-policies}). Then, for any $\gamma > 0$, under decoder realizability $(\phistar \in \Phi)$, base model realizability $(\Mstarlat \in \cMlat)$, and mismatch function completeness with class $\cLlat$ (\cref{ass:hphi-completeness}), the estimator in \cref{alg:debiased-mle-optimistic} with inputs $\Phi, \cMlat, \cL_\lat,$ and $\gamma$ satisfies \cref{ass:optim-replearn} with\footnote{In this section, the notations $\wt\cO$ and $\lesssim$ ignores constants and logarithmic factors of: $H, \Ccovs, \abs{\cA},T,$ and $\log\prn{\abs{\cMlat}\abs{\cLlat}\abs{\Phi}}$.}
\[
  \Estsimopt(T,\gamma) = \wt{\cO}\prn*{\sqrt{H\Ccovs\abs{\cA} T}\log(\abs{\cMlat}\abs{\cL_\lat}\abs{\Phi})}.
\]
      \end{restatable}

Instantiating \cref{thm:online-reduction-main} with the above representation learning oracle, we obtain the following algorithmic modularity result. 

\begin{restatable}[Algorithmic modularity via \SelfPred]{corollary}{selfpredmodular}\label{cor:selfpredmodular}
Under the same conditions as in \cref{lem:implementing-optim-replearn}, and for any base algorithm $\Alglat$,  \olr with inputs $T,K,\Phi,\SelfPred,$ and $ \Alglat$ achieves %
		\[
	\En\brk*{\Riskobs(TK)} \lesssim c_1\cdot{}\Riskbase(K) + c_2 \gamma \cdot \frac{K}{\sqrt{T}} \sqrt{H\Ccovs\abs{\cA}}\log(\abs{\cMlat}\abs{\cL_\lat}\abs{\Phi}) + c_3\gammainv\cdot KH,
	\]
	for absolute constants $c_1,c_2,c_3$. Consequently, for any $\Alglat$ with base risk $\Riskbase(K)$, setting $\gamma$ and $T$ appropriately gives
	\[
	\En\brk*{\Riskobs(TK)} \lesssim \Riskbase(K),
	\]
	with a number of trajectories $TK = \wt\cO\prn{\nicefrac{K^5H^3 \Ccovs\abs{\cA}\log^2(\abs{\cMlat}\abs{\cL_\lat}\abs{\Phi})}{(\Riskbase(K))^4}}$.
	\loose
\end{restatable}

For example, if \Alglat is a base algorithm with $\Riskbase(K) = \cO(K^{-1/2})$, %
	setting $\gamma$ and $T$ appropriately gives an expected risk of $\veps$ with a number of trajectories $TK = \wt\cO\prn*{\nicefrac{H^3\Ccovs\abs{\cA}(\log(\abs{\cMlat}\abs{\cLlat}\abs{\Phi}))^2}{\veps^{14}}}$. 	This result shows that statistical modularity can be achieved \emph{up to $\log(\abs{\cLlat})$ factors} for every base MDP class
$\cMlat$ which is subsumed by coverability, including tabular MDPs and low-rank MDPs.\footnote{This provides a partial answer to the ``Model Class + Coverability'' open question of \cref{fig:lb-table}.} Compared to our positive result for the case of pushforward coverability (\cref{sec:statistical-results-positive}), this imposes less dynamics assumptions (since coverability is implied by pushforward coverability) but requires more representational assumptions (namely, access to the mismatch-complete class $\cLlat$). We further remark that the mismatch completeness assumption always holds for i) the Block MDP setting, since we can always construct $\cLlat$ such that $\log(\abs{\cLlat})=\cO(HS^2)$, and ii) every MDP class $\cMlat$ whenever we also have a realizable set of emission processes ($\psistar \in \Psi$), since we can construct $\cLlat$ such that $\log(\abs{\cLlat})=\log(\abs{\Phi}\abs{\cMlat}\abs{\Psi})$. However, the mismatch completeness assumption may be more general than either of these settings. %

Our results can be viewed as providing a theoretical justification for self-predictive representation learning, which has been widely used in empirical works  \cite{gelada2019deepmdp,schwarzer2020data}. We consider self-prediction's ability to obtain universal observable-to-latent reductions as a strong indicator that it merits further theoretical study. In particular, many empirical works propose heuristics to alleviate the degeneracy/non-uniqueness issues inherent with self-prediction \cite{gelada2019deepmdp,schwarzer2020data,hafner2023mastering,tang2023understanding}. Our methods provide a principled way to address these, and it would be interesting to investigate whether this is also \emph{empirically} effective. In general, however, it is unclear whether our loss admits a computationally efficient implementation, due to the presence of optimism. Towards this, 
 a fascinating direction for future work is understanding how self-predictive estimation can be used to obtain algorithmic modularity \emph{without} the addition of optimism over the base (latent) models.  %

\section{Discussion}\label{sec:discussion}
Our work initiates the study of statistical and algorithmic modularity
for reinforcement learning under general latent dynamics. Our positive and
negative results serve as a first step toward a unified theory for 
reinforcement learning in the presence of high-dimensional observations. To this
end, we close with some important future directions and open problems.

\paragraph{Statistical modularity.} 
  
Can we obtain a unified characterization for the statistical
  complexity of RL under latent dynamics with a given class of base
  MDPs $\cMlat$? Our results in \cref{sec:statistical-results} suggest
  that this will require new tools that go beyond existing notions of
  statistical complexity. Toward resolving this problem, concrete questions that
  are not yet understood include: (i) Is coverability
  \citep{xie2022role} (as
  opposed to pushforward coverability) sufficient for learnability
  under latent dynamics? (ii) Is the \emph{Exogenous Block MDP}
  problem \citep{efroni2021provable,mhammedi2024power}---a special case of our general framework---statistically
  tractable? Lastly, are there additional types of feedback that are weaker than
  hindsight observability, yet suffice to bypass the hardness results
  in \cref{sec:statistical-results}?

\textbf{Algorithmic modularity.} Can we derive a unified representation learning objective that enables algorithmic modularity whenever statistical modularity is
  possible? Ideally, such an objective would be computationally tractable. Alternatively, can we show that algorithmic modularity
  fundamentally requires stronger modeling assumptions than
  statistical modularity? Toward addressing the problems above, a first step might be to
  understand: (i) What are the minimal statistical assumptions under
  which we can minimize the self-predictive objective in
  \neurips{\cref{sec:online-rl-teaser}}\arxiv{\cref{sec:online-rl}}? (ii)  How can we encourage finding good representations via
  self-prediction beyond the use of optimism over the base (latent)
  models; and (iii) When can we minimize self-prediction in a computationally efficient fashion?

\subsubsection*{Acknowledgements} Nan Jiang acknowledges funding support from NSF IIS-2112471, NSF CAREER IIS-2141781, Google Scholar Award, and Sloan Fellowship.

\newpage 

\printbibliography

@inproceedings{song2024rich,
  title={Rich-Observation Reinforcement Learning with Continuous Latent Dynamics},
  author={Song, Yuda and Wu, Lili and Foster, Dylan J and Krishnamurthy, Akshay},
  booktitle={International Conference on Machine Learning},
  year={2024}
}

@inproceedings{agarwal2022model,
  title={Model-based RL with Optimistic Posterior Sampling: Structural Conditions and Sample Complexity},
  author={Agarwal, Alekh and Zhang, Tong},
  booktitle={Neural Information Processing Systems},
  year={2022}
}

@inproceedings{ni2024bridging,
  title={Bridging State and History Representations: Understanding Self-Predictive RL},
  author={Ni, Tianwei and Eysenbach, Benjamin and Seyedsalehi, Erfan and Ma, Michel and Gehring, Clement and Mahajan, Aditya and Bacon, Pierre-Luc},
  booktitle={arXiv:2401.08898},
  year={2024}
}

@inproceedings{lee2023learning,
  title={Learning in POMPDs Is Sample-Efficient With Hindsight Observability},
  author={Lee, Jonathan and Agarwal, Alekh and Dann, Christoph and Zhang, Tong},
  booktitle={International Conference on Machine Learning},
  year={2023},
}

@inproceedings{foster2024online,
  title={Online Estimation via Offline Estimation: An Information-Theoretic Framework},
  author={Foster, Dylan J and Han, Yanjun and Qian, Jian and Rakhlin, Alexander},
  booktitle={arXiv:2404.10122},
  year={2024}
}

@inproceedings{mhammedi2023representation,
  title={Representation Learning With Multi-Step Inverse Kinematics: An Efficient and Optimal Approach to Rich-Observation RL},
  author={Mhammedi, Zakaria and Foster, Dylan J and Rakhlin, Alexander},
  booktitle={International Conference on Machine Learning},
  year={2023},
}

@inproceedings{song2022hybrid,
  title={Hybrid RL: Using Both Offline and Online Data Can Make RL Efficient},
  author={Song, Yuda and Zhou, Yifei and Sekhari, Ayush and Bagnell, J Andrew and Krishnamurthy, Akshay and Sun, Wen},
  booktitle={International Conference on Learning Representations},
  year={2023}
}

@misc{agarwal2019reinforcement,
  title={Reinforcement learning: Theory and algorithms},
  author={Agarwal, Alekh and Jiang, Nan and Kakade, Sham M and Sun, Wen},
  howpublished ={\url{https://rltheorybook.github.io/}},
  year={2022},
  note={Version: January 31, 2022}
}

@inproceedings{zhang2022efficient,
  title={Efficient Reinforcement Learning in Block MDPs: A Model-Free Representation Learning Approach},
  author={Zhang, Xuezhou and Song, Yuda and Uehara, Masatoshi and Wang, Mengdi and Agarwal, Alekh and Sun, Wen},
  booktitle={International Conference on Machine Learning},
  year={2022},
}

@inproceedings{zhou2021nearly,
  title={Nearly Minimax Optimal Reinforcement Learning for Linear Mixture Markov Decision Processes},
  author={Zhou, Dongruo and Gu, Quanquan and Szepesvari, Csaba},
  booktitle={Conference on Learning Theory},
  year={2021},
}

@inproceedings{schwarzer2020data,
  title={Data-Efficient Reinforcement Learning with Self-Predictive Representations},
  author={Schwarzer, Max and Anand, Ankesh and Goel, Rishab and Hjelm, R Devon and Courville, Aaron and Bachman, Philip},
  booktitle={International Conference on Learning Representations},
  year={2020}
}

@inproceedings{amortila2024scalable,
  title={Scalable Online Exploration via Coverability},
  author={Amortila, Philip and Foster, Dylan J and Krishnamurthy, Akshay},
  booktitle={International Conference on Machine Learning},
  year={2024}
}

@inproceedings{chen2019information,
  title={Information-Theoretic Considerations in Batch Reinforcement Learning},
  author={Chen, Jinglin and Jiang, Nan},
  booktitle={International Conference on Machine Learning},
  year={2019},
}

@book{cesa2006prediction,
  title={Prediction, Learning, and Games},
  author={Cesa-Bianchi, Nicolo and Lugosi, G{\'a}bor},
  year={2006},
  publisher={Cambridge university press}
}

@book{Sara00,
	Author = {S. A. van de Geer},
	Publisher = {Cambridge University Press},
	Title = {Empirical Processes in {M}-{E}stimation.},
	Year = {2000}}

@inproceedings{agarwal2014taming,
  title={Taming the monster: A fast and simple algorithm for contextual bandits},
  author={Agarwal, Alekh and Hsu, Daniel and Kale, Satyen and Langford, John and Li, Lihong and Schapire, Robert},
  booktitle={International Conference on Machine Learning},
  year={2014}
}

@inproceedings{rigollet2023high,
  title={High-dimensional statistics},
  author={Rigollet, Philippe and H{\"u}tter, Jan-Christian},
  booktitle={arXiv:2310.19244},
  year={2023}
}

@book{boucheron2013concentration,
  title={Concentration inequalities: A nonasymptotic theory of independence},
  author={Boucheron, St{\'e}phane and Lugosi, G{\'a}bor and Massart, Pascal},
  publisher={Oxford university press},
  year={2013},
}

@inproceedings{gelada2019deepmdp,
  title={DeepMDP: Learning Continuous Latent Space Models for Representation Learning},
  author={Gelada, Carles and Kumar, Saurabh and Buckman, Jacob and Nachum, Ofir and Bellemare, Marc G},
  booktitle={International Conference on Machine Learning},
  year={2019},
}

@inproceedings{jiang2024note,
  title={A Note on Loss Functions and Error Compounding in Model-based Reinforcement Learning},
  author={Jiang, Nan},
  booktitle={arXiv:2404.09946},
  year={2024}
}

@inproceedings{russo2013eluder,
  title={Eluder Dimension and the Sample Complexity of Optimistic Exploration},
  author={Russo, Daniel and Van Roy, Benjamin},
  booktitle={Neural Information Processing Systems},
  year={2013}
}

@inproceedings{foster2020beyond,
  title={Beyond {UCB}: Optimal and Efficient Contextual Bandits With Regression Oracles},
  author={Foster, Dylan J and Rakhlin, Alexander},
  booktitle={International Conference on Machine Learning},
  year={2020}
}

@inproceedings{ayoub2020model,
  title={Model-Based Reinforcement Learning with Value-Targeted Regression},
  author={Ayoub, Alex and Jia, Zeyu and Szepesvari, Csaba and Wang, Mengdi and Yang, Lin},
  booktitle={International Conference on Machine Learning},
  year={2020},
}

@inproceedings{wang2020provably,
  title={Reinforcement Learning with General Value Function Approximation: Provably Efficient Approach via Bounded Eluder Dimension},
  author={Wang, Ruosong and Salakhutdinov, Russ R and Yang, Lin},
  booktitle={Neural Information Processing Systems},
  year={2020}
}

@inproceedings{krishnamurthy2016pac,
  title={{PAC} Reinforcement Learning With Rich Observations},
  author={Krishnamurthy, Akshay and Agarwal, Alekh and Langford, John},
  booktitle={Neural Information Processing Systems},
  year={2016}
}

@inproceedings{jiang2017contextual,
  title={Contextual Decision Processes With Low {Bellman} Rank Are {PAC}-Learnable},
  author={Jiang, Nan and Krishnamurthy, Akshay and Agarwal, Alekh and Langford, John and Schapire, Robert E},
  booktitle={International Conference on Machine Learning},
  year={2017}
}

@inproceedings{du2019latent,
  title={Provably Efficient {RL} With Rich Observations via Latent State Decoding},
  author={Du, Simon and Krishnamurthy, Akshay and Jiang, Nan and Agarwal, Alekh and Dudik, Miroslav and Langford, John},
  booktitle={International Conference on Machine Learning},
  year={2019},
}

@inproceedings{misra2020kinematic,
  title={Kinematic State Abstraction and Provably Efficient Rich-Observation Reinforcement Learning},
  author={Misra, Dipendra and Henaff, Mikael and Krishnamurthy, Akshay and Langford, John},
  booktitle={International Conference on Machine Learning},
  year={2020},
}

@inproceedings{agarwal2020flambe,
  title={{FLAMBE}: Structural Complexity and Representation Learning of Low Rank {MDP}s},
  author={Agarwal, Alekh and Kakade, Sham and Krishnamurthy, Akshay and Sun, Wen},
  booktitle={Neural Information Processing Systems},
  year={2020}
}

@inproceedings{azar2017minimax,
  title={Minimax Regret Bounds for Reinforcement Learning},
  author={Azar, Mohammad Gheshlaghi and Osband, Ian and Munos, R{\'e}mi},
  booktitle={International Conference on Machine Learning},
  year={2017}
}

@inproceedings{zhang2006from,
  title={From $\epsilon$-entropy to {KL}-entropy: Analysis of minimum information complexity density estimation},
  author={Zhang, Tong},
  booktitle={The Annals of Statistics},
  volume={34},
  number={5},
  pages={2180--2210},
  year={2006},
  publisher={Institute of Mathematical Statistics}
}

@inproceedings{du2021bilinear,
  title={Bilinear Classes: A Structural Framework for Provable Generalization in {RL}},
  author={Du, Simon S and Kakade, Sham M and Lee, Jason D and Lovett, Shachar and Mahajan, Gaurav and Sun, Wen and Wang, Ruosong},
  booktitle={International Conference on Machine Learning},
  year={2021}
}

@book{li2009unifying,
  title={A Unifying Framework for Computational Reinforcement Learning Theory},
  author={Li, Lihong},
  year={2009},
  publisher={Rutgers, The State University of New Jersey}
}

@inproceedings{sun2019model,
  title={Model-Based {RL} in Contextual Decision Processes: {PAC} Bounds and Exponential Improvements Over Model-Free Approaches},
  author={Sun, Wen and Jiang, Nan and Krishnamurthy, Akshay and Agarwal, Alekh and Langford, John},
  booktitle={Conference on Learning Theory},
  year={2019},
  
}

@inproceedings{jin2021bellman,
  title={Bellman Eluder Dimension: New Rich Classes of {RL} Problems, and Sample-Efficient Algorithms},
  author={Jin, Chi and Liu, Qinghua and Miryoosefi, Sobhan},
  booktitle={Neural Information Processing Systems},
  year={2021}
}

@inproceedings{weisz2021exponential,
  title={Exponential Lower Bounds for Planning in {MDPs} With Linearly-Realizable Optimal Action-Value Functions},
  author={Weisz, Gell{\'e}rt and Amortila, Philip and Szepesv{\'a}ri, Csaba},
  booktitle={Algorithmic Learning Theory},
  year={2021},
}

@inproceedings{amortila2024mitigating,
  title={Mitigating Covariate Shift in Misspecified Regression With Applications to Reinforcement Learning},
  author={Amortila, Philip and Cao, Tongyi and Krishnamurthy, Akshay},
  booktitle={Conference on Learning Theory},
  year={2024}
}

@inproceedings{weisz2021query,
  title={On Query-Efficient Planning in MDPs Under Linear Realizability of the Optimal State-Value Function},
  author={Weisz, Gellert and Amortila, Philip and Janzer, Barnab{\'a}s and Abbasi-Yadkori, Yasin and Jiang, Nan and Szepesv{\'a}ri, Csaba},
  booktitle={Conference on Learning Theory},
  year={2021},
}

@inproceedings{modi2020sample,
  title={Sample Complexity of Reinforcement Learning Using Linearly Combined Model Ensembles},
  author={Modi, Aditya and Jiang, Nan and Tewari, Ambuj and Singh, Satinder},
  booktitle={International Conference on Artificial Intelligence and Statistics},
  year={2020},
}

@inproceedings{dong2019provably,
  title={Provably Efficient Reinforcement Learning With Aggregated States},
  author={Dong, Shi and Van Roy, Benjamin and Zhou, Zhengyuan},
  booktitle={arXiv:1912.06366},
  year={2019}
}

@book{lattimore2020bandit,
  title={Bandit Algorithms},
  author={Lattimore, Tor and Szepesv{\'a}ri, Csaba},
  year={2020},
  publisher={Cambridge University Press}
}

@inproceedings{zhang2021feel,
  title={Feel-Good Thompson Sampling for Contextual Bandits and Reinforcement Learning},
  author={Zhang, Tong},
  booktitle={SIAM Journal on Mathematics of Data Science},
  year={2022},
}

@inproceedings{foster2021statistical,
  title={The Statistical Complexity of Interactive Decision Making},
  author={Foster, Dylan J and Kakade, Sham M and Qian, Jian and Rakhlin, Alexander},
  booktitle={arXiv:2112.13487},
  year={2021}}

@inproceedings{xie2022role,
  title={The Role of Coverage in Online Reinforcement Learning},
  author={Xie, Tengyang and Foster, Dylan J and Bai, Yu and Jiang, Nan and Kakade, Sham M},
  booktitle={International Conference on Learning Representations},
  year={2023}
}

@inproceedings{amortila2023harnessing,
  title={Harnessing Density Ratios for Online Reinforcement Learning},
  author={Amortila, Philip and Foster, Dylan J. and Jiang, Nan and Sekhari, Ayush and Xie, Tengyang},
  booktitle={International Conference on Learning Representations},
  year={2024}
}

@inproceedings{foster2023tight,
  title={Tight Guarantees for Interactive Decision Making with the Decision-Estimation Coefficient},
  author={Foster, Dylan J and Golowich, Noah and Han, Yanjun},
  booktitle={Conference on Learning Theory},
  year=2023
}

@inproceedings{osband2016lower,
  title={On Lower Bounds for Regret in Reinforcement Learning},
  author={Osband, Ian and Van Roy, Benjamin},
  booktitle={arXiv:1608.02732},
  year={2016}
}

@inproceedings{domingues2021episodic,
  title={Episodic Reinforcement Learning in Finite MDPs: Minimax Lower Bounds Revisited},
  author={Domingues, Omar Darwiche and M{\'e}nard, Pierre and Kaufmann, Emilie and Valko, Michal},
  booktitle={Algorithmic Learning Theory},
  year={2021},
}

@inproceedings{jin2020provably,
  title={Provably Efficient Reinforcement Learning With Linear Function Approximation},
  author={Jin, Chi and Yang, Zhuoran and Wang, Zhaoran and Jordan, Michael I},
  booktitle={Conference on Learning Theory},
  year={2020},
}

@inproceedings{laskin2020curl,
  title={CURL: Contrastive Unsupervised Representations for Reinforcement Learning},
  author={Laskin, Michael and Srinivas, Aravind and Abbeel, Pieter},
  booktitle={International Conference on Machine Learning},
  year={2020},
}

@inproceedings{nair2022r3m,
  title={R3M: A Universal Visual Representation for Robot Manipulation},
  author={Nair, Suraj and Rajeswaran, Aravind and Kumar, Vikash and Finn, Chelsea and Gupta, Abhinav},
  booktitle={arXiv:2203.12601},
  year={2022}
}

@inproceedings{lamb2022guaranteed,
  title={Guaranteed Discovery of Control-Endogenous Latent States with Multi-Step Inverse Models},
  author={Lamb, Alex and Islam, Riashat and Efroni, Yonathan and Didolkar, Aniket Rajiv and Misra, Dipendra and Foster, Dylan J and Molu, Lekan P and Chari, Rajan and Krishnamurthy, Akshay and Langford, John},
  booktitle={Transactions on Machine Learning Research},
  year={2024}
}

@inproceedings{yarats2021image,
  title={Image Augmentation Is All You Need: Regularizing Deep Reinforcement Learning From Pixels},
  author={Yarats, Denis and Fergus, Rob and Kostrikov, Ilya},
  booktitle={International Conference on Learning Representations},
  year={2021}
}

@inproceedings{hafner2023mastering,
  title={Mastering Diverse Domains Through World Models},
  author={Hafner, Danijar and Pasukonis, Jurgis and Ba, Jimmy and Lillicrap, Timothy},
  booktitle={arXiv:2301.04104},
  year={2023}
}

@inproceedings{brohan2022rt,
  title={RT-1: Robotics Transformer for Real-World Control at Scale},
  author={Brohan, Anthony and Brown, Noah and Carbajal, Justice and Chebotar, Yevgen and Dabis, Joseph and Finn, Chelsea and Gopalakrishnan, Keerthana and Hausman, Karol and Herzog, Alex and Hsu, Jasmine and Ibarz, Julian and Ichter, Brian and Irpan, Alex and Jackson, Tomas and Jesmonth, Sally and Joshi, Nikhil J and Julian, Ryan and Kalashnikov, Dmitry and Kuang, Yuheng and Leal, Isabel and Lee, Kuang-Huei and Levine, Sergey and Lu, Yao and Malla, Utsav and Manjunath, Deeksha and Mordatch, Igor and Nachum, Ofir and Parada, Carolina and Peralta, Jodilyn and Perez, Emily and Pertsch, Karl and Quiambao, Jornell and Rao, Kanishka and Ryoo, Michael and Salazar, Grecia and Sanketi, Pannag and Sayed, Kevin and Singh, Jaspiar and Sontakke, Sumedh and Stone, Austin and Tan, Clayton and Tran, Huong and Vanhoucke, Vincent and Vega, Steve and Vuong, Quan and Xia, Fei and Xiao, Ted and Xu, Peng and Xu, Sichun and Yu, Tianhe and Zitkovich, Brianna},
  booktitle={arXiv:2212.06817},
  year={2022}
}

@inproceedings{zhang2020learning,
  title={Learning Invariant Representations for Reinforcement Learning Without Reconstruction},
  author={Zhang, Amy and McAllister, Rowan and Calandra, Roberto and Gal, Yarin and Levine, Sergey},
  booktitle={International Conference on Learning Representations},
  year={2021}
}

@inproceedings{tang2017exploration,
  title={\# Exploration: A Study of Count-Based Exploration for Deep Reinforcement Learning},
  author={Tang, Haoran and Houthooft, Rein and Foote, Davis and Stooke, Adam and Xi Chen, OpenAI and Duan, Yan and Schulman, John and DeTurck, Filip and Abbeel, Pieter},
  booktitle={Neural Information Processing Systems},
  year={2017}
}

@inproceedings{pathak2017curiosity,
  title={Curiosity-Driven Exploration by Self-Supervised Prediction},
  author={Pathak, Deepak and Agrawal, Pulkit and Efros, Alexei A and Darrell, Trevor},
  booktitle={International Conference on Machine Learning},
  pages={2778--2787},
  year={2017},
  
}

@inproceedings{guo2022byol,
  title={BYOL-Explore: Exploration by Bootstrapped Prediction},
  author={Guo, Zhaohan and Thakoor, Shantanu and P{\^\i}slar, Miruna and Avila Pires, Bernardo and Altch{\'e}, Florent and Tallec, Corentin and Saade, Alaa and Calandriello, Daniele and Grill, Jean-Bastien and Tang, Yunhao and Valko, Michal and Munos, R{\'e}mi and Azar, Mohammad Gheshlaghi and Piot, Bilal},
  booktitle={Neural Information Processing Systems},
  year={2022}
}

@inproceedings{dean2020certainty,
  title={Certainty Equivalent Perception-Based Control},
  author={Dean, Sarah and Recht, Benjamin},
  booktitle={Learning for Dynamics and Control},
  year={2021},
}

@inproceedings{mhammedi2020learning,
  title={Learning the Linear Quadratic Regulator From Nonlinear Observations},
  author={Mhammedi, Zakaria and Foster, Dylan J and Simchowitz, Max and Misra, Dipendra and Sun, Wen and Krishnamurthy, Akshay and Rakhlin, Alexander and Langford, John},
  booktitle={Neural Information Processing Systems},
  year={2020}
}

@inproceedings{misra2021provable,
  title={Provable Rich Observation Reinforcement Learning With Combinatorial Latent States},
  author={Misra, Dipendra and Liu, Qinghua and Jin, Chi and Langford, John},
  booktitle={International Conference on Learning Representations},
  year={2021}
}

@inproceedings{efroni2021provable,
  title={Provable {RL} With Exogenous Distractors via Multistep Inverse Dynamics},
  author={Efroni, Yonathan and Misra, Dipendra and Krishnamurthy, Akshay and Agarwal, Alekh and Langford, John},
  booktitle={International Conference on Learning Representations},
  year={2022}
}

@inproceedings{mhammedi2024power,
  title={The Power of Resets in Online Reinforcement Learning},
  author={Mhammedi, Zakaria and Foster, Dylan J and Rakhlin, Alexander},
  booktitle={arXiv:2404.15417},
  year={2024}
}

@inproceedings{mhammedi2023efficient,
  title={Efficient Model-Free Exploration in Low-Rank MDPs},
  author={Mhammedi, Zak and Block, Adam and Foster, Dylan J and Rakhlin, Alexander},
  booktitle={Neural Information Processing Systems},
  year={2023}
}

@inproceedings{uehara2022representation,
	author    = {Masatoshi Uehara and
	Xuezhou Zhang and
	Wen Sun},
	title     = {Representation Learning for Online and Offline {RL} in Low-rank MDPs},
	booktitle = {The Tenth International Conference on Learning Representations},
	year      = {2022},
}

@inproceedings{modi2021model,
  title={Model-Free Representation Learning and Exploration in Low-Rank Mdps},
  author={Modi, Aditya and Chen, Jinglin and Krishnamurthy, Akshay and Jiang, Nan and Agarwal, Alekh},
  booktitle={Journal of Machine Learning Research},
  year={2024}
}

@inproceedings{foster2023model,
  title={Model-Free Reinforcement Learning with the Decision-Estimation Coefficient},
  author={Foster, Dylan J and Golowich, Noah and Qian, Jian and Rakhlin, Alexander and Sekhari, Ayush},
  booktitle={Neural Information Processing Systems},
  year={2023}
}

@inproceedings{dean2019robust,
  title={Robust Guarantees for Perception-Based Control},
  author={Dean, Sarah and Matni, Nikolai and Recht, Benjamin and Ye, Vickie},
  booktitle={Learning for Dynamics and Control},
  year={2020}
}

@inproceedings{efroni2022sample,
  title={Sample-Efficient Reinforcement Learning in the Presence of Exogenous Information},
  author={Efroni, Yonathan and Foster, Dylan J and Misra, Dipendra and Krishnamurthy, Akshay and Langford, John},
  booktitle={Conference on Learning Theory},
  year={2022},
}

@inproceedings{feng2020provably,
  title={Provably Efficient Exploration for Reinforcement Learning Using Unsupervised Learning},
  author={Feng, Fei and Wang, Ruosong and Yin, Wotao and Du, Simon S and Yang, Lin},
  booktitle={Neural Information Processing Systems},
  year={2020}
}

@inproceedings{wu2021reinforcement,
  title={On Reinforcement Learning With Adversarial Corruption and Its Application to Block MDP},
  author={Wu, Tianhao and Yang, Yunchang and Du, Simon and Wang, Liwei},
  booktitle={International Conference on Machine Learning},
  year={2021},
  
}

@inproceedings{xie2021batch,
  title={Batch Value-Function Approximation With Only Realizability},
  author={Xie, Tengyang and Jiang, Nan},
  booktitle={International Conference on Machine Learning},
  year={2021},
  
}

@inproceedings{golowich2023exploring,
  title={Exploring and Learning in Sparse Linear MDPs Without Computationally Intractable Oracles},
  author={Golowich, Noah and Moitra, Ankur and Rohatgi, Dhruv},
  booktitle={Symposium on Theory of Computing},
  year={2024}
}

@inproceedings{hao2021online,
  title={Online Sparse Reinforcement Learning},
  author={Hao, Botao and Lattimore, Tor and Szepesv{\'a}ri, Csaba and Wang, Mengdi},
  booktitle={International Conference on Artificial Intelligence and Statistics},
  year={2021},
}

@inproceedings{guo2023sample,
  title={Sample-Efficient Learning of POMDPs with Multiple Observations In Hindsight},
  author={Guo, Jiacheng and Chen, Minshuo and Wang, Huan and Xiong, Caiming and Wang, Mengdi and Bai, Yu},
  booktitle={International Conference on Learning Representations},
  year={2024}
}

@inproceedings{shi2023theoretical,
  title={Theoretical Hardness and Tractability of POMDPs in RL with Partial Hindsight State Information},
  author={Shi, Ming and Liang, Yingbin and Shroff, Ness},
  booktitle={arXiv:2306.08762},
  year={2023}
}

@inproceedings{lanier2024learning,
  title={Learning Interpretable Policies in Hindsight-Observable POMDPs through Partially Supervised Reinforcement Learning},
  author={Lanier, Michael and Xu, Ying and Jacobs, Nathan and Zhang, Chongjie and Vorobeychik, Yevgeniy},
  booktitle={arXiv:2402.09290},
  year={2024}
}

@inproceedings{foster2023foundations,
  title={Foundations of Reinforcement Learning and Interactive Decision Making},
  author={Foster, Dylan J and Rakhlin, Alexander},
  booktitle={arXiv:2312.16730},
  year={2023}
}

@inproceedings{hafner2019dream,
  title={Dream to Control: Learning Behaviors by Latent Imagination},
  author={Hafner, Danijar and Lillicrap, Timothy and Ba, Jimmy and Norouzi, Mohammad},
  booktitle={International Conference on Learning Representations},
  year={2019}
}

@inproceedings{hafner2019learning,
  title={Learning Latent Dynamics for Planning From Pixels},
  author={Hafner, Danijar and Lillicrap, Timothy and Fischer, Ian and Villegas, Ruben and Ha, David and Lee, Honglak and Davidson, James},
  booktitle={International Conference on Machine Learning},
  year={2019},
}

@inproceedings{hafner2020mastering,
  title={Mastering Atari With Discrete World Models},
  author={Hafner, Danijar and Lillicrap, Timothy and Norouzi, Mohammad and Ba, Jimmy},
  booktitle={International Conference on Learning Representations},
  year={2021}
}

@inproceedings{schrittwieser2020mastering,
  title={Mastering Atari, Go, Chess and Shogi by Planning With a Learned Model},
  author={Schrittwieser, Julian and Antonoglou, Ioannis and Hubert, Thomas and Simonyan, Karen and Sifre, Laurent and Schmitt, Simon and Guez, Arthur and Lockhart, Edward and Hassabis, Demis and Graepel, Thore and Lillicrap, Timothy and Silver, David},
  booktitle={Nature},
  year={2020},
}

@inproceedings{tang2023understanding,
  title={Understanding Self-Predictive Learning for Reinforcement Learning},
  author={Tang, Yunhao and Guo, Zhaohan Daniel and Richemond, Pierre Harvey and Avila Pires, Bernardo and Chandak, Yash and Munos, Remi and Rowland, Mark and Gheshlaghi Azar, Mohammad and Le Lan, Charline and Lyle, Clare and Gy\"{o}rgy, Andr\'{a}s and Thakoor, Shantanu and Dabney, Will and Piot, Bilal and Calandriello, Daniele and Valko, Michal},
  booktitle={International Conference on Machine Learning},
  year={2023},
}

@inproceedings{block2024performance,
  title={On the Performance of Empirical Risk Minimization with Smoothed Data},
  author={Block, Adam and Rakhlin, Alexander and Shetty, Abhishek},
  booktitle={Conference on Learning Theory},
  year={2024}
}

@inproceedings{wahlstrom2015pixels,
  title={From Pixels to Torques: Policy Learning With Deep Dynamical Models},
  author={Wahlstr{\"o}m, Niklas and Sch{\"o}n, Thomas B and Deisenroth, Marc Peter},
  booktitle={International Conference on Machine Learning},
  year={2015}
}

@inproceedings{levine2016end,
  title={End-To-End Training of Deep Visuomotor Policies},
  author={Levine, Sergey and Finn, Chelsea and Darrell, Trevor and Abbeel, Pieter},
  booktitle={The Journal of Machine Learning Research},
  year={2016},
}

@inproceedings{kumar2021rma,
  title={RMA: Rapid Motor Adaptation for Legged Robots},
  author={Kumar, Ashish and Fu, Zipeng and Pathak, Deepak and Malik, Jitendra},
  booktitle={Robotics: Science and Systems},
  year={2021}
}

@inproceedings{baker2022video,
  title={Video PreTraining (VPT): Learning to Act by Watching Unlabeled Online Videos},
  author={Baker, Bowen and Akkaya, Ilge and Zhokov, Peter and Huizinga, Joost and Tang, Jie and Ecoffet, Adrien and Houghton, Brandon and Sampedro, Raul and Clune, Jeff},
  booktitle={Neural Information Processing Systems},
  year={2022}
}

\newpage

\appendix

\renewcommand{\contentsname}{Contents}
\addtocontents{toc}{\protect\setcounter{tocdepth}{2}}
{\hypersetup{hidelinks}
\tableofcontents
}

\neurips{
\newpage
\section{Omitted Results from \cref{sec:algorithmic-results}: Algorithmic Modularity via Self-predictive Estimation}\label{sec:online-rl}

}

\newpage
\section{Additional Discussion of Related Work}\label{app:additional_related}
In this section, we discuss aspects of related work not already covered
in greater detail.

\paragraph{Reinforcement learning under latent dynamics (or, with rich observations).}
Reinforcement learning under latent dynamics (or, with rich observations) has received extensive
investigation in recent years, however most works have been focused on the \emph{Block MDP} model in which the latent state space is tabular/finite
\citep{krishnamurthy2016pac,du2019latent,misra2020kinematic,zhang2022efficient,mhammedi2023representation} (see also the the closely related framework of \emph{Low-Rank MDPs}
\citep{agarwal2020flambe,modi2021model,zhang2022efficient,uehara2022representation,mhammedi2023efficient}). Beyond tabular spaces, \citet{dean2019robust,dean2020certainty,mhammedi2020learning} consider continuous \emph{linear}
dynamics, \citet{misra2021provable} considers factored (but
discrete) latent dynamics, 
\citet{efroni2021provable,efroni2022sample,mhammedi2024power} consider the Exogenous Block MDP problem 
in which a tabular latent state space is augmented with a non-controllable (``exogenous'') factor, and \citet{song2024rich} consider Lipshitz continuous dynamics. To our knowledge, our work is
the first to: i) explore reinforcement learning under general latent dynamics, in particular in settings where the latent space itself admits function approximation, and ii) take a more modular approach (cf. the taxonomy of \cref{sec:statistical-results}).

On the algorithmic side, the works of \citet{uehara2022representation}
and \citet{zhang2022efficient}, which consider Low-Rank MDPs and Block
MDPs respectively, can be viewed as interleaving representation
learning with ``latent'' reinforcement learning algorithms that assume
access to a good representation, and were an inspiration for this
work. However, the algorithmic details and analyses are highly
specialized to Block/Low-Rank MDPs, and unlikely to be directly
applicable to reinforcement learning under general latent
dynamics. Other works with a modular flavor include:
\begin{itemize}
\item \citet{feng2020provably} solve tabular Block MDPs by combining
  a black-box latent algorithm with an ``unsupervised learning
  oracle'' for representation learning. This approach only leads to
  guarantees for tabular Block MDPs, and it is unclear whether the unsupervised learning oracle
  their approach requires can be constructed in natural settings.
\item \citet{wu2021reinforcement} solve tabular block MDPs by
  combining a corruption-robust latent algorithm with a representation
  learning procedure based on clustering. Again, this work is
  restricted to the tabular setting, and requires a separation
  condition which may not be satisfied in general.
\end{itemize}

\paragraph{General complexity measures for reinforcement learning.}
Another line of research provides general complexity measures that
enable sample-efficient reinforcement learning, including Bellman rank
\citep{jiang2017contextual,sun2019model,du2021bilinear,jin2021bellman},
eluder dimension \citep{russo2013eluder}, coverability
\citep{xie2022role}, and the Decision-Estimation Coefficient (DEC)
\citep{foster2021statistical,foster2023tight,foster2023model}. Bellman
rank and other complexity measures based on average Bellman error \citep{jiang2017contextual,sun2019model,du2021bilinear,jin2021bellman} are
insufficient to characterize learnability under general latent
dynamics, as there are classes $\cMlat$ that are known to be
learnable, yet do not have bounded Bellman rank or Bellman-Eluder
dimension \citep{efroni2021provable,xie2022role}. Meanwhile, variants
of Bellman rank based on squared Bellman error or related notions of
error can \citep{xie2022role,amortila2023harnessing} address this
problem for some settings, but satisfying the modeling/realizability
assumptions (e.g., Bellman completeness) required by these methods in
the latent-dynamics setting is non-trivial. For example, the crux
of our sample complexity bounds under latent pushforward coverability
in \cref{sec:statistical-results} (\cref{thm:pushforwardgolf}) is to
prove a rather involved structural result which shows that Bellman
completeness can indeed be satisfied under this assumption, but it is
unclear whether these techniques can be applied to more general latent
dynamics classes. We expect that it is possible to bound
the Decision-Estimation Coefficient
\citep{foster2021statistical,foster2023tight,foster2023model} for the
framework, but deriving efficient algorithms using this
framework is non-trivial.

\newpage
\section{Technical Tools}
 \begin{lemma}\label{lem:log-exp-concentration}
 For any sequence of real-valued random variables $(X_t)_{t \leq T}$ adapted to a filtration $\prn{\filt_t}_{t\leq T}$, it holds that with probability at least $1-\delta$,
 	\[
 	\sum_{t=1}^T X_t \leq \sum_{t=1}^T \log\prn*{\En_{t-1}\brk*{e^{X_t}}} + \log(\deltainv).
 	\]
 \end{lemma}

  \begin{lemma}[Freedman's inequality (e.g., \citet{agarwal2014taming})]
  \label{lem:freedman}
  Let $(X_t)_{t\leq{T}}$ be a real-valued martingale difference
  sequence adapted to a filtration $\prn{\filt_t}_{t\leq{}T}$. If
  $\abs*{X_t}\leq{}R$ almost surely, then for any $\eta\in(0,1/R)$, with probability at least $1-\delta$,
    \[
      \sum_{t=1}^{T}X_t \leq{} \eta\sum_{t=1}^{T}\En_{t-1}\brk*{X_t^{2}} + \frac{\log(\delta^{-1})}{\eta}.
    \]
  \end{lemma}

  \begin{lemma}[Corollary of \cref{lem:freedman}]
      \label{lem:multiplicative_freedman}
            Let $(X_t)_{t\leq{T}}$ be a sequence of random
      variables adapted to a filtration $\prn{\filt_{t}}_{t\leq{}T}$. If
  $0\leq{}X_t\leq{}R$ almost surely, then with probability at least
  $1-\delta$,
  \begin{align*}
    &\sum_{t=1}^{T}X_t \leq{}
                        \frac{3}{2}\sum_{t=1}^{T}\En_{t-1}\brk*{X_t} +
                        4R\log(2\delta^{-1}),
    \intertext{and}
      &\sum_{t=1}^{T}\En_{t-1}\brk*{X_t} \leq{} 2\sum_{t=1}^{T}X_t + 8R\log(2\delta^{-1}).
  \end{align*}
    \end{lemma}
    
  \begin{lemma}[Lemma D.2 of \citet{foster2024online}]\label{lem:chain-rule-hellinger}
	Let $(\cX_1, \fF_1), \ldots, (\cX_n, \fF_n)$ be a sequence of measurable spaces, and let $\cX\ind{i} = \prod_{t=1}^i \cX_t$ and $\fF\ind{i} = \otimes_{t=1}^i \fF_t$. For each $i$, let $P\ind{i}$ and $Q\ind{i}$ be probability kernels from $(\cX\ind{i-1},\fF\ind{i-1})$ to $(\cX_i, \fF_i)$. Let $P$ and $Q$ be the laws of $X_1, \ldots, X_n$ under $X_i \sim P\ind{i}(\cdot \mid X_{1:i-1})$ and $X_i \sim Q\ind{i}(\cdot \mid X_{1:i-1})$, respectively. Then it holds that
  		\[
  		\Dhels{P}{Q} \leq 7 \En_P\brk*{\sum_{i=1}^n \Dhels{P\ind{i}(\cdot \mid X_{1:i-1})}{Q\ind{i}(\cdot \mid X_{1:i-1})}}
  		\]
  \end{lemma}

	\begin{lemma}[Lemma A.11 of \citet{foster2021statistical}]\label{lem:a-11-dec}
		Let $\bbP$ and $\bbQ$ be probability measures on $(\cX, \fF)$. For all $h: \cX \rightarrow \bbR$ with $0 \leq h(X) \leq R$ almost surely under $\bbP$ and $\bbQ$, we have
		\[
			\En_\bbP\brk*{h(X)} \leq 3\En_\bbQ\brk*{h(X)} + 4R\Dhels{\bbP}{\bbQ}.
		\]
	\end{lemma}

\begin{lemma}[Lemma 1 of \citet{jiang2017contextual}]\label{lem:lemma1jiang}
For any $f: \cX \times \cA \rightarrow [0,1]$, $\pi: \cS \times [H] \rightarrow \Delta(\cA)$, we have
	\[
		\En_{x_1}\brk*{f(x_1,\pi(x_1))} - J(\pi) = \sum_{h=1}^H \En^\pi\brk*{f(x_h,a_h) - \cT^\pi f(x_h,a_h)}.
	\]
\end{lemma}

\begin{lemma}[Offline-to-online conversion under coverability \cite{xie2022role,foster2024online}]\label{lem:offline-to-online}
	Let $M$ be an MDP over state space $\cZ$, $\Pi$ be a policy set, and $\Ccov = \Ccov(M,\Pi)$ be the (state-action) coverability coefficient for $M$ and $\Pi$ (\cref{def:coverability}). Let $p\ind{t} \in \Delta(\Pi)$ be a sequence of distributions over $\Pi$, and $g\ind{t}_h: \cZ \times \cA \rightarrow [0,1]$ be a sequence of functions. Then we have that %
	\[
		\sum_{t=1}^T \sum_{h=1}^H \En_{\pi\ind{t} \sim p\ind{t}} \En^{\pi\ind{t}}\brk*{g\ind{t}_h(x_h,a_h)} \leq \cO\prn*{\sqrt{H\Ccov \log(T)\sum_{t=1}^T \sum_{h=1}^H \sum_{i=1}^{t-1}\En_{\pi\ind{i} \sim p\ind{i}}\En^{\pi\ind{i}}\brk*{g\ind{t}_h(x_h,a_h)}} + H\Ccov}.
	\]
\end{lemma}

\newpage
\section{Structural Properties of Coverability and Mismatch Functions}
This appendix contains structural results regarding coverability and the mismatch functions. We firstly recall the definition of the mismatch functions.

\begin{definition}[Mismatch functions]\label{def:h_phi_redux}
	For decodable emission process $\psistar$, decoder $\phi \in \Phi$ and $h \in [H]$, we define the \emph{mismatch function} for $\phi$, $\hphih:\cS \rightarrow \Delta(\cS)$, as the probability kernel
	\[
		\hphih(s'_h \mid s_h) \coloneqq \bbP_{x_h \sim \psistarh(s_h)}\prn*{\phih(x_h) = s'_h}.
	\]
\end{definition}
We also recall the definition of state coverability.
\begin{definition}[State Coverability]\label{def:state-cov-redux}\label{def:coverability} The coverability coefficient for an MDP $M$ and a policy class $\Pi$ defined over a state space $\cZ$, $\Ccovs(M,\Pi)$, is given by
	\begin{equation}\label{eq:coverability_redux}
		\Ccovs(M,\Pi) \coloneqq \max_{h \in [H]} \min_{\mu \in \Delta(\cZ)}\max_{\pi \in \Pi} \max_{z \in \cZ} \crl*{\frac{d^{\sM,\pi}_h(z)}{\mu(z)}}.
	\end{equation}
\end{definition}

We also define the related notion of state-action coverability. 
\begin{definition}[State-Action Coverability]\label{def:state-action-cov}\label{def:coverability} The coverability coefficient for an MDP $M$ and a policy class $\Pi$ defined over a state space $\cZ$ and action space $\cA$, $\Ccov(M,\Pi)$, is given by
	\begin{equation}\label{eq:coverability_redux}
		\Ccov(M,\Pi) \coloneqq \max_{h \in [H]} \min_{\mu \in \Delta(\cZ \times \cA)}\max_{\pi \in \Pi} \max_{z,a \in \cZ \times \cA} \crl*{\frac{d^{\sM,\pi}_h(z,a)}{\mu(z,a)}}.
	\end{equation}
\end{definition}

In the remainder of the section, we let $\Pilat \subseteq \crl{\cS \times [H] \rightarrow \Delta(\cA)}$ denote an arbitrary set of latent policies, and 
\begin{equation}\label{eq:hPhi-Pilat-policies-redux}
	\hPhi \circ \Pilat = \crl*{[\hphi\circ \pilat]_h(a \mid s) \coloneqq \sum_{s' \in \cS} \hphih(s' \mid s) \pi_{\lat,h}(a \mid s') \mid \phi \in \Phi, \pilat \in \Pilat}. %
\end{equation}
\begin{restatable}[State coverability is invariant to rich observations]{lemma}{covinvariance}\label{lem:covinvariance}
Let $\Mstarobs = \dtri*{\Mstarlat, \psistar}$. Then, we have 
	\[
	\Ccovs(\Mstarobs, \Pilat \circ \Phi) = \Ccovs(\Mstarlat, \hPhi \circ \Pilat).
	\]
Furthermore, letting $\crl*{\mu_{\lat,h} \in \Delta(\cS)}_{h \in [H]}$ denote the distribution which witnesses the right-hand-side, the left-hand-side is witnessed by the distribution 
\[
\mu_{\obs,h}(x) = \psistarh(x \mid \phistarh(x))\mu_{\lat,h}(\phistarh(x)).
\]
\end{restatable}

The lemma follows from the following two observations.

\begin{restatable}{lemma}{hphistateocc}\label{lem:hphi-state-occ}
	Let $\crl*{\hphi}_{\phi \in \Phi}$ denote the mismatch functions for emission $\psistar$, and let $\Mobs = \dtri*{\Mlat, \psistar}$. Then, for any $\pilat \in \Pilat$, $\phi \in \Phi$, $h \in [H]$, $x \in \cX$, we have
	\[
		 d_{h}^{\Mobs, \pilat \circ \phi}(x) = \psistarh(x \mid \phistarh(x))d_{h}^{\Mlat, \hphi \circ \pilat}(\phistarh(x)).
	\]
\end{restatable}

\begin{proof}[\pfref{lem:hphi-state-occ}]
	Below, we write $s_h = \phistar(x_h)$. We proceed by induction, simply writing $d_{\obs,h} \coloneqq d_{h}^{\Mobs,\pilat \circ \phi}$ and $d_{\lat,h} \coloneqq  d_{h}^{\Mlat,\hphi \circ \pilat}$. The base case $(h=1)$ is obtained by noting that $d_{\lat,1}(s) = P_{\lat,1}(s \mid \emptyset)$ while $d_{\obs,1}(x) = P_{\obs,1}(x \mid \emptyset) = \psistar_1(x \mid s) P_{\lat,1}(s \mid \emptyset)$. For the general case, via the Bellman flow equations, we have
	\begin{align*}
	d_{\obs,h}(x_h) &= \sum_{x_{h-1},a_{h-1} \in \cX \times \cA} \Pobsh(x_h \mid x_{h-1},a_{h-1}) d_{\obs,h-1}(x_{h-1})\pilat(a_{h-1} \mid \phi(x_{h-1})) \\
	&=  \psi(x_h \mid s_h) \sum_{x_{h-1},a_{h-1} \in \cX \times \cA} \Plath(s_h \mid s_{h-1},a_{h-1}) d_{\lat,h-1}(s_{h-1}) \psi(x_{h-1} \mid s_{h-1})\pilat(a_{h-1} \mid \phi(x_{h-1})) \\
	&= \psi(x_h \mid s_h) \sum_{s_{h-1},a_{h-1} \in \cS \times \cA} \Plath(s_h \mid s_{h-1},a_{h-1}) d_{\lat,h-1}(s_{h-1}) \times \\
		&\hspace{2in} \prn*{\sum_{x_{h-1}: \phistar(x_{h-1}) = s_{h-1}}\psi(x_{h-1} \mid s_{h-1})\pilat(a_{h-1} \mid \phi(x_{h-1}))}.
	\end{align*}
	The result is obtained by noting that 
	\begin{align*}
	\hphi \circ \pilat(a_{h-1} \mid s_{h-1}) &= \sum_{s' \in \cS} \hphi(s' \mid s_{h-1}) \pilat(a_{h-1} \mid s') \\
	&= \sum_{s' \in \cS} \sum_{x_{h-1}: \phistar(x_{h-1})=s_{h-1}} \psi(x_{h-1} \mid s_{h-1}) \indic\crl*{\phi(x_{h-1}) = s'} \pilat(a_h \mid s') \\
	&= \sum_{x_{h-1}: \phistar(x_{h-1}) = s_{h-1}}\psi(x_{h-1} \mid s_{h-1}) \pilat(a_{h-1} \mid \phi(x_{h-1})),
	\end{align*}
where the second line follows from the definition of the mismatch functions. 
\end{proof}

\begin{lemma}[Equivalence of state coverability and cumulative state reachability]\label{lem:cov-state-reachability}
	Let $M$ be an MDP defined over a state space $\cZ$.  The following definition is equivalent to \cref{def:state-cov-redux}:
	\begin{equation}\label{eq:state-cov-equivalence}
		\Ccovs(M,\Pi) \coloneqq \max_{h \in [H]} \sum_{z \in \cZ} \max_{\pi \in \Pi} d^{M,\pi}_h(z). 
	\end{equation}
\end{lemma}
\begin{proof}[\pfref{lem:cov-state-reachability}]
Straightforward adaptation of the proof of Lemma 3 from \citet{xie2022role}. 
\end{proof}

\begin{proof}[\pfref{lem:covinvariance}] Using \cref{lem:hphi-state-occ} and \cref{lem:cov-state-reachability}, we have
\begin{align*}
\Ccovs(\Mobs,\Pilat \circ \Phi) &= \max_{h \in [H]} \sum_{x \in \cX} \max_{\pilat, \phi} d^{\pilat \circ \phi}_\obs(x) \\
&= \max_{h \in [H]} \sum_{x \in \cX} \max_{\pilat, \phi} \psistar(x \mid \phistar(x)) d^{\hphi \circ \pilat}_\lat(\phistar(x)) \\
&= \max_{h \in [H]} \sum_{s \in \cS} \sum_{x : \phistar(x)=s} \max_{\pilat, \phi} \psistar(x \mid s) d^{\hphi \circ \pilat}_\lat(s) \\
&= \max_{h \in [H]} \sum_{s \in \cS} \max_{\pilat, \phi} d^{\hphi \circ \pilat}_\lat(s) \sum_{x : \phistar(x)=s}  \psistar(x \mid s) \\
&= \Ccovs(\Mlat, \hPhi \circ \Pilat).
\end{align*}
\end{proof}

Lastly, we show that state-action coverability is bounded by state coverability times the size of the action set.

\begin{lemma}[State-action coverability bound]\label{lem:state-action-cov-state-cov}
	For any MDP $M$ and policy set $\Pi$, we have
	\[
		\Ccov(M,\Pi) \leq \Ccovs(M,\Pi)\abs{\cA}.
	\]
\end{lemma}

\begin{proof}[\pfref{lem:state-action-cov-state-cov}]
	Let $\mu_s \in \Delta(\cZ)$ witness $\Ccovs(M,\Pi)$. Fix $h \in [H]$, which we omit below for cleanliness. Then, we have
	\begin{align*}
		\min_{\mu_{s,a} \in \Delta(\cZ \times \cA)} \max_{\pi \in \Pi} \max_{z,a \in \cZ \times \cA} \crl*{\frac{d^{M,\pi}(z,a)}{\mu_{s,a}(z,a)}} &\leq \max_{\pi \in \Pi} \max_{z,a \in \cZ \times \cA} \crl*{\frac{d^{M,\pi}(z)\pi(a \mid z)}{\mu_{s}(z) \nicefrac{1}{\abs{\cA}}}}\\
		&\leq \abs{\cA} \max_{\pi \in \Pi} \max_{z \in \cZ} \crl*{\frac{d^{M,\pi}(z)}{\mu_{s}(z)}} \\
		&= \Ccovs(M,\Pi)\abs{\cA}.
	\end{align*}

\end{proof}

\begin{lemma}[Pushforward coverability is invariant to rich observations]\label{lem:pushforward-invariant}
Let $\Cpush(M)$ denote the pushforward coverability parameter for an MDP $M$ (\cref{def:pushforward_def_main}), and $\Mobsstar \coloneqq \dtri*{\Mstarlat,\psistar}$. Then, we have
\[
	\Cpush(\Mstarobs) = \Cpush(\Mstarlat).
\]	
Furthermore, letting $\crl*{\mu_{\lat,h} \in \Delta(\cS)}_{h \in [H]}$ denote the distribution which witnesses the right-hand-side, the left-hand-side is witnessed by the distribution 
\[
\mu_{\obs,h}(x) = \psistarh(x \mid \phistarh(x))\mu_{\lat,h}(\phistarh(x)).
\]
\end{lemma}

This follows from an analogous equivalence of pushforward coverability and cumulative \textit{conditional reachability}.

\begin{lemma}[Equivalence of pushforward coverability and cumulative conditional reachability]\label{lem:push-reachability}
Let $M$ be an MDP defined over a state space $\cZ$ with transition kernel $P$. The following definition is equivalent to pushforward coverability (\cref{def:pushforward_def_main}):
\[
	\Cpush(M) \coloneqq \max_{h \in [H]}\sum_{z' \in \cZ} \max_{z,a \in \cZ \times \cA} P_h(z' \mid z,a).
\]
\end{lemma}
\begin{proof}[\pfref{lem:push-reachability}]
Fix $h \in [H]$, whose dependence we omit below. For the first direction, letting $\mu$ denote the pushforward coverability distribution, we have:
\[
\sum_{z' \in \cZ} \max_{z,a \in \cZ \times \cA} P(z' \mid z,a) = \sum_{z' \in \cZ} \max_{z,a \in \cZ \times \cA} \frac{P(z' \mid z,a)}{\mu(z')} \mu(z') \leq \Cpush \sum_{z' \in \cZ} \mu(z') = \Cpush.
\]
For the second direction, taking $\mu(z') \propto \max_{z,a} P(z' \mid z,a)$,
we have
\[
\min_{\mu \in \Delta(\cZ)} \max_{z,a,z' \in \cZ \times \cA \times \cZ} \frac{P(z' \mid z,a)}{\mu(z')} \leq \max_{z,a,z' \in \cZ \times \cA \times \cZ} \frac{P(z' \mid z,a)}{\max_{\tilde{z},\tilde{a}} P(z' \mid \tilde{z},\tilde{a})} \sum_{\tilde{z}'
}  \max_{\tilde{z},\tilde{a}}P(\tilde{z}' \mid \tilde{z},\tilde{a}) \leq  \sum_{z'}  \max_{z,a}P(z' \mid z,a).
\] 
\end{proof}

\begin{proof}[\pfref{lem:pushforward-invariant}]
This result follows by \cref{lem:push-reachability} since,
\begin{align*}
\Cpush(\Mobs) &= \sum_{x' \in \cX} \max_{x,a} \Pobs(x' \mid x,a) \\
	&= \sum_{s' \in \cS} \sum_{x': \phistar(x') = s'} \max_{x,a} \psistar(x' \mid s')\Plat(s' \mid \phistar(x),a)\\
	&= \sum_{s' \in \cS} \max_{x,a} \Plat(s' \mid \phistar(x),a) \sum_{x': \phistar(x') = s'}  \psistar(x' \mid s')\\
	&= \sum_{s' \in \cS} \max_{s,a} \Plat(s' \mid s,a) = \Cpush(\Mlat).
\end{align*}

\end{proof}

We next show that the mismatch functions can be used to express the observation-level backups for any function of the decoders. For any $g: \cS \rightarrow \bbR$, $h \in [H]$, we define the function $\brk{\hphih \circ g}: \cS \rightarrow \bbR$
\[
\brk{\hphih \circ g}(s) \coloneqq \sum_{s' \in \cS} \hphih(s' \mid s) g(s').
\]
We further overload the Bellman operator notation and define, for any $g: \cS \rightarrow \bbR$ and $\Mlat = (\rlat,\Plat)$, 
\[
\brk{\cT^{\Mlat}_h g}(s,a) = \rlat(s,a) + \En_{s' \sim \Plat(s,a)}\brk*{g(s')}.
\]

\begin{lemma}\label{lem:hphi-bellman}
Let $\Mobs = \dtri*{\Mlat, \psistar}$, $\phistar \coloneqq (\psistar)^{-1}$, $\phi \in \Phi$, and $\hphi$ be the mismatch function for emission $\psistar$ (\cref{def:h_phi_redux}). Then, for any $f_\lat: \cS \times \cA \rightarrow \bbR$, $h \in [H]$, and $(x,a) \in \cX \times \cA$, we have
\[
\brk*{\cT^{\Mobs}_h(f_\lat \circ \phi_{h+1})}(x,a) = \brk*{\cT^{\Mlat}_h(\hphihpo \circ V_{f_\lat})}(\phistar_h(x),a).
\]
\end{lemma}
\begin{proof}[\pfref{lem:hphi-bellman}]
Let $f \coloneqq f_\lat$, $h \in [H]$, and $(x,a) \in \cX \times \cA$ be given. Then, we have:  
\begin{align*}
	\brk*{\cT^{\Mobs}_{h}(f \circ \phi_{h+1})}(x,a) &= \rlath(\phistarh(x),a) + \En_{s_{h+1}\sim\Plath(\phistarh(x),a)}\En_{x_{h+1}\sim{}\psistar_{h+1}(s_{h+1})}\brk*{V_f(\phi(x_{h+1}))}\\
	&= \rlath(\phistarh(x),a) + \En_{s_{h+1}\sim\Plath(\phistarh(x),a)}\brk*{ \sum_{x_{h+1} \in \cX} \psistar(x_{h+1} \mid s_{h+1}) V_f(\phi(x_{h+1})) } \\
	&= \rlath(\phistarh(x),a) + \En_{s_{h+1}\sim\Plath(\phistarh(x),a)}\brk*{ \sum_{s' \in \cS} \hphi(s' \mid s_{h+1}) V_f(s') } \\
	&= \rlath(\phistarh(x),a) + \En_{s_{h+1}\sim\Plath(\phistarh(x),a)}\brk*{ \hphi \circ V_f (s_{h+1})} \\
	&= \brk*{\cT^{\Mlat}_h(\hphi \circ V_f)}(\phistarh(x),a),
 \end{align*}
 where the third line follows from the definition of the mismatch function $\hphi$. 
\end{proof}

We next show that the mismatch functions can be used to realize the pushforward dynamics $\phi\sharp\Mstarobs$, which we recall are defined as:
	\begin{equation}\label{eq:pushforward-model-redux}
\brk*{\phi\sharp\Mstarobsh}(r,s' \mid x,a) = %
	\sum_{x':\phi(x')=s'} \Mstarobsh(r,x' \mid x,a). %
\end{equation}
We also recall the notation $\brk*{\hphihpo \circ \Mlat}_h$, defined via:
\[
\brk*{\hphi \circ \Mlat}_h(r_h,s_{h+1} \mid s_h,a_h) \coloneqq \sum_{s'_{h+1} \in \cS} \Mlath(r_h,s'_{h+1} \mid s_h,a_h)
 \hphihpo(s_{h+1} \mid s'_{h+1}).
 \]

\begin{restatable}[Pushforward model realizability via mismatch functions]{lemma}{phicompressedrealizable}\label{lem:phi-compressed-model-realizable}
	For all $\phi \in \Phi$, $h \in [H]$, we have:
	\begin{equation}\label{eq:phi-compressed-model-realizable}
\brk{\phi_{h+1}\sharp\Mstarobsh}(\cdot \mid x,a) = \brk*{\brk*{\hphi \circ \Mstarlat}_h \circ \phistarh}(\cdot \mid x,a)
	\end{equation}
\end{restatable}

\begin{proof}[\pfref{lem:phi-compressed-model-realizable}]
Note that $\hphi$ can alternatively be written as:
\[
	\hphih(s'_h \mid s_h) = \sum_{x_h: \phi(x_h) = s'_h} \psistarh(x_h \mid s_h). 
\]
We have
\begin{align*}
\phi_{h+1}\sharp\Mstarobsh(r_{h+1}, s_{h+1} \mid x_h,a_h) &= \sum_{x_{h+1}: \phi_{h+1}(x_{h+1})=s_{h+1}} \Mstarobsh(r_{h+1}, x_{h+1} \mid x_h,a_h) \\
	&= \sum_{x_{h+1}: \phi_{h+1}(x_{h+1})=s_{h+1}} \prn*{\sum_{r,s' \in \bbR \times \cS}\Mstarlath(r,s' \mid \phistarh(x_h),a_h) \psistarhpo(x_{h+1} \mid s')} \\
	&=  \sum_{r, s' \in \bbR \times \cS}\Mstarlath(r,s' \mid \phistarh(x_h),a_h) \sum_{x_{h+1}: \phi_{h+1}(x_{h+1})=s_{h+1}} \psistarhpo(x_{h+1} \mid s') \\
	&=  \sum_{r,s' \in \bbR \times \cS}\Mstarlath(r,s' \mid \phistarh(x_h),a_h) \hphihpo(s' \mid s_{h+1}) \\
	&=  \brk*{\hphi \circ \Mstarlat}_{h}(r,s_{h+1} \mid \phistarh(x_h),a_h),
\end{align*}
as desired. 
\end{proof}

\newpage
\section{Proofs and Additional Results for \cref{sec:statistical-lower}: Impossibility Results}\label{app:impossibility}
This section contains additional information and proofs related to our impossibility results regarding statistical modularity (\cref{sec:statistical-lower}), and is organized as follows:
\begin{itemize}
\item \cref{app:additional-lb} contains the statement for an additional lower bound that is useful for establishing the impossibility results of \cref{fig:lb-table}.
\item \cref{sec:filling-out-lb-table} contains details for each entry of \cref{fig:lb-table}.
\item \cref{app:lb-proofs} contains for proofs for our main lower bound (\cref{thm:stochastic-tree-lb}) and the additional lower bound (\cref{thm:stochastic-cb-lb}).
\end{itemize}

\subsection{Additional Lower Bound}\label{app:additional-lb}

\begin{restatable}[Alternative lower bound]{theorem}{cblb}\label{thm:stochastic-cb-lb}
For every $N \geq 4$, there exists an emission class $\Psi$ and a decoder class $\Phi$ with $|\Psi|=|\Phi| = N$ and a family of latent MDPs $\cMlat$
	satisfying (i) $|\cMlat|=1$, (ii) $H = 1$, (iii) $|\cS| = |\cX| = N$, (iv) $|\cA| = N$, and %
	such that
	  \begin{enumerate}
  \item For all $\veps,\delta>0$, we have $\comp(\cMlat,\veps,\delta) = 0$.
  \item For an absolute constant $c>0$, $\comp(\dtri*{\cMlat,\Phi},c,c) \geq \Omega(N/\log(N))$.
  \end{enumerate}
\end{restatable}

\begin{proof}[Proof of \cref{thm:stochastic-cb-lb}]
See \cref{app:cb-lb-proof}. 
\end{proof}

\subsection{Details for \cref{fig:lb-table}}\label{sec:filling-out-lb-table}
\newcommand{\paragraphu}[1]{\paragraph{\underline{#1}}}

Below, we provide details on each entry in \cref{fig:lb-table}. More precisely, for each latent class $\cMlat$, we will give a (brief) description of the MDP class $\cMlat$, give our choice of latent complexity $\comp$ for $\cMlat$, and prove that the class is or is not statistically modular for that choice of latent complexity. We view our choices of latent complexities as natural complexities for the respective classes.

\paragraphu{Tabular MDPs (\greencheck).} 
\begin{itemize}
	\item Latent class $\cMlat$: Tabular MDPs $\Mlat = (\cS,\cA,\Plat,\Rlat,H)$. \citep{azar2017minimax}
	\item Latent complexity $\comp$: We take $\comp(\cMlat,\veps,\delta) = \poly(\abs{\cS},\abs{\cA},H,\varepsilon^{-1},\log\delta^{-1})$, which is attainable, for example, via the \textsc{Ucb-Vi} algorithm of \citet{azar2017minimax}
	\item Statistical modularity (\greencheck): Known Block MDP algorithms (e.g. \textsc{Musik} \citep{mhammedi2023representation}, \textsc{Briee} \citep{zhang2022efficient}) have sample complexities of $\poly(\abs{\cS},\abs{\cA},H,\varepsilon^{-1},\log\delta^{-1},\log|\Phi|)$.
\end{itemize}

\paragraphu{Contextual Bandits (\greencheck).} 
\begin{itemize}
	\item Latent class $\cMlat$: Contextual bandits with context space $\cS$, action space $\cA$, reward function $\rstarlat: \cS \times \cA \rightarrow [0,1]$ and a finite realizable function class satisfying $r^\star \in \cFlat$. 
	\item Latent complexity $\comp$: We take $\comp(\cMlat,\veps,\delta) = \poly(\abs{\cA},\log\abs{\cFlat},\varepsilon^{-1},\log\delta^{-1})$, attainable via, e.g., the \textsc{Square-Cb} algorithm \citep{foster2020beyond}.
	\item Statistical modularity (\greencheck): We note that $\cFlat \circ \Phi = \crl*{\brk*{f \circ \phi} \mid f \in \cF, \phi \in \Phi}$ is a realizable function class for the observation-level reward function $\rstarobs$, since $\rstarobs = \brk*{\rstarlat \circ \phistar} \in \cFlat \circ \Phi$. Thus, applying the \textsc{Square-Cb} algorithm directly on the observations $x\ind{t},a\ind{t},r\ind{t}$ will give complexity $\poly(\abs{\cA}\log(|\cFlat||\Phi|),\varepsilon^{-1},\log\delta^{-1}) = \poly(\abs{\cA},\log|\cFlat|,\log|\Phi|,\varepsilon^{-1},\log\delta^{-1})$.
\end{itemize}

\paragraphu{Low-rank MDP (\greencheck).}
	\begin{itemize}
	\item Latent class $\cMlat$: MDPs $\Mlat = (\cS,\cA,H,\Plat,\rlat)$ such that there exists $\mu^\star_{\lat,h} \in \bbR^d$, $\theta^\star_{\lat,h} \in \bbR^d$, and a known set of features $\Xi_\lat = \crl*{\xi_\lat = \crl*{\xi_{\lat,h}: \cS \times \cA \rightarrow \bbR^d}_{h=1}^H}$ such that for all $h \in [H]$ we have $r_\lat(s_h,a_h) = \langle \xi^\star_{\lat,h}(s_h,a_h), \theta^\star_{\lat,h} \rangle$ as well as
	\begin{equation}\label{eq:low-rank-mdp}
 \Plath(s_{h+1} \mid s_h,a_h) = \langle \xi^\star_{\lat,h}(s_h,a_h), \mu_{\lat,h+1}^\star(s_{h+1}) \rangle 
\end{equation}
 for some $\xi^\star_\lat \in \Xi_\lat$. 
 	\item Latent complexity $\comp$: We take $\comp(\cMlat,\veps,\delta) = \poly(d, \abs{\cA}, H, \log|\Xi_\lat|, \varepsilon^{-1},\log\delta^{-1})$, which is attainable via the \textsc{VoX} algorithm of \citet{mhammedi2023efficient}.
 	\item Statistical modularity (\greencheck): This is obtained by noting that the observation-level dynamics also satisfy the low-rank property with the same dimension. Formally, letting $\Pobs$ be the transition kernel for $\dtri*{\Mlat, \psistar}$ and $\phistar = (\psistar)^{-1}$, we have
\begin{align*}
	\Pobsh(x_{h+1} \mid x_h,a_h) &= \sum_{s_{h+1} \in \cS} \Plath(s_{h+1} \mid \phistarh(x_h),a_h) \psistarhpo(x_{h+1} \mid s_{h+1}) \\
		&= \sum_{s_{h+1} \in \cS} \tri*{\xi^\star_{\lat,h}(\phistarh(x),a), \mu^\star_{\lat,h+1}(s_{h+1})} \psistarhpo(x_{h+1} \mid s_{h+1}) \\
		&= \tri*{\xi^\star_{\lat,h}(\phistarh(x),a),\sum_{s_{h+1} \in \cS}\mu^\star_{\lat,h+1}(s_{h+1})\psistarhpo(x_{h+1} \mid s_{h+1}) }.
\end{align*}
Thus, the transition kernel $\Pobs$ is a low-rank MDP with $\mu_{\obs,h+1}(x_{h+1}) \coloneqq \sum_{s_{h+1}}\mu^\star_{\lat,h+1}(s_{h+1})\psistarhpo(x_{h+1} \mid s_{h+1})$ and feature class \[
	\Xi_\lat \circ \Phi = \crl*{ \xi_\lat \circ \phi = \crl*{\xi_h \circ \phi_h: x,a \mapsto \xi_h(\phi_h(x),a)}_{h=1}^H \mid \xi_\lat \in \Xi_\lat, \phi \in \Phi}.
	\]
	Lastly, since $r_\obs = \brk*{\rlat \circ \phistar}$, the reward function is also linear with the same unknown feature class. Thus we can apply \textsc{Vox} directly on top of the observations, with the feature class $\Xi_\lat \circ \Phi$, which will achieve a complexity $\poly(d,\abs{\cA},H,\log\abs{\Xi_\lat},\log\abs{\Phi},\vepsinv,\logdelinv)$.
	\end{itemize}
	
\paragraphu{Known Deterministic MDP ($\abs{\cMlat}=1$) (\greencheck).}
\begin{itemize}
	\item Latent class $\cMlat$: $\cMlat = \crl{\Mlat=(\cS,\cA,\Plat,\Rlat,H)}$ is a set of MDPs of size 1 with both deterministic rewards and deterministic transitions. 
	\item Latent complexity $\comp$: We take $\comp(\cMlat,\veps,\delta) = 0$, which is attainable as $\Mlat$ is known and we can simply deploy its optimal policy.
	\item Statistical modularity (\greencheck): We note that, due to determinism, the latent optimal policy can be chosen to be \emph{open-loop} without loss of generality, and thus will always experience the same trajectory $(s^\star_1, a^\star_1, \ldots, s^\star_H, a^\star_H)$. We can define the observation-level policy which commits to this same sequence of actions, i.e. $\pi_{\obs,h}(x_h)=a^\star_h$ for all $x_h$. This will be an optimal policy for any $\Mobs = \dtri*{\Mlat,\psi}$, and can also be learned in $0$ samples.
\end{itemize}

\paragraphu{Low State Occupancy ($\forall\,\pi:\cS \rightarrow \Delta(\cA))$ (\greencheck).}
\begin{itemize}
	\item Latent class $\cMlat$: $\cMlat = \crl{\Mlat=(\cS,\cA,\Plat,\Rlat,H)}$ is a set of MDPs for which we have a realizable value function class, and %
		such that there exists a feature map $\zeta_{\lat} = \crl*{\zeta_{\lat,h}: \cS\rightarrow \bbR^d}_{h=1}^H$ such that for all $\pi: \cS \rightarrow \Delta(\cA)$ and for all $\Mlat \in \cMlat$, we have
	\[
		\forall h \in [H] \,\, \exists \theta^{\Mlat,\pi}_h: \quad d^{\Mlat,\pi}_h(s) = \tri*{\zeta_{\lat,h}(s), \theta^{\Mlat,\pi}_h}.
	\]
	Note that the feature map does not need to be known.
      \item Latent complexity $\comp$: We take $\comp(\cMlat,\veps,\delta) = \poly(d,\abs{\cA},H,\log\abs{\cFlat},\vepsinv,\logdelinv)$, which is attainable by the \textsc{Bilin-Ucb} algorithm of \citeauthor{du2021bilinear}, since i) MDPs with this property have Bilinear rank bounded by $d \abs{\cA}$ (see Definition 4.3 and Lemma 4.6 of \citet{du2021bilinear}), and ii) one can construct the value function class $\cFlat = \crl{Q^{\Mlat,\star} \mid \Mlat \in \cMlat}$, which is realizable and has size $\log\abs{\cFlat} = \log\abs{\cMlat}$.
	\item Statistical modularity (\greencheck): We firstly note that one can construct a realizable value function class for the set $\dtri*{\cMlat,\Phi}$, via the set $\cF_\obs = \crl*{Q^{\Mlat,\star} \circ \phi \mid \Mlat \in \cMlat, \phi \in \Phi}$. This is realizable since, for any $\Mobs \coloneqq \dtri*{\Mlat,\psi}$, letting $\phistar = \psiinv$, we have $Q^{\Mobs,\star} = Q^{\Mlat,\star} \circ \phistar$, and that this class has size $\log\abs{\cMlat}\abs{\Phi}$. We can then show that the occupancies $d^{\Mobs,\pi_{f_\obs}}$, for $f_\obs \in \cF_\obs$, can also be expressed as $d$-dimensional linear function for an appropriate choice of features, which will imply that the \textsc{Bilin-Ucb} algorithm run directly on $\Mobs$ will attain a complexity of $\poly(d,\abs{\cA},H,\log\cMlat,\log\Phi,\vepsinv,\logdelinv)$. To obtain this, we recall the following lemma:
		\hphistateocc*
	Thanks to the above lemma, we have
	\begin{align*}
	d_\obs^{\pi_{f \circ \phi}}(x_h) &= 	\psi(x_h \mid \phistar(x_h))d_\lat^{\hphi \circ \pi_f}(\phistar(x_h)) \\
	&= \psi(x_h \mid \phistar(x_h))\tri*{ \brk*{\zeta_{\lat,h} \circ \phistarh}(x_h), \theta^{\Mlat,\hphi \circ \pi_f}_h} \\
	&= \tri*{\psi(x_h \mid \phistar(x_h))\brk*{\zeta_{\lat,h} \circ \phistarh}(x_h), \theta^{\Mlat,\hphi \circ \pi_f}_h}
	\end{align*}
	and so $d_\obs^{\pi_{f \circ \phi}}$ is linear with feature mapping $\psi(x_h \mid \phistar(x_h))\brk*{\zeta_{\lat,h} \circ \phistarh}$ and parameter $\theta^{\Mlat,\hphi \circ \pi_f}$. Recall that the feature map need not be known, so that \textsc{Bilin-Ucb} can still be applied despite not knowing $\psi$ and $\phistar$. 
\end{itemize}

\paragraphu{Model class + Pushforward Coverability (\greencheck).} 
\begin{itemize}
	\item Latent class $\cMlat$: $\cMlat = \crl{\Mlat=(\cS,\cA,\Plat,\Rlat,H)}$ is a set of MDPs that all satisfy pushforward coverability $\Cpush(\Mlat) \leq \Cpush$ (cf. \eqref{eq:pushforward-coverability} for the definition). 
	\item Latent complexity $\comp$: We take $\comp(\cMlat,\veps,\delta) = \poly(\Cpush,\abs{\cA},H,\log\abs{\cMlat},\vepsinv,\logdelinv)$, which is attainable by the \Golf algorithm via the results of \citet{xie2022role} (see also \cref{lem:golf-onpolicy}). We obtain this by noting that i) $\Ccov \leq \Cpush\abs{\cA}$, where $\Ccov$ is defined in Definition 2 of \citet{xie2022role}, and ii) a realizable model class can be used to construct a realizable value function class $\cF$ and a Bellman-complete value function helper class $\cG$ with sizes $\log\abs{\cF} = \log\abs{\cM}$ and $\log\abs{\cG} = \cO(\log\abs{\cM})$.
	\item Statistical modularity (\greencheck): This is obtained via \cref{thm:pushforwardgolf}.
\end{itemize}

\paragraphu{Linear CB/MDP (\redx$^\star$).}
	\begin{itemize}
		\item Latent class $\cMlat$: MDPs $\Mlat = (\cS,\cA,\Plat,\Rlat,H)$ that are linear with respect to a known feature map $\xi^\star_\lat: \cS \times \cA \rightarrow \bbR^d$ (i.e. such that \eqref{eq:low-rank-mdp} holds for $\xi^\star_\lat$).
		\item Latent complexity $\comp$: We take $\comp(\cMlat,\veps,\delta) = \poly(d,H,\vepsinv,\logdelinv)$, which is attainable via the \textsc{Lsvi-Ucb} algorithm of \citet{jin2020provably}. \dfedit{Note that this guarantee does not depend on the number of actions.}
		\item Statistical intractability (\redx): The latent model used in the construction of \cref{thm:stochastic-cb-lb} is a set (of size 1) of linear MDPs with $d=1$. In particular, that construction was a contextual bandit so we only have to realize a reward function, and since there is only one latent model so we can trivially embed this with $d=1$ via $\xi^\star_\lat(s,a) = r_\lat(s,a)$, where $r_\lat$ is the reward function of the MDP used in \cref{thm:stochastic-cb-lb}.
		\item Statistical modularity with additional $\abs{\cA}$-dependence: As in the Low-rank MDP case above, $\dtri*{\Mlat,\psi}$ is low-rank with unknown feature set $\Phi' = \crl*{\xi^\star_\lat \circ \phi \mid \phi \in \Phi}$. Thus, by the same conclusion, a the \textsc{Vox} algorithm will have complexity $\poly(d, \abs{\cA}, H, \log|\Phi|)$, which is of the desired form if we allow suboptimal dependence on $\abs{\cA}$.
	\end{itemize}

\paragraphu{Model class + Coverability ($\forall\, \pi_{\sss{M}}: M \in \cM$) (\redx).} 
	\begin{itemize}
	 \item Latent assumption: $\cMlat = \crl*{\Mlat = (\cS,\cA,\Plat,\Rlat,H)}$ is a set of MDPs that all satisfy coverability with respect to the policy class $\Pi_\cM = \{\pi_{\sM} \mid M \in  \cM\}$, i.e. we have
\[
	\forall \Mlat \in \cMlat: \quad \Ccov(\Mlat) = \inf_{\mu_h \in \Delta(\cS \times \cA)}\sup_{h \in [H]}  \sup_{\pi\in \Pi_{\cM}} \nrm*{\frac{d^{\sMlat,\pi}_{h}}{\mu_h}}_\infty < \infty 
\]
	\item Latent complexity $\comp$: We take $\comp(\cMlat,\veps,\delta) = \poly(\Ccov,H,\log\abs{\cMlat},\vepsinv,\logdelinv)$, which is attainable by the \Golf algorithm via the results of \citet{xie2022role} (see also \cref{lem:golf-onpolicy}). We obtain this by noting that a realizable model class can be used to construct a realizable value function class $\cF$ and a complete value function class $\cG$ of sizes $\log\abs{\cF} = \log\abs{\cM}$ and $\log\abs{\cG} = \cO(\log\abs{\cM})$.
		\item Statistical intractability (\redx): The latent models used in the construction of \cref{thm:stochastic-tree-lb} are a set of coverable MDPs -- in particular, these are trivially coverable with $\Ccov = 1$ since there is a single latent model and we can take $\mu = d^{\Mstarlat,\pi_{\Mstarlat}}$. \dfedit{We remark that it is an interesting open question whether this impossibility result continues to hold if we require coverability with respect to the class $\Pi$ of all possible latent policies.}
	\end{itemize}

\paragraphu{Known Stochastic MDP ($\abs{\cMlat}=1$) (\redx).}
\begin{itemize}
	\item Latent class $\cMlat$: $\cMlat = \crl{\Mlat=(\cS,\cA,\Plat,\Rlat,H)}$ is a set of MDPs of size 1.
	\item Latent complexity $\comp$: We take $\comp(\cMlat,\veps,\delta) = 0$, which is attainable as $\Mlat$ is known and we can simply deploy its optimal policy.
	\item Statistical intractability (\redx): This is precisely the setting of \cref{thm:stochastic-tree-lb}, which shows that at least $\Omega(N/\log(N))$ samples will be needed, where $N = \abs{\Phi}$.
	\end{itemize}

\paragraphu{Bellman rank ($Q$-type or $V$-type) (\redx)} 
\begin{itemize}
\item Latent assumption: $\cMlat = \crl*{\Mlat = (\cS,\cA,\Plat,\Rlat,H)}$ is a set of latent models such that each $\Mlat \in \cMlat$ has $Q$-type Bellman rank $d$ or $V$-type Bellman rank $d$ \citep{jin2021bellman}. Letting $\cF$ be a realizable value function class for $\cMlat$, in the $Q$-type case, this means that the $\abs{\Pi_\cF} \times \abs{\cF}$ matrix %
	\[
		\cE^{Q}_h(\pi,f) = \En^{\pi}\brk*{f_h(s_h,a_h) - r_h - \max_{a'}f_{h+1}(s_{h+1},a')},
	\]
	admits a rank $d$ factorization. In the $V$-type case, the matrix
	\[
		\cE^{V}_h(\pi,f) =\En_{s_h \sim d^\pi_h, a_h \sim \pi_f}\brk*{f_h(s_h,a_h) - r_h - \max_{a'}f_{h+1}(s_{h+1},a')}
	\]
	admits a rank-$d$ matrix factorization. 
\item Latent complexity $\comp$: We take $\comp(\cMlat,\veps,\delta) = \poly(d,H,\abs{\cA}\log\abs{\cF},\vepsinv,\logdelinv)$ for the $V$-type Bellman rank case, which is achievable by the \textsc{Olive} algorithm of \citet{jiang2017contextual}, and $\comp(\cMlat,\veps,\delta) = \poly(d,H,\log\abs{\cF},\vepsinv,\logdelinv)$ for $Q$-type Bellman rank, which is achievable by the \textsc{Bilin-Ucb} algorithm of \citet{du2021bilinear}.
\item Statistical intractability (\redx): We note that the construction in \cref{thm:stochastic-tree-lb} has $\abs*{\cMlat}=1$, which trivially has Bellman rank equal to $1$, so \cref{thm:stochastic-tree-lb} precludes statistical modularity with complexity $\comp$. %
\end{itemize}

\paragraphu{Eluder dimension + Bellman Completeness (\redx)} 
\begin{itemize}
\item Latent class $\cMlat$: $\cMlat = \crl*{\Mlat = (\cS,\cA,\Plat,\Rlat,H)}$ is a set of MDPs such that there is a function class $\cFlat$ satisfying
	\[
	\forall f_\lat \in \cFlat, \Mlat \in \cMlat: \quad		\cT^{\Mlat} f_\lat \in \cFlat.
	\] 
	Furthermore, each $\Mlat \in \cMlat$ has Bellman-Eluder dimension bounded by $d$ (see Definition 8 of \cite{jin2021bellman}). %
\item Latent complexity $\comp$: We take $\comp(\cMlat, \veps,\delta) = \poly(d, H, \log\abs{\cF}, \vepsinv,\logdelinv)$, which is attainable by the \Golf algorithm of \citet{jin2021bellman}.
\item Statistical intractability (\redx): As in the Bellman rank case, the construction in \cref{thm:stochastic-tree-lb} %
	has $\abs{\cMlat} = 1$, so we can take $\cFlat = \{Q^{\Mlat,\star} \mid \Mlat \in \cMlat\}$ which is evidently complete for $\cT^{\Mlat}$, and has Eluder dimension $1$, so \cref{thm:stochastic-tree-lb} precludes statistical modularity with complexity $\comp$. %
\end{itemize}

\paragraphu{$Q^\star$-irrelevant State Abstraction (\redx)} 
\begin{itemize}
\item Latent class $\cMlat$: $\Mlat = (\cS,\cA,\Plat,\Rlat,H)$ such that there is a known state abstraction function $\zeta_\lat : \cS \rightarrow \cZ$ such that %
	$\zeta_\lat(s) = \zeta_\lat(s')$ implies that $Q^{\sMlat,\star}(s,a) = Q^{\sMlat,\star}(s',a)$ for all $a \in \cA$. 
\item Latent complexity $\comp$: We take $\comp(\cMlat,\veps,\delta) = \poly(|\cZ|,|\cA|,H,\vepsinv,\logdelinv)$ which is attainable by the \textsc{Olive} algorithm of \citet{jiang2017contextual}.
\item Statistical intractability (\redx): We take $\cMlat = \crl*{\Mlat}$ as the MDP class from the construction of \cref{thm:stochastic-tree-lb}. Let $Q^\star_\lat \coloneqq Q^{\sMlat,\star}$. Note that we have $Q^\star_\lat(s,a) \in \{0,1\}$ for all $s,a$, so we can take a latent abstract state space $\cZ = \{(0,0),(0,1),(1,0),(1,1)\}$ and a state abstraction function $\zeta_\lat$ such that $\zeta_\lat(s) = (i,j)$ if $\Qstarlat(s,0) = i$ and $\Qstarlat(s,1) = j$. This satisfies the property of a $\Qstar$-irrelevant abstraction, since $\zeta_\lat(s) = \zeta_\lat(s') = (i,j)$ implies that $\Qstarlat(s,0) = \Qstarlat(s',0) = i$ and $\Qstarlat(s,1) = \Qstarlat(s',1) = j$. This has a constant-sized abstract space ($|\cZ| = 4$) and $|\cA|=2$, so \cref{thm:stochastic-tree-lb} precludes statistical modularity with complexity $\comp$. %
\end{itemize}

\paragraphu{Linear Mixture MDP (\redx).}
\begin{itemize}
\item Latent class $\cMlat$: MDPs $\Mlat = (\cS, \cA, \Plat, \Rlat, H)$ such that there is a known feature map $\zeta_\lat = \{\zeta_{\lat,h}: s',s,a \mapsto \bbR^d \}_{h=1}^H$ such that
	\[
		\forall h \in [H], \exists \theta_h \in \bbR^d : \quad \Plath(s' \mid s,a) = \tri*{\zeta_{\lat,h}(s' \mid s,a), \theta_h}
	\]
\item Latent complexity $\comp$: We take $\comp(\cMlat,\veps,\delta) = \poly(d,H,\vepsinv,\logdelinv)$, which is attainable by the \textsc{Ucrl-Vtr$^+$} algorithm of \citet{zhou2021nearly}
\item Statistical intractability (\redx): We take $\cMlat = \crl{\Mlat}$ to be the construction of \cref{thm:stochastic-tree-lb}. Here, there is a single latent model, so this is trivially embeddable with $\zeta_{\lat,h}(s' \mid s,a) = \Pstarlath(s'\mid s,a) \in \bbR^1$. This has dimension $d=1$, so \cref{thm:stochastic-tree-lb} precludes statistical modularity with complexity $\comp$. 
\end{itemize}

\paragraphu{Linear $\Qstar/\Vstar$ (\redx).}\,
\begin{itemize}
\item Latent class $\cMlat$: MDPs $\Mlat = (\cS,\cA,\Plat,\Rlat,H)$ such that there are known features maps $\alpha_\lat: \cS \times \cA \rightarrow \bbR^d$ and $\beta_\lat: \cS \rightarrow \bbR^d$  such that for all $\Mlat \in \cMlat$, there exists unknown parameters $\theta_Q, \theta_V \in \bbR^d$ such that $Q^{\sMlat,\star}(s,a) = \tri*{\alpha_\lat(s,a), \theta_Q}$ and $V^{\sMlat,\star}(s) = \tri*{\beta_\lat(s), \theta_V}$. 
\item Latent complexity $\comp$: We take $\comp(\cMlat,\veps,\delta)  = \poly(d,H,\vepsinv,\logdelinv)$, which is attainable by the \textsc{Bilin-Ucb} algorithm of \citet{du2021bilinear}.
\item Statistical intractability (\redx): We can take $\cMlat$ to be the latent MDP class from the construction of \cref{thm:stochastic-tree-lb}. Since there is a single latent model, this is trivially embeddable with dimension $1$, i.e. we can take $\zeta_\lat(s,a) = \Qstarlat(s,a)$ and $\beta_\lat(s) = \Vstarlat(s)$. This has dimension $d=1$, so \cref{thm:stochastic-tree-lb} precludes statistical modularity with complexity $\comp$. %
\end{itemize}

\paragraphu{Low State or State-Action Occupancy ($\forall\,\pi_\sM: M \in \cM)$  ({\redx}).}
\begin{itemize}
	\item Latent class $\cMlat$: In the Low State Occupancy model, $\cMlat = \crl{\Mlat=(\cS,\cA,\Plat,\Rlat,H)}$ is a set of MDPs such that there exists a feature map $\zeta^V_{\lat} = \crl*{\zeta_{\lat,h}: \cS\rightarrow \bbR^d}_{h=1}^H$ such that for all $\pi \in \crl{\pi_{\Mlat} \mid \Mlat \in \cMlat}$ and for all $\Mlat \in \cMlat$, we have
	\[
		\forall h \in [H] \,\, \exists \theta^{\Mlat,\pi}_h: \quad d^{\Mlat,\pi}_h(s) = \tri*{\zeta^V_{\lat,h}(s), \theta^{\Mlat,\pi}_h}.
	\]
	For the State-Action Occupancy model, we have that there exists a feature map $\zeta^Q_{\lat} = \crl*{\zeta_{\lat,h}: \cS \times \cA \rightarrow \bbR^d}_{h=1}^H$ such that for all $\pi \in \crl{\pi_{\Mlat} \mid \Mlat \in \cMlat}$ and for all $\Mlat \in \cMlat$, we have
	\[
		\forall h \in [H] \,\, \exists \theta^{\Mlat,\pi}_h: \quad d^{\Mlat,\pi}_h(s,a) = \tri*{\zeta^Q_{\lat,h}(s,a), \theta^{\Mlat,\pi}_h}.
	\]
	Note that the feature map does not need to be known in either case.
      \item Latent complexity $\comp$: We take $\comp(\cMlat,\veps,\delta) = \poly(d,\abs{\cA},H,\log\abs{\cFlat},\vepsinv,\logdelinv)$ for the state occupancy case and  $\comp(\cMlat,\veps,\delta) = \poly(d,H,\log\abs{\cMlat},\vepsinv,\logdelinv)$. Both are attainable by the \textsc{Bilin-Ucb} algorithm of \citeauthor{du2021bilinear}, since i) MDPs with this property have Bilinear rank bounded by $d \abs{\cA}$  and $d$ respectively (see Definition 4.3 and Lemma 4.6 of \citep{du2021bilinear}), and ii) one can construct the value function class $\cFlat = \crl{Q^{\Mlat,\star} \mid \Mlat \in \cMlat}$ which is realizable and has size $\log\abs{\cFlat} = \log\abs{\cMlat}$.

\item Intractability: We can take the construction of \cref{thm:stochastic-tree-lb}, which has $\abs{\cMlat}=1$ and thus is trivially embeddable with dimension $1$, i.e. we can take $\zeta^V_{\lat}(s) = d^{\Mlat,\pi_{\Mlat}}(s)$ and $\zeta^Q_\lat(s,a) = d^{\Mlat,\pi_{\Mlat}}(s,a)$.
\end{itemize}

\paragraphu{Bisimulation (\yellowq)}
\begin{itemize}
\item Latent class $\cMlat$: MDPs $\Mlat = (\cS,\cA,\Plat,\Rlat,H)$ such that there is a known state abstraction function $\zeta_\lat : \cS \rightarrow \cZ$ such that %
	$\zeta_\lat(s) = \zeta_\lat(\wt{s})$ implies that $\Rlat(s,a) = \Rlat(\wt{s},a)$ for all $a \in \cA$ as well as $\sum_{s': \zeta_\lat(s') = z'}\Plat(s' \mid s,a) = \sum_{s': \zeta_\lat(s') = z'}\Plat(s' \mid \wt{s},a)$ for all $z'$. 
\item Latent complexity $\comp$: We take $\comp(\cMlat,\veps,\delta) = \poly(|\cZ|,|\cA|,H,\vepsinv,\logdelinv)$ which is attainable by the \textsc{Olive} algorithm of \cite{jiang2017contextual}.
\item Openness (\yellowq): A negative result does not follow from existing constructions, since the dynamics from the tree-based construction of \cref{thm:stochastic-tree-lb} are not bisimilar unless $\abs{\cZ} = \abs{\cS}$, which allows for the application of tabular methods. At the same time, a positive result does not follow from existing methods, since it is non-trivial to extend existing Block MDP methods to use the bisimulation state abstraction in a way that only pays for $\abs{\cZ}$.
\end{itemize}

\paragraphu{Low State-Action Occupancy ($\forall \pi: \cS \rightarrow \Delta(\cA)$) (\yellowq$^\star$)} 
\begin{itemize}
	\item Latent class $\cMlat$: $\cMlat = \crl{\Mlat=(\cS,\cA,\Plat,\Rlat,H)}$ is a set of MDPs such that there exists a feature map $\zeta^Q_{\lat} = \crl*{\zeta_{\lat,h}: \cS \times \cA \rightarrow \bbR^d}_{h=1}^H$ such that for all $\pi: \cS \rightarrow \Delta(\cA)$ and for all $\Mlat \in \cMlat$, we have
	\[
		\forall h \in [H] \,\, \exists \theta^{\Mlat,\pi}_h: \quad d^{\Mlat,\pi}_h(s,a) = \tri*{\zeta^Q_{\lat,h}(s,a), \theta^{\Mlat,\pi}_h}.
	\]
	Note that the feature map does not need to be known.
	\item We take $\comp(\cMlat,\veps,\delta) = \poly(d,H,\log\abs{\cMlat},\vepsinv,\logdelinv)$, which is attainable by the \textsc{Bilin-Ucb} algorithm of \citeauthor{du2021bilinear}, since i) MDPs with this property have Bilinear rank bounded by $d$ (see Definition 4.3 and Lemma 4.6 of \citep{du2021bilinear}), and ii) one can construct a realizable value function class of size $\log\abs{\cF} = \log\abs{\cM}$. 
	\item Openness (\yellowq): A negative result does not follow from existing constructions, since the dynamics from the tree-based construction of \cref{thm:stochastic-tree-lb} do not have linear occupancies for all $\pi:\cS\to\Delta(\cA)$ unless $d = \abs{\cS}$, which allows for the application of tabular methods, and the dynamics from the bandit-based construction \cref{thm:stochastic-cb-lb} do not have linear occupancies for all $\pi:\cS\to\Delta(\cA)$ unless $d=\abs{\cA}$. At the same time, unlike the low state occupancy case, a positive result does not follow as it is unclear if we can express the observation-space occupancies linearly. 
	\item Statistical tractability with additional (suboptimal) $\abs{\cA}$-dependence (\greencheck): Note that we can reduce to the Low State Occupancy case (\greencheck), since 
		\[
		d^\pi(s) = \sum_{a \in \cA} d^\pi(s,a) = \tri*{\theta^\pi, \sum_{a \in \cA} \zeta^Q_\lat(s,a)} \coloneqq \tri*{\theta^\pi, \zeta^V_\lat(s)}.
		\]
		However, this blows up the feature norm bound of the feature map $\zeta^V_\lat(s)$ by a factor of $\abs{\cA}$, which will appear logarithmically in the bound obtained by \textsc{Bilin-Ucb}. 
\end{itemize}

\paragraphu{Model class + Coverability ($\forall\, \pi: \cS \rightarrow \Delta(\cA)$) (\yellowq).} 
	\begin{itemize}
	 \item Latent class $\cMlat$: $\cMlat = \crl*{\Mlat = (\cS,\cA,\Plat,\Rlat,H)}$ is a set of MDPs that all satisfy coverability with respect to all policies $\pilat: \cS \rightarrow \Delta(\cA)$, i.e. we have
\[
	\forall \Mlat \in \cMlat: \quad \Ccov(\Mlat) = \inf_{\mu_h \in \Delta(\cS \times \cA)}\sup_{h \in [H]}  \sup_{\pi: \cS \rightarrow \Delta(\cA)} \nrm*{\frac{d^{\sMlat,\pi}_{h}}{\mu_h}}_\infty < \infty 
\]
	\item Latent complexity $\comp$: We take $\comp(\cMlat,\veps,\delta) = \poly(\Ccov,H,\log\abs{\cMlat},\vepsinv,\logdelinv)$, which is attainable by the \Golf algorithm via the results of \citeauthor{xie2022role} (see also \cref{lem:golf-onpolicy}). We obtain this by noting that a realizable model class can be used to construct a realizable value function class $\cF$ and a complete value function class $\cG$ of sizes $\log\abs{\cF} = \log\abs{\cM}$ and $\log\abs{\cG} = \cO(\log\abs{\cM})$.
		\item Openness (\yellowq): A negative result does not follow from the existing constructions. The tree-based construction of \cref{thm:stochastic-tree-lb} satisfies coverability with $\Ccov = \exp(\Omega(H))$ and the bandit-based construction of \cref{thm:stochastic-cb-lb} satisfies coverability with $\Ccov = \abs{\cA}$. In both cases, the lower bounds cannot be used to rule out statistical modularity with the above latent complexity. Similarly, it unclear how to obtain a positive result for the latent-dynamics class $\dtri{\Mlat,\Phi}$.
	\end{itemize}

\subsection{Proofs for Lower Bounds (\cref{thm:stochastic-tree-lb,thm:stochastic-cb-lb})}\label{app:lb-proofs}

\subsubsection{Main lower bound (\cref{thm:stochastic-tree-lb})}\label{app:tree-lb-pf}

We will prove the following result.

\mainlb*

\begin{proof} \, Let $N$ be given and assume without loss
  of generality that it is a power of $2$. We first construct the
  class of latent-dynamics MDPs, following \citet{song2024rich}.

\paragraph{Latent MDP.}  The construction has a single ``known''
latent MDP $\Mlat$, so that the only uncertainty in the family of latent-dynamics MDPs we
        construct arises from the emission processes. We set $\cMlat = \crl*{\Mlat}$. Set
        $H=\log_2(N)+1$ and $\cA=\crl{0,1}$. We define the state space
        and latent transition dynamics as follows.
        \begin{itemize}
        \item The state space can be partitioned as
          $\cS=\cS^{1},\ldots,\cS^{N}$.
        \item Each block $\cS^{i}$ corresponds to a standard
          depth-$H$ binary
          tree MDP with deterministic dynamics (e.g.,
          \citet{osband2016lower,domingues2021episodic}). There is a
          single ``root'' node at layer $h=1$, which we denote by
          $\sinit^{i}$, and $N$ ``leaf'' nodes at layer $H$, which
          we denote by $\crl[\big]{\sfin^{i,j}}_{j\in\brk{N}}$. For
          each $h=1,\ldots,H-1$, choosing action $0$ leads to the left
          successor of the current state deterministically, and
          choosing action $1$ leads to the right sucessor; this
          process continues until we reach a leaf node at layer $H$.\loose
        \item The initial state distribution is
          $P_{\lat,1}(\emptyset)=\unif(\sinit^{1},\ldots,\sinit^{N})$.
        \item There are no rewards for layers $1,\ldots,H-1$. For
          layer $H$, the reward is
          \begin{align}
            \label{eq:lower_rewar}
            R_H(\sfin^{i,j},\cdot) = \indic\crl*{j=i}.
          \end{align}
        \end{itemize}
This construction can summarized as follows. At layer $1$,
we draw the index of one of $N$ binary trees uniformly at random, and
initialize into the root of the tree. From here, we receive a reward
of $1$ if we successfully navigate to the leaf node whose index agrees with the
index of the tree itself, and receive a reward of $0$ otherwise.

Note that the total number of latent states in this construction is
$\abs{\cS}=N\cdot\abs{\cS_1}=N(2N-1)$	

\paragraph{Observation space and decoder class.} 

Let us introduce some additional notation. For each block
$\cS^{i}$, let
$\cS_h^{i}\ldef{}\crl{s_h^{i,j}}_{j\in\brk{2^{h-1}}}$ denote the
states in block $i$ that are reachable at layer $h$, so that
$\cS^{i}_1=\crl*{\sinit^{i}}$ and
$\cS^{i}_H=\crl{\sfin^{i,j}}_{j\in\brk{N}}$. We define
$\cX=\cS$ so that $\abs{\cX}\leq{}4N^2$, and consider a class of emission processes corresponding to
deterministic maps. Let $\Sigma$ denote the set of cyclic permutations on $N$ elements, excluding the identity permutation. That is, each $\sigma_i \in \Sigma$ takes the form
\[
\sigma_i:k\mapsto{}k+i\mod{}N \quad \text{ for } i\in\crl{1,\ldots,N}.
\]
  For each
$\sigma \in \Sigma$, we consider the emission process
\[
\psi^{\sigma}_h(\cdot\mid{}s_h\ind{i,j}) = \indic_{s_h\ind{\sigma(i),j}}.
\]
That is, $\psi^{\sigma}$ shifts the index of the binary tree containing
$s_h\ind{i,j}$ according to $\sigma$. Let $\Psi = \{\psi^\sigma \mid \sigma \in \Sigma\}$. Consider the decoder class
\begin{align*}
  \Phi = \Psi^{-1} \coloneqq \crl[\big]{s^{i}\mapsto{}s^{\psi^{-1}(i)}\mid{}\psi\in\Psi},
\end{align*}
which has $\abs{\Phi}=N$.  We consider the class of rich-observation MDPs given by
\begin{align}
  \label{eq:lower_bound_class}
  \dtri*{\cMlat, \Phi} \ldef{} \crl*{M^{i}\ldef{} \llangle \Mlat, \psi^{\sigma_i} \rrangle \mid{}\sigma_i\in \Sigma}.
\end{align}
It is clear that this class of rich-observation MDPs satisfies the
decodability assumption for emissions $\Psi$.

\paragraph{Sample complexity lower bound.}%
To lower bound the sample
complexity, we prove a lower bound on the constrained PAC
Decision-Estimation Coefficient (DEC) of
\cite{foster2023tight}. For an arbitrary MDP $\Mbar$ (defined over
the space $\cX$) and $\veps\in\brk{0,2^{1/2}}$, define\footnote{For
  measures $\bbP$ and $\bbQ$, we define squared Hellinger distance by $\Dhels{\bbP}{\bbQ}=\int(\sqrt{d\bbP}-\sqrt{d\bbQ})^2$.}
\begin{align*}
  \dec(\cM,\Mbar)
  = \inf_{p,q\in\Delta(\Pi)}\sup_{M\in\cM}
  \crl*{\En_{\pi\sim{}p}\brk*{\Jm(\pim)-\Jm(\pi)}
  \mid{} \En_{\pi\sim{}q}\brk*{\Dhels{M(\pi)}{\Mbar(\pi)}}\leq\veps^2},
\end{align*}
where $M(\pi)$ denotes the law over trajectories
$(x_1,a_1,r_1),\ldots,(x_H,a_H,r_H)$ induced by executing
the policy $\pi$ in the MDP $M$, $\Jm(\pi)$ denotes the expected
reward for policy $\pi$ under $M$, and $\pim$ denotes the optimal
policy for $M$. We further define
\begin{align*}
  \dec(\cM) = \sup_{\Mbar}\dec(\cM,\Mbar),
\end{align*}
where the supremum ranges over all MDPs defined over $\cX$ and
$\cA$. We now appeal to the following technical lemma.
\begin{lemma}
  \label{lem:dec_lower_tree}
  For all $\veps^2\geq{}4/N$, we have that
  $\dec(\dtri*{\cMlat, \Phi}) \geq \frac{1}{2}$.
\end{lemma}
In light of \cref{lem:dec_lower_tree}, it follows from Theorem 2.1 in \citet{foster2023tight}\footnote{Theorem 2.1 in
  \citet{foster2023tight} is stated with respect to
  $\sup_{\Mbar\in\texttt{conv}(\cM)}\dec(\cM,\Mbar)$, but the actual proof (Section
  2.2) gives a stronger result that scales with $\sup_{\Mbar}\dec(\cM,\Mbar)$.}
that any PAC RL algorithm that uses $T$ episodes
of interaction for $T\log(T)\leq{}c\cdot{}N$ must have
$\En\brk*{\Jm(\pim)-\Jm(\pihat)}\geq{}c'$ for a worst-case MDP in
$\cM$, where $c,c'>0$ are absolute constants. This implies that any PAC RL which has $\En\brk*{\Jm(\pim)-\Jm(\pihat)}\leq{}c'$ must have $T\log(T)\geq{}c\cdot{}N$ and thus $T \geq{} c \cdot N/\log(N)$.

\end{proof}

\begin{proof}[\pfref{lem:dec_lower_tree}]
  Define $\Mbarlat$ as the latent-space MDP that has identical
  dynamics to $\Mlat$ but, has zero reward in every state, and define
  $\Mbar \ldef \dtri*{\Mbarlat,\id}$ as the rich-observation MDP obtained
  by composing $\Mbarlatent$ with the ``identity'' emission process $\id$ that
  sets $x_h=s_h$. Observe that $\Mbar$ and $M^i$, induce identical dynamics
  in observation space if rewards are ignored: For all policies $\pi$,   %
  \begin{equation}
    \label{eq:identical_dynamics}
    \bbP^{\sss{\Mbar},\pi}\brk*{(x_1,a_1),\ldots,(x_H,a_H)=\cdot}
    = \bbP^{\sss{M^{i}},\pi}\brk*{(x_1,a_1),\ldots,(x_H,a_H)=\cdot}.
  \end{equation}
It follows that for each $i$, for all policies $\pi$, we have
  \begin{align}
    &\Dhels{M^i(\pi)}{\Mbar(\pi)}\notag\\
    &= \Dhels{(\dtri*{\Mlat,\psi_i})(\pi)}{(\dtri*{\Mbarlat,\id})(\pi)} \notag\\
    &= \sum_{j=1}^{N}
      \bbP^{\sMbar,\pi}\brk*{x_H=\sfin\ind{\psi_i(j),j}}\cdot\Dhels{\indic_{1}}{\indic_{0}}\notag\\
    &= 2\sum_{j=1}^{N}
      \bbP^{\sMbar,\pi}\brk*{x_H=\sfin\ind{\psi_i(j),j}}\notag\\
    &= \frac{2}{N}\sum_{j=1}^{N}
      \bbP^{\sMbar,\pi}\brk*{x_H=\sfin\ind{\psi_i(j),j}\mid{}x_1=\sinit\ind{\psi_i(j)}},\\
    &= \frac{2}{N}\sum_{j=1}^{N}
      \bbP^{\sMbar,\pi}\brk*{x_H=\sfin\ind{j,\psi_i^{-1}(j)}\mid{}x_1=\sinit\ind{j}},\label{eq:lower_step1}
  \end{align}
  since the learner receives identical feedback in the MDPs
  $M^i$ and $\Mbar$ unless they reach the observation
  $x_H=\sfin\ind{\psi_i(j),j}$ for some $j$ (corresponding to latent
  state $\sfin\ind{j,j}$ in $M^i$), in which case they receiver
  reward $1$ in $M^i$ but reward $0$ in $\Mbar$.
  We now claim that for any $q\in\Delta(\Pi)$, there exists a set of
  at least $N/2$ indices $\cI_q\subset\brk{N}$ such that
  \begin{align}
    \En_{\pi\sim{}q}\brk*{\Dhels{M^i(\pi)}{\Mbar(\pi)} } \leq
    \frac{4}{N}
    \label{eq:lower_step2}
  \end{align}
  for all $i\in\cI_q$. To see this, note that by
  \eqref{eq:lower_step1}, we have
  \begin{align*}
    \En_{i\sim\unif(\brk{N})}\En_{\pi\sim{}q}\brk*{
    \Dhels{M^i(\pi)}{\Mbar(\pi)}
    }
    &\leq{}
      \En_{\pi\sim{}q}\brk*{\frac{2}{N}\sum_{j=1}^{N}
      \frac{1}{N}\sum_{i=1}^{N}\bbP^{\sMbar,\pi}\brk*{x_H=\sfin\ind{j,\psi_i^{-1}(j)}\mid{}x_1=\sinit\ind{j}}
      }\\
    &\leq{}
      \En_{\pi\sim{}q}\brk*{\frac{2}{N}\sum_{j=1}^{N}
      \frac{1}{N}
      }
      = \frac{2}{N},
  \end{align*}
  where the second inequality uses that
  $\sum_{i=1}^{N}\bbP^{\sMbar,\pi}\brk*{x_H=\sfin\ind{j,\psi_i^{-1}(j)}\mid{}x_1=\sinit\ind{j}}\leq{}1$, %
  as the events in the sum are mutually exclusive (and the event we
  condition on does not depend on $i$). We conclude by Markov's
  inequality that
  $ \bbP_{i\sim\unif(\brk{N})}\brk*{\En_{\pi\sim{}q}\brk*{
      \Dhels{M^i(\pi)}{\Mbar(\pi)}}\geq{}4/N } \leq{} 1/2$,
  giving $\cI_q\geq{}N/2$.

  From \eqref{eq:lower_step2}, we conclude that for all
  $\veps^2\geq{}4/N$,
  \begin{align*}
    \dec(\cM,\Mbar)
    \geq{} \inf_{q\in\Delta(\Pi)}\inf_{p\in\Delta(\Pi)}\sup_{i\in\cI_q}
    \crl*{\En_{\pi\sim{}p}\brk*{J^{\sss{M^i}}(\pi_{\sss{M^i}})-J^{\sss{M^i}}(\pi)}}.
  \end{align*}
  To lower bound this quantity, observe that for any index $i$ and any
  policy $\pi$, we have
  \begin{align*}
    J^{\sss{M^i}}(\pi_{\sss{M^i}})-J^{\sss{M^i}}(\pi)
    &= \frac{1}{N}\sum_{j=1}^{N}
      \bbP^{\sss{M}\ind{i},\pi}\brk*{x_H\neq\sfin\ind{\psi_i(j),j}\mid{}x_1=\sinit\ind{\psi_i(j)}}\\
    &= 1-\frac{1}{N}\sum_{j=1}^{N}
      \bbP^{\sss{M}\ind{i},\pi}\brk*{x_H=\sfin\ind{\psi_i(j),j}\mid{}x_1=\sinit\ind{\psi_i(j)}}\\
    &= 1-\frac{1}{N}\sum_{j=1}^{N}
      \bbP^{\sMbar,\pi}\brk*{x_H=\sfin\ind{\psi_i(j),j}\mid{}x_1=\sinit\ind{\psi_i(j)}}\\
    &= 1-\frac{1}{N}\sum_{j=1}^{N}
      \bbP^{\sMbar,\pi}\brk*{x_H=\sfin\ind{j,\psi_i^{-1}(j)}\mid{}x_1=\sinit\ind{j}},
  \end{align*}
  where the third inequality uses \eqref{eq:identical_dynamics}.
  We conclude that for any distribution $p,q\in\Delta(\Pi)$,
  \begin{align*}
    &\sup_{i\in\cI_q}
      \crl*{\En_{\pi\sim{}p}\brk*{J^{\sss{M^i}}(\pi_{\sss{M^i}})-J^{\sss{M^i}}(\pi)}}
    \\
    &\geq\En_{i\sim\unif(\cI_q)}
      \crl*{\En_{\pi\sim{}p}\brk*{J^{\sss{M^i}}(\pi_{\sss{M^i}})-J^{\sss{M^i}}(\pi)}}\\
    &\geq
      1- \frac{1}{N}\sum_{j=1}^{N}
      \En_{i\sim\unif(\cI_q)}\bbP^{\sMbar,\pi}\brk*{x_H=\sfin\ind{j,\psi_i^{-1}(j)}\mid{}x_1=\sinit\ind{j}}\\
    &=
      1- \frac{1}{N}\sum_{j=1}^{N}
      \frac{1}{\abs*{\cI_q}}\sum_{i\in\cI_q}\bbP^{\sMbar,\pi}\brk*{x_H=\sfin\ind{j,\psi_i^{-1}(j)}\mid{}x_1=\sinit\ind{j}}
    \geq{} 
    1- \frac{1}{\abs*{\cI_q}} \geq{} \frac{1}{2}
  \end{align*}
  as long as $N\geq{}4$, where the second-to-last inequality uses that
  for all $j$, the events
  $\crl[\big]{x_H=\sfin\ind{j,\psi_i^{-1}(j)}\mid{}x_1=\sinit\ind{j}}$
  are disjoint for all $i$. Since this lower bound holds uniformly for
  all $q,p\in\Delta(\Pi)$, we conclude that
  \begin{align*}
    \dec(\dtri*{\cMlat, \Phi},\Mbar) \geq \frac{1}{2}.
  \end{align*}
\end{proof}

\subsubsection{Proof of alternative lower bound (\cref{thm:stochastic-cb-lb})}\label{app:cb-lb-proof}

We will prove the following result. 

\cblb*

\begin{proof}[\pfref{thm:stochastic-cb-lb}]
We repeat more or less repeat the same proof as
\cref{thm:stochastic-tree-lb}, but with the appropriate modifications
to translate from the contextual tree-based construction in
\cref{thm:stochastic-tree-lb} to the contextual bandit-based
construction in the theorem statement. Let $N$ be given and assume without loss of generality that it is a power of $2$. %

\paragraph{Latent MDP.}  Our construction has a single ``known'' latent MDP $\Mlat$; that is,
        the only uncertainty in the family of rich-observation MDPs we
        construct arises from the emission processes. Set $\cMlat = \crl*{\Mlat}$. Set
        $H=1$ and $\cA=\brk{N}$. We define the state space
        and latent transition dynamics as follows.
        \begin{itemize}
        \item The state space can be partitioned as
          $\cS=\cS^{1},\ldots,\cS^{N}$.
        \item Each block $\cS^{i}$ corresponds to a single state $s^i$ with $N$ actions denoted by $a^i$, $i \in [N]$.
        \item The initial state distribution is
          $P_{\lat,1}(\emptyset)=\unif(s^{1},\ldots,s^{N})$.
        \item The reward function is
          \begin{align}
            \label{eq:lower_reward_cb}
            R_1(s^{i},a^j) = \indic\crl*{j=i}.
          \end{align}
        \end{itemize}
Informally, this construction can summarized as a contextual bandit (with uniform context distribution), with a reward of $1$ if and only if we play the action corresponding to the index of the context drawn. 

Note that the total number of latent states in this construction is
$\abs{\cS}=N$ and the number of actions is $\abs{\cA} = N$. 	

\paragraph{Observation space and decoder class.} 

We define
$\cX=\cS$ so that $\abs{\cX}=\abs{\cS}$, and consider a class of emission processes corresponding to
deterministic maps. Let $\Sigma$ denote the set of cyclic permutations on $N$ elements, excluding the identity permutation. That is, each $\sigma_i \in \Sigma$ takes the form
\[
\sigma_i:k\mapsto{}k+i\mod{}N,  \quad \text{ for } i\in\crl{1,\ldots,N}.
\]
  For each
$\sigma \in \Sigma$, we consider the emission process
\[
\psi^{\sigma}(\cdot\mid{}s^{i}) = \indic_{s^{\sigma(i)}}(\cdot)
\]
That is, $\psi^{\sigma}$ shifts the context $s^i$ %
according to $\sigma$. Let $\Psi = \{\psi^\sigma \mid \sigma \in \Sigma\}$. Consider the decoder class
\begin{align*}
  \Phi = \Psi^{-1} \coloneqq \crl[\big]{s^{i}\mapsto{}s^{\psi^{-1}(i)}\mid{}\psi\in\Psi},
\end{align*}
which has $\abs{\Phi}=N$.  We consider the class of rich-observation MDPs given by
\begin{align}
  \label{eq:lower_bound_class_cb}
  \dtri*{\cMlat, \Phi} \ldef{} \crl*{M^{i}\ldef{} \llangle \Mlat, \psi^{\sigma_i} \rrangle \mid{}\sigma_i\in \Sigma}.
\end{align}
It is clear that this class of rich-observation MDPs satisfies the
decodability assumption for emissions $\Psi$.

\paragraph{Sample complexity lower bound.}%
To lower bound the sample
complexity, we prove a lower bound on the constrained PAC
Decision-Estimation Coefficient (DEC) of
\cite{foster2023tight}. For an arbitrary MDP $\Mbar$ (defined over
the space $\cX$) and $\veps\in\brk{0,2^{1/2}}$, define\footnote{For
  measures $\bbP$ and $\bbQ$, we define squared Hellinger distance by $\Dhels{\bbP}{\bbQ}=\int(\sqrt{d\bbP}-\sqrt{d\bbQ})^2$.}
\begin{align*}
  \dec(\cM,\Mbar)
  = \inf_{p,q\in\Delta(\Pi)}\sup_{M\in\cM}
  \crl*{\En_{\pi\sim{}p}\brk*{\Jm(\pim)-\Jm(\pi)}
  \mid{} \En_{\pi\sim{}q}\brk*{\Dhels{M(\pi)}{\Mbar(\pi)}}\leq\veps^2},
\end{align*}
where $M(\pi)$ denotes the law over observations
$(x_1,a_1,r_1)$ induced by executing
the policy $\pi$ in the MDP $M$, $\Jm(\pi)$ denotes the expected
reward for policy $\pi$ under $M$, and $\pim$ denotes the optimal
policy for $M$. We further define
\begin{align*}
  \dec(\cM) = \sup_{\Mbar}\dec(\cM,\Mbar),
\end{align*}
where the supremum ranges over all MDPs defined over $\cX$ and
$\cA$. We now appeal to the following technical lemma. %
\begin{lemma}
  \label{lem:dec_lower_cb}
  For all $\veps^2\geq{}4/N$, we have that
  $\sup_{\Mbar}\dec(\cM,\Mbar) \geq \frac{1}{2}$.
\end{lemma}
In light of \cref{lem:dec_lower_cb}, it follows from Theorem 2.1 in \citet{foster2023tight}\footnote{Theorem 2.1 in
  \citet{foster2023tight} is stated with respect to
  $\sup_{\Mbar\in\texttt{conv}(\cM)}\dec(\cM,\Mbar)$, but the actual proof (Section
  2.2) gives a stronger result that scales with $\sup_{\Mbar}\dec(\cM,\Mbar)$.}
that any PAC RL algorithm that uses $T$ episodes
of interaction for $T\log(T)\leq{}c\cdot{}N$ must have
$\En\brk*{\Jm(\pim)-\Jm(\pihat)}\geq{}c'$ for a worst-case MDP in
$\cM$, where $c,c'>0$ are absolute constants. This implies that any PAC RL which has $\En\brk*{\Jm(\pim)-\Jm(\pihat)}\leq{}c'$ must have $T\log(T)\geq{}c\cdot{}N$ and thus $T \geq{} c \cdot N/\log(N)$.
\end{proof}

\begin{proof}[\pfref{lem:dec_lower_cb}]
  Define $\Mbarlat$ as the latent-space MDP that has identical
  dynamics to $\Mlat$ but, has zero reward for every state-action pair, and define
  $\Mbar \ldef \dtri*{\Mbarlat,\id}$ as the rich-observation MDP obtained
  by composing $\Mbarlatent$ with the identity emission process that
  sets $x_h=s_h$. In the rest of the proof, we use the shorthand $\psi_i \coloneqq \psi^{\sigma_i}$. Observe that $\Mbar$ and $M^i$, induce identical dynamics
  in observation space if rewards are ignored, i.e. for all policies $\pi: \cX \rightarrow \Delta(\cA)$,   %
  \begin{equation}
    \label{eq:identical_dynamics_cb}
    \bbP^{\sss{\Mbar},\pi}\brk*{(x_1,a_1)=\cdot}
    = \bbP^{\sss{M^{i}},\pi}\brk*{(x_1,a_1)=\cdot}.
  \end{equation}
It follows that for each $i$, for all policies $\pi$, we have
  \begin{align}
    &\Dhels{M^i(\pi)}{\Mbar(\pi)}\notag\\
    &= \Dhels{(\dtri*{\Mlat,\psi_i})(\pi)}{(\dtri*{\Mbarlat,\id})(\pi)} \notag\\
    &= \sum_{j=1}^{N}
      \bbP^{\sMbar,\pi}\brk*{x_1 = s^{\psi_i(j)}, a_1=a^j}\cdot\Dhels{\indic_{1}}{\indic_{0}}\notag\\
    &= 2\sum_{j=1}^{N}
      \bbP^{\sMbar,\pi}\brk*{x_1 = s^{\psi_i(j)}, a_1=a^j}\notag\\
     &= \frac{2}{N}\sum_{j=1}^{N}
      \bbP^{\sMbar,\pi}\brk*{a_1=a^j \mid x_1 = s^{\psi_i(j)}}\notag\\
     &= \frac{2}{N}\sum_{j=1}^{N}
      \bbP^{\sMbar,\pi}\brk*{a_1=a^{\psiinv_i(j)} \mid x_1 = s^j}\notag
  \end{align}

  since the learner receives identical feedback in the MDPs
  $M^i$ and $\Mbar$ unless they play the action
  $a_1=a^j$ given observation $x_1 = s^{\psi_i(j)}$ (corresponding to latent
  state $s^i$ in $M^i$), in which case they receiver
  reward $1$ in $M^i$ but reward $0$ in $\Mbar$.
  We now claim that for any $q\in\Delta(\Pi)$, there exists a set of
  at least $N/2$ indices $\cI_q\subset\brk{N}$ such that
  \begin{align}
    \En_{\pi\sim{}q}\brk*{\Dhels{M^i(\pi)}{\Mbar(\pi)} } \leq
    \frac{4}{N}
    \label{eq:lower_step2}
  \end{align}
  for all $i\in\cI_q$. To see this, note that by
  \eqref{eq:lower_step1}, we have
  \begin{align*}
    \En_{i\sim\unif(\brk{N})}\En_{\pi\sim{}q}\brk*{
    \Dhels{M^i(\pi)}{\Mbar(\pi)}
    }
    &\leq{}
      \En_{\pi\sim{}q}\brk*{\frac{2}{N}\sum_{j=1}^{N}
      \frac{1}{N}\sum_{i=1}^{N}\bbP^{\sMbar,\pi}\brk*{a_1=a^{\psiinv_i(j)} \mid x_1 = j}
      }\\
    &\leq{}
      \En_{\pi\sim{}q}\brk*{\frac{2}{N}\sum_{j=1}^{N}
      \frac{1}{N}
      }
      = \frac{2}{N}.
  \end{align*}
   We conclude by Markov's
  inequality that
  $ \bbP_{i\sim\unif(\brk{N})}\brk*{\En_{\pi\sim{}q}\brk*{
      \Dhels{M^i(\pi)}{\Mbar(\pi)}}\geq{}4/N } \leq{} 1/2$,
  giving $\cI_q\geq{}N/2$.

  From \eqref{eq:lower_step2}, we conclude that for all
  $\veps^2\geq{}4/N$,
  \begin{align*}
    \dec(\dtri*{\cMlat,\Phi},\Mbar)
    \geq{} \inf_{q\in\Delta(\Pi)}\inf_{p\in\Delta(\Pi)}\sup_{i\in\cI_q}
    \crl*{\En_{\pi\sim{}p}\brk*{J^{\sss{M^i}}(\pi_{\sss{M^i}})-J^{\sss{M^i}}(\pi)}}.
  \end{align*}
  To lower bound this quantity, observe that for any index $i$ and any
  policy $\pi$, we have
  \begin{align*}
    J^{\sss{M^i}}(\pi_{\sss{M^i}})-J^{\sss{M^i}}(\pi)
    &= 1-\frac{1}{N}\sum_{j=1}^{N}
      \bbP^{\sss{M}\ind{i},\pi}\brk*{a_1=a\ind{j}\mid{}x_1=s\ind{\psi_i(j)}}\\
    &= 1-\frac{1}{N}\sum_{j=1}^{N}
      \bbP^{\sMbar,\pi}\brk*{a_1=a\ind{j}\mid{}x_1=s\ind{\psi_i(j)}}\\
    &= 1-\frac{1}{N}\sum_{j=1}^{N}
      \bbP^{\sMbar,\pi}\brk*{a_1=a\ind{\psiinv_i(j)}\mid{}x_1=s\ind{j}},
  \end{align*}
  where the third inequality uses \eqref{eq:identical_dynamics_cb}.
  We conclude that for any distribution $p,q\in\Delta(\Pi)$,
  \begin{align*}
    &\sup_{i\in\cI_q}
      \crl*{\En_{\pi\sim{}p}\brk*{J^{\sss{M^i}}(\pi_{\sss{M^i}})-J^{\sss{M^i}}(\pi)}}
    \\
    &\geq\En_{i\sim\unif(\cI_q)}
      \crl*{\En_{\pi\sim{}p}\brk*{J^{\sss{M^i}}(\pi_{\sss{M^i}})-J^{\sss{M^i}}(\pi)}}\\
    &\geq
      1- \frac{1}{N}\sum_{j=1}^{N}
      \En_{i\sim\unif(\cI_q)}\bbP^{\sMbar,\pi}\brk*{a_1=a\ind{\psiinv_i(j)}\mid{}x_1=s\ind{j}}\\
    &=
      1- \frac{1}{N}\sum_{j=1}^{N}
      \frac{1}{\abs*{\cI_q}}\sum_{i\in\cI_q}\bbP^{\sMbar,\pi}\brk*{a_1=a\ind{\psiinv_i(j)}\mid{}x_1=s\ind{j}}    \geq{} 
    1- \frac{1}{\abs*{\cI_q}} \geq{} \frac{1}{2}
  \end{align*}
  as long as $N\geq{}4$. %
  Since this lower bound holds uniformly for
  all $q,p\in\Delta(\Pi)$, we conclude that
  \begin{align*}
    \dec(\dtri*{\cMlat, \Phi},\Mbar) \geq \frac{1}{2}.
  \end{align*}
\end{proof}

\newpage
\section{Proofs for \cref{sec:statistical-results-positive}: Positive Results}\label{app:statistical}
This section is dedicated to the proof of our upper bound establishing that pushforward-coverable MDPs are statistically modular (\cref{thm:pushforwardgolf}).

\subsection{Proofs for Latent Model Class + Pushforward Coverability (\cref{thm:pushforwardgolf})}\label{sec:pushforward-golf}

In this section, we establish positive results under latent MDP classes which satisfy pushforward coverability. %
We assume that every model in $\cMlat$ satisfies \emph{pushforward coverability}, defined as follows:
\begin{definition}[Pushforward coverability]\label{def:pushforward_def_app} The pushforward coverability coefficient $\Cpush$ for an MDP $M$ with transition kernel $P$ is defined by
	\begin{equation}\label{eq:pushforward-coverability}
	\Cpush(M) = \max_{h \in [H]} \inf_{\mu \in \Delta(\cS)} \sup_{(s,a,s') \in\cS \times \cA \times\cS} \frac{P_{h-1}(s' \mid s,a)}{\mu(s')}.
	\end{equation}
The pushforward coverability coefficient for an MDP class $\cM$ is defined by
\[
	\Cpush(\cM) = \max_{M \in \cM} \Cpush(M).
\]
\end{definition}

Note that for any MDP $M$ we always have
\begin{equation}\label{eq:coverability-pushforward-bound}
	\Ccov(M,\PiRNS) \leq \Cpush(M) \abs{\cA},
\end{equation}
where $\Ccov$ is the state-action coverability coefficient (\cref{def:state-action-cov}). Thus, an MDP with low pushforward coverability is also an MDP with low state-action coverability for all policies (upto a dependence on $\abs{\cA}$).

We will show the show the following result.

\pushforwardgolf*

The proof comes in three parts. We will firstly show that MDP that satisfies pushforward coverability admit low-dimensional feature maps that can approximate Bellman backups (\cref{sec:pushforward-structural}), then establish that a regret bound for the \Golf algorithm \citep{xie2022role} under misspecification (\cref{sec:golf-onpolicy-misspecification}), and then combine these ingredients (\cref{sec:sample-efficient-pushforward}).

\subsubsection{A structural result: Pushforward-coverable MDPs are
  approximately low-rank}\label{sec:pushforward-structural}
  
 Our central technical result for this section is \cref{lem:jl_cov_push}, which is
      based on a variant of the Johnson-Lindenstrauss lemma and establishes that under pushforward coverability, we can define a linear feature class which satisfies
      an approximate form of Bellman completeness. We define the clipping operator via
      \[
      	\clip(x) \coloneqq \max\crl*{\min\crl*{x,2},0}.
      \]

 	\begin{restatable}[Existence of a low-dimensional embedding]{lemma}{jlcovpush}
  \label{lem:jl_cov_push}
  Let $M $ be a known MDP with reward function $r$, transition kernel $P$, and pushforward coverability parameter $\Cpush$. Let $\mu = \crl{\mu_h}_{h \in [H]}$ denote its pushforward coverability distribution (i.e. the minimizer of \cref{def:pushforward_def_main}) and %
  $\cF\subseteq(\cS \times [H] \to[0,1])$ be an arbitrary class of functions. Suppose that we
  sample $W\in\pmo^{d\times\cS}$ as a matrix of independent Rademacher
  random variables, and define
  \begin{align*}
    \psi_h(s,a) = r_h(s,a) \oplus \frac{1}{\sqrt{d}}W\prn*{P_h(\cdot \mid s,a)/\mu^{1/2}_h(\cdot)}_{\cdot \in \cS} \in \bbR^{d+1}%
  \end{align*}
  and
  \begin{align*}
    w_{f,h} = 1 \oplus \frac{1}{\sqrt{d}}W\prn*{\mu^{1/2}_h(\cdot)f_{h+1}(\cdot)}_{\cdot \in \cS}\in\bbR^{d+1}.
  \end{align*}
  Then for any $\vepsapx \in(0,1)$, as long as we set
  \[
 	d \geq 2^9\frac{\Cpush\log(16\abs{\cF}H\deltainv/\vepsapx)}{\vepsapx},
 \]
 we have that with probability at least $1-\delta$, for all $f\in\cF$ and $h \in [H]$,
  \begin{align*}
	\En_{\mu_h \otimes \unif(\cA)}\brk*{\prn*{\clip\brk*{\tri*{w_{f,h},\psi_h(s,a)}-\cT_h f_{h+1}(s,a)}}^2}\leq\vepsapx,
  \end{align*}
  as well as $\max_{s,a,h} \nrm*{\psi_h(s,a)}_2^2 \leq \Cpush(16\log\prn{\abs{\cS}\abs{\cA}H} + 11)$ and $\max_{f,h} \nrm*{w_{f,h}}_2^2\leq{} 16\log\prn{\abs{\cF}H}+ 11$. We emphasize that the feature map $\psi = \crl*{\psi_h}_{h=1}^H$ is
  oblivious to $\cF$, in the sense that it can
  be computed directly from $M$ without any knowledge of $\cF$.
\end{restatable}

\begin{proof}[\pfref{lem:jl_cov_push}]%
Fix $h \in [H]$, whose dependence we omit for cleanliness. We begin by verifying that, in expectation, $\tri{w_f, \psi(s,a)}$ is equal to $\cT f(s,a)$. For this, note that 
  \begin{align*}
    \tri*{w_f,\psi(s,a)}
    &=
      r(s,a) + \frac{1}{d}\sum_{i=1}^{d}\prn*{\sum_{s' \in \cS}W_{i,s'}\frac{P(s'\mid s,a)}{\mu^{1/2}(s')}}\prn*{\sum_{s'' \in \cS}W_{i,s''}\mu^{1/2}(s'')f(s'')}\\
    &= r(s,a) + \sum_{s' \in \cS}P(s'\mid s,a) f(s') + \frac{1}{d}\sum_{i=1}^{d}\sum_{s' \in \cS}\sum_{\substack{s'' \in \cS\\ s'' \neq s'}}W_{i,s'}\frac{P(s'\mid s,a)}{\mu^{1/2}(s')}W_{i,s''}\mu^{1/2}(s'')f(s'').
  \end{align*}
  Consequently, we have 
  \begin{align}\label{eq:concentration-1}
    \abs*{\cT f(s,a)-\tri*{w_f,\psi(s,a)}}
    &=
      \abs*{\frac{1}{d}\sum_{i=1}^{d}\sum_{s' \in \cS}\sum_{\substack{s'' \in \cS\\ s'' \neq s'}}W_{i,s'}\frac{P(s'\mid s,a)}{\mu^{1/2}(s')}W_{i,s''}\mu^{1/2}(s'')f(s'')}.
  \end{align}
  Note that this remaining noise term is zero-mean -- we will show in the sequel that it can be made small by picking $d$ appropriately. We next examine the norms of the vectors $\psi(s,a)$ and $w_f$. Note that we have
\begin{align} 
\nrm*{\psi(s,a)}_2^2
  &=
    \frac{1}{d}\sum_{i=1}^d\prn*{\sum_{s'\in\cS}W_{i,s'}\frac{P(s' \mid s,a)}{\mu^{1/2}(s')}}^2 \nonumber \\
   &=
     \sum_{s' \in \cS} \frac{P^2(s' \mid s,a) }{\mu(s')} + \frac{1}{d}\sum_{i=1}^d\sum_{s' \in \cS}\sum_{\substack{s'' \in \cS\\ s'' \neq s'}}W_{i,s'}W_{i,s''}\frac{P(s' \mid s,a)}{\mu^{1/2}(s')}\frac{P(s'' \mid s,a)}{\mu^{1/2}(s'')} \nonumber \\ 
   &\leq
     \Cpush  + \frac{1}{d}\sum_{i=1}^d\sum_{s' \in \cS}\sum_{\substack{s'' \in \cS\\ s'' \neq s'}}W_{i,s'}W_{i,s''}\frac{P(s' \mid s,a)}{\mu^{1/2}(s')}\frac{P(s'' \mid s,a)}{\mu^{1/2}(s'') },\label{eq:concentration-2}
\end{align}
where we have used that 
\[
\sum_{s' \in \cS} \frac{P^2(s' \mid s,a) }{\mu(s')} \leq \Cpush \sum_{s' \in \cS} P(s' \mid s,a) = \Cpush
	 \]
	 by definition of pushforward coverability. Further note that we have
    \begin{align}
  \nrm*{w_f}^2_2
  &=
    \frac{1}{d}\sum_{i=1}^d\prn*{\sum_{s'\in\cS}W_{i,s'}\mu^{1/2}(s')f(s')}^2 \nonumber \\
  &= \En_{s' \sim \mu}[f(s')] + \frac{1}{d}\sum_{i=1}^d\sum_{s' \in \cS}\sum_{\substack{s'' \in \cS\\ s'' \neq s'}}W_{i,s'}W_{i,s''}\mu^{1/2}(s')f(s')\cdot \mu^{1/2}(s'')f(s'') \nonumber \\
  &\leq 1 + \frac{1}{d}\sum_{i=1}^d\sum_{s' \in \cS}\sum_{\substack{s'' \in \cS\\ s'' \neq s'}}W_{i,s'}W_{i,s''}\mu^{1/2}(s')f(s')\cdot \mu^{1/2}(s'')f(s''). \label{eq:concentration-3}
\end{align}
We will now appeal to the following technical lemma to upper bound \eqref{eq:concentration-1}, \eqref{eq:concentration-2}, and \eqref{eq:concentration-3} by establishing that the Rademacher noise terms concentrate to their expectations. The proof of the lemma will be given in the sequel.
  \begin{lemma}
    \label{lem:rademacher}
    Let $u,v\in\bbR^{n}$, and let $W\in\pmo^{d\times{}n}$ have
    independent Rademacher entries. Then with probability at least $1-\delta$,
    \begin{equation}\label{eq:rademacher}
      \abs*{\frac{1}{d}\sum_{i \in [d]}\sum_{j \in [n]}\sum_{\substack{k \in [n]\\k\neq{}j}}W_{i,j}W_{i,k}u_jv_k}
      \leq{} \nrm*{u}_2\nrm*{v}_2\cdot\sqrt{\frac{32\log(2\delta^{-1})}{d}}
      + \nrm*{u}^2_2\nrm*{v}^2_2\cdot\frac{64\log(2\delta^{-1})}{d}.
    \end{equation}
   Furthermore, for any set of vectors $\cV \subset \bbR^n$, we also have  
       \[
   \frac{1}{d}\max_{v \in \cV}\sum_{i \in [d]}\sum_{j \in [n]}\sum_{\substack{k \in [n]\\k\neq{}j}}W_{i,j}W_{i,k}v_jv_k
      \leq{} \max_{v \in \cV} \nrm{v}_2^2 (16\log\abs{\cV} + 9) + \max_{v \in \cV}\nrm*{v}^2_2\cdot\sqrt{\frac{32\log(2\delta^{-1})}{d}}
      + \max_{v \in \cV}\nrm*{v}^4_2\cdot\frac{64\log(2\delta^{-1})}{d}.
    \]
  \end{lemma}
	Let $(s,a) \in \cS \times \cA$ and $f \in \cF$. To bound $\abs{\tri*{\psi(s,a),w_f} - \cT f(s,a)}$ (cf. \eqref{eq:concentration-1}), we apply the first bound of
        \cref{lem:rademacher} with $u = \prn*{P(s' \mid s,a)/
          \mu^{1/2}(s')}_{s' \in \cS}$ and $v =
        \prn*{\mu^{1/2}(s')f(s')}_{s' \in \cS}$, which gives
	\begin{equation}\label{eq:northamericanscum}
		\abs*{\tri*{\psi(s,a),w_f} - \cT f(s,a)} \leq \sqrt{\frac{32\Cpush \log(2\deltainv)}{d}} + 64\Cpush\frac{\log(2\deltainv)}{d} \coloneqq \veps(\deltainv), 
	\end{equation}
	where we have again used that $
	 \nrm*{u}^2_2 = \sum_{s' \in \cS} \frac{P^2(s' \mid s,a) }{\mu(s')} \leq \Cpush
	 $ and also that $\nrm*{v}^2_2 = 1$ since $\nrm{f}_\infty \leq 1$ for all $f \in \cF$. To bound \eqref{eq:concentration-2}, we apply the second bound of  \cref{lem:rademacher} with $\cV = \crl*{\prn*{\frac{P_{h-1}(s' \mid s,a)}{\mu^{1/2}_{h}(s')}}_{s' \in \cS}}_{\substack{s,a \in \cS \times \cA \\
	 	h\times [H]}}$, which gives
	\[
	\max_{s,a \in \cS \times \cA, h\in [H]} \nrm*{\psi_h(s,a)}_2^2
  \leq
         \Cpush(16\log\abs{\cS}\abs{\cA}H+9) + \Cpush\sqrt{\frac{32\log(2\delta^{-1})}{d}}
      + \Cpush^2\frac{64\log(2\delta^{-1})}{d} \coloneqq B_1.
         \]
	Lastly, to bound \eqref{eq:concentration-3}, we take $\cV = \crl*{\prn*{\mu^{1/2}_h(s')f_h(s')}_{s' \in \cS}}_{\substack{f \in \cF \\ h \in [H]}}$ in \cref{lem:rademacher}, which establishes that
	\[
	\max_{f \in \cF, h \in [H]} \nrm{w_{f,h}}^2_2 \leq 9 + 16\log\abs{\cF}H + \sqrt{\frac{32\log(2\delta^{-1})}{d}}
      + \frac{64\log(2\delta^{-1})}{d} \coloneqq B_2.
	\]
Note that \eqref{eq:northamericanscum} establishes that the Bellman backup $\cT f(s,a)$ is well-approximated by  $\tri{\psi(s,a),w_f}$ only at a single state-action pair $(s,a)$. We can obtain an $L_\infty$-approximation guarantee by taking a union bound over $\cS$ and $\cA$, which would incur a dependence on $\log\abs{\cS}$ in the final sample complexity. Here, we bypass this by instead requiring only an approximation guarantee under the $L_2(\mu \otimes \unif(\cA))$ norm.  Via (pushforward) coverability, this will ensure that $\En^{\pi}\brk*{\prn*{\tri{w_f,\psi(s,a)} - \cT f (s,a)}^2}$ is well-controlled for all policies $\pi$, which will be sufficient for our downstream sample-complexity analysis of \Golf. However, directly establishing an $L_2(\mu \otimes \unif(\cA))$ approximation guarantee is technically challenging since it would require establishing a fourth-order (rather than second-order) equivalent of \eqref{eq:rademacher}. 
The remainder of the proof will obtain an $L_2(\mu \otimes \unif(\cA))$ approximation guarantee by instead sampling a dataset of size $n$ from $\mu \otimes \unif(\cA)$ and taking a union bound over that dataset to ensure a uniform bound on all state-action pairs in that dataset. Via an additional concentration bound, this will ensure that the error is well-behaved under the $L_2(\mu \otimes \unif(\cA))$ norm. %
        
For each $h \in [H]$, sample a dataset $D = \{ (s\ind{i}_h,a\ind{i}_h) \}_{i=1}^n$ i.i.d. from $\mu_h \otimes \unif(\cA)$. By a union bound over $n$, $\cF$, and $H$, we have that
		\begin{equation}\label{eq:theparty}
	\forall i \in [n], f \in \cF, h \in [H]: \quad	\abs*{\tri*{\psi_h(s\ind{i}_h,a\ind{i}_h),w_{f,h}} - \cT_h f_{h+1}(s\ind{i}_h,a\ind{i}_h)} \leq \veps(n\abs{\cF}H\deltainv), 
	\end{equation}
	where we recall the definition of $\veps(\cdot)$ from \eqref{eq:northamericanscum}.
	Now, let 
               	\[
	X_{f,h}(s,a) \coloneqq
        \prn*{\clip\brk*{\tri*{\psi_h(s,a),w_{f,h}}} - \cT_h
          f_{h+1}(s,a)}^2.
          \]         
          Note that $\abs{X_{f,h}(s,a)} \leq 4$ and 
           \[
           X_{f,h}(s,a) \leq\prn*{\tri*{\psi_h(s,a),w_{f,h}} - \cT_h
          f_{h+1}(s,a)}^2,
                \]
                since $\cT_h f_{h+1}(s,a) \in [0,2]$ and the clipping operator is 1-Lipshitz. Note that 
        \[
        \En_{(s,a) \sim \mu_h \otimes \unif(\cA)}\brk*{X_{f,h}(s,a)} \coloneqq
        \En_{\mu_h \otimes \unif(\cA)}\brk*{\prn*{\clip\brk*{\tri*{\psi_h(s,a),w_f}} - \cT_h
          f_{h+1}(s,a)}^2},
          \]
          where this expectation is only over the sampling of the data point $(s,a)$ (and not the Rademacher matrix $W$).
        Let 
        \[
        X_{i,f,h} \coloneqq X_{f,h}(s\ind{i}_h,a\ind{i}_h).
          \]
          By boundedness of $X_{f,h}(s,a)$ and Hoeffding's inequality, we have that with probability at least
       $1-\delta$: %
	\[
		\abs*{\frac{1}{n} \sum_{i=1}^n X_{i,f,h} - \En_{\mu \otimes \unif(\cA)}\brk*{X_{f,h}(s,a)}} \leq 4\sqrt{\frac{\log(2\deltainv)}{n}}.
              \]

	Taking another union bound over $\cF$ and $H$ as well as the event in \eqref{eq:theparty} gives that
		\begin{align}
		\forall f \in \cF, h \in [H]: \qquad &\abs*{\frac{1}{n} \sum_{i=1}^n X_{i,f,h} - \En_{\mu \otimes \unif(\cA)}\brk*{X_{f,h}(s,a)}} \leq 4\sqrt{\frac{\log(2\abs{\cF}H\deltainv)}{n}}\label{eq:usvthem}, \\
		\text{ and }\forall i \in [n], f \in \cF, h \in [H]:\qquad  &X_{i,f,h} \leq  \veps^2(n\abs{\cF}H\deltainv),
		\end{align}
		recalling the definition of $\veps(\cdot)$ from \eqref{eq:northamericanscum}.
      Then, re-arranging \eqref{eq:usvthem} gives us that
\begin{align}
 \En_{\mu \otimes \unif(\cA)}\brk*{\prn*{\clip\brk{\tri*{\psi_h(s_h,a_h),w_f}} - \cT_h
          f_{h+1}(s_h,a_h)}^2} &\leq \frac{1}{n} \sum_{i=1}^n X_{i,f,h} + 4\sqrt{\frac{\log(2\abs{\cF}H\deltainv)}{n}} \nonumber \\
		&\leq \veps^2(n\abs{\cF}H\deltainv) + 4\sqrt{\frac{\log(2\abs{\cF}H\deltainv)}{n}} \label{eq:blackmessiah},
	\end{align}

	 We now conclude the proof by picking $n$ and $d$ appropriately to ensure that the right-hand-side is bounded by $\vepsapx$, which will ensure the desired claim that
	\[
	 \En_{\mu \otimes \unif(\cA)}\brk*{\prn*{\clip\brk*{\tri*{\psi_h(s_h,a_h),w_f}} - \cT_h
          f_{h+1}(s_h,a_h)}^2} \leq \vepsapx.
	\]
	For convenience, we introduce absolute constants $c$ and $c'$ whose precise values may change from line to line. We pick $n = 64\log(2\abs{\cF}H\deltainv) / \vepsapx^2$. %
Plugging this into \cref{eq:blackmessiah} gives
		\begin{align}
 \En_{\mu \otimes \unif(\cA)}\brk*{\prn*{\clip\brk*{\tri*{\psi_h(s_h,a_h),w_f}} - \cT_h
          f_{h+1}(s_h,a_h)}^2} \leq  \veps^2(n\abs{\cF}H\deltainv) + c\cdot\veps
	\end{align}
	 Noting that $n \leq 128\frac{\abs{\cF}H\deltainv}{\vepsapx^2}$ and plugging this into $\veps$ (\eqref{eq:northamericanscum}) gives
	 \begin{equation}\label{eq:sanctified}
	\veps(n\abs{\cF}H\deltainv) \leq \Cpush^{1/2}\sqrt{\frac{64\log(16\abs{\cF}H\deltainv/\vepsapx)}{d}} + \Cpush\frac{128\log(16\abs{\cF}H\deltainv/\vepsapx)}{d}.
	 \end{equation}
	 Setting%
	 	 \[
	 	d \geq 2^9 \frac{\Cpush \log(16\abs{\cF}H\deltainv/\vepsapx)}{\vepsapx}
	 \]
	ensures that 
	\begin{equation}\label{eq:sunglasses}
	\veps^2(n\abs{\cF}H\deltainv)\leq \veps(n\abs{\cF}H\deltainv) \leq \frac{\vepsapx}{2}
	\end{equation}
	by \eqref{eq:sanctified}. %
	 Combining \eqref{eq:blackmessiah} and \eqref{eq:sunglasses}, we get 
\begin{align}
	\En_{\mu \otimes \unif(\cA)}\brk*{\prn*{\clip\brk*{\tri*{\psi_h(s_h,a_h),w_f}} - \cT_h f_{h+1}(s_h,a_h)}^2} \leq \vepsapx,
	\end{align}
as desired. It only remains to establish the concentration results of \cref{lem:rademacher}. %
\end{proof}

 \begin{proof}[\pfref{lem:rademacher}]
    We establish the first claim. Let $i \in [d]$ be fixed, and consider the random variable
    \begin{align*}
      Z_i\ldef{}
      \sum_{j \in [n]} \sum_{\substack{k\in[n]\\k\neq{}j}}W_{i,j}W_{i,k}v_ju_k. %
    \end{align*}
    Note that $\En\brk*{Z_i} = 0$ by independence of $W_{i,j}$ and $W_{i,k}$ for every $j \neq k$. By Exercise 6.9 of \citet{boucheron2013concentration}, we have that
    \begin{align*}
      \log\En\brk*{\exp\prn*{\lambda{}Z_i}}
      &\leq{}
        \frac{16\lambda^2}{2(1-64\nrm{u}_2^2\nrm{v}_2^2\lambda)}\nrm{u}_2^2\nrm{v}_2^2.
    \end{align*}
    Since $Z_i$ are independent, it follows that
    \begin{align*}
      \log\En\brk*{\exp\prn*{\lambda{}\sum_{i=1}^{d}Z_i}}
      &\leq{}
        \frac{16\lambda^2}{2(1-64\nrm{u}_2^2\nrm{v}_2^2\lambda)}\nrm{u}_2^2\nrm{v}_2^2d.
    \end{align*}
Hence, $\sum_{i=1}^{d}Z_i$ is a sub-Gamma random variable with parameters
$\nu=16\nrm*{u}_2^2\nrm*{v}_2^2d$ and
$c=64\nrm*{u}_2^2\nrm*{v}_2^2$, and it follows from Equation (2.5) on page 29 of
\citet{boucheron2013concentration} that for all $\veps>0$,
\begin{align*}
  \bbP\prn*{
  \sum_{i=1}^{d}Z_i
\geq{} \nrm{u}_2\nrm{v}_2\sqrt{32d\veps}+64\nrm*{u}_2^2\nrm*{v}_2^2\veps
  }\leq{}e^{-\veps}.
\end{align*}
Taking a union bound, and using that the random variable is symmetric, we obtain the desired claim. 

We now establish the second claim. Let $\cV \subset \bbR^{n}$ be a subset of vectors. Let $i \in [d]$ be fixed, and re-consider the random variable
    \begin{align*}
      Z_i\ldef{} \max_{v \in \cV}
      \sum_{j \in [n]} \sum_{\substack{k \in [n]\\k\neq j}}W_{i,j}W_{i,k}v_jv_k.
    \end{align*}
    Again appealing to Exercise 6.9 of \citet{boucheron2013concentration}, %
    we have that
    \begin{align*}
      \log\En\brk*{\exp\prn*{\lambda{}\prn*{Z_i-\En\brk{Z_i}}}}
      &\leq{}
        \frac{16\lambda^2}{2(1-64B\lambda)}\En\brk*{\max_{v \in \cV} \sum_{j \in [n]} \sum_{\substack{k \in [n]\\k\neq j}} W_{i,j}W_{i,k}v^2_jv^2_k} \\
      &\leq{}
        \frac{16\lambda^2}{2(1-64B\lambda)}\En\brk*{\max_{v \in \cV}\sum_{j,k=1}^n v^2_jv^2_k} \\
      &=
        \frac{16\lambda^2}{2(1-64B\lambda)}\max_{v \in \cV}\nrm*{v}_2^4
    \end{align*}
    where $B \ldef \max_{v \in \cV} \nrm{v}_2^4$.
    Since $Z_i$ are independent, it follows that
    \begin{align*}
      \log\En\brk*{\exp\prn*{\lambda{}\sum_{i=1}^{d}\prn*{Z_i - \En\brk{Z_i}}}}
      &\leq{}
        \frac{16\lambda^2}{2(1-64B\lambda)}\max_{v \in \cV}\nrm*{v}_2^4d.
    \end{align*}
Hence, $\sum_{i=1}^{d}Z_i$ is a sub-Gamma random variable with parameters
$\nu=16\max_{v \in \cV}\nrm*{v}_2^4 d$ and
$c=64\max_{v \in \cV}\nrm*{v}_2^4$, and it follows from Equation (2.5) on page 29 of
\citet{boucheron2013concentration} that for all $\veps>0$,
\begin{align*}
  \bbP\prn*{
  \frac{1}{d} \sum_{i=1}^{d}Z_i
\geq{} \En[Z_i] + \max_{v \in \cV}\nrm*{v}^2_4\sqrt{\frac{32\veps}{d}}+64\max_{v \in \cV}\nrm*{v}_2^4\frac{\veps}{d}
  }\leq{}e^{-\veps}.
\end{align*}

To conclude, it remains only to show the bound $\En\brk{Z_i} \leq \max_v \nrm{v}_2^2 (16\log\abs{\cV} + 9)$. %
 This follows by a standard log-sum-exp approach. Below, we abbreviate $\rho_j \coloneqq W_{i,j}$. We can observe that for any $\lambda > 0$:
\begin{align}
 	\En\brk{Z_i} &= \En\brk*{\max_{v \in \cV} \sum_{j \in [n]}\sum_{\substack{k \in [n] \\ k \neq j}}\rho_j \rho_k v_j v_k} \nonumber \\
 	&\leq \frac{1}{\lambda} \log\prn*{\sum_{v \in \cV} \En\brk*{\exp(\lambda \sum_{j \in [n]} \sum_{\substack{k \in [n] \\ k \neq j}} \rho_j \rho_k v_j v_k)} } \nonumber\\
 	&\leq \frac{1}{\lambda} \log\prn*{\sum_{v \in \cV} \En\brk*{\exp(\lambda \prn*{\sum_{j=1}^n \rho_j v_j}^2)}} \label{eq:squaredrademacher}
\end{align}
Note that $X \coloneqq \sum_j \rho_j v_j$ is subGaussian with parameter $\nrm{v}_2^2$, since:
\[
	\En\brk*{\exp(\lambda \sum_{j=1}^n \rho_j v_j )} = \prod_{j=1}^n \En\brk*{\exp(\lambda \rho_j v_j)} \leq \prod_{j=1}^n \exp(\frac{\lambda^2 v^2_j}{2}) = \exp(\frac{\lambda^2}{2} \nrm{v}_2^2).
\]

Then, it follows (e.g. Lemma 1.12 of \citet{rigollet2023high}) that $X^2 - \En\brk{X^2}$ satisfies a sub-exponential MGF bound with parameter $16\nrm{v}_2^2$, i.e.
\[
\En\brk{\exp(\lambda(X^2 - \En\brk{X^2}))} \leq \exp(\frac{256}{2}\lambda^2 \nrm{v}_2^4) \qquad \forall \abs{\lambda} \leq \frac{1}{16\nrm{v}_2^2}.
\]

We also note that
\[
\En\brk{X^2} = \sum_{i,j=1}^n v_i v_j \En\brk{\veps_i \veps_j} = \nrm{v}_2^2. 
\]
Adding and subtracting $\En\brk{X^2}$ in \eqref{eq:squaredrademacher} gives
\begin{align*}
	&\leq \frac{1}{\lambda} \log\prn*{\sum_{v \in \cV} \En\brk*{\exp(\lambda \prn*{X^2 - \nrm{v}_2^2} + \lambda\nrm{v}_2^2)}} \\
	 &= \frac{1}{\lambda} \log\prn*{\sum_{v \in \cV} \En\brk*{\exp(\lambda \prn*{X^2 - \nrm{v}_2^2})}\exp(\lambda\nrm{v}_2^2)} \\
	 &\leq \frac{1}{\lambda} \log\prn*{\sum_{v \in \cV} \exp(128\lambda^2\nrm{v}_2^4 + \lambda\nrm{v}_2^2)} \qquad \forall \abs{\lambda} \leq \frac{1}{16\max_v \nrm{v}_2^2}\\
	&\leq \frac{1}{\lambda}\log\abs{\cV} + \max_v 128\lambda\nrm{v}_2^4 + \max_v \nrm{v}_2^2 \qquad \forall \abs{\lambda} \leq \frac{1}{16\max_v \nrm{v}_2^2}
\end{align*}
Picking $\lambda = \frac{1}{16 \max_v \nrm{v}_2^2}$ concludes the proof.
\end{proof}

\subsubsection{\textsc{Golf} with on-policy misspecification}\label{sec:golf-onpolicy-misspecification}
Consider the version of \Golf \cite{jin2021bellman} in
\cref{alg:golf}. We have the following guarantee for the regret of
\Golf, which extends \citet{jin2021bellman} to allow for \emph{on-policy} misspecification. 

\begin{algorithm}[t]
\caption{\textsc{GOLF} \citep{jin2021bellman}}
\label{alg:golf}
{\bfseries input:} Function classes $\cF$ and $\cG$, confidence width $\beta>0$. \\
{\bfseries initialize:} $\cF\ind{0} \leftarrow \cF$, $\cD_{h}\ind{0} \leftarrow \emptyset\;\;\forall h \in [H]$. 
\begin{algorithmic}[1]
\For{episode $t = 1,2,\dotsc,T$}
    \State Select policy $\pi\ind{t} \leftarrow \pi_{f\ind{t}}$, where $f\ind{t} \ldef{} \argmax_{f \in \cF\ind{t-1}}f(x_1,\pi_{f,1}(x_1))$. \label{step:glof_optimism}
    \State Execute $\pi\ind{t}$ for one episode and obtain trajectory $(x_1\ind{t},a_1\ind{t},r_1\ind{t}),\ldots,(x_H\ind{t},a_H\ind{t},r_H\ind{t})$. \label{step:glof_sampling}
    \State Update dataset: $\cD_{h}\ind{t} \leftarrow \cD_{h}\ind{t-1} \cup \crl[\big]{\prn[\big]{x_h\ind{t},a_h\ind{t},x_{h+1}\ind{t}}}\;\;\forall h \in [H]$.
    \State Compute confidence set:
    \begin{gather}
      \nonumber
      \cF\ind{t} \leftarrow \crl[\bigg]{ f \in \cF: \cL_{h}\ind{t}(f_h,f_{h+1}) - \min_{g_h \in \cG_h} \cL_{h}\ind{t}(g_h,f_{h+1}) \leq \beta\;\;\forall h \in [H] },
    \\
    \nonumber
    \text{where \quad } \cL_{h}\ind{t}(f,f') \coloneqq \sum_{(x,a,r,x') \in \cD_{h}\ind{t}}\prn[\Big]{ f(x,a) - r - \max_{a' \in \cA} f'(x',a') }^2 ,~\forall f,f' \in \cF.
    \end{gather}
\EndFor
\State Output $\wh{\pi} = \unif(\pi\ind{1:T})$. 
\end{algorithmic}
\end{algorithm}

\newcommand{\apx}{\mathsf{apx}}

\begin{lemma}
  \label{lem:golf-onpolicy}
  
Suppose that $Q^{\Mobsstar,\star} \in\cF$ and $\cG$ satisfies
  $\vepsapx$-completeness in the sense that for all $h \in [H]$ and  $f\in\cF_{h+1}$,
  there exists $g \in\cG_h$ such that
  $\En^{\pi}\prn*{g - \cT^{\Mstarobs}_h f}^2\leq{}\vepsapx^2$ for all $\pi \in \Pi_\cF \coloneqq \crl*{\pi_f : f \in \cF}$.  Let $\Ccov \coloneqq \Ccov(\Mobsstar,\Pi_\cF)$ (\cref{def:coverability}). Then for an
  appropriate choice of $\beta$, \cref{alg:golf} ensures that
  \begin{align*}
    \Reg
    \leq{} H\sqrt{\Ccov{}T\log(\abs{\cF}\abs{\cG}HT/\delta)} + HT\sqrt{\Ccov \log(T)}\vepsapx.  \end{align*}
  
\end{lemma}
\begin{proof}[\pfref{lem:golf-onpolicy}]%
  For each $f_{h+1} \in \cF_{h+1}$, let $\apx\brk{f_h} = \argmin_{g_h \in \cG_h} \sup_{\pi \in \Pi} \En^\pi\brk*{\prn*{g_h - \cT_h f_{h+1}}^2}$. Let 
\[
\delta\ind{t}_h(\cdot,\cdot) \coloneqq f\ind{t}_h(\cdot,\cdot) - \cT_f f\ind{t}_{h+1}(\cdot,\cdot) \quad \& \quad \wt\delta\ind{t}_h(\cdot,\cdot) \coloneqq f\ind{t}_h(\cdot,\cdot) - \apx\brk*{f\ind{t}_{h+1}}(\cdot,\cdot),
\]
and note that by Jensen's inequality we have that for all $\pi$, $\En^\pi\brk*{\delta\ind{t}_h(\cdot,\cdot)} \leq \En^\pi\brk*{\wt\delta\ind{t}_h(\cdot,\cdot)} + \vepsapx$. We further adopt the shorthand $d\ind{t}_h(x,a) \coloneqq d^{\pi\ind{t}}_h(x,a)$ and $\wtd\ind{t}_h(x,a) \coloneqq \sum_{i<t}d\ind{t}_h(x,a)$. 
As a consequence of realizability ($\Qstarobsh \in \cF_h$) and
approximate Bellman completeness, standard concentration arguments
(proved in the sequel) lead to the following result.

\begin{lemma}[Optimism and small in-sample squared Bellman errors]\label{lem:golf-all-policy-concentration}
With probability at least $1-\delta$, by taking $\beta = c \log(TH\abs{\cF}\abs{\cG}/\delta) + T\vepsapx$, we have that for all $t \in [T]$,  
\[
	(i) \,\, \Qstarobsh \in \cF\ind{t}, \quad \text{and} \quad (ii) \,\, \sum_{x,a} \wtd\ind{t}_h(x,a) \prn*{\wt{\delta}\ind{t}_h(x,a)}^2 \leq \cO(\beta).
\]
\end{lemma}

The rest of the proof proceeds similarly to the analysis of Section
3.2 in \citet{xie2022role}. Namely, by optimism (\cref{lem:golf-all-policy-concentration}) and a standard Bellman error decomposition (\cref{lem:lemma1jiang}) we have %
\[
	\Reg \leq \sum_{t=1}^T \sum_{h=1}^H \En_{d\ind{t}_h}\brk*{\delta\ind{t}_h(x,a)} \leq TH\cdot \vepsapx + \sum_{t=1}^T\sum_{h=1}^H \En_{d\ind{t}_h}\brk*{\wt\delta\ind{t}_h(x,a)}.
\]
Let us defining the burn-in time
\[
	\tau_h(x,a) = \min\{ t \mid \wtd\ind{t}_h(x,a) \geq \Ccov \mustar_h(x,a)\},
\]
where $\mustarh$ is the coverability distribution for the set of
policies $\Pi_\cF$ (i.e., the distribution $\mustarh$ that achieves
the minimum in the coverability definition). Using the same decomposition into ``burn-in
phase'' and ``stable phase'' in \citet{xie2022role}, we have:
\[
\sum_{t=1}^T\sum_{h=1}^H \En_{d\ind{t}_h}\brk*{\wt\delta\ind{t}_h(x,a)} \leq 2H\Ccov + \sum_{t=1}^T\sum_{h=1}^H \En_{d\ind{t}_h}\brk*{\wt\delta\ind{t}_h(x,a) \indic\crl*{t \geq \tau_h(x,a)}}.
\]
Applying a change of measure argument on the second term then gives:
\begin{align*}
\sum_{t=1}^T\sum_{h=1}^H \En_{d\ind{t}_h}\brk*{\wt\delta\ind{t}_h(x,a) \indic\crl*{t \geq \tau_h(x,a)}} \leq H \underbrace{\sqrt{\sum_{t=1}^T\sum_{x,a} \frac{\prn*{\indic\{t \geq \tau_h(x,a)\}d\ind{t}_h(x,a)}^2}{\wtd\ind{t}_h(x,a)}}}_{\texttt{(A)}} \underbrace{\sqrt{\sum_{t=1}^T\sum_{x,a} \wtd\ind{t}_h(x,a) \prn*{\wt\delta\ind{t}_h(x,a)}^2}}_{\texttt{(B)}}
\end{align*}
By the same reasoning as in \citet{xie2022role}, we have $\texttt{(A)} \leq \cO(\sqrt{\Ccov\log(T)})$, and by \cref{lem:golf-all-policy-concentration} we have $\texttt{(B)} \leq \cO(\sqrt{\beta T})$. Using that $\beta = \log(TH\abs{\cF}/\delta) + T\vepsapx^2$  gives the desired result. It remains to establish the concentration results of \cref{lem:golf-all-policy-concentration}.  
\end{proof}

\begin{proof}[\pfref{lem:golf-all-policy-concentration}]
  For any function $f$, define a random variable
	\[
	X_t(h,f) = \prn*{f_h(s\ind{t}_h,a\ind{t}_h) - r\ind{t}_h - f_{h+1}(s\ind{t}_{h+1})}^2- \prn*{\cT_h f_{h+1}(s\ind{t}_h,a\ind{t}_h) - r\ind{t}_h - f_{h+1}(s\ind{t}_{h+1})}^2.
	\]
	Let  $\fF_{t,h} = \{s\ind{i}_1,a\ind{i}_1,r\ind{i}_1,\ldots,s\ind{i}_H,a\ind{i}_H,r\ind{i}_H\}_{i<t}$. Note that
	\begin{align}
		 \En\brk*{r\ind{t}_h + f_{h+1}(s\ind{t}_{h+1}) \mid \fF_{t,h}} =  \En^{\pi\ind{t}}\brk*{\cT_h f(s_h,a_h)}.
	\end{align}
	and thus that
	\[
	\En\brk*{X_t(h,f) \mid \fF_{t,h}} = \En^{\pi\ind{t}}\brk*{\prn*{f_h(s_h,a_h) - \cT_h f_h(s_h,a_h)}^2}.
	\]
	Next, note that
	\begin{align*}
			\Var\brk*{X_t(h,f) \mid \fF_{t,h}} &\leq \En\brk*{\prn*{X_t(h,f)}^2 \mid \fF_{t,h}} \\
			&\hspace{-3em}\leq \En\brk*{\prn*{f_h(s\ind{t}_h,a\ind{t}_h) - \cT_h f_h(s\ind{t}_h,a\ind{t}_h)}^2\prn*{f_h(s\ind{t}_h,a\ind{t}_h) + \cT_h f_h(s\ind{t}_h,a\ind{t}_h) + 2\prn*{r\ind{t}_h - f_{h+1}(s\ind{t}_{h+1})}}^2 \mid \fF_{t,h}} \\
			&\hspace{-3em}\leq 16\En\brk*{\prn*{f_h(s\ind{t}_h,a\ind{t}_h) - \cT_h f_h(s\ind{t}_h,a\ind{t}_h)}^2\mid \fF_{t,h}} = 16\En\brk*{X_t(h,f) \mid \fF_{t,h}}.
	\end{align*}
	By Freedman's inequality (\cref{lem:freedman}, \cref{lem:multiplicative_freedman}), we have that with probability at least $1-\delta$:
	\[
		\abs*{\sum_{i<t} X_i(h,f) - \sum_{i<t} \En\brk*{X_i(h,f) \mid \fF_{i,h}}} \leq \cO\prn*{\sqrt{\log(1/\delta)\sum_{i<t}\En\brk*{X_i(h,f) \mid \fF_{i,h}}} + \log(1/\delta)}
	\]
	Taking a union bound over $[T] \times [H] \times \cF$, we have
        that for all $t,h,f$, with probability at least $1-\delta$:
	\begin{equation}\label{eq:x-t-minus-bellman-errors-sec3} 
		\abs*{\sum_{i<t} X_i(h,f) - \sum_{i<t} \En^{\pi\ind{i}}\brk*{\prn*{f_h(s_h,a_h) - \cT_h f_h(s_h,a_h)}^2}} \leq \cO\prn*{\sqrt{\iota\sum_{i<t}\En^{\pi\ind{i}}\brk*{\prn*{f_h(s_h,a_h) - \cT_h f_h(s_h,a_h)}^2}} + \iota},
	\end{equation}
	where $\iota = \log(\abs*{\cF}HT/\delta)$. We now show that 
	\begin{equation}\label{eq:latent-golf-x-t-bounded-sec3}
		\sum_{i<t} X_i(h,f\ind{t}) \leq \beta + \cO\prn*{T\vepsapx^2 + \iota} = \cO\prn*{\beta},
	\end{equation}
	which will imply, from \eqref{eq:x-t-minus-bellman-errors-sec3}, that
	\[
	\sum_{i<t} \En^{\pi\ind{t}}\brk*{\prn*{f_h(s_h,a_h) - \cT_h f_h(s_h,a_h)}^2} \leq \cO\prn*{\iota + \beta} = \cO(\beta),
	\]
	as desired. To see \eqref{eq:latent-golf-x-t-bounded-sec3}, let 
	\[
	 \Delta_t = \sum_{i<t} \prn*{\apx\brk*{\cT_h f\ind{t}_{h+1}}(s\ind{i}_h,a\ind{i}_h) - r\ind{i}_h - f\ind{t}_{h+1}(s\ind{i}_{h+1})}^2	- \prn*{\cT_h f\ind{t}_h(s\ind{i}_h,a\ind{i}_h) - r\ind{i}_h - f\ind{t}_{h+1}(s\ind{i}_{h+1})}^2
	\]
	and then note that:
	\begin{align*}
	\sum_{i<t} X_i(h,f\ind{t}) &= \sum_{i<t} \prn*{f\ind{t}_h(s\ind{i}_h,a\ind{i}_h) - r\ind{i}_h - f\ind{t}_{h+1}(s\ind{i}_{h+1})}^2- \prn*{\cT_h f\ind{t}_h(s\ind{i}_h,a\ind{i}_h) - r\ind{i}_h - f\ind{t}_{h+1}(s\ind{i}_{h+1})}^2	\\
	&= \sum_{i<t} \prn*{f\ind{t}_h(s\ind{i}_h,a\ind{i}_h) - r\ind{i}_h - f\ind{i}_{h+1}(s\ind{i}_{h+1})}^2- \prn*{\apx\brk*{\cT_h f\ind{t}_{h+1}}(s\ind{i}_h,a\ind{i}_h) - r\ind{i}_h - f\ind{t}_{h+1}(s\ind{i}_{h+1})}^2	 + \Delta_t \\ 
	&\leq \sum_{i<t} \prn*{f\ind{t}_h(s\ind{i}_h,a\ind{i}_h) - r\ind{i}_h - f\ind{t}_{h+1}(s\ind{i}_{h+1})}^2- \inf_{g_h \in \cG_h}\sum_{i<t} \prn*{g(s\ind{i}_h,a\ind{i}_h) - r\ind{i}_h - f\ind{t}_{h+1}(s\ind{i}_{h+1})}^2 + \Delta_t\\ 
	&\leq \beta + \Delta_t.
	\end{align*}
where the second-to-last line follows from $\apx\brk*{\cT_h f\ind{t}_{h+1}} \in \cG$ and the last line follows from the definition of the confidence set. It remains to show that $\Delta_t \leq \cO(T\vepsapx^2 + \iota)$, which we do via a similar concentration argument. Namely, let
	\[
		Y_t(h,f) = \prn*{\apx\brk*{\cT_h f_{h+1}}(s\ind{t}_h,a\ind{t}_h) - r\ind{t}_h - f\ind{k}_{h+1}(s\ind{t}_{h+1})}^2	- \prn*{\cT_h f_h(s\ind{t}_h,a\ind{t}_h) - r\ind{t}_h - f\ind{k}_{h+1}(s\ind{t}_{h+1})}^2,
	\]
	and note that, as before,
	\[
		\En\brk*{Y_t(h,f) \mid \fF_{t,h}} = \En^{\pi\ind{t}}\brk*{\prn*{\apx\brk*{\cT_h f_{h+1}}(s_h,a_h) - \cT_h f_h(s_h,a_h)}^2},
	\]
	and 
	\[
		\Var\brk*{Y_t(h,f) \mid \fF_{t,h}} \leq 16\En\brk*{Y_t(h,f) \mid \fF_{t,h}},
	\]
	by the same calculation as earlier.
	Thus, by Freedman's inequality and a union bound, we have that, with probability at least $1-\delta$,
		\begin{align}\label{eq:y-t-minus-bellman-errors} 
&\abs*{\sum_{i<t} Y_t(h,f) - \sum_{i<t} \En^{\pi\ind{t}}\brk*{\prn*{\apx\brk*{\cT_h f_{h+1}}(s_h,a_h) - \cT_h f_h(s_h,a_h)}^2}} \\ 
	&\leq \cO\prn*{\sqrt{\iota\sum_{i<t}\En^{\pi\ind{t}}\brk*{\prn*{\apx\brk*{\cT_h f_{h+1}}(s_h,a_h) - \cT_h f_h(s_h,a_h)}^2}} + \iota},
	\end{align}
	where $\iota = \log(\abs*{\cF}HT/\delta)$. Recalling the misspecification assumption, this implies that
	\[
	\sum_{i<t} Y_t(h,f) \leq \cO\prn*{t \vepsapx^2 + \iota},
	\]
	for all $h,f,t$, with high probability. This concludes the result for $(ii)$. For $(i)$, this follows identically to the proof of Lemma 40 in \citet{jin2021bellman}, since this only uses the property that $\Qstar \in \cF$.

\end{proof}

\subsubsection{Sample-efficient latent-dynamics RL under pushforward coverability}\label{sec:sample-efficient-pushforward}

We conclude by combining the previous two results to obtain the main result for this section.

\pushforwardgolf*

\begin{proof}[\pfref{thm:pushforwardgolf}]%
Let $\Mobsstar \coloneqq \dtri*{\Mstarlat,\psistar} \in \cMlatphi$ be the unknown latent-dynamics MDP. Define observation-level value functions
\begin{align*}
  \cF = \crl*{\QMstarlat\circ\phi\mid{}\Mlat \in \cMlat, \phi\in\Phi},
\end{align*}
so that $Q^{\Mobsstar,\star}= Q^{\Mstarlat,\star} \circ \phistar \in\cF$ via decoder and model realizability, and
$\log\abs{\cF_h}\leq{}\log\abs{\cMlat}\abs{\Phi}$. Consider any function class $\cL\subseteq\crl{\cS\to\brk{0,1}}$ and MDP $\Mlat =(\rlat,\Plat)$. For a given value $\vepsapx>0$, setting $d$ according to \cref{lem:jl_cov_push} implies that there exists a $d$-dimensional feature map $\vphi_{\sMlat,h}(s,a)\in\bbR^{d+1}$ %
  such that for all $\ell\in\cL$ and $h \in [H]$,
  there exists $w_{\ell,h}\in\bbR^{d+1}$ %
  such that
  \begin{align}\label{eq:w-ell}
    \En_{\mu_{\Mlat} \otimes \unif(\cA)}\brk*{\prn*{\clip\brk*{\tri*{\vphi_{\sMlat,h}(s,a),w_{\ell,h}}}
    -\cT^{\Mlat}_h \ell_{h+1}(s,a)
    }^2} \leq \vepsapx,
  \end{align}
  where $\mu_{\Mlat}$ is the pushforward coverability distribution for $\Mlat$.
  Moreover, the map $\vphi_h$ is explicitly computed as a function of
  $\Mlat$ by a randomized algorithm with success probability $1-\delta$, with no knowledge of the class $\cL$ required. We
  consider the class
  \begin{align}\label{eq:L-function-class}
  	\cL = \crl*{\hphi \circ Q^{\Mlat,\star}(s,a) \coloneqq \sum_{s' \in \cS} \hphi(s' \mid s) Q^{\Mlat,\star}(s',a) \mid \phi \in \Phi, \Mlat \in \cMlat},
  \end{align}
  where $\hphi : \cS \rightarrow \Delta(\cS)$ is the mismatch function for decoder $\phi$ and emission $\psistar$, defined in \cref{def:h_phi}. %
  Note that $\cL$ has size $\log\abs*{\cL}\leq\log\abs{\cMlat}\abs{\Phi}$, and that we have
  \[
  	\cT^{\Mobsstar}_h(Q^{\Mlat,\star}_h \circ \phi_h)(x,a) = \cT^{\Mstarlat}_h(\hphihpo \circ V^{\Mlat,\star}_h)(\phistar_h(x),a)
  \]
  by \cref{lem:hphi-bellman}.
   By \cref{lem:covinvariance} we have that $\mu_{\Mobsstar,h}(x) = \psistarh(x \mid \phistarh(x)) \mu_{\Mstarlat,h}(\phistarh(x))$ is the coverability distribution for MDP $\Mstarobs$, and 
  \[
  	\En_{\mu_{\Mstarlat} \otimes \unif(\cA)}\brk*{f(s,a)} = \En_{\mu_{\Mobsstar} \otimes \unif(\cA)}\brk*{f(\phistar(x),a)}.
  \]
  Now, define 
  \begin{align*}
    \cG_{\sMlat,h} = \crl*{(x,a)\mapsto{}\clip\brk*{\tri*{\vphi_{\Mlat,h}(\phi(x),a),w}}\mid{}\phi\in\Phi,\nrm*{w}^2_2\leq{} 11+16\log\prn{\abs{\cMlat}\abs{\Phi}H}}.
  \end{align*}
  Recall the definition of $w_f$ (for $f :\cS \times \cA \rightarrow [0,1]$) from \cref{lem:jl_cov_push}, and note that by the norm bound $\max_{\ell \in \cL} \nrm{w_\ell}_2^2 \leq 11+16\log\prn{\abs{\cMlat}\abs{\Phi}H}$ given by \cref{lem:jl_cov_push}, we have $(x,a)\mapsto{}\tri*{\vphi_{\Mlat,h}(\phi(x),a),w_\ell} \in \cG_h$ for every $\ell \in \cL$. Next, note that by the norm bound $\max_{s,a} \nrm{\psi(s,a)}^2_2 \leq \Cpush(11+16\log\prn{\abs{\cS}\abs{\cA}H})$, given by \cref{lem:jl_cov_push}, we have every $g_h \in \cG_{\Mlat,h}$ satisfies $\nrm{g_h}_\infty \leq c\Cpush^{1/2}\log\prn*{\abs{\cMlat}\abs{\Phi}\abs{\cS}\abs{\cA}H} \coloneqq B$ for some absolute constant $c$. Therefore, $\cG_{\Mlat,h}$ has size \paedit{$\log\abs{\cG_{\Mlat,h}}\leq\bigoht\prn*{d \cdot \log(B) + \log\abs{\Phi}} = \bigoht\prn*{d \log\log(\abs{\cS}) + \log\abs{\Phi}}$}, where the $\bigoht$ notation ignores logarithmic factors of $\Cpush$, $\abs{\cA}$, $\log\abs{\cMlat}$, and $\log\abs{\Phi}$.\footnote{Formally, this requires a standard covering number argument; we omit the details.} Define $\cG_h = \cup_{\Mlat \in \cMlat} \cG_{\Mlat,h}$, which has size $\log\abs{\cG_h} \leq \log\abs{\Mlat}\dfedit{+}(\bigoht\prn*{d\log\log(\abs{\cS}) + \log\abs{\Phi}})$. Together, these results with \cref{lem:jl_cov_push} imply that for all $f_{h+1} \in\cF_{h+1}$, there exists $g_h \in \cG_{h}$ such that
  \begin{align*}
    \En_{\mu_{\Mobsstar,h} \otimes \unif(\cA)}\brk*{\prn*{g_h(x_h,a_h)
    - \brk*{\cT^{\Mobsstar}_{h}f_{h+1}}(x_h,a_h)}^2}\leq\vepsapx.
  \end{align*}

  This, in turn, implies that for all $\piobs \in \PiRNS$ we have
   \begin{align*}
    \En^{\piobs}\brk*{\prn*{g_h(x_h,a_h)
    - \brk*{\cT^{\sMstarobs}_{h}f_{h+1}}(x_h,a_h)}^2} \leq\Cpush \abs{\cA}\vepsapx,
  \end{align*}
  since $\mu_{\Mobsstar,h} \otimes \unif(\cA)$ satisfies coverability (\cref{def:coverability}) with parameter $\Ccov(\Mobsstar,\PiRNS) \leq \Cpush \abs{\cA}$ (\eqref{eq:coverability-pushforward-bound}). 
   
  Then, it follows by \cref{lem:golf-onpolicy} that if we run
  \cref{alg:golf} with the classes $\cF$ and $\cG$ we will get
       \begin{align*}
    \Reg
    &\leq{} H\sqrt{\Cpush\abs{\cA}{}T\log(\abs{\cMlat}\abs{\Phi}HT/\delta)(d\log\log(\abs{\cS})+\log\abs{\Phi})} + HT\sqrt{\Cpush^2\abs{\cA}^2\log(T)\vepsapx} \\
    &\leq{} H\sqrt{\Cpush^5\abs{\cA}{}T\log(\abs{\cMlat}\abs{\Phi}HT/\delta)\frac{\log(\Cpush^2\abs{\cMlat}\abs{\Phi}^2H\deltainv/\vepsapx)\log\log(\abs{\cS})}{\vepsapx}} + HT\sqrt{\Cpush^2\abs{\cA}^2\log(T)\vepsapx}
       \end{align*}

     Choosing $\vepsapx = \frac{1}{\sqrt{T}}$ to balance leads to
    \begin{align*}
          \Reg &\approxleq{} HT^{3/4}\sqrt{\Cpush^5\abs{\cA}{}\log(\abs{\cMlat}\abs{\Phi}HT/\delta)\log(\Cpush^2\abs{\cMlat}\abs{\Phi}^2H\deltainv T)\log\log(\abs{\cS})} + HT^{3/4}\sqrt{\Cpush^2\abs{\cA}^2\log(T)} \\
          &\lesssim HT^{3/4}\sqrt{\Cpush^5\abs{\cA}^2\log(\abs{\cMlat}\abs{\Phi}HT/\delta)\log(T\Cpush^2\abs{\cMlat}\abs{\Phi}^2H/\delta)\log\log(\abs{\cS})},
     \end{align*}
     which gives a risk bound of 
         \begin{align*}
          \Risk 
          &\lesssim \frac{1}{T^{1/4}}H\sqrt{\Cpush^5\abs{\cA}^2\log(\abs{\cMlat}\abs{\Phi}HT/\delta)\log(T\Cpush^2\abs{\cMlat}\abs{\Phi}^2H/\delta)\log\log(\abs{\cS})}.
     \end{align*}
     Equating this to $\veps$ gives a sample complexity of 
     \[
     T = \poly(\Cpush,A,H,\log\abs{\cMlat},\log\abs{\Phi}, \vepsinv,\logdelinv,\log\log(\abs{\cS})),
     \]
     as desired. 
    \paedit{We have not made much effort to optimize the rate; in particular, a faster rate is possible by using the \textsc{Golf.Dbr} algorithm of \citet{amortila2024mitigating}, which improves over the \textsc{Golf} algorithm when there is misspecification.}
\end{proof}

\newpage
\section{Proofs and Additional Information for \cref{sec:hindsight-rl}: Hindsight RL}\label{app:hindsight}
This appendix contains additional information and proofs related to algorithmic modularity under hindsight observations (\cref{sec:hindsight-rl}), and is organized as follows:

\begin{itemize}
\item \cref{app:implementing-hsightlearn} contains the pseudocode and proofs related to the online representation learning oracle \ExpWeights (\cref{lem:det-online-classifier}).
\item \cref{app:o2l-hindsight-proofs} contains the proof for our risk bound of the \olr algorithm under hindsight observability (\cref{thm:hindsightreduction}).
\end{itemize}

\subsection{Pseudocode and Proofs for \ExpWeights (\cref{lem:det-online-classifier})}\label{app:implementing-hsightlearn}

\begin{algorithm}[h]
    \begin{algorithmic}
		\State \textbf{input}: Decoder set $\Phi$
  \For{$t=1, 2, \cdots, T$} 
  	\State Get dataset $\crl*{x\ind{i}_h, \phistar(x\ind{i}_h)}_{i \in [t-1], h \in [H]}$
  	\For{$h=1,\ldots,H$}
  \State For $\phi \in \Phi$, compute 
  	\[
	q\ind{t}_h(\phi_h) \propto \exp\prn*{-\sum_{i=1}^{t-1}\indic\brk*{\phi_h(x\ind{i}_h) \neq \phistarh(x\ind{i}_h)} },
	\]
	\State and set 
	\begin{equation}\label{eq:derandom-expweights}
	\bar{\phi}\ind{t}_h(x) = \argmax_{s \in \cS} \bbP_{\phi_h \sim q\ind{t}_h}(\phi_h(x)=s).
	\end{equation}  
	\EndFor
	
    \State Return $\barphi\ind{t} = \crl{\barphi\ind{t}_h}_{h=1}^H$.
\EndFor
\end{algorithmic}
\caption{Derandomized Exponential Weights (\ExpWeights)}
\label{alg:derandomized-expweights}
\end{algorithm}

The main result for this estimator is the following.

\detonlineclassifier*

\begin{proof}[\pfref{lem:det-online-classifier}]
For each $h \in [H]$, consider the realizable online classification problem where $x\ind{t}_h \sim d^{\pi\ind{t}}_h$, for $\pi\ind{t}$ chosen adversarially, and $y\ind{t}_h = \phistarh(x\ind{t}_h)$. Consider the exponential weights estimator 
\[
	q\ind{t}_h(\phi) \propto \exp\prn*{-\sum_{i=1}^{t-1}\indic\brk*{\phi(x\ind{i}_h) \neq \phistarh(x\ind{i}_h)} }.
\]
For every sequence $(x\ind{t}_h)_{t=1}^T$, these %
distributions %
	satisfy the deterministic regret bound
\[
	\sum_{t=1}^T \En_{\hatphi\ind{t}_h \sim q\ind{t}_h} \brk*{\indic\brk*{\hatphi\ind{t}_h(x\ind{t}_h) \neq \phistarh(x\ind{t}_h)}} \leq  2\log|\Phi|,
      \]
by Corollary 2.3 of \citet{cesa2006prediction}. Taking conditional expectations over $x\ind{t}_h \sim d^{\pi\ind{t}}_h$ and using \cref{lem:multiplicative_freedman}
	gives that with probability at least $1-\delta$:
\[
	\sum_{t=1}^T \En_{\hatphi\ind{t}_h \sim q\ind{t}_h} \En^{\pi\ind{t}}\brk*{\indic\brk*{\hatphi\ind{t}_h(x_h) \neq \phistarh(x_h)}} \leq  4\log|\Phi| + 8\log(2\deltainv).
\]
Taking a union bound over $h \in [H]$ and summing over $h \in [H]$ we obtain that with probability at least $1-\delta$:
\[
	\sum_{t=1}^T  \sum_{h=1}^H \En_{\hatphi\ind{t}_h \sim q\ind{t}_h} \En^{\pi\ind{t}}\brk*{\indic\brk*{\hatphi\ind{t}_h(x_h) \neq \phistarh(x_h)}} \leq  4H\log|\Phi| + 8H\log(2H\deltainv).
\]

Now, recall that at each time $t$, we define the improper decoder $\barphi\ind{t}_h$ via:
\begin{equation}\label{eq:mode-feature}
	\bar{\phi}\ind{t}_h(x) = \argmax_{s \in \cS} \bbP_{\phi\ind{t}_h \sim q\ind{t}_h}(\phi\ind{t}_h(x)=s)
\end{equation}
Let $\ell_h(x_h,q\ind{t}_h) = \bbP_{\phi\ind{t}_h \sim q\ind{t}_h}(\phi\ind{t}_h(x_h) \neq \phistarh(x_h))$. Note that $\ell$ satisfies 
\begin{align}
	\sum_{t=1}^T \sum_{h=1}^H \En_{\phi\ind{t}_h \sim q\ind{t}_h}\En^{\pi\ind{t}} \brk*{\indic\brk*{\phi\ind{t}_h(x_h) \neq \phistarh(x_h)}} &= \sum_{t=1}^T \sum_{h=1}^H \En^{\pi\ind{t}} \En_{\phi\ind{t}_h \sim q\ind{t}_h}\brk*{\indic\brk*{\phi\ind{t}_h(x_h) \neq \phistarh(x_h)}}\\
		&= \sum_{t=1}^T \sum_{h=1}^H \En^{\pi\ind{t}}\brk{\ell_h(x_h,q\ind{t}_h)} \label{eq:bump}.
\end{align}
 By abuse of notation we also denote $\ell_h(x_h, \barphi_h) = \indic\brk*{\barphi_h(x) \neq \phistar(x)}$. We will show that
\begin{equation}\label{eq:boop}
 	\forall x, t, h: \ell_h(x_h,\barphi\ind{t}_h) \leq 2\ell_h(x_h,q\ind{t}_h),
\end{equation}
from which we will obtain that with probability at least $1-\delta$:
\[
	\Regrep(T) = \sum_{t=1}^T  \sum_{h=1}^H \En^{\pi\ind{t}}\brk*{\indic\brk*{\barphi\ind{t}_h(x_h) \neq \phistarh(x_h)}} \leq  8H\log|\Phi| + 16H\log(2H\deltainv).
\]
Integrating the high-probability regret bound gives
\[
	\En\brk*{\Regrep(T)} = \cO\prn*{H\log\prn{H|\Phi|}},
\]
as desired.
Towards establishing \eqref{eq:boop}, let us fix $x$ and let $s_{\max}$ denote the argmax in \eqref{eq:mode-feature}. There are two cases: 
\begin{itemize}
	\item $\bbP_{\phi\ind{t}_h \sim q\ind{t}_h}(\phi\ind{t}_h(x) = s_{\max}) \geq \frac{1}{2}$: 
	\begin{itemize}
		\item[$\rightarrow$] If $s_{\max} = \phi^\star(x)$, $\ell(x,\barphi\ind{t}_h) = 0$ so we are done. 
		\item[$\rightarrow$] Otherwise, $s_{\max} \neq \phistar(x)$ and we have $\ell(x,\barphi\ind{t}_h) = 1$. However, since $\phistar(x) \neq s_{\max}$ we have $\phi\ind{t}_h(x) = s_{\max} \implies \phi\ind{t}_h(x) \neq \phistarh(x)$ and so
		\[
			\bbP_{\phi\ind{t}_h \sim q\ind{t}_h}(\phi\ind{t}_h(x) \neq \phistarh(x)) \geq \bbP_{\phi\ind{t}_h \sim q\ind{t}_h}(\phi\ind{t}_h(x) = s_{\max}) \geq \frac{1}{2} = \frac{1}{2}\ell(x,\barphi\ind{t}_h).
		\]
	\end{itemize}
	\item $\bbP_{\phi\ind{t}_h \sim q\ind{t}_h}(\phi\ind{t}_h(x) = s_{\max}) < \frac{1}{2}$:
	\begin{itemize}
		\item[$\rightarrow$] If $s_{\max} = \phistarh(x)$, $\ell(x,\barphi\ind{t}_h) = 0$ so we are done. 
		\item[$\rightarrow$] Otherwise, $s_{\max} \neq \phistar(x)$ and we have $\ell(x,\barphi\ind{t}_h) = 1$. However, by definition of $s_{\max}$ as the mode we also have
		\[
			\bbP_{\phi\ind{t}_h \sim q\ind{t}_h}(\phi\ind{t}_h(x) = \phistarh(x)) \leq \bbP_{\phi\ind{t}_h \sim q\ind{t}_h}(\phi\ind{t}_h(x) = s_{\max}) < \frac{1}{2},
		\]
		so in particular we have 
		\[
			\ell(x,q\ind{t}_h) = \bbP_{\phi\ind{t}_h \sim q\ind{t}_h}(\phi\ind{t}_h(x) \neq \phistarh(x)) > \frac{1}{2} = \frac{1}{2}\ell(x,\barphi\ind{t}_h).
		\]
	\end{itemize}
\end{itemize}

\end{proof}

\subsection{Proofs for \olr Under Hindsight Observability (\cref{thm:hindsightreduction})}\label{app:o2l-hindsight-proofs}

\hindsightreduction*

\begin{proof}[\pfref{thm:hindsightreduction}]
Let $\prn{\hatphi\ind{t}}_{t \in [T]}$ denote the decoders chosen by $\Hsightlearn$, and let $\rho\ind{t}$ denote the distribution over decoders induced at time $t$ from the interaction of $\Hsightlearn, \Alglat, $ and $\Mstarobs$. Let $\pi\ind{t,k}_\obs \coloneqq \pi\ind{t,k}_\lat \circ \hatphi\ind{t}$ and $p\ind{t,k}_\obs$ denote the distribution over (observation-space) policies played at epoch $t$ and episode $k$, induced by the interaction of $\Hsightlearn, \Alglat$, and $\Mstarobs$. We adopt the notation $\pi\ind{t,K+1}_\lat \coloneqq \hatpi\ind{t}_\lat \sim p\ind{t,K+1}_\lat$ for the final policy output by \Alglat in epoch $t$ and $(x\ind{t,K+1}_h,a\ind{t,K+1}_h,r\ind{t,K+1}_h)$ for the trajectory collected from that (observation-level) policy $\hatpi\ind{t}_\lat \circ \hatphi\ind{t}$.  
We firstly note that by assumption, we have the guarantee
\begin{equation}\label{eq:teebs}
	\En\brk*{\sum_{t=1}^T \sum_{k=1}^{K+1} \sum_{h=1}^H \En_{\pi\ind{t,k}_\obs \sim p\ind{t,k}_\obs} \En^{\pi\ind{t,k}_\obs}\brk*{\indic\brk*{\hatphi\ind{t}_h(x_h) \neq \phistarh(x_h)}}} \leq (K+1) \Estrep(T) \leq 2K \Estrep(T).
\end{equation}
which follows by applying \cref{ass:online-classifier} to the distributions $\bar{p}\ind{t}_\obs = \frac{1}{(K+1)}\sum_{k=1}^{K+1} p\ind{t,k}_\obs$ and noting that 
\begin{align*}
&\sum_{t=1}^T \sum_{h=1}^H \frac{1}{K+1} \sum_{k=1}^{K+1} \En_{\pi\ind{t,k}_\obs \sim p\ind{t,k}_\obs}\En^{\pi\ind{t,k}_\obs}\brk*{\indic\brk*{\hatphi\ind{t}_h(x_h) \neq \phistarh(x_h)}} \\
	&\qquad = \sum_{t=1}^T \sum_{h=1}^H \En_{\bar{\pi}\ind{t}_\obs \sim \bar{p}\ind{t}_\obs}\En^{\bar{\pi}\ind{t}_\obs}\brk*{\indic\brk*{\hatphi\ind{t}_h(x_h) \neq \phistarh(x_h)}} \leq \Estrep(T).
\end{align*}
Let $\Risk(K,\Alglat,\phi,\Mstarobs) = J^{\Mstarobs}(\pistar_{\Mstarobs}) - J^{\Mstarobs}(\hatpilat \circ \phi)$ be the random variable denoting the risk of the final policy output by $\Alglat$ after $K$ rounds of interaction with $\Mstarobs$ when given feature $\phi$ in any epoch $t$. 
For any $\phi: \cX \rightarrow \cS$, let $\En_{\phi}$ denote the law over trajectories $(x\ind{k}_h,a\ind{k}_h,r\ind{k}_h)_{k \in [K+1],h \in [H]}$ and policies $(\pi\ind{k}_\lat \circ \phi)_{k \in [K+1]}$ generated after $K$ rounds of interaction when $\Alglat$ is given feature $\phi$ in any epoch. (Recall that, for all of the above definitions, a new instance of $\Alglat$ is initialized at every epoch, so we do not have to specify \emph{which epoch it is}, only the current feature $\phi$). %
	Finally, let $G_t$ be the ``good'' event
\[
	G_t = \crl*{ \forall k \in [K+1], \forall h \in [H]: \,\, \hatphi\ind{t}_h(x\ind{t,k}_h) = \phistarh(x\ind{t,k}_h)}.
      \]
   Recall that, in any round $t$, $\Alglat$ only observes the latent (``compressed'') trajectories $(\hatphi_h\ind{t}(x\ind{t,k}_h),a\ind{t,k}_h,r\ind{t,k}_h)$ as history for choosing policies. We can therefore conclude that, when $\hatphi\ind{t}(x\ind{t,k}_h) = \phistar(x\ind{t,k}_h)$ for all $k \in [K+1], h \in [H]$, the distribution over final policies $\hatpi\ind{t}_\lat$ chosen by $\Alglat$ will be identical as if we had chosen $\phistar$ as our decoder. In particular, this implies
\begin{align}
	\En_{\phihat\ind{t}}\brk*{\indic\crl*{G_t}\Risk(K,\Alglat,\hatphi\ind{t},\Mstarobs)} &= 	\En_{\phistar}\brk*{\indic\crl*{G_t}\Risk(K,\Alglat,\phistar,\Mstarobs)} \nonumber \\
		&\leq \Riskstar(K) \label{eq:caribou},
\end{align}
where the second line simply follows by removing the indicator function, recalling that $\Riskstar(K) = \En\brk*{\Risk(K,\Alglat,\Mstarlat)}$, and using that $ \Risk(K,\Alglat,\phistar,\Mstarobs) = \Risk(K,\Alglat,\Mstarlat)$. %

 Then, we have:
  \begin{align*}
\En\brk*{\Risk_\obs(TK)} &= \frac{1}{T}\sum_{t=1}^T \En_{\hatphi\ind{t} \sim \rho\ind{t}}\brk*{\En_{\hatphi\ind{t}}\brk*{\Risk(K,\Alglat,\hatphi\ind{t},\Mstarobs)}}\\
	&\leq \frac{1}{T}\sum_{t=1}^T \En_{\hatphi\ind{t} \sim \rho\ind{t}}\brk*{\En_{\hatphi\ind{t}}\brk*{\indic\crl*{G_t}\Risk(K,\Alglat,\hatphi\ind{t},\Mstarobs)}} \\
		&\qquad + \frac{1}{T}\sum_{t=1}^T  \En_{\hatphi\ind{t} \sim \rho\ind{t}}\brk*{\En_{\hatphi\ind{t}}\brk*{\indic\crl*{\neg G_t}}} \\
	&\leq \frac{1}{T}\sum_{t=1}^T \Riskstar(K) + \frac{1}{T}\sum_{t=1}^T \bbP(\neg G_t) \\
	&= \Riskstar(K) + \frac{1}{T}\sum_{t=1}^T \bbP(\neg G_t),
\end{align*}
where the first equality applies the tower rule for conditional expectation, the second equality applies linearity of conditional expectations and the upper bound $\Risk(K,\Alglat,\hatphi\ind{t},\Mstarobs) \leq 1$, and the third lines applies the upper bound \eqref{eq:caribou}.%
 It remains to bound the last term. Here, note that 
by a union bound,
\[
\bbP\prn{\neg G_t} \leq \En\brk*{\sum_{k=1}^{K+1} \sum_{h=1}^H  \En_{\pi\ind{t,k} \sim p\ind{t,k}}\En^{\pi\ind{t,k}}\indic\crl*{\hatphi\ind{t}(x\ind{t,k}_h) \neq \phistar(x\ind{t,k}_h)}},
\]
where we have used that trajectory $k$ in round $t$ is sampled from policy $\pi\ind{t,k}$, which is in turn sampled from $p\ind{t,k}$. Summing over $t$ and using the bound in \eqref{eq:teebs} concludes the proof.

\end{proof}

\newpage
\section{Proofs for \cref{sec:online-rl}: Self-Predictive Estimation}\label{app:online}

This appendix contains additional information and proofs related to algorithmic modularity under self-predictive estimation (\cref{sec:online-rl}), and is organized as follows:

\begin{itemize}
\item \cref{sec:implementing-replearn} contains the pseudocode and proofs related to the online representation learning oracle \SelfPred (\cref{lem:implementing-optim-replearn}).
\item \cref{app:olr-online-main-risk} contains the proof for our risk bound of the \olr algorithm under self-predictive estimation (\cref{thm:online-reduction-main}).
\end{itemize}

\subsection{Pseudocode and Proofs for \SelfPred (\cref{lem:implementing-optim-replearn})}\label{sec:implementing-replearn}

The pseudocode for our self-predictive estimation procedure is given in \cref{alg:debiased-mle-optimistic}.
\begin{algorithm}[h]
    \begin{algorithmic}[1]
	\State \textbf{input}: Decoder set $\Phi$, Latent model class $\cMlat$, Mismatch-complete class $\cLlat$, Optimism parameter $\gamma$
	\State Set $\beta \coloneqq \frac{1}{2}\sqrt{\nicefrac{\Ccov H \log(T)}{T}}$
 \For{$t=1, 2, \cdots, T$} 
	\State Get dataset $\cD\ind{t} = \crl{x\ind{i}_h,a\ind{i}_h,r\ind{i}_h,x\ind{i}_{h+1}}_{i\in[t-1],h\in[H]}$
  \State Compute 
\begin{align}\label{eq:debiased-mle-optimistic}
\wh{M}\ind{t} \circ \hatphi\ind{t} = \argmax_{\brk{M \circ \phi} \in \cMlat \circ \Phi}  \Bigg\{ &(\gamma\beta)^{-1} J_M(\pi_M) + \sum_{h=1}^H \sum_{i=1}^n\log\prn*{\brk*{M_h \circ \phi_h}(r\ind{i}_{h},\phi_{h+1}(x\ind{i}_{h+1})\mid x\ind{i}_h,a\ind{i}_h)} \\
&\quad-\max_{\brk{M' \circ \phi'} \in \cL_{\lat} \circ \Phi} \sum_{i=1}^n \log\prn*{\brk*{M'_h \circ \phi'_h}(r\ind{i}_{h},\phi_{h+1}(x\ind{i}_{h+1})\mid x\ind{i}_h,a\ind{i}_h)} \Bigg\}. \label{eq:debiasing}
\end{align}
\State Return $\hatphi\ind{t} = \crl*{\hatphi\ind{t}_h}_{h \in [H]}$.
\EndFor
\end{algorithmic}
\caption{Optimistic Self-Predictive Latent Model Estimation (\SelfPred)}
\label{alg:debiased-mle-optimistic}
\end{algorithm}

Our main result concerning the \SelfPred estimator for online optimistic self-predictive estimation is the following. We recall our notation for the instantaneous self-prediction error 
\[
\brk*{\Delta_h(\Mlat, \phi)}(x_h,a_h) \coloneqq \Dhels{M\sub{\lat,h}(\phi\sub{h}(x_h),a_h)}{\brk[\big]{\phi\sub{h+1}\sharp M^\star_{\obs,h}}(x_h,a_h)}.
\]
 
\optreplearn*

\begin{proof}[\pfref{lem:implementing-optim-replearn}]
We will firstly establish that the algorithm obtains low \textit{offline} estimation error.
\begin{lemma}[\SelfPred attains low offline estimation error]\label{lem:debiased-mle-offline}
	For any $\gamma > 0$, under decoder realizability ($\phistar \in \Phi$), model realizability ($\Mstarlat \in \cMlat$), and mismatch function completeness with class $\cLlat$ (\cref{ass:hphi-completeness}), the estimator in \cref{alg:debiased-mle-optimistic} with inputs $\Phi$, $\Mlat$, $\cLlat$, and $\gamma$ satisfies that for all $t \in [T]$, with probability at least $1-\delta$, 
\begin{align}
\sum_{h=0}^H \sum_{i=1}^{t-1} &\En_{\pi\ind{i} \sim p\ind{i}}\En^{\pi\ind{i}}\brk*{\brk{\Delta_h(\wh{M}\ind{t},\hatphi\ind{t})}(x_h,a_h)} + \gammainv\prn*{J^{\Mstarlat}(\pi_{\Mstarlat}) - J^{\wh{M}\ind{t}}(\pi_{\wh{M}\ind{t}})} \nonumber  \\
	&\qquad \leq \cO\prn*{\log(\abs{\cMlat}\abs{\cL_\lat}\abs{\Phi}HT\deltainv)}. \label{eq:pulsar}
\end{align}
\end{lemma}

Given this result, we can appeal to offline-to-online conversions to establish the final result. Let $\Ccov \coloneqq \Ccov(\Mobsstar,\Pilat \circ \Phi)$ denote the (state-action) coverability coefficient in $\Mobsstar$ over the set of policies $\Pilat \circ \Phi$. Note that by \cref{lem:covinvariance} we have $\Ccovs(\Mobsstar, \Pilat \circ \Phi) = \Ccovs$ and therefore by \cref{lem:state-action-cov-state-cov} we have $\Ccov(\Mobsstar,\Pilat \circ \Phi) \leq \Ccovs \abs{\cA}$. Let $\eta > 0$ be a parameter to be chosen later, and $\beta_\texttt{off} = \cO\prn*{\log(\abs{\cMlat}\abs{\cL_\lat}\abs{\Phi}HT\deltainv)}$ be the offline estimation error guaranteed by \cref{lem:debiased-mle-offline}. We abbreviate $\alpha \coloneqq \sqrt{\Ccov H \log(T)}$, $\En^{p\ind{t}}\brk{\cdot} \coloneqq \En_{\pi\ind{t} \sim p\ind{t}}\En^{\pi\ind{t}}\brk{\cdot}$, and $\En^{\wt{p}\ind{t}} \coloneqq \sum_{i=1}^{t-1} \En_{\pi\ind{i} \sim p\ind{i}}\En^{\pi\ind{i}}\brk{\cdot}$. Then, we have:
\begin{align*}
&\sum_{t=1}^T \sum_{h=1}^H \En^{p\ind{t}}\brk*{\brk{\Delta_h(\wh{M}\ind{t},\hatphi\ind{t})}(x_h,a_h)} + \gammainv\prn*{J^{\Mstarlat}(\pi_{\Mstarlat}) - J^{\wh{M}\ind{t}}(\pi_{\wh{M}\ind{t}})}\\
	&\leq \alpha\sqrt{\sum_{t=1}^T \sum_{h=1}^H \En^{\wt{p}\ind{t}}\brk*{\brk{\Delta_h(\wh{M}\ind{t},\hatphi\ind{t})}(x_h,a_h)}} + \cO(H\Ccov) + \gammainv\sum_{t=1}^T \prn*{J^{\Mstarlat}(\pi_{\Mstarlat}) - J^{\wh{M}\ind{t}}(\pi_{\wh{M}\ind{t}})} \\
	&\leq \alpha\prn*{ \frac{\eta}{2} \sum_{t=1}^T \sum_{h=1}^H \En^{\wt{p}\ind{t}}\brk*{\brk{\Delta_h(\wh{M}\ind{t},\hatphi\ind{t})}(x_h,a_h)} + \frac{1}{2\eta}} + \gammainv\sum_{t=1}^T \prn*{J^{\Mstarlat}(\pi_{\Mstarlat}) - J^{\wh{M}\ind{t}}(\pi_{\wh{M}\ind{t}})} + \cO(H\Ccov)
\end{align*}
where in the first inequality we have used \cref{lem:offline-to-online} with $g\ind{t}_h = \Delta_h(\wh{M}\ind{t},\hatphi\ind{t})$ and in the second inequality we have used the AM-GM inequality with parameter $\eta$. Collecting terms, we proceed via:
\begin{align*}
	&= \frac{\alpha\eta}{2} \sum_{t=1}^T \prn*{\sum_{h=1}^H \En^{\wt{p}\ind{t}}\brk*{\Delta_h(\wh{M}\ind{t},\hatphi\ind{t})(x_h,a_h)} + (\frac{\gamma \eta \alpha}{2})^{-1} \prn*{J^{\Mstarlat}(\pi_{\Mstarlat}) - J^{\wh{M}\ind{t}}(\pi_{\wh{M}\ind{t}})}} + \frac{\alpha}{2\eta} + \cO(H\Ccov)\\
	&\leq \frac{\alpha\eta}{2} T \beta_{\texttt{off}}  + \frac{\alpha}{2\eta} + \cO(H\Ccov)\\
	&\leq \cO\prn*{\sqrt{\Ccovs\abs{\cA} H \log(T) T} \beta_{\texttt{off}} + H\Ccovs\abs{\cA}}\\
	&\leq \cO\prn*{\sqrt{H\Ccovs \abs{\cA} T\log(T)}\log(\abs{\cMlat}\abs{\cL_\lat}\abs{\Phi}HT\deltainv)},
\end{align*}

where in the first inequality we have used \cref{lem:debiased-mle-offline} and the definition of $\gamma$ in \cref{alg:debiased-mle-optimistic} (cf. \eqref{eq:debiased-mle-optimistic}) and in the second inequality we have chosen $\eta = \nicefrac{1}{\sqrt{T}}$ to balance the terms and used the bound $\Ccov \leq \Ccovs\abs{\cA}$.
We convert to an expected regret bound by picking $\delta$ appropriately, which gives the final result. It remains to show \cref{lem:debiased-mle-offline}.
\end{proof}

\begin{proof}[\pfref{lem:debiased-mle-offline}]
Fix an iteration $t \in [T]$, and abbreviate $\wh{M} \coloneqq \wh{M}\ind{t}$ and $\hatphi \coloneqq \hatphi\ind{t}$. We follow the analysis of maximum likelihood estimation from \citet{Sara00,zhang2006from,agarwal2020flambe}. In particular, we quote Lemma 24 of \cite{agarwal2020flambe}, which in an abstract conditional estimation framework with density class $\cF$ states the following. 

\begin{lemma}[Lemma 24 of \citet{agarwal2020flambe}]
Let $D=\{(x_i,y_i)\}$ be a dataset collected with $x_i \sim p\ind{i}(x_{1:i-1},y_{1:i-1})$ and $y_i \sim f^\star(\cdot \mid x_i)$, $L(f,D) = \sum_{i=1}^n \ell(f,(x_i,y_i))$ be any loss function that decomposes additively, $\wh{f} : D \rightarrow \cF$ be an estimator, $D'$ be a tangent sequence $D' = \{(\wt{x}_i,\wt{y}_i)\}$ sampled independently via $\wt{x}_i \sim p\ind{i}(x_{1:i-1},y_{1:i-1})$ and $\wt{y}_i \sim f^\star(\cdot \mid \wt{x}_i)$. Then, with probability at least $1-\delta$, we have
\begin{equation}\label{eq:decoupling}
	-\log\En_{D'} \exp\prn*{L(\wh{f}(D),D')} \leq -L(\wh{f}(D),D) + \log(\abs{\cF}\deltainv),
\end{equation} %
\end{lemma}

For our purposes, we have that $\cF = \cMlat \circ \Phi$, the data distribution is collected adaptively (for each $h \in [H]$) via $\pi\ind{i} \sim  p\ind{i}$, $x\ind{i}_h, a\ind{i}_h \sim d^{\Mstarobs,\pi\ind{i}}_h$, and $r\ind{i}_h, x\ind{i}_{h+1} \sim \Mstarobs(\cdot \mid x\ind{i}_h,a\ind{i}_h)$. For the loss function $L$, we take
\begin{align*}
L((M,\phi),D) &= -\sum_{h=0}^H\sum_{i=1}^t\log(\frac{\Mstarobs(r\ind{i}_{h+1},\phi_{h+1}(x\ind{i}_{h+1})\mid x\ind{i}_h,a\ind{i}_h)}{\brk{M_h\circ\phi_h}(r\ind{i}_h,\phi_{h+1}(x\ind{i}_{h+1}) \mid x\ind{i}_h,a\ind{i}_h)}) -\frac{\gammainv}{2}(J^{\Mstarlat}(\pi_{\Mstarlat}) - J^M(\pi_M)).
\end{align*}
We begin by upper bounding the quantity $-L((\wh{M}, \hatphi)(D),D)$ appearing on the \rhs of \eqref{eq:decoupling}, or equivalently lower bounding $L((\wh{M}, \hatphi)(D),D)$. Let us abbreviate $\wh{V} = J^{\wh{M}}(\pi_{\wh{M}})$ and  $V^\star = J^{\Mstarlat}(\pi_{\Mstarlat})$. Towards this, note that %
\begin{align*}
L((\wh{M}, \hatphi)(D),D) &=\sum_{h=0}^H\sum_{i=1}^t \log\prn*{\brk*{\wh{M}_h\circ\hatphi_h}(r\ind{i}_h,\hatphi_{h+1}(x\ind{i}_{h+1}) \mid x\ind{i}_h,a\ind{i}_h)} \\
	 &\quad- \sum_{h=0}^H\sum_{i=1}^t\log\prn*{\Mstarobs(r\ind{i}_h,\wh{\phi}_{h+1}(x\ind{i}_{h+1})\mid x\ind{i}_h,a\ind{i}_h)} + \frac{\gammainv}{2}(\wh{V} - V^\star) \\
	   &\geq 	 \sum_{h=0}^H\sum_{i=1}^t \log\prn*{\brk*{\wh{M}_h\circ\hatphi_h}(r\ind{i}_h,\hatphi_{h+1}(x\ind{i}_{h+1}) \mid x\ind{i}_h,a\ind{i}_h)} \\
	   &\quad- \sum_{h=0}^H\max_{\brk{M'\circ\phi'} \in \cL_\lat \circ \Phi} \sum_{i=1}^t\log\prn*{\brk*{M'_h \circ \phi'_h}(r\ind{i}_h,\wh{\phi}_{h+1}(x\ind{i}_{h+1})\mid x\ind{i}_h,a\ind{i}_h)} + \frac{\gammainv}{2}(\wh{V} - V^\star) \\
	  & \geq \sum_{h=0}^H \sum_{i=1}^t \log\prn*{\brk*{\Mstarlath\circ\phistarh}(r\ind{i}_h,\phi^\star_{h+1}(x\ind{i}_{h+1}) \mid x\ind{i}_h,a\ind{i}_h)} \\
	  &\quad - \sum_{h=0}^H \max_{\brk{M'\circ\phi'} \in \cL_\lat \circ \Phi} \sum_{i=1}^t\log\prn*{\brk*{M'_h \circ \phi'_h}(r\ind{i}_h,\phi^\star_{h+1}(x\ind{i}_{h+1})\mid x\ind{i}_h,a\ind{i}_h)} + \frac{\gammainv}{2}(V^\star - V^\star) \\
	  	  &= \sum_{h=0}^H \sum_{i=1}^t \log\prn*{\brk*{\Mstarlath\circ\phistarh}(r\ind{i}_h,\phi^\star_{h+1}(x\ind{i}_{h+1}) \mid x\ind{i}_h,a\ind{i}_h)} \\
	  &\quad - \sum_{h=0}^H \max_{\brk{M'\circ\phi'} \in \cL_\lat \circ \Phi} \sum_{i=1}^t\log\prn*{\brk*{M'_h \circ \phi'_h}(r\ind{i}_h,\phi^\star_{h+1}(x\ind{i}_{h+1})\mid x\ind{i}_h,a\ind{i}_h)},  
\end{align*}
where in the second line we have used \cref{lem:phi-compressed-model-realizable} with \cref{ass:hphi-completeness} and in the third line we have used the ERM property of $\wh{M} \circ \hatphi$ together with decoder and model realizability. We claim that this implies 
\begin{equation}\label{eq:mle-concentration-lb}
L((\wh{M}, \hatphi)(D),D) \geq -\log(\abs{\cL_\lat \circ\Phi}H\deltainv)
\end{equation}
by concentration. Indeed, for each $h \in [H],i \in [t],$ and $\brk{M' \circ \phi'} \in \cLlat \circ \Phi$, let

\[
 Z^{\brk{M'\circ\phi'}}_{i,h} = -\frac{1}{2}\log(\frac{\Mstarobs(r\ind{i}_h,\phistar_{h+1}(x\ind{i}_{h+1}) \mid x\ind{i}_h,a\ind{i}_h)}{\brk{M' \circ \phi'}(r\ind{i}_h,\phistar_{h+1}(x\ind{i}_{h+1}) \mid x\ind{i}_h,a\ind{i}_h)})
\]

Applying \cref{lem:log-exp-concentration}, we have that
\begin{align}
	&\sum_{i=1}^t\log(\frac{\Mstarobs(r\ind{i}_h,\phistar_{h+1}(x\ind{i}_{h+1}) \mid x\ind{i}_h,a\ind{i}_h)}{\brk{M' \circ \phi'}(r\ind{i}_h,\phistar_{h+1}(x\ind{i}_{h+1}) \mid x\ind{i}_h,a\ind{i}_h)})\nonumber \\
	&\geq \sum_{i=1}^t -2\log( \En_{\pi\ind{i}\sim p\ind{i}}\En^{\pi\ind{i}}\brk*{\exp\prn*{-\frac{1}{2}\log\prn*{\frac{\Mstarobs(r\ind{i}_h,\phistar_{h+1}(x\ind{i}_{h+1}) \mid x\ind{i}_h,a\ind{i}_h)}{\brk{M' \circ \phi'}(r\ind{i}_h,\phistar_{h+1}(x\ind{i}_{h+1}) \mid x\ind{i}_h,a\ind{i}_h)}}}})  - \logdelinv \label{eq:tantdamourperdu},
\end{align}
with probability at least $1-\delta$, where %
we have recalled that data is gathered adaptively according to $\pi\ind{i} \sim p\ind{i}$.%
We now quote the following lemma from \citet{zhang2006from,agarwal2020flambe}.

\begin{lemma}[Lemma 25 of \citet{agarwal2020flambe}]\label{lem:log-exp-stuff-unperturbed}
	For any $\cD \in \Delta(\cX)$ and $p,q \in [\cX \rightarrow \Delta(\cY)]$, we have
	\[
	-2\log \En_{x \sim \cD, y \sim q(\cdot \mid x)} \exp(-\frac{1}{2} \log(\nicefrac{q(y \mid x)}{p(y \mid x)})) \geq \En_{x \sim \cD}\brk*{\Dhels{q(\cdot \mid x)}{p(\cdot \mid x)}}
	\]
\end{lemma}

\begin{proof}[\pfref{lem:log-exp-stuff-unperturbed}]
We include the proof for completeness. The result follows via the following steps.
	\begin{align*}
	-2\log \En_{x \sim \cD, y \sim q(\cdot \mid x)} \exp(-\frac{1}{2} \log(\nicefrac{q(y \mid x)}{p(y \mid x)}))	&= -2\log \En_{x \sim \cD, y \sim q(\cdot \mid x)} \sqrt{\nicefrac{p(y \mid x)}{q(y \mid x)}} \\
		&\geq 2\prn*{1 - \En_{x \sim \cD, y \sim q(\cdot \mid x)} \sqrt{\nicefrac{p(y \mid x)}{q(y \mid x)}}} \tag{$\forall x:\log(x) \leq x-1$} \\
		&= \En_{x \sim \cD}\brk*{2\prn*{1 - \En_{y \sim q(\cdot \mid x)} \sqrt{\nicefrac{p(y \mid x)}{q(y \mid x)}}}} \\
		&= \En_{x \sim \cD}\brk*{\Dhels{p(\cdot \mid x)}{q(\cdot \mid x)}}
	\end{align*}
\end{proof}

By \cref{lem:log-exp-stuff-unperturbed}, we have that the right-hand-side of \eqref{eq:tantdamourperdu} is further lower bounded by
\begin{align*}
	&\sum_{i=1}^t\log(\frac{\Mstarobs(r\ind{i}_h,\phistar_{h+1}(x\ind{i}_{h+1}) \mid x\ind{i}_h,a\ind{i}_h)}{\brk{M' \circ \phi'}(r\ind{i}_h,\phistar_{h+1}(x\ind{i}_{h+1}) \mid x\ind{i}_h,a\ind{i}_h)}) \\
	 &\quad\geq \sum_{i=1}^t \En_{\pi\ind{i}\sim p\ind{i}}\En^{\pi\ind{i}}\brk*{\Dhels{\phistar_{h+1}\sharp\Mstarobs(\cdot \mid x\ind{i}_h,a\ind{i}_h)}{\brk{M' \circ \phi'}(\cdot \mid x\ind{i}_h,a\ind{i}_h)}} - \logdelinv \\
	&\quad\geq - \logdelinv,
\end{align*}

where the last line follows from the non-negativity of squared Hellinger. Taking a union bound over $M' \circ \phi' \in \cLlat \circ \Phi$ and $h \in [H]$ gives the desired lower bound in \eqref{eq:mle-concentration-lb}. 

To conclude the proof, it remains to lower bound the \lhs in \eqref{eq:decoupling}. Here, note that: 
\begin{align}
&-\log\En_{D'} \exp\prn*{L((\wh{M},\hat\phi)(D),D')} + \frac{\gammainv}{2}(V^\star - \wh{V}) \nonumber \\ 
	&=-\log\En_{D'}\brk*{\exp\prn*{-\frac{1}{2}\sum_{h=1}^H\sum_{i=1}^t\log(\frac{\Mstarobs(\wt{r}\ind{i}_h,\wh{\phi}_{h+1}(\wt{x}\ind{i}_{h+1})\mid x\ind{i}_h,a\ind{i}_h)}{\brk*{\wh{M}_h \circ \hatphi_h}(r\ind{i}_h,\wh{\phi}_{h+1}(x\ind{i}_{h+1}) \mid x\ind{i}_h,a\ind{i}_h)})}} \nonumber\\
&= -\sum_{h=1}^H \sum_{i=1}^t \log\En_{\pi\ind{i}\sim p\ind{i}}\En^{\pi\ind{i}}\brk*{ \exp\prn*{-\frac{1}{2}\log(\frac{\Mstarobs(r\ind{i}_h,\wh{\phi}_{h+1}(x\ind{i}_{h+1})\mid x\ind{i}_h,a\ind{i}_h)}{\brk*{\wh{M}_h\circ\hatphi_h}(r\ind{i}_h,\wh{\phi}_{h+1}(x\ind{i}_{h+1}) \mid x\ind{i}_h,a\ind{i}_h)})}}, \label{eq:sum-exp-log-loss} 
\end{align}
where we have used that in the ``tangent sequence'' $D'$ the current sample $(\wt{r}\ind{i}_{h}, \wt{x}\ind{i}_{h+1})$ is independent of $(r\ind{i}_{h}, x\ind{i}_{h+1})$. To bound this term, we again appeal to \cref{lem:log-exp-stuff-unperturbed}, concluding that
\begin{align*}
&-\sum_{h=1}^H \sum_{i=1}^t \log\En_{\pi\ind{i} \sim p\ind{i}}\En^{\pi\ind{i}}\brk*{ \exp\prn*{-\frac{1}{2}\log(\frac{\Mstarobs(r\ind{i}_h,\wh{\phi}_{h+1}(x\ind{i}_{h+1})\mid x\ind{i}_h,a\ind{i}_h)}{\brk*{\wh{M}_h \circ \hatphi_h}(r\ind{i}_h,\wh{\phi}_{h+1}(x\ind{i}_{h+1}) \mid x\ind{i}_h,a\ind{i}_h)})}} \\
	&\quad \geq \frac{1}{2}\sum_{h=1}^H \sum_{i=1}^t  \En_{\pi\ind{i} \sim p\ind{i}}\En^{\pi\ind{i}}\brk*{\Dhels{\brk{\wh{M}_h \circ \hatphi_h}(x_h,a_h)}{\hatphi_{h+1}\sharp\Mstarobs(x_h,a_h)}} %
\end{align*}
Combining everything, we have:
\begin{align*}
&\frac{1}{2}\prn*{\sum_{h=1}^H \sum_{i=1}^t  \En_{\pi\ind{i} \sim p\ind{i}}\En^{\pi\ind{i}}\brk*{\Dhels{\brk*{\wh{M}_h \circ \hatphi_h}(x_h,a_h)}{\hatphi_{h+1}\sharp\Mstarobs(x_h,a_h)}}+ \gammainv(V^\star - \wh{V})} \\
&\qquad \leq \log(\abs{\cL_\lat}\abs{\Phi}H\deltainv) +\log(\abs{\cMlat}\abs{\Phi}\deltainv)
\end{align*}

Taking an additional union bound over $t \in [T]$, we have that with probability at least $1-\delta$:
\begin{align*}
\sum_{h=1}^H \sum_{i=1}^t  \En_{\pi\ind{i} \sim p\ind{i}_h}&\En^{\pi\ind{i}}\brk*{\Dhels{\brk*{\wh{M}_h \circ \hatphi_h}(x_h,a_h)}{\hatphi_{h+1}\sharp\Mstarobs(x_h,a_h)}} + \gammainv\prn*{J^{\Mstarlat}(\pi_{\Mstarlat}) - J^{M\ind{t}}(\pi_{M\ind{t}})}  \\
	&\leq \cO\prn*{\log(\abs{\cMlat}\abs{\cL_\lat}\abs{\Phi}HT\deltainv)},
\end{align*}

for all $t \in [T]$, as desired.

\end{proof}

\selfpredmodular*

\begin{proof}[\pfref{cor:selfpredmodular}]
	The first inequality simply follows by plugging the bound of $\Estsimopt$ from \cref{lem:implementing-optim-replearn} into \cref{thm:online-reduction-main}. For the second inequality, let $\Delta = c_2\sqrt{H\Ccovs\abs{\cA}}\log(\abs{\cMlat}\abs{\cL_\lat}\abs{\Phi})$. The result follows by setting $\gamma$ s.t. $c_3 \gammainv H K = \Riskbase(K)$ i.e. $\gamma = c_3 \frac{KH}{\Riskbase(K)}$, 
	and $T$ such that $\frac{\gamma K \Delta}{\sqrt{T}} = \Riskbase(K)$ i.e. $T = \frac{K^4\Delta^2 \gamma^2}{\Riskbase(K)^2} = \frac{K^4\Delta^2 H^2}{\prn*{\Riskbase(K)}^4}$.
		Then the result follows by direct substitution and by noting that $\frac{K}{T} \leq 1$ since $\Riskbase(K) \leq 1$.
	
\end{proof}

\subsection{Proofs for Main Risk Bound (\cref{thm:online-reduction-main})}\label{app:olr-online-main-risk}

Our main risk bound (\cref{thm:online-reduction-main}) follows as a special case of a more general theorem (\cref{thm:online-reduction-alpha}), which holds for algorithm that satisfies a property we refer to as \Crobust-ness (\cref{def:alglat-robust}). We now state the more general theorem, postponing its proof (and a formal definition of corruption robustness) until \cref{app:online-corruption}.

\begin{restatable}[Risk bound for \olr under self-predictive estimation and \Crobustness]{theorem}{onlinereductionalpha}\label{thm:online-reduction-alpha}
Assume \OptReplearn satisfies \cref{ass:optim-replearn} with parameter $\gamma > 0$ and that $\cMlat$ is realizable (i.e. $\Mstarlat \in \cMlat$). 
	Furthermore, let \Alglat be \Crobust (\cref{def:alglat-robust}) with parameter $\alpha$. Then, \olr (\cref{alg:obs-to-lat}) with inputs $T,K,\Phi,\Alglat$, and $\OptReplearn$ has expected risk
\begin{equation}
	\En\brk*{\Riskobs(TK)} \leq c_1 \cdot \Riskbase(K) + c_2 \gamma \cdot \frac{K}{T} \Estsimopt(T,\gamma) + c_3\gammainv \cdot \prn*{\alpha^2 + H}%
\end{equation}
for absolute constants $c_1,c_2,c_3 > 0$.

\end{restatable}

Our main risk bound (\cref{thm:online-reduction-main}) follows from the following lemma, which establishes that \emph{any} \alglat is \Crobust in the sense of \cref{def:alglat-robust} for a sufficiently large corruption robustness parameter. Below, for any POMDP $\wtM$ over state-action space $\cS \times \cA$, we write $\wtM(s_{1:h},a_{1:h})$ for the conditional probability over reward $r_h$ and $s_{h+1}$ given $s_{1:h},a_{1:h}$, i.e. $\wtM_h(s_{1:h},a_{1:h}) = \wtM_h(r_{h},s_{h+1} = \cdot \mid s_{1:h},a_{1:h})$. %

\begin{restatable}{lemma}{generalalgsimlemma}\label{lem:general-alg-sim-lemma}
Let $\Mstar$ be any reference MDP and $\wtM$ be any POMDP \dfedit{with the same state and action space}. Then for any algorithm \alglat, we have 
\begin{align*}
\En^{\wtM,\alglat}\brk*{\Risk_{\Mstar}(K)} &\leq 	 c_1 \En^{\Mstar,\alglat}\brk*{\Risk_{\Mstar}(K)} \\
	&\quad + c_2 \En^{\wtM,\alglat}\brk*{\sum_{k=1}^K \sum_{h=1}^H \En^{\wtM,\pi\ind{k}}\brk*{\Dhels{\Mstar_h(s_h,a_h)}{\wtM_h(s_{1:h},a_{1:h})}}},
\end{align*}
where $c_1,c_2>0$ are absolute constants. In particular, $\alglat$ is \Crobust (\cref{def:alglat-robust}) with $\alpha =c_2 \sqrt{K H}$.\loose %
\end{restatable}

\begin{proof}[\pfref{lem:general-alg-sim-lemma}]
	Let us abbreviate $\alg \coloneqq \alglat$. For $i \in [K]$, let $\tau\ind{i}$ denote the trajectory $ (s\ind{i}_1,a\ind{i}_1,r\ind{i}_1, \ldots, s\ind{i}_H,a\ind{i}_H,r\ind{i}_H)$. Let $\bbP \coloneqq \bbP^{\Mstar, \alg}$ denote the law of $\crl*{(\pi\ind{i}, \tau\ind{i})}_{i\in\brk{K}}$ under $\alg$ in the true MDP $\Mstar$, and $\bbQ \coloneqq \bbP^{\wtM, \alg}$ denote the law of $\crl*{(\pi\ind{i}, \tau\ind{i})}_{i\in\brk{K}}$ under $\alg$ under the POMDP $\wtM$. Let us write $M^\star(\pi)$ and $\wtM(\pi)$ for the laws of trajectory $\tau$ sampled from policy $\pi$ in $M^\star$ or $\wtM$ respectively. Let $\hatpi$ denote the policy output by the algorithm after $K$ rounds of interaction with the environment. By \cref{lem:a-11-dec} we have
	\[
	 \En^{\wtM,\alg}\brk*{J^{\Mstar}(\pi_{\Mstar}) - J^{\Mstar}(\hatpi)} \leq 	 3\En^{\Mstar,\alg}\brk*{J^{\Mstar}(\pi_{\Mstar}) - J^{\Mstar}(\hatpi)} + 4\Dhels{\bbP^{\Mstar,\alg}}{\bbP^{\wtM,\alg}}.
	\]
	By the subadditivity property for squared Hellinger distance (\cref{lem:chain-rule-hellinger}) applied to the sequence $\pi\ind{1},\tau\ind{1},\ldots,\pi\ind{K},\tau\ind{K}$, we have
	\begin{align*}
		\Dhels{\bbP^{\Mstar,\alg}}{\bbP^{\wtM,\alg}} %
		&\leq 7 \En^{\wtM,\alg}\Bigg[\sum_{k=1}^K \Dhels{\bbP(\pi\ind{k} \mid \pi\ind{1:k-1},\tau\ind{1:k-1})}{\bbQ(\pi\ind{k} \mid \pi\ind{1:k-1},\tau\ind{1:k-1})} + \\
		&\hspace{7.5em} \Dhels{\bbP(\tau\ind{k} \mid \pi\ind{1:k},\tau\ind{1:k-1})}{\bbQ(\tau\ind{k} \mid\pi\ind{1:k},\tau\ind{1:k-1})}\Bigg] \\
		&= 7 \En^{\wtM,\alg}\brk*{\sum_{k=1}^K \Dhels{\bbP(\tau\ind{k} \mid \pi\ind{1:k},\tau\ind{1:k-1})}{\bbQ(\tau\ind{k} \mid\pi\ind{1:k},\tau\ind{1:k-1})}} \\
		&= 7 \En^{\wtM,\alg}\brk*{\sum_{k=1}^K \Dhels{\Mstar(\pi\ind{k})}{\wtM(\pi\ind{k})}} \\
		&\leq 49 \En^{\wtM,\alg}\brk*{\sum_{k=1}^K \sum_{h=1}^H \En^{\wtM,\pi\ind{k}}\brk*{\Dhels{\Mstar_h(s_h,a_h)}{\wtM_h(s_{1:h},a_{1:h})}}}
	\end{align*}
	where in the second step we have used that $\bbP(\pi\ind{k} \mid \pi\ind{1:k},\tau\ind{1:k-1}) = \bbQ(\pi\ind{k} \mid\pi\ind{1:k},\tau\ind{1:k-1})$ since the histories are equivalent, in the third step we have used that the trajectories are generated by the MDP/PODMP $\Mstar$ and $\wtM$, respectively, in the fourth step we have again applied the subadditivity property of the squared Hellinger distance (\cref{lem:chain-rule-hellinger}) to the sequence $(s_1,a_1,r_1,\ldots,s_{H},a_H,r_H)$.%
\end{proof}

\onlinereductionmain*

\begin{proof}[\pfref{thm:online-reduction-main}]
This follows from \cref{thm:online-reduction-alpha} as well as \cref{lem:general-alg-sim-lemma}, by taking %
	$\alpha = c_2\sqrt{KH}$ and simplifying. %
\end{proof}

\newpage
\section{Additional Results for \cref{sec:online-rl}: Self-Predictive Estimation}\label{app:online-corruption}
This section contains a more general result for algorithmic modularity
under self-predictive estimation (\cref{thm:online-reduction-alpha}),
from which our main result is derived as a special case, along with
associated background, applications, and proofs. This section is organized as follows.

\begin{itemize}
\item \cref{sec:crobust-algorithms} presents: definitions for the
  $\phi$-compressed POMDP and \Crobust algorithms
  (\cref{sec:confounded-pomdps}), statements for properties of the
  $\phi$-compressed dynamics (\cref{app:properties-phi-pomdp}). The
  risk bound for \olr under self-predictive estimation and
  \Crobustness (\cref{thm:online-reduction-alpha}) is given in \cref{sec:risk-bound-crobust}, and a statement that the \textsc{Golf} algorithm is \Crobust (\cref{sec:crobustexamples}).
\item \cref{app:proofs-properties-phi-pomdp} presents for the proofs for the properties of the $\phi$-compressed POMDPs.
\item \cref{app:online-reduction-alpha-pf} presents a proof for the risk bound of \olr under self-predictive estimation and \Crobustness. %
\item \cref{app:crobust-examples} presents a proof that the \textsc{Golf} algorithm is \Crobust. 
\end{itemize}

\subsection{$\olr$ with Self-predictive Estimation and \Crobust Base Algorithms}\label{sec:crobust-algorithms}

\subsubsection{Definitions: $\phi$-compressed POMDP and \Crobustness}\label{sec:confounded-pomdps}

Consider iteration $k\in[K]$ of epoch $t \in [T]$ within \olr. Suppose that
\GenReplearn has chosen decoder $\phi = \phi\ind{t}: \cX \rightarrow \cS$. Then,
the latent algorithm has observed the data $\cD\ind{t,k} = \{
\phi(x\ind{t,k}_h),
a\ind{t,k}_h,r\ind{t,k}_h,\phi(x\ind{t,k}_{h+1})\}$ collected from the
preceding policies in the epoch: $\pi\ind{t,1}_\lat \circ \phi\ind{t}, \ldots, \pi\ind{t,k-1}_\lat \circ \phi\ind{t}$ (\cref{line:phi-compressed-dynamics}). Due to possible inaccuracies in the decoder $\phi$, the dataset $\cD\ind{t,k}$ may not be generated from a Markovian process and must instead be viewed as being generated from a PODMP, formally  %
	defined as follows. 
\begin{definition}[$\phi$-compressed POMDP]\label{def:phi-compressed-pomdp}
The $\phi$-compressed POMDP $\Mphi$ induced by $\Mstarobs$ and $\phi$ is defined by:\loose
\begin{enumerate}
	\item \underline{Latent} state space $\cX$
	\item Action space $\cA$
	\item \underline{Observation} state space $\cS$
	\item Latent reward functions $\Rstarobsh: \cX \times \cA \rightarrow [0,1]$
	\item Latent dynamics $\Pstarobsh: \cX \times \cA \rightarrow \Delta(\cX)$
	\item (Deterministic) observation function $\cO_h: \cX \rightarrow \cS$ defined by $\cO_h(x) = \phi_h(x)$, 
	\item Horizon $H$
	\item Initial latent distribution $\Pstarobs(x_0 \mid \emptyset)$
\end{enumerate}	
\end{definition}

Note that the latent space \textit{for the POMDP} is the observation
space of the latent-dynamics MDP $\Mstarobs$, and vice-versa; \dfedit{we adopt
this terminology because---from the perspective of the base algorithm,
the observations $x_h$ can be viewed as a Markovian (yet partially
observed process) that generates the learned states $\phi(x_h)$ on
which the algorithm acts.}
We write $\Pphipilat \coloneqq \bbP^{\Mphi,\pilat}$ for the probability distribution over trajectories $(x_h,s_h,a_h,r_h)_{h \in [H]}$ in the $\phi$-compressed POMDP when playing policy $\pilat: \cS \times [H] \rightarrow \Delta(\cA)$, where $x_h \in \cX$ are the POMDP's latent states, $s_h \in \cS$ are the observed states, and $a_h \in \cA$ are the actions. %
We let $\Ephipilat \coloneqq \En^{\Mphi,\pilat}$ denote the corresponding expectation. We write $\wtP_{\phi,h}(s_{h+1} \mid s_{1:h},a_{1:h}) = \wtP^{\pilat}_\phi(s_{h+1} \mid s_{1:h},a_{1:h})$ and $\wtr_{\phi,h}(r_h \mid s_{1:h},a_{1:h}) = \wtP^{\pilat}_\phi(r_h \mid s_{1:h},a_{1:h})$ for the conditional distributions of next states and rewards given the first $h$ state-action pairs, which are policy-independent. We also write $\Mphi(r_h,s_{h+1} \mid s_{1:h},a_{1:h}) =  \wtr_{\phi,h}(r_h \mid s_{1:h},a_{1:h})\wtP_{\phi,h}(s_{h+1} \mid s_{1:h},a_{1:h})$ for the joint one-step probability. We will abbreviate $\Mphi(s_{1:h},a_{1:h}) \coloneqq \Mphi(r_h,s_{h+1} = \cdot \mid s_{1:h},a_{1:h})$. 

Note that for any $\pilat$, $\Pphipih(s_{h+1} \mid s_h,a_h)$ is a well-defined (Markovian, policy-dependent) probability kernel, which is equivalent to 
\begin{align}\label{eq:Pphi-pomdp-definition}
	\Pphipih(s_{h+1} \mid s_h,a_h) &= \sum_{s_{1:h-1},a_{1:h-1}}\Pphipih(s_{1:h-1},a_{1:h-1} \mid s_h,a_h) \wtP_{\phi,h}(s_{h+1} \mid s_{1:h},a_{1:h}) \\
	&= \wtE^{\pilat}_\phi\brk*{\wtP_{\phi,h}(s_{h+1} \mid s_{1:h},a_{1:h}) \mid s_h,a_h}
\end{align}
Similarly, $\rphipih(r_h \mid s_h,a_h)$ is a Markovian and policy-dependent reward distribution which is equivalent to
\begin{align}\label{eq:rphi-pomdp-definition}
	\rphipih(r_h \mid s_h,a_h) &= \sum_{s_{1:h-1},a_{1:h-1}}\Pphipih(s_{1:h-1},a_{1:h-1} \mid s_h,a_h) \wtr_{\phi,h}(r_h \mid s_{1:h},a_{1:h}) \\
	&= \wtE^{\pilat}_\phi\brk*{\wtr_{\phi,h}(r_h \mid s_{1:h},a_{1:h}) \mid s_h,a_h}.
\end{align}

 Finally, we let 
\begin{equation}\label{eq:mphi-pomdp-definition}
\wtM^{\pilat,\star}_{\phi,h}(r_h, s_{h+1} \mid s_h,a_h) = \wtE^{\pilat}_\phi\brk*{\Mphi(r_h,s_{h+1} \mid s_{1:h},a_{1:h}) \mid s_h,a_h}
\end{equation}
denote the associated one-step model over joint rewards and transitions.

Our \Crobustness condition asserts that the agent---when
observing data from the $\phi\ind{t}$-compressed dynamics
$\Mphi$---attains a risk bound for $\Mlat$ which is proportional to
its risk when observing data from $\Mlat$ itself, plus a term that
captures the degree of misspecification between $\Mphi$ and $\Mlat$. %

\begin{definition}[\Crobust algorithm]\label{def:alglat-robust}
We say that $\Alglat$ is \Crobust with parameters $\alpha$ and $\Riskbase$ if there exists a constant $c_1$ such that, for any $(\phi,M_\lat) \in \Phi \times \cMlat$, we have
\begin{align*}
	\En^{\Mphi,\Alglat}\brk*{\Risk(K,\Alglat,\Mlat)} &\leq c_1\cdot\Riskbase(K) \\
	&\hspace{-3em}+\alpha\En^{\Mphi,\Alglat}\brk*{\sqrt{\sum_{k=1}^K \sum_{h=1}^H \En_{\pi\ind{k}_\lat \sim p\ind{k}}\wtE_{\phi}^{\pi\ind{k}_\lat}\brk*{\Dhels{\Mlath(s_h,a_h)}{\Mphih(s_{1:h}, a_{1:h})}}}},
\end{align*}

where we recall the definition of the random variable $\Risk(K,\Alglat,\Mlat)$ from \eqref{eq:risk-alg},
   the expectation $\En^{\Mphi,\alglat}$ denotes the interaction protocol of \Alglat in the $\phi$-compressed dynamics $\Mphi$, and $p\ind{k}$ denotes the randomization distribution over latent policies that $\Alglat$ plays.%
\end{definition}

\subsubsection{Basic properties of the $\phi$-compressed dynamics (\cref{def:phi-compressed-pomdp})}\label{app:properties-phi-pomdp}

We establish a number of basic properties for the $\phi$-compressed POMDP and their relation to the self-prediction guarantee obtained by \OptReplearn. These properties are proved in \cref{app:proofs-properties-phi-pomdp}. Firstly, we have the following change-of-measure lemma:
\begin{restatable}[Change of measure lemma]{lemma}{changeofmeasure}\label{lem:change-of-measure}
For any $\phi \in \Phi$, $f \in [\cS \times \cA \rightarrow [0,1]]$, $h \in [H]$, and $\pilat \in [\cS \times [H] \rightarrow \Delta(\cA)]$, we have:
\begin{equation}
	 \wtE^{\pilat}_\phi\brk*{f(s_h,a_h)} = \En^{\pilat \circ \phi}\brk*{\brk{f \circ \phi}(x_h,a_h)}.\label{eq:com}
\end{equation}
\end{restatable}

The next lemma states that the kernels of the $\phi$-compressed POMDP are well-approximated by the (Markovian) latent model fit by \OptReplearn. We recall the instantaneous self-prediction error 
\[
\brk*{\Delta_h(\Mlat, \phi)}(x_h,a_h) \coloneqq \Dhels{M\sub{\lat,h}(\phi\sub{h}(x_h),a_h)}{\brk[\big]{\phi\sub{h+1}\sharp M^\star_{\obs,h}}(x_h,a_h)}.
\]

\begin{restatable}[Near-markovianity of the $\phi$-compressed dynamics]{lemma}{nearmarkov}\label{lem:near-markov}
For any decoder $\phi$, base model $\Mlat$, and policy $\pi_\lat: \cS \times [H] \rightarrow \Delta(\cA)$, we have:  
\begin{align}\label{eq:near-markov-podmp}
	&\sum_{h=0}^H \wtE^{\pilat}_{\phi}\brk*{\Dhels{\Mlath(s_h,a_h)}{\Mphih(s_{1:h},a_{1:h})}} \leq \sum_{h=0}^H \En^{\pilat \circ \phi}\brk*{\brk{\Delta_h(\Mlat, \phi)}(x_h,a_h)}.
\end{align}
Furthermore, we also have
\begin{align}\label{eq:near-markov-podmp-2}
	\sum_{h=0}^H \wtE^{\pilat}_{\phi}\brk*{\Dhels{\Mlath(s_h,a_h)}{\Mphipih(s_{h},a_{h})}}  \leq \sum_{h=0}^H \En^{\pilat \circ \phi}\brk*{\brk{\Delta_h(M_{\lat}, \phi)}(x_h,a_h)}.
\end{align}
\end{restatable}

A corollary is the following lemma establishing errors between expectations under $\Mlat$, the model estimated by $\Algsim$, and those under the $\phi$-compressed POMDP $\Mphi$.%
\begin{restatable}[Simulation lemma]{lemma}{simlemma}\label{lem:simlemma}
For any latent model $\Mlat$ with Markovian transition kernel $\crl{\Plath}_{h \in [H]}$, latent policy $\pilat: \cS \times [H] \rightarrow \Delta(\cA)$, and decoder $\phi \in \Phi$, we have that for all $f: \cS \times \cA \rightarrow [0,1]$:
	\begin{equation}\label{eq:fourtet}
	\abs{\En^{\Mlat,\pi_\lat}[f(s_h,a_h)] - \wtE^{\pi_\lat}_{\phi}[f(s_h,a_h)]} \leq \sum_{h'<h}\En^{\pilat \circ \phi}\brk*{\nrm*{\brk*{\Plat \circ \phi}_h(x_{h'},a_{h'}) - \phi_{h+1}\sharp\Pstarobsh(x_{h'},a_{h'})}_\tv},
\end{equation}
and thus for any sequence of policies $\pi\ind{t}_\lat$, latent models $\Mlat\ind{t}$, and decoders $\phi\ind{t}$, we have:
\[
	\sum_{t=1}^T\sum_{h=0}^H \abs{\En^{\Mlat\ind{t},\pi\ind{t}_\lat}[f(s_h,a_h)] - \wtE^{\pi\ind{t}_\lat}_{\phi\ind{t}}[f(s_h,a_h)]} \leq H\sqrt{TH}\sqrt{ \sum_{t=1}^T\sum_{h=0}^H \En^{\pi\ind{t}_\lat \circ \phi\ind{t}}\brk*{\brk{\Delta_h(M\ind{t}_{\lat}, \phi\ind{t})}(x_h,a_h)}}.
      \]
      
\end{restatable}

\subsubsection{Risk bound for $\olr$ under \texttt{CorruptionRobustness}}\label{sec:risk-bound-crobust}

We state the main risk bound for $\olr$ under self-predictive estimation and the above definition of corruption robustness.

\onlinereductionalpha*

\subsubsection{Examples of \Crobust algorithms}\label{sec:crobustexamples}

In this section, we establish that the \Golf algorithm  satisfies the
\Crobust definition (\cref{def:alglat-robust}) with a parameter
$\alpha \approx K^{-1/2}$. This improves upon the rate that would be
obtained by invoking the generic guarantee in
\cref{lem:general-alg-sim-lemma}. We expect that several other
algorithms can be analyzed in a similar way, thereby leading to tight
rates in the same fashion. We restate the pseudocode in \cref{alg:golf-redux} for convenience. 

\begin{algorithm}[th]
\caption{\textsc{GOLF} \citep{jin2021bellman}}
\label{alg:golf-redux}
{\bfseries input:} Function classes $\cF$ and $\cG$, confidence width $\beta>0$. \\
{\bfseries initialize:} $\cF\ind{0} \leftarrow \cF$, $\cD_{h}\ind{0} \leftarrow \emptyset\;\;\forall h \in [H]$. 
\begin{algorithmic}[1]
\For{episode $t = 1,2,\dotsc,T$}
    \State Select policy $\pi\ind{t} \leftarrow \pi_{f\ind{t}}$, where $f\ind{t} \ldef{} \argmax_{f \in \cF\ind{t-1}}f(x_1,\pi_{f,1}(x_1))$. \label{step:glof_optimism}
    \State Execute $\pi\ind{t}$ for one episode and obtain trajectory $(x_1\ind{t},a_1\ind{t},r_1\ind{t}),\ldots,(x_H\ind{t},a_H\ind{t},r_H\ind{t})$. \label{step:glof_sampling}
    \State Update dataset: $\cD_{h}\ind{t} \leftarrow \cD_{h}\ind{t-1} \cup \crl[\big]{\prn[\big]{x_h\ind{t},a_h\ind{t},x_{h+1}\ind{t}}}\;\;\forall h \in [H]$.
    \State Compute confidence set:
    \begin{gather}
      \nonumber
      \cF\ind{t} \leftarrow \crl[\bigg]{ f \in \cF: \cL_{h}\ind{t}(f_h,f_{h+1}) - \min_{g_h \in \cG_h} \cL_{h}\ind{t}(g_h,f_{h+1}) \leq \beta\;\;\forall h \in [H] },
    \\
    \nonumber
    \text{where \quad } \cL_{h}\ind{t}(f,f') \coloneqq \sum_{(x,a,r,x') \in \cD_{h}\ind{t}}\prn[\Big]{ f(x,a) - r - \max_{a' \in \cA} f'(x',a') }^2 ,~\forall f,f' \in \cF.
    \end{gather}
\EndFor
\State Output $\wh{\pi} = \unif(\pi\ind{1:T})$. %
\end{algorithmic}
\end{algorithm}

Let $\Mlat = (\rlat,\Plat)$ be given, and we let $\Qstarlat \coloneqq
Q^{\Mlat,\star}$, and $\Tlath f(s,a) \coloneqq \rlath(s,a) + \En_{s' \sim \Plath(s,a)}\brk*{V_f(s')}$. We assume that the algorithm has a latent function class $\cF_{\texttt{alg}}$ which realizes $\Qstarlat$, as well as a helper class $\cG_{\texttt{alg}}$ which is $\Tlat$-complete for $\cFalg$.

\begin{assumption}[$\cT_\lat$-completeness]\label{ass:trep-completeness}We have:
\[
\Qstarlat \in \cF_{\texttt{alg}}, \quad \text{and} \quad 	\cT_\lat \cFalg \subseteq \cGalg.
\]
\end{assumption}

For our analysis of \Golf, it is most natural to quantify the corruption levels in the following way.
\begin{assumption}[Corruption levels of $\Mlat$ and $\Mphi$]\label{ass:corruption-level-golf}
Let $\vepsrep^2$ be such that, for any sequence of policies
$\pi\ind{k}_\lat$ played by the algorithm \dfedit{when interacting
  with the $\phi$-compressed POMDP}, we have
	\begin{align}\label{eq:corruption-level-golf}
	\sum_{k=1}^K \sum_{h=1}^H \En_{\pi\ind{k}_\lat \sim p\ind{k}_\lat} \wtE^{\pi\ind{k}_\lat}_{\phi}\brk*{\prn{\rlath(s_h,a_h) - \wtr^{\pi\ind{k}}_{\phi,h}(s_{h},a_{h})}^2 + \norm{\Plath(s_h,a_h) - \Pphipikh(s_{h}, a_{h})}^2_\tv} \leq \vepsrep^2.
	\end{align}
\end{assumption}
We note that 
\begin{align*}
\vepsrep^2 &\lesssim \sum_{k=1}^K \sum_{h=1}^H \wtE^{\pi\ind{k}_\lat}_\phi\brk*{\Dhels{\Mlath(s_{h},a_{h})}{\wtM_{\phi,h}^{\star,\pi\ind{k}_\lat}(s_{h},a_h)}} \\
	&\leq \sum_{k=1}^K \sum_{h=1}^H \wtE^{\pi\ind{k}_\lat}_\phi\brk*{\Dhels{\Mlath(s_{h},a_{h})}{\Mphih(s_{1:h},a_{1:h})}} 
\end{align*} 
by the data-processing inequality (cf. \eqref{eq:qpi-minus-tqpi} and \eqref{eq:getinnocuous}) and the inequality $\nrm{p-q}^2_\tv \leq \Dhels{p}{q}$, and thus a \Crobustness bound in terms of $\vepsrep$ implies a \Crobustness bound in the sense of \cref{def:alglat-robust}. 
\begin{restatable}[Latent GOLF is \Crobust]{theorem}{latentgolfcrobust}\label{thm:latentgolfcrobust}
	Under \cref{ass:trep-completeness} and \cref{ass:corruption-level-golf},  \cref{alg:golf-redux} with $\beta = c \prn*{\log\prn*{\abs{\cF}\abs{\cG}KH\deltainv} + \vepsrep}$, has regret
	  \[
	\sum_{k=1}^K J^{\Mlat}(\pi_{\Mlat}) - J^{\Mlat}(\pi\ind{k}) \leq \cO\prn*{H\sqrt{\Ccov K \log(K)\log(\nicefrac{\abs{\cF}\abs{\cG}HK}{\delta})}} + \cO\prn*{H^{3/2}\sqrt{K \Ccov \log(K) \vepsrep^2 }} ,
\]
and consequently is \Crobust (\cref{def:alglat-robust}) with parameters 
\[
\alpha = \frac{H^{3/2}}{\sqrt{K}}\sqrt{\Ccov\log(K)} \text{ and } \Riskbase(K) = \cO\prn*{\frac{H}{\sqrt{K}}\sqrt{\Ccov  \log(K)\log(\abs{\cF}\abs{\cG}HK)}}.
\]
\end{restatable}

\begin{corollary}[\textsc{Golf} applied in \olr]
 Let us suppose that the appropriate assumptions for the estimator in \cref{alg:debiased-mle-optimistic} to have regret bounded by $\Estsim(T,\gamma) = \cO\prn*{\sqrt{H\Ccov T}\log(\Ccov\abs{\cMlat}\abs{\cL_\lat}\abs{\Phi}HT)}$ (\cref{lem:implementing-optim-replearn}) hold. Then,  we can take $\gamma \approx K^{-1/2}$ and $T \approx K^{4}$, and the bound \cref{thm:online-reduction-alpha} gives
	an expected risk of $\veps$ with a number of trajectories $
        TK =
        \poly(\Ccov,H,\log\abs{\cMlat},\log\abs{\Phi},\log\abs{\cLlat})
        \cdot\nicefrac{1}{\veps^{10}}$, improving over the
        $1/\veps^{14}$ rate of the universal result
        (\cref{cor:selfpredmodular}). %
\end{corollary}

\subsection{Proofs for \cref{app:properties-phi-pomdp}: Properties of $\phi$-compressed POMDPs}\label{app:proofs-properties-phi-pomdp}

\changeofmeasure*

\begin{proof}[\pfref{lem:change-of-measure}]
Recall that $\wtP^{\pilat}_\phi$ denotes the law of $(x_h,s_h,a_h)_{h
  \in [H]}$ in the $\phi$-compressed POMDP when playing policy
$\pilat$. For clarity, and to differentiate a random variable from its
realization, in the proofs below we will use upper-case notation such
as $\crl*{S_h = s_h, A_h = a_h, X_h = x_h}$ to indicate realizations

of random variables in the POMDP. %

Let $\wtd^{\pilat}_h(s,a) = \wtP^{\pilat}_\phi(S_h=s,A_h=a)$ be the marginalized occupancy measure for in the \compressedPOMDP\, $\Mphi$. We write $d^{\pilat \circ \phi}_h \coloneqq d^{\Mstarobs,\pilat\circ \phi}_h$.
The \lhs in \eqref{eq:com} is equal to:
\[
	\wtE^{\pilat}_{\phi}\brk{f(s_h,a_h)} = \sum_{s \in \cS,a \in \cA} \wtd^{\pilat}_h(s,a) f(s,a),
\]
Meanwhile, the \rhs is equal to:
	\[
	\En^{\pilat \circ \phi}_h\brk{\brk{f \circ \phi}(x_h,a_h)} = \sum_{s \in \cS,a \in \cA} f(s,a) \sum_{x:\phi(x)=s} d^{\pilat \circ \phi}_h(x,a).
\]
So it only remains to show that, for each $s \in \cS$ and $a \in \cA$,
we have $\wtd^{\pilat}_h(s,a) = \sum_{x:\phi(x)=s} d^{\pilat \circ
  \phi}_h(x,a)$. Firstly, note that it is enough to show that
$\sum_{x_{h}: \phi(x_{h})=s_{h}} d^{\pilat \circ \phi}_{h}(x_{h}) =
\wtd^{\pilat}_{h}(s_h)$, since $\wtd^{\pilat}_{h}(s_h,a_h) = \wtd^{\pilat}_{h}(s_h)\pilat(a_h \mid s_h)$ and $\sum_{x_{h}: \phi(x_{h})=s_{h}} d^{\pilat \circ \phi}_{h}(x_{h},a_h) = \sum_{x_{h}: \phi(x_{h})=s_{h}} d^{\pilat \circ \phi}_{h}(x_{h})\pilat(a_h \mid \phi(x_h)) = \pilat(a_h \mid s_h) \sum_{x_{h}: \phi(x_{h})=s_{h}} d^{\pilat \circ \phi}_{h}(x_{h})$. 
  Toward this, we have:
\begin{align*}
\sum_{x_{h}: \phi(x_{h})=s_{h}} d^{\pilat \circ \phi}_{h}(x_{h}) &= \sum_{x_{h}: \phi(x_{h})=s_{h}} \prn*{\sum_{x_{h-1},a_{h-1} \in \cX \times \cA} d_{h-1}^{\pilat\circ\phi}(x_{h-1},a_{h-1}) \Pstarobsh(x_{h} \mid x_{h-1},a_{h-1})} \\ 
	&= \sum_{x_{h-1},a_{h-1} \in \cX \times \cA} d_{h-1}^{\pilat\circ\phi}(x_{h-1},a_{h-1}) \sum_{x_{h}: \phi(x_{h}) = s_{h}} \Pstarobsh(x_{h} \mid x_{h-1},a_{h-1}) \\
	&= \sum_{x_{h-1},a_{h-1} \in \cX \times \cA} d_{h-1}^{\pilat\circ\phi}(x_{h-1},a_{h-1}) \Pstarobsh(\phi(x_{h}) = s_h \mid x_{h-1},a_{h-1}) \\
\end{align*}
At the same time, %
\begin{align*}
\wtd^{\pilat}_{h}(s_h) = \wtP^{\pilat}_\phi(S_{h} = s_h) &= \sum_{\wt{x},\wt{a}} \wtP^{\pilat}_\phi(X_{h-1}=\wt{x},A_{h-1}=\wt{a})\wtP^{\pilat}_\phi(S_{h} = s_{h} \mid X_{h-1} = \wt{x},A_{h-1} = \wt{a})  \\
	&= \sum_{\wt{x},\wt{a}} \wtP^{\pilat}_\phi(X_{h-1}=\wt{x},A_{h-1}=\wt{a}) \Pstarobs(\phi(x_h) = s_h \mid x_{h-1},a_{h-1}),
\end{align*}
where in the second equality we have used the definition of the observation function $s_h = \cO(x_h) = \phi(x_h)$. 

To conclude, it remains to show that for all $h$, we have:
\[
d_{h}^{\pilat\circ\phi}(x_{h},a_{h}) = \wtP^{\pilat}_\phi(X_{h}=x_{h},A_{h}=a_{h}).
\]
We do this by induction. Again, note that it is sufficient to establish $d_{h}^{\pilat\circ\phi}(x_{h}) = \wtP^{\pilat}_\phi(X_h = x_h)$.  The case $h=1$ is clear. For the general case, we have:
\begin{align*}
	d_{h}^{\pilat\circ\phi}(x_{h}) &= \sum_{x_{h-1},a_{h-1} \in \cX \times \cA}	d_{h-1}^{\pilat\circ\phi}(x_{h-1},a_{h-1})\Pstarobs(x_h \mid x_{h-1},a_{h-1}) \\
		&= \sum_{x_{h-1},a_{h-1} \in \cX \times \cA}	\wtP^{\pilat}_\phi(X_h=x_{h-1},A_{h-1}=a_{h-1})\Pstarobs(x_h \mid x_{h-1},a_{h-1}) \\
		&= \sum_{x_{h-1},a_{h-1} \in \cX \times \cA}	\wtP^{\pilat}_\phi(X_h=x_{h-1},A_{h-1}=a_{h-1})\wtP^{\pilat}_\phi(X_h = x_h \mid X_{h-1}=x_{h-1}, A_{h-1}=a_{h-1}) \\
		&= \wtP^{\pilat}_\phi(X_h = x_h).
\end{align*}
\end{proof}

\nearmarkov*

\begin{proof}[\pfref{lem:near-markov}]
	We begin with the first event. Note that, for any $\pilat$, the PODMP kernel $\Mphih(r_h,s_{h+1} = \cdot \mid s_{1:h},a_{1:h})$ can be written as:
		\begin{align*}
		\Mphih(r_h,s_{h+1} = \cdot \mid s_{1:h},a_{1:h}) &= \sum_{x_h,a_h \in \cX \times \cA} \wtP^{\pilat}_\phi(r_h,s_{h+1} = \cdot \mid x_h,a_h,s_{1:h},a_{1:h})\wtP^{\pilat}_{\phi}(x_h,a_h \mid s_{1:h},a_{1:h}) \\%\wtP_\phi(x_h,a_h \mid s_{1:h},a_{1:h}) \\
			&= \sum_{x_h,a_h \in \cX \times \cA} \wtP^{\pilat}_\phi(r_h,s_{h+1} = \cdot \mid x_h,a_h)\wtP^{\pilat}_{\phi}(x_h,a_h \mid s_{1:h},a_{1:h}),
	\end{align*}
	where we have used $\wtM(r_h,s_{h+1} = \cdot \mid s_{1:h},a_{1:h}) = \wtP^{\pilat}_\phi(r_h,s_{h+1} = \cdot \mid s_{1:h},a_{1:h})$, the law of total probability, and that $x_h,a_h$ is a sufficient statistic for $r_h$ and $ s_{h+1}$. %
	We further note that
	\begin{equation}\label{eq:justice}
		\wtP^{\pilat}_\phi(r_h,s_{h+1} = \cdot \mid x_h,a_h) = \Mstarobsh(r_h,\phi_{h+1}(x_{h+1}) = \cdot \mid x_h,a_h),
	\end{equation}
	since $s_{h+1} = \cO_{h+1}(x_{h+1}) = \phi_{h+1}(x_{h+1})$ is a deterministic function of $x_{h+1}$ and $r_h,x_{h+1} \sim \Mstarobsh(x_h,a_h)$.
	Thus, for a fixed $h$ and $t$, and omitting the $h$ indices on the decoder $\phi$ for cleanliness, the expectation in equation \eqref{eq:near-markov-podmp} becomes:
	\begin{align*}
	&\wtE_\phi^{\pilat}\brk*{\Dhels{\Mlath(s_h,a_h)}{\Mphih(r_h,s_{h+1}=\cdot \mid s_{1:h},a_{1:h})}} \\
		&\leq \sum_{s_{1:h},a_{1:h} \in (\cS \times \cA)^h} \wtP_\phi^{\pilat}(s_{1:h},a_{1:h}) \sum_{x_h,a_h} \wtP_\phi^{\pilat}(x_h,a_h \mid s_{1:h},a_{1:h})\Dhels{\Mlath(s_h,a_h)}{\wtP^{\pilat}_\phi(r_h,s_{h+1} = \cdot \mid x_h,a_h)}\tag{Jensen} \\
		& = \sum_{\substack{s_{1:h},a_{1:h} \in (\cS \times \cA)^h \\ x_h,a_h \in \cX \times \cA}} \wtP_\phi^{\pilat}(s_{1:h},a_{1:h})
           \wtP^{\pilat}_\phi(x_h,a_h \mid s_{1:h},a_{1:h})\Dhels{\Mlath(\phi(x_h),a_h)}{
           \wtP^{\pilat}_\phi(r_h,\phi(x_{h+1}) = \cdot \mid x_h,a_h)}
          \\
         & = \sum_{x_h,a_h \in \cX \times \cA}
           \wtP_\phi^{\pilat}(x_h,a_h)\Dhels{\Mlat(\phi(x_h),a_h)}{
           \Mstarobsh(r_h,\phi(x_{h+1}) = \cdot \mid x_h,a_h)} \tag{Simplifying \& \eqref{eq:justice}}
          \\
		& = \En^{\pilat \circ \phi}\brk*{\Dhels{\Mlat(\phi(x_h),a_h)}{
           \Mstarobsh(r_h,\phi(x_{h+1}) = \cdot \mid x_h,a_h)}} \tag{Change of measure (\cref{lem:change-of-measure})} \\
		& = \En^{\pilat \circ \phi}\brk*{\Dhels{\Mlat(\phi(x_h),a_h)}{
           \phi\sharp\Mstarobsh(x_h,a_h)}} \tag{By definition of $\phi\sharp\Mstarobs$},
\end{align*}
as desired. %
Summing over $h \in [H]$ we obtain the desired bound. The bound \eqref{eq:near-markov-podmp-2}
is a consequence of \eqref{eq:near-markov-podmp} and the data-processing inequality. Namely, using the definition of $\Mphipih$ from \eqref{eq:mphi-pomdp-definition} and the joint convexity of the squared Hellinger distance we have:
 \begin{equation}\label{eq:getinnocuous}
 \Dhels{\Mlath(\cdot \mid s_h,a_h)}{\Mphipih(\cdot \mid s_{h},a_{h})}  \leq \wtE^{\pilat}_{\phi}\brk*{\Dhels{\Mlath(\cdot \mid s_h,a_h)}{\Mphih(\cdot \mid s_{1:h},a_{1:h})} \mid s_h,a_h}.
 \end{equation}
Thus, we have
 \begin{align*}
 &\En^{\pilat}_{\phi}\brk*{ \Dhels{\Mlath(\cdot \mid s_h,a_h)}{\Mphipih(\cdot \mid s_{h},a_{h})}} \\
 &\quad \leq  \En^{\pilat}_{\phi}\brk*{ \En^{\pilat}_{\phi}\brk*{\Dhels{\Mlath(\cdot \mid s_h,a_h)}{\Mphih(\cdot \mid s_{1:h},a_{1:h})} \mid s_h,a_h } } \\
 &\quad =  \En^{\pilat}_{\phi}\brk*{ \Dhels{\Mlath(\cdot \mid s_h,a_h)}{\Mphih(\cdot \mid s_{1:h},a_{1:h})} },
 \end{align*}
 as desired. 
\end{proof}

\simlemma*

\begin{proof}[\pfref{lem:simlemma}]
Firstly note that, from \cref{lem:change-of-measure}, the left-hand-side of \eqref{eq:fourtet} is equivalent to 
\begin{equation}\label{eq:fourtet2}
\abs{\En^{\Mlat,\pi_\lat}[f(s_h,a_h)] - \wtE^{\pi_\lat}_{\phi}[f(s_h,a_h)]} = \abs{\En^{\Mlat,\pi_\lat}[f(s_h,a_h)] - \En^{\Mstarobs,\pi_\lat \circ \phi}\brk*{[f \circ \phi](x_h,a_h)}} 
\end{equation}
For any $\pilat: \cS \times [H] \rightarrow \Delta(\cA)$, let $\dlath^\pilat = d^{\Mlat,\pilat}_h$ denote the occupancy in $\Mlat$, and similarly for any $\pi_\obs: \cX \times [H] \rightarrow \Delta(\cA)$ let $d_{\obs,h}^\piobs(x_h,a_h) = d^{\Mobsstar,\piobs}_h(x_h,a_h)$ denote the occupancy in $\Mobsstar$. We overload notation by letting $d_{\obs,h}^{\pilat \circ \phi}(s,a) \coloneqq \sum_{x: \phi(x)=s}d_{\obs,h}^{\pi\circ\phi}(x,a)$. We will establish the stronger result that 
\begin{equation}
  \label{eq:direct}
\nrm*{\dlath^{\pilat}(\cdot) - d_{\obs,h}^{\pilat \circ \phi}(\cdot)}_\tv \leq  \sum_{h'<h}\En^{\pilat \circ \phi}\brk*{\nrm*{\brk*{\Plat \circ \phi}(x_{h'},a_{h'}) - \phi\sharp\Pstarobs(x_{h'},a_{h'})}_\tv},
\end{equation}
where the $\tv$ norm on the left-hand-side is over $\cS \times
\cA$. Note that this implies the desired bound on \eqref{eq:fourtet2} by Holder's inequality.  We prove this by induction over $h$. For the base case ($h=0$), we have:
  \begin{align*}
&\sum_{s_1,a_1} \Bigg|d^{\pilat}_{\lat,1}(s_1,a_1) - d^{\pilat \circ \phi}_{\obs}(s_1,a_1)\Bigg|\\
&\qquad = \sum_{s_1,a_1} \Bigg|P_{\lat,0}(s_1 \mid \emptyset)\pilat(a_1 \mid s_1) - \sum_{x_1 = \phi(x_1)=s_1} d^{\pilat \circ \phi}_{\obs}(x_1,a_1)\Bigg| \\
&\qquad = \sum_{s_1,a_1} \Bigg|P_{\lat,0}(s_1 \mid \emptyset)\pilat(a_1 \mid s_1) - \sum_{x_1 = \phi(x_1)=s_1} P^\star_{\obs,0}(x_1 \mid \emptyset)\pilat(a_1 \mid \phi(x_1))\Bigg| \\
&\qquad = \sum_{s_1} \Bigg|P_{\lat,0}(s_1 \mid \emptyset) - \phi_1\sharp P^\star_{\obs,0}(s_1 \mid \emptyset)\Bigg| \sum_{a_1} \pilat(a_1 \mid s_1) \\
&\qquad = \nrm*{P_{\lat,0}(\emptyset) - \phi_1\sharp P^\star_{\obs,0}(\emptyset)}_\tv.
\end{align*}

For the general case, let us further overload notation by letting $ d^{\pi\circ\phi}_{\obs,h}(s_h) = \sum_{a_h} d^{\pi \circ \phi}_{\obs,h}(s_h,a_h)$ and $\Pstarobs(s_h \mid x_{h-1},a_{h-1}) =  \phi\sharp\Pstarobs(s_h \mid x_{h-1},a_{h-1}) = \sum_{x_h: \phi(x_h)=s_h} \Pstarobs(x_h \mid x_{h-1},a_{h-1})$. Let us also abbreviate $\pi \coloneqq \pilat$. Firstly note that it is sufficient to establish the
result for $\sum_{s_h \in \cS} \abs*{\dlath^\pi(s_h) -
  d^{\pi\circ\phi}_{\obs,h}(s_h)}$, since
	\begin{align*}
	\sum_{s_h,a_h \in \cS \times \cA} \abs*{\dlath^\pi(s_h,a_h) - d^{\pi\circ\phi}_{\obs,h}(s_h,a_h)} &= \sum_{s_h,a_h \in \cS \times \cA} \abs*{\dlath^\pi(s_h) - d^{\pi\circ\phi}_{\obs,h}(s_h)}\pi(a_h\mid s_h) \\ 
		&= \sum_{s_h \in \cS} \abs*{\dlath^\pi(s_h) - d^{\pi\circ\phi}_{\obs,h}(s_h)}.
	\end{align*}
Below, all summations over $s_h$ (resp. $x_h$) with domain unspecified
are over $\cS$ (resp. $\cX$), and likewise for summations over
$s_h,a_h$ or $x_h,a_h$. We have: %
\begin{align*}
&\sum_{s_h}\Bigg|d^\pi_{\lat,h}(s_h) - d_{\obs,h}^{\pi \circ \phi}(s_h)\Bigg| \\
	&=\sum_{s_h} \Bigg| \dlath^{\pi}(s_h) - \sum_{x_h: \phi(x_h)=s_h}d_{\obs,h}^{\pi \circ \phi}(x_h)\Bigg| \\
	&= \sum_{s_h} \Bigg|\sum_{s_{h-1},a_{h-1}}\dlath^{\pi}(s_{h-1},a_{h-1})\Plath(s_h \mid s_{h-1},a_{h-1}) \\
	&\hspace{10em} -\sum_{x_h: \phi(x_h)=s_h}\sum_{x_{h-1},a_{h-1}}d_{\obs,h}^{\pi \circ \phi}(x_{h-1},a_{h-1})\Pstarobsh(x_h \mid x_{h-1},a_{h-1})\Bigg| \\
	&= \sum_{s_h}\Bigg|\sum_{s_{h-1},a_{h-1}}\dlath^{\pi}(s_{h-1},a_{h-1})\Plath(s_h \mid s_{h-1},a_{h-1}) \\
	&\hspace{10em} -\sum_{x_{h-1},a_{h-1}}d_{\obs,h}^{\pi \circ \phi}(x_{h-1},a_{h-1})\Pstarobsh(s_h \mid x_{h-1},a_{h-1})\Bigg| \\
	&= \sum_{s_h}\Bigg|\sum_{s_{h-1},a_{h-1}}\dlath^{\pi}(s_{h-1},a_{h-1})\Plath(s_h \mid s_{h-1},a_{h-1}) - \sum_{x_{h-1},a_{h-1}}d_{\obs,h}^{\pi \circ \phi}(x_{h-1},a_{h-1})\Plath(s_h \mid \phi(x_{h-1}),a_{h-1}) \\
	&\quad+ \sum_{x_{h-1},a_{h-1}}d_{\obs,h}^{\pi \circ \phi}(x_{h-1},a_{h-1})\Plath(s_h \mid \phi(x_{h-1}),a_{h-1}) - \sum_{x_{h-1},a_{h-1}}d_{\obs,h}^{\pi \circ \phi}(x_{h-1},a_{h-1})\Pstarobsh(s_h \mid x_{h-1},a_{h-1})\Bigg| \\
	&\leq \sum_{s_{h-1},a_{h-1}} \abs*{\dlath^{\pi}(s_{h-1},a_{h-1})- \sum_{x_{h-1}:\phi(x_{h-1})=s_{h-1}}d_{\obs,h}^{\pi \circ \phi}(x_{h-1},a_{h-1})}\sum_{s_{h}}\Plath(s_h \mid s_{h-1},a_{h-1})  \\
	&\quad + \sum_{s_h}\abs*{\sum_{x_{h-1},a_{h-1}}d_{\obs,h}^{\pi \circ \phi}(x_{h-1},a_{h-1})\prn*{(\Plath \circ \phi)(s_h \mid x_{h-1},a_{h-1}) - \Pstarobsh(s_h \mid x_{h-1},a_{h-1})}} \\
	&\leq \nrm*{\dlathmo^{\pi}(\cdot) - d_{\obs,h-1}^{\pi \circ \phi}(\phi^{-1}(\cdot))}_\tv \\
	&\quad + \sum_{x_{h-1},a_{h-1}}d_{\obs,h}^{\pi \circ \phi}(x_{h-1},a_{h-1}) \sum_{s_h}\Bigg|(\Plath \circ \phi)(s_h \mid x_{h-1},a_{h-1}) - \Pstarobsh(s_h \mid x_{h-1},a_{h-1})\Bigg| \\
	&\quad \leq \nrm*{\dlathmo^{\pi}(\cdot) - d_{\obs,h-1}^{\pi \circ \phi}(\phi^{-1}(\cdot))}_\tv +\En^{\pi \circ \phi}\brk*{\nrm*{\brk*{\Plath \circ \phi}(x_{h-1},a_{h-1}) - \phi\sharp\Pstarobsh(x_{h-1},a_{h-1})}_\tv}.
\end{align*}

From which it follows that, for each $h$, we have: %
\begin{align*}
	 \nrm*{\dlath^{\pi}(\cdot) - d_{\obs,h}^{\pi \circ \phi}(\phi^{-1}(\cdot))}_\tv &\leq \sum_{h'<h}\En^{\pi \circ \phi}\brk*{\nrm*{\brk*{\Plat \circ \phi}_{h'}(x_{h'},a_{h'}) - \phi_{h'+1}\sharp\Pstar_{\obs,h'}(x_{h'},a_{h'})}_\tv} \\
	 &\leq \sum_{h' \in [H]}\En^{\pi \circ \phi}\brk*{\nrm*{\brk*{\Plat \circ \phi}_{h'}(x_{h'},a_{h'}) - \phi_{h'+1}\sharp\Pstar_{\obs,h'}(x_{h'},a_{h'})}_\tv}.
\end{align*}
\end{proof}

\subsection{Proofs for \cref{sec:risk-bound-crobust}: Risk Bound Under \Crobustness (\cref{thm:online-reduction-alpha})}\label{app:online-reduction-alpha-pf}

\onlinereductionalpha*

\begin{proof}[\pfref{thm:online-reduction-alpha}]
Let us write $\pi\ind{t,K+1}_\lat = \hatpi\ind{t}_\lat$ and, for any $t,k \in [T] \times [K+1]$, $\pi\ind{t,k}_\obs \coloneqq \pi\ind{t,k}_\lat \circ \phi\ind{t}$. Let $p\ind{t,k}_\obs$ denote the distributions of played policies $\pi\ind{t,k}_\obs$ induced by the interaction of $\Alglat$ and $\OptReplearn$ inside the \olr algorithm. Let us write the online sum of self-prediction errors as
\begin{equation}\label{eq:epsrep}
\vepsrep^2 \coloneqq \sum_{t=1}^T \sum_{k=1}^{K+1} \sum_{h=0}^H \En_{\pi\ind{t,k}_\obs \sim p\ind{t,k}}\En^{\pi\ind{t,k}_\obs}\brk*{\Dhels{\brk*{M\ind{t}\sub{\lat}\circ \phi\ind{t}}_h(x_h,a_h)}{\phi\ind{t}\sub{h+1}\sharp M^\star\sub{\obs,h}(x_h,a_h)}}
\end{equation}
Since the final output policy of $\olr$ satisfies $\hatpilat = \unif(\hatpilat\ind{1},\ldots,\hatpilat\ind{T})$ (\cref{line:pac-output}), we have 
\[
\En\brk*{\Riskobs(TK)} = \frac{1}{T}\sum_{t=1}^T\En\brk*{ J^{\Mstarobs}(\pistarobs) - J^{\Mstarobs}(\hatpi\ind{t}_\obs)}.
\]
We take the following decomposition on the risk
\begin{align}\label{eq:risk-decomp}
	J^{\Mstarobs}(\pistar_\obs) - J^{\Mstarobs}(\hatpi\ind{t}_\obs) &= J^{\Mstarlat}(\pi_{\Mstarlat}) - J^{M\ind{t}_\lat}(\pi_{M\ind{t}_\lat}) + \underbrace{J^{M\ind{t}_\lat}(\pi_{M\ind{t}_\lat}) - J^{M\ind{t}_\lat}(\hatpi\ind{t}_\lat)}_{\mathrm{A_t}} + \underbrace{J^{M\ind{t}_\lat}(\hatpi\ind{t}_\lat) - J^{\Mstarobs}(\hatpi\ind{t}_\obs)}_{\mathrm{B_t}}.
\end{align}
We will show that $\En\brk*{\sum_{t=1}^T \mathrm{A_t}} \lesssim T\Regbase(K) + \alpha \sqrt{T} \En\brk{\vepsrep}$ and that $\En\brk*{\sum_{t=1}^T
\mathrm{B_t}} \lesssim \sqrt{TH}\En\brk{\vepsrep}$, then return to the first
term $J^{\Mstarlat}(\pi_{\Mstarlat}) -
J^{M\ind{t}_\lat}(\pi_{M\ind{t}_\lat})$ at the end of the proof.

To bound $\En\brk*{\sum_{t=1}^{T}A_t}$, we note that
\begin{align*}
\sum_{t=1}^T\En\brk*{A_t} &\leq c_1T \Riskbase(K) +  \alpha \sum_{t=1}^T \En\brk*{\sqrt{\sum_{k=1}^K\sum_{h=1}^H \En_{\pi\ind{t,k}_\lat \sim p\ind{t,k}_\lat}\wtE^{\pi\ind{t,k}_\lat}_{\phi\ind{t}} \brk*{ \Dhels{\Mlath\ind{t}(s_h,a_h)}{\wt{M}^{\star}_{\phi\ind{t},h}(s_{1:h},a_{1:h})} } }} \\
	&\leq c_1 T\Riskbase(K) +  \alpha \sum_{t=1}^T\En\brk*{\sqrt{\sum_{k=1}^K\sum_{h=1}^H \En_{\pi\ind{t,k}_\lat \sim p\ind{t,k}_\lat} \En^{\pi\ind{t,k}_\lat \circ \phi\ind{t}}\brk*{\brk{\Delta_h(M\ind{t}_{\lat}, \phi\ind{t})}(x_h,a_h)} } } \\
	&\leq c_1 T \Riskbase(K) +  \alpha\sqrt{T} \En\brk*{\sqrt{\sum_{t=1}^T \sum_{k=1}^K\sum_{h=1}^H \En_{\pi\ind{t,k}_\obs \sim p\ind{t,k}_\obs} \En^{\pi\ind{t,k}_\obs}\brk*{\brk{\Delta_h(M\ind{t}_{\lat}, \phi\ind{t})}(x_h,a_h)} } } \\
	&\leq c_1 T \Riskbase(K) +  \alpha\sqrt{T} \En\brk*{\vepsrep}.
\end{align*}
where the first line follows from the \Crobust definition (\cref{def:alglat-robust}), the second line follows from \cref{lem:near-markov}, the third line follows by Cauchy-Schwartz, and the last line recalls the definition of $\vepsrep$ from \eqref{eq:epsrep}. %

For the term $\sum_{t=1}^{T}B_t$, for any $\pilat: \cS \times [H] \rightarrow \Delta(\cA)$ we let $Q_{\lat\ind{t},h}^{\pilat} = \cT^{\Mlat\ind{t}}_h Q_{\lat\ind{t},h+1}^{\pilat}$ be the $Q^{\pilat}$ function of the latent MDP $\Mlat\ind{t}$. %
	Note that 
\begin{align}
&\sum_{t=1}^T \crl*{J^{M\ind{t}_\lat}(\hatpi\ind{t}_\lat) - \En^{\hatpi\ind{t}_\lat \circ \phi\ind{t}}\brk*{\brk{Q\sub{\lat\ind{t}}^{\hatpi\ind{t}} \circ \phi\ind{t}}_1(x_1,a_1)}} \\
&\qquad = \sum_{t=1}^T \En^{M\ind{t}_\lat,\hatpi\ind{t}_\lat}\brk*{Q\sub{\lat\ind{t},1}^{\hatpi\ind{t}}(s_1,a_1)} - \En^{\hatpi\ind{t}_\lat\circ\phi\ind{t}}\brk*{\brk{Q\sub{\lat\ind{t}}^{\hatpi\ind{t}} \circ \phi\ind{t}}_1(x_1,a_1)} \nonumber \\
&\qquad \leq \sum_{t=1}^T \En^{\hatpi\ind{t}_\lat \circ \phi\ind{t}}\brk*{\nrm*{\brk{P\ind{t}\sub{\lat}\circ \phi\ind{t}}_0(\emptyset) - \phi\ind{t}\sub{1}\sharp P^\star\sub{\obs,0}(\emptyset)}_\tv} \tag{by \cref{lem:simlemma}} \\
&\qquad \leq \sum_{t=1}^T \sum_{h=0}^H \En^{\hatpi\ind{t}_\lat \circ \phi\ind{t}}\brk*{\nrm*{\brk{P\ind{t}\sub{\lat}\circ \phi\ind{t}}_h(x_h,a_h) - \phi\ind{t}\sub{h+1}\sharp P^\star\sub{\obs,h}(x_h,a_h)}_\tv} \nonumber \\
&\qquad \leq \sqrt{TH}\vepsrep \tag{by Cauchy-Schwartz},
\end{align}
so it is enough to bound 
\[
\sum_{t=1}^T \crl*{\En^{\hatpi\ind{t}_\lat\circ\phi\ind{t}}\brk*{\brk{Q\sub{\lat\ind{t}}^{\hatpi\ind{t}_\lat} \circ \phi\ind{t}}_1(x_1,a_1)} - J^{\Mstarobs}(\hatpi\ind{t}_\obs)}.
\]

Fix $t$ and $h$, whose indexing we omit below for cleanliness. Note that, for any $\pilat: \cS \times [H] \rightarrow \Delta(\cA)$, we have: 
\begin{align}
&\En^{\pilat \circ \phi}\brk*{\prn*{\brk*{Q_{\lat}^\pilat \circ \phi}_h(x_h,a_h) - \cT^{\Mstarobs,\pilat \circ \phi}_h\brk*{Q_{\lat}^\pilat \circ \phi}_{h+1}(x_h,a_h)}^2} \\
&\, \leq 2\En^{\pilat \circ \phi}\brk*{\prn*{\brk*{r_{\lat} \circ \phi}_h - r^\star_{\obs,h}}^2(x_h,a_h)} \\
	&\quad + 2\En^{\pilat \circ \phi}\brk*{\prn*{\En_{\Plath(\phi(x_h),a_h)}\brk*{Q^\pilat_{\lat,h+1}(\cdot,\pilat)} - \En_{\Pstarobsh(x_h,a_h)}\brk*{\brk*{Q_\lat^\pilat \circ \phi}_{h+1}(\cdot,\pilat)}}^2} \\
&\, \leq 2\En^{\pilat \circ \phi}\brk*{\prn*{\brk*{r_{\lat} \circ \phi}_h - r^\star_{\obs,h}}^2(x_h,a_h) + \nrm*{\Plath(\phi(x_h),a_h) - \phi_{h+1}\sharp\Pstarobsh(x_h,a_h)}^2_\tv} \\
&\, \leq 4\En^{\pilat \circ \phi}\brk*{\Dhels{\Mlath(\phi_h(x_h),a_h)}{\phi_{h+1}\sharp\Mstarobsh(x_h,a_h)}} \label{eq:qpi-minus-tqpi},
\end{align}
where the final line follows from two applications of the
data-processing inequality (since $\Mlath(r_h,s_{h+1} \mid \phi_h(x_h),a_h)= \Rlath(r_h \mid \phi_h(x_h),a_h)\Plath(s_{h+1} \mid \phi_h(x_h),a_h)$ and $\phi_{h+1}\sharp\Mstarobsh(r_h,s_{h+1}\mid x_h,a_h)=\Rstarobsh(r_h \mid x_h,a_h)\phi_{h+1}\sharp\Pstarobsh(s_{h+1} \mid x_h,a_h)$)  as well as the bound $\nrm*{p - q}^2_\tv \leq \Dhels{p}{q}$.
Summing this over $t,h$ and using a standard decomposition for regret
(\cref{lem:lemma1jiang}) gives:
\begin{align}
&\sum_{t=1}^T \crl*{\En^{\hatpi\ind{t}_\lat\circ\phi\ind{t}}\brk*{\brk{Q\sub{\lat\ind{t}}^{\hatpi\ind{t}_\lat} \circ \phi\ind{t}}_1(x_1,a_1)} - J^{\Mstarobs}(\hatpi\ind{t}_\obs)} \nonumber \\
&\quad = \sum_{t=1}^T\sum_{h=1}^H \En^{\hatpi\ind{t}_\lat\circ\phi\ind{t}}\brk*{\brk{Q\sub{\lat\ind{t}}^{\hatpi\ind{t}_\lat} \circ \phi\ind{t}}_h(x_h,a_h) - \cT^{\Mstarobs,\hatpi\ind{t}_\lat \circ \phi\ind{t}}_h\brk{Q\sub{\lat\ind{t}}^{\hatpi\ind{t}_\lat} \circ \phi\ind{t}}_{h+1}(x_h,a_h)} \tag{\cref{lem:lemma1jiang}} \\
&\quad \leq \sqrt{TH} \sqrt{\sum_{t=1}^T \sum_{h=1}^H  \En^{\hatpi\ind{t}_\lat\circ\phi\ind{t}}\brk*{\prn*{\brk{Q\sub{\lat\ind{t}}^{\hatpi\ind{t}_\lat} \circ \phi\ind{t}}_h(x_h,a_h) - \cT^{\Mstarobs,\hatpi\ind{t}_\lat \circ \phi\ind{t}}_h\brk{Q\sub{\lat\ind{t}}^{\hatpi\ind{t}_\lat} \circ \phi\ind{t}}_{h+1}(x_h,a_h)}^2}} \nonumber \\
&\quad \leq \sqrt{4TH} \sqrt{\sum_{t=1}^T \sum_{h=1}^H  \En^{\hatpi\ind{t}_\lat\circ\phi\ind{t}}\brk*{\Dhels{\brk*{M\ind{t}_{\lat,h} \circ \phi\ind{t}\sub{h}}(x_h,a_h)}{\phi\ind{t}\sub{h+1}\sharp\Mstarobsh(x_h,a_h)}}} \tag{By \eqref{eq:qpi-minus-tqpi}} \\
&\quad \leq \sqrt{4TH}\vepsrep \nonumber.
\end{align}

Returning to the decomposition of \eqref{eq:risk-decomp} and combining everything gives:
\begin{align*}
	\En\brk*{\Riskobs} &\leq \frac{1}{T}\crl*{\sum_{t=1}^T \En\brk*{J^{\Mstarlat}(\pi_{\Mstarlat}) - J^{\Mlat\ind{t}}(\pi_{M\ind{t}_\lat})}} + \frac{1}{T}\prn*{\alpha\sqrt{T} + 4\sqrt{TH}}\En\brk{\vepsrep} + c_1\cdot\Riskbase(K)\\
		&\leq \frac{1}{T} \crl*{\sum_{t=1}^T \En\brk*{J(\pistar) - J^{\Mlat\ind{t}}(\pi_{\Mlat\ind{t}}) + \gamma \vepsrep^2}} + \frac{\gammainv}{T} \prn*{\alpha\sqrt{T} + 4\sqrt{TH}}^2 + c_1\cdot\Riskbase(K)\\
		&\leq \gamma \frac{2K}{T} \Estsimopt(T,\gamma)  + 2\gammainv \prn*{\alpha^2 + 16H} + c_1\cdot\Riskbase(K),
\end{align*}

where the second inequality follows by AM-GM applied to the middle term and the third inequality follows from: i) Jensen's inequality, ii) \cref{ass:optim-replearn} applied to the distributions $\bar{p}\ind{t}_\obs = \frac{1}{K}\sum_{k=1}^K p\ind{t,k}_\obs$, iii) the bound $K+1 \leq 2K$, and  iv) the inequality $(x+y)^2 \leq 2(x^2+y^2)$. 
\end{proof}

\subsection{Proofs for \cref{sec:crobustexamples}: Examples of \Crobust Algorithms}\label{app:crobust-examples}

\latentgolfcrobust*

\begin{proof}[\pfref{thm:latentgolfcrobust}] Recall that the agent is observing data from the $\phi$-compressed POMDP $\Mphi$, and thus the datasets are of the form $\cD_{h}\ind{k} = \cD_{\phi,h}\ind{k} = \{ \phi(x\ind{i}_h),a\ind{i}_h,r\ind{i}_h,\phi(x\ind{i}_{h+1})\}_{i=1}^{k-1}$. For any $\pilat \in \Pilat$, we define 
\[
\Tphipilath f(s_h,a_h) = \rphipih(s_h,a_h) + \En_{s' \sim \Pphipih(s_h,a_h)}\brk*{f(s')},
\]
where $\rphipih$ and $\Pphipih$ are the policy-dependent Markov operators defined in \eqref{eq:Pphi-pomdp-definition} and \eqref{eq:rphi-pomdp-definition}.

As a consequence, we observe the following misspecification guarantee for $\cT_\lat$.
\begin{lemma}[Misspecification guarantee for
  $\Tlat$]\label{lem:latent-golf-misspecification}
\[
	\forall f: \cS \times \cA \rightarrow [0,1]: \quad \sum_{k=1}^K \sum_{h=1}^H \wtE^{\pi\ind{k}}_\phi \brk*{\prn*{\Tlath f(s_h,a_h) - \wtT^{\pi\ind{k}}_{\phi,h}f(s_h,a_h)}^2} \leq \cO(\vepsrep^2).
\]
\end{lemma}
\begin{proof}[\pfref{lem:latent-golf-misspecification}]
	Follows from \cref{ass:corruption-level-golf} and the definitions of $\wtT^{\pi\ind{k}}_{\phi,h}$ and $\Tlath$.
\end{proof}

We recall that \cref{ass:corruption-level-golf} implies the following simulation lemma from $\Mlat$ to $\Mphi$. 
\begin{lemma}[Simulation lemma]\label{lem:mlat-to-pomdp-simulation}
	\[
	\forall f: \cS \times \cA \rightarrow [0,1] \quad 	\abs*{\En^{\Mlat,\pilat}\brk*{f(s_h,a_h)} - \Ephipilat\brk*{f(s_h,a_h)}} \leq \sum_{h' < h} \Ephipilat\brk*{\nrm*{P_{\lat,h'}(s_{h'},a_{h'}) - \wtP^{\pilat}_{\phi,h'}(s_{h'},a_{h'})}_\tv} 
	\]
	and thus %
		\[
	\sum_{k=1}^K \sum_{h=1}^K \En^{\Mlat,\pi\ind{k}}\brk*{f(s_h,a_h)} \leq \sum_{k=1}^K \sum_{h=1}^K \Ephipi{k}\brk*{f(s_h,a_h)} + H^{3/2}\sqrt{K\vepsrep^2}.
	\]
\end{lemma}

We begin with the following lemmas, which will be proved in the sequel. 

\begin{lemma}[Optimism]\label{lem:latent-golf-optimism}
  For the choice of $\beta$ in \cref{thm:latentgolfcrobust}, with probability at least $1-\delta$, we have that for all $k \in [K]$:
	\[
	\Qstarlat \in \cF\ind{k}.
	\]
\end{lemma}

\begin{lemma}[Small in-sample squared Bellman errors]\label{lem:latent-golf-in-sample}
With probability at least $1-\delta$, we have that for all $k \in [K]$, $h \in [H]$, and $f \in \cF\ind{k}$:
	\[
		\sum_{i=1}^{k-1} \Ephipi{i}\brk*{\prn*{f(s_h,a_h) - \Tphipih{i} f(s_h,a_h)}^2} \leq \cO(\beta).
	\]
\end{lemma}

Let us write $\pi\ind{k}_\obs \coloneqq \pi\ind{k} \circ \phi$. Let us introduce the shorthand $\wtd_{\obs,h}\ind{k} \coloneqq \sum_{i=1}^{k-1} d^{\pi\ind{k}_\obs}_{\obs,h}$, where $d^\pi_\obs$ is the occupancy for $\Mstarobs$, and also the burn-in time
\[
\kappa_h(x,a) \coloneqq \min\crl*{k:  \sum_{i=1}^{k-1} d^{\pi\ind{k}}_{\obs,h}(x,a) \geq \Ccov \mustarh(x,a)}.
\]	

Let us recall, from the analysis of \cite{xie2022role}, that for any $h \in [H]$ and $f: \cS \times \cA \rightarrow [0,1]$ we have
\begin{equation}\label{eq:burn-in-time}
	\sum_{k=1}^K \En^{\pi\ind{k}}\brk*{f(s_h,a_h) \indic\crl*{k < \kappa_h(s_h,a_h)}} \leq 2\Ccov,
\end{equation}
as well as 
\begin{equation}\label{eq:cov-potential}
\sum_{h=1}^H \sum_{k=1}^K \sum_{s,a} \frac{(d^{
\pi\ind{k}_\obs}_{h}(x,a)\indic\{k \geq \kappa_h(x,a)\})^2}{\wtd\ind{k}_{h}(x,a)} \leq \bigoh\prn*{H\Ccov\log(K)}.
\end{equation}

\begin{align*}
\sum_k J^{\Mlat}(\pi_{\Mlat}) - J^{\Mlat}(\pi\ind{k}) &\leq \sum_{k=1}^K\sum_{h=1}^H \En^{\Mlat,\pi\ind{k}}\brk*{f\ind{k}(s_h,a_h) - \cT_\lat f\ind{k}(s_h,a_h)} \tag{Optimism (\cref{lem:latent-golf-optimism})}\\
&\leq \sum_{k=1}^K\sum_{h=1}^H \wtE^{\pi\ind{k}}_\phi\brk*{f\ind{k}(s_h,a_h) - \cT_\lat f\ind{k}(s_h,a_h)} + H^{3/2}\sqrt{K\vepsrep^2}\tag{\cref{lem:mlat-to-pomdp-simulation}}\\
&= \sum_{k=1}^K\sum_{h=1}^H \En^{\pi\ind{k}\circ \phi}\brk*{\brk*{\prn{f\ind{k} - \cT_\lat f\ind{k}}\circ \phi}(x_h,a_h)} + H^{3/2}\sqrt{K\vepsrep^2} \tag{\cref{lem:change-of-measure}} \\
&\leq \sum_{k=1}^K\sum_{h=1}^H \En^{\pi\ind{k}\circ \phi}\brk*{\brk*{\prn{f\ind{k} - \cT_\lat f\ind{k}}\circ \phi}(x_h,a_h)\indic\{k \geq \kappa_h(x_h,a_h)\}} \\
	&\qquad + 2H\Ccov + H^{3/2}\sqrt{K\vepsrep^2} \tag{burn-in time \eqref{eq:burn-in-time}} \\
	&\leq \underbrace{\sum_{k=1}^K\sum_{h=1}^H \En^{\pi\ind{k}\circ \phi}\brk*{\brk*{\prn{f\ind{k} - \Tphipikh f\ind{k}}\circ \phi}(x_h,a_h)\indic\{k \geq \kappa_h(x_h,a_h)\}}}_{\mathrm{(I)}} \\
	&\quad +  \underbrace{\sum_{k=1}^K\sum_{h=1}^H \En^{\pi\ind{k}\circ \phi}\brk*{\brk*{\prn{\Tphipikh f\ind{k}- \Tlath f\ind{k}}\circ \phi}(x_h,a_h)}}_{\mathrm{(II)}} + 2H\Ccov+ H^{3/2}\sqrt{K\vepsrep^2}
\end{align*}

Note that, by change of measure (\cref{lem:change-of-measure}) and the misspecification guarantee (\cref{lem:latent-golf-misspecification}), the second term is bounded by:
\[
	\mathrm{(II)} = \sum_{k=1}^K\sum_{h=1}^H\Ephipi{k}\brk*{\prn{\Tphipikh f\ind{k}- \Tlath f\ind{k}}(s_h,a_h)} \leq \sqrt{KH\vepsrep^2}.
\]
Turning to the first term, we have:
\begin{align}
	&\sum_{h=1}^H \sum_{k=1}^K \En^{\pi\ind{k}_\obs}\brk*{\brk*{\prn{f\ind{k} - \Tphipikh f\ind{k}}\circ \phi}(x_h,a_h)\indic\{k \geq \kappa_h(x_h,a_h)\}} \\
	&\quad\quad \leq \sqrt{\sum_{h=1}^H \sum_{k=1}^K \sum_{x,a} \frac{(d^{
\pi\ind{k}_\obs}_{h}(x,a)\indic\{k \geq \kappa_h(x,a)\})^2}{\wtd\ind{k}_{h}(x,a)} }\sqrt{\sum_{h=1}^H \sum_{k=1}^K  \En_{\wtd\ind{k}_\obs}\brk*{\prn*{\prn{f\ind{k} - \Tphipikh f\ind{k}}\circ \phi}^2(x_h,a_h)}} \\
	&\quad\quad \leq \sqrt{H\Ccov\log(K)}\sqrt{\sum_{h=1}^H \sum_{k=1}^K  \En_{\wtd\ind{k}_\obs}\brk*{\prn*{\prn{f\ind{k} - \Tphipikh f\ind{k}}\circ \phi}^2(x_h,a_h)}}\tag{coverability potential \eqref{eq:cov-potential}} \\
	&\quad\quad = \sqrt{H\Ccov\log(K)}\sqrt{\sum_{h=1}^H \sum_{k=1}^K \sum_{i=1}^{k-1} \wtE^{\pi\ind{i}}_{\phi}\brk*{\prn*{f\ind{k}(s_h,a_h) - \Tphipikh f\ind{k}(s_h,a_h)}^2}} \tag{change of measure, \cref{lem:change-of-measure}} \\
	&\quad\quad \leq \cO\prn*{H\sqrt{\Ccov K \log(K)\beta}},\label{eq:latent-golf-final-on-policy}
\end{align}

where we have used that, from \cref{lem:latent-golf-in-sample}, we have: 
\[
\sum_{h=1}^H \sum_{k=1}^K \sum_{i=1}^{k-1} \wtE^{\pi\ind{i}}_\phi\brk*{\prn*{f\ind{k}(s_h,a_h) - \wtT^{\pi\ind{i}}_\phi f\ind{k}(s_h,a_h)}^2} \leq \cO(\beta H K).
\]
This gives an upper bound on the regret of
\[
\sum_{k=1}^K J^{\Mlat}(\pi_{\Mlat}) - J^{\Mlat}(\pi\ind{k}) \leq \cO\prn*{H\sqrt{\Ccov K \log(K)\beta} +  H^{3/2}\sqrt{K\vepsrep^2}}.
\]
Using that $\beta = \cO\prn*{\log(\frac{\abs{\cF}\abs{\cG}HK}{\delta}) + \vepsrep}$%
	and simplifying gives
\[
	\sum_{k=1}^K J^{\Mlat}(\pi_{\Mlat}) - J^{\Mlat}(\pi\ind{k}) \leq \cO\prn*{H\sqrt{\Ccov K \log(K)\log(\nicefrac{\abs{\cF}\abs{\cG}HK}{\delta})}} + \cO\prn*{H^{3/2}\sqrt{K \Ccov \log(K) \vepsrep^2}},
\]
as desired. It only remains to establish the concentrations results. 

\paragraph{Concentration analysis.}

We establish the concentration results of \cref{lem:latent-golf-optimism} and \cref{lem:latent-golf-in-sample}. 

\begin{proof}[\pfref{lem:latent-golf-in-sample}]
	Let 
	\[
	X_k(h,f) = \prn*{f_h(s\ind{k}_h,a\ind{k}_h) - r\ind{k}_h - f_{h+1}(s\ind{k}_{h+1})}^2- \prn*{\wtT^{\pi\ind{k}}_\phi f_h(s\ind{k}_h,a\ind{k}_h) - r\ind{k}_h - f_{h+1}(s\ind{k}_{h+1})}^2.
	\]
	Let $\fF_{k,h} = \{s\ind{i}_1,a\ind{i}_1,r\ind{i}_1,\ldots,s\ind{i}_H,a\ind{i}_H,r\ind{i}_H\}_{i=1}^k$. Note that
	\begin{align*}
		 \En\brk*{r\ind{k}_h + f_{h+1}(s\ind{k}_{h+1}) \mid \fF_{k,h}} &= \En\brk*{r\ind{k}_h + f_{h+1}(s\ind{k}_{h+1}) \mid \pi\ind{k}} \\
		 	&= \En\brk*{\En\brk*{r\ind{k}_h + f_{h+1}(s\ind{k}_{h+1}) \mid s\ind{k}_h, a\ind{k}_h,\pi\ind{k}} \mid \pi\ind{k}} \\
		 	&= \En\brk*{\wtT^{\pi\ind{k}}_\phi f(s\ind{k}_h,a\ind{k}_h) \mid \pi\ind{k}}\\
		 	&= \wtE^{\pi\ind{k}}_\phi\brk*{\wtT^{\pi\ind{k}}_\phi f(s_h,a_h)},
	\end{align*}
	and thus that
	\[
	\En\brk*{X_k(h,f) \mid \fF_{k,h}} = \wtE^{\pi\ind{k}}_\phi\brk*{\prn*{f_h(s_h,a_h) - \wtT^{\pi\ind{k}}_\phi f_h(s_h,a_h)}^2}.
	\]
	Next, note that
	\begin{align*}
			\Var\brk*{X_k(h,f) \mid \fF_{k,h}} &\leq \En\brk*{\prn*{X_k(h,f)}^2 \mid \fF_{k,h}} \\
			&\leq 16\En\brk*{\prn*{f_h(s\ind{k}_h,a\ind{k}_h) - \wtT^{\pi\ind{k}}_\phi f_h(s\ind{k}_h,a\ind{k}_h)}^2\mid \fF_{k,h}} \\
			&= 16\En\brk*{X_k(h,f) \mid \fF_{k,h}}.
	\end{align*}
	By Freedman's inequality (\cref{lem:freedman}, \cref{lem:multiplicative_freedman}), we have that with probability at least $1-\delta$:
	\[
		\abs*{\sum_{t<k} X_t(h,f) - \sum_{t<k} \En\brk*{X_t(h,f) \mid \fF_{t,h}}} \leq \cO\prn*{\sqrt{\log(1/\delta)\sum_{t<k}\En\brk*{X_t(h,f) \mid \fF_{t,h}}} + \log(1/\delta)}
	\]
	Taking a union bound over $[K] \times [H] \times \cF$, we have that for all $k,h,f$, with probability at least $1-\delta$:
	\begin{align}\label{eq:x-t-minus-bellman-errors} 
		&\abs*{\sum_{t<k} X_t(h,f) - \sum_{t<k} \wtE^{\pi\ind{k}}_\phi\brk*{\prn*{f_h(s_h,a_h) - \wtT^{\pi\ind{k}}_\phi f_h(s_h,a_h)}^2}} \\
		&\quad\leq \cO\prn*{\sqrt{\iota\sum_{t<k}\wtE^{\pi\ind{k}}_\phi\brk*{\prn*{f_h(s_h,a_h) - \wtT^{\pi\ind{k}}_\phi f_h(s_h,a_h)}^2}} + \iota},
	\end{align}
	where $\iota = \log(\abs*{\cF}HK/\delta)$. We now show that 
	\begin{equation}\label{eq:latent-golf-x-t-bounded}
		\sum_{t<k} X_t(h,f\ind{k}) \leq \beta + \cO\prn*{\vepsrep + \iota} = \cO\prn*{\beta},
	\end{equation}
	which will imply, from \eqref{eq:x-t-minus-bellman-errors}, that
	\[
	\sum_{t<k} \wtE^{\pi\ind{k}}_\phi\brk*{\prn*{f_h(s_h,a_h) - \wtT^{\pi\ind{k}}_\phi f_h(s_h,a_h)}^2} \leq \cO\prn*{\iota + \beta} = \cO(\beta),
	\]
	as desired. To see \eqref{eq:latent-golf-x-t-bounded}, let 
	\[
	 \Delta_k = \sum_{t<k} \prn*{\cT_\lat f\ind{k}_h(s\ind{t}_h,a\ind{t}_h) - r\ind{t}_h - f\ind{k}_{h+1}(s\ind{t}_{h+1})}^2	- \prn*{\wtT^{\pi\ind{t}}_\phi f\ind{k}_h(s\ind{t}_h,a\ind{t}_h) - r\ind{t}_h - f\ind{k}_{h+1}(s\ind{t}_{h+1})}^2
	\]
	and then note that:
	\begin{align*}
	\sum_{t<k} X_t(h,f\ind{k}) &= \sum_{t<k} \prn*{f\ind{k}_h(s\ind{t}_h,a\ind{t}_h) - r\ind{t}_h - f\ind{k}_{h+1}(s\ind{t}_{h+1})}^2- \prn*{\wtT^{\pi\ind{t}}_\phi f\ind{k}_h(s\ind{t}_h,a\ind{t}_h) - r\ind{t}_h - f\ind{k}_{h+1}(s\ind{t}_{h+1})}^2	\\
	&= \sum_{t<k} \prn*{f\ind{k}_h(s\ind{t}_h,a\ind{t}_h) - r\ind{t}_h - f\ind{k}_{h+1}(s\ind{t}_{h+1})}^2- \prn*{\cT_\lat f\ind{k}_h(s\ind{t}_h,a\ind{t}_h) - r\ind{t}_h - f\ind{k}_{h+1}(s\ind{t}_{h+1})}^2	 + \Delta_k \\ 
	&\leq \sum_{t<k} \prn*{f\ind{k}_h(s\ind{t}_h,a\ind{t}_h) - r\ind{t}_h - f\ind{k}_{h+1}(s\ind{t}_{h+1})}^2- \inf_{g_h \in \cG_h}\sum_{t<k} \prn*{g(s\ind{t}_h,a\ind{t}_h) - r\ind{t}_h - f\ind{k}_{h+1}(s\ind{t}_{h+1})}^2 + \Delta_k\\ 
	&\leq \beta + \Delta_k.
	\end{align*}
where the second-to-last line follows from $\cT_\lat \cF \subseteq \cG$ and the last line follows from the definition of the confidence set. It remains to show that $\Delta_k \leq \cO(\vepsrep + \iota)$, which we do via a similar concentration argument. Namely, let
	\[
		Y_t(h,f) = \prn*{\cT_\lat f_h(s\ind{t}_h,a\ind{t}_h) - r\ind{t}_h - f\ind{k}_{h+1}(s\ind{t}_{h+1})}^2	- \prn*{\wtT^{\pi\ind{t}}_\phi f_h(s\ind{t}_h,a\ind{t}_h) - r\ind{t}_h - f\ind{k}_{h+1}(s\ind{t}_{h+1})}^2,
	\]
	and note that, as before,
	\[
		\En\brk*{Y_t(h,f) \mid \fF_{t,h}} = \wtE^{\pi\ind{t}}_\phi\brk*{\prn*{\cT_\lat f_h(s_h,a_h) - \wtT^{\pi\ind{t}}_\phi f_h(s_h,a_h)}^2},
	\]
	and 
	\[
		\Var\brk*{Y_t(h,f) \mid \fF_{t,h}} \leq 16\En\brk*{Y_t(h,f) \mid \fF_{t,h}},
	\]
	by the same calculation as earlier.
	Thus, by Freedman's inequality and a union bound, we have that, with probability at least $1-\delta$,
		\begin{align}\label{eq:y-t-minus-bellman-errors} 
			&\abs*{\sum_{t<k} Y_t(h,f) - \sum_{t<k} \wtE^{\pi\ind{k}}_\phi\brk*{\prn*{\cT_\lat f_h(s_h,a_h) - \wtT^{\pi\ind{k}}_\phi f_h(s_h,a_h)}^2}} \\
			&\quad \leq \cO\prn*{\sqrt{\iota\sum_{t<k}\wtE^{\pi\ind{k}}_\phi\brk*{\prn*{\cT_\lat f_h(s_h,a_h) - \wtT^{\pi\ind{k}}_\phi f_h(s_h,a_h)}^2}} + \iota},
	\end{align}
	where $\iota = \log(\abs*{\cF}HK/\delta)$. Recalling the misspecification assumption \cref{lem:latent-golf-misspecification}, this implies that
	\[
	\sum_{t<k} Y_t(h,f) \leq \cO\prn*{\vepsrep + \iota},
	\]
	for all $h,f,k$, with high probability. Applying this to $f = f\ind{k}$ concludes the result.
	\end{proof}
	
	\begin{proof}[\pfref{lem:latent-golf-optimism}]%
          We use similar arguments to the preceding lemma. Let $\Qstarlath \coloneqq Q^\star_{\Mlat,h}$. The aim is to show that, for all $h \in [H], k \in [K], g \in \cG$, we have:
	\[
		\sum_{t<k} \prn*{g(s\ind{t}_h,a\ind{t}_h) - r\ind{t}_h - \Qstarlathpo(s\ind{t}_{h+1})}^2 - \prn*{\Qstarlath(s\ind{t}_h,a\ind{t}_h) - r\ind{t}_h - \Qstarlath(s\ind{t}_{h+1})}^2 \geq -\beta,
	\]
	from which the conclusion will follow. We show that
	\begin{equation}\label{eq:latent-golf-optimism-1}
	\sum_{t<k} \underbrace{\prn*{g(s\ind{t}_h,a\ind{t}_h) - r\ind{t}_h - \Qstarlathpo(s\ind{t}_{h+1})}^2 - \prn*{\wtT^{\pi\ind{t}}_\phi \Qstarlath(s\ind{t}_h,a\ind{t}_h) - r\ind{t}_h - \Qstarlath(s\ind{t}_{h+1})}^2}_{\coloneqq W_t(h,g)} \geq -\beta/2,
	\end{equation}
	and also that
	\begin{equation}\label{eq:latent-golf-optimism-2}
		\sum_{t<k} \underbrace{\prn*{\wtT^{\pi\ind{t}}_\phi \Qstarlath(s\ind{t}_h,a\ind{t}_h) - r\ind{t}_h - \Qstarlath(s\ind{t}_{h+1})}^2  - \prn*{\Qstarlath(s\ind{t}_h,a\ind{t}_h) - r\ind{t}_h - \Qstarlath(s\ind{t}_{h+1})}^2}_{\coloneqq V_t(h)} \geq -\beta/2.
	\end{equation}
	
	For \eqref{eq:latent-golf-optimism-1}, note that 
	\begin{equation}\label{eq:non-neg-w-t}
		\En\brk*{ W_t(h,g) \mid \cF_{t,h}} = \wtE^{\pi\ind{t}}_\phi\brk*{\prn*{g_h(s_h,a_h) - \wtT^{\pi\ind{t}}_{\phi,h}\Qstarlath(s_h,a_h)}^2},
	\end{equation}
	and that $\Var\brk*{W_t(h,g) \mid \cF_{t,h}} \leq 16\En\brk*{ W_t(h,g) \mid \cF_{t,h}}$. By Freedman, this gives
	\[
	\abs*{ \sum_{t<k} W_t(h,g) - \sum_{t<k} \En\brk*{W_t(h,g) \mid \fF_{t,h}}} \leq \cO\prn*{\sqrt{\iota\sum_{t<k}\En\brk*{ W_t(h,g) \mid \cF_{t,h}}} + \iota} \leq \frac{1}{2}\En\brk*{ W_t(h,g) \mid \cF_{t,h}} + \cO(\iota),
	\]
	or in other words
	\[
		\sum_{t<k} W_t(h,g) \geq \frac{1}{2}\sum_{t<k} \En\brk*{W_t(h,g) \mid \fF_{t,h}} - \cO(\iota) \geq -\cO(\iota), 
	\]
	using the non-negativity of \eqref{eq:non-neg-w-t}.
	For \eqref{eq:latent-golf-optimism-2}, note that
	\begin{equation}\label{eq:non-neg-v-t}
		\En\brk*{V_t(h) \mid \fF_{t,h}} = - \wtE^{\pi\ind{t}}_\phi\brk*{\prn*{\Tlath \Qstarlath - \wtT^{\pi\ind{t}}_{\phi,h} \Qstarlath}^2} \geq -\vepsrep,
	\end{equation}
	and that $\Var\brk*{V_t(h) \mid \fF_{t,h}} \leq 16\wtE^{\pi\ind{t}}_\phi\brk*{\prn*{\Tlath \Qstarlath - \wtT^{\pi\ind{t}}_{\phi,h} \Qstarlath}^2}$.
	By Freedman, this gives
	\begin{align}
	\abs*{ \sum_{t<k} V_t(h) - \sum_{t<k} \En\brk*{V_t(h) \mid \fF_{t,h}}} &\leq \cO\prn*{\sqrt{\iota\sum_{t<k}\wtE^{\pi\ind{t}}_\phi\brk*{\prn*{\cT_\lat \Qstarlathpo(s_h,a_h) - \wtT^{\pi\ind{t}}_{\phi,h} \Qstarlathpo(s_h,a_h)}^2}} + \iota} \\
		&= \cO\prn*{\vepsrep + \iota},
	\end{align}
	or in other words
	\[
		\sum_{t<k} V_t(h) \geq \sum_{t<k} \En\brk*{V_t(h) \mid \fF_{t,h}} - \cO\prn*{\vepsrep + \iota} \geq -\cO(\vepsrep + \iota),
	\]
	where we have used \eqref{eq:non-neg-v-t}.

\end{proof}
	
	\end{proof}

\newpage

\neurips{\section*{NeurIPS Paper Checklist}

\begin{enumerate}

\item {\bf Claims}
    \item[] Question: Do the main claims made in the abstract and introduction accurately reflect the paper's contributions and scope?
    \item[] Answer: \answerYes{} %
    \item[] Justification: All results in this paper are of a theoretical nature, and the stated contributions in the abstract and introduction are given in a precise, formal language.

\item {\bf Limitations}
    \item[] Question: Does the paper discuss the limitations of the work performed by the authors?
    \item[] Answer: \answerYes{} %
    \item[] Justification: All results in this paper are of a theoretical nature -- we precisely state the conditions under which our results hold. %

\item {\bf Theory Assumptions and Proofs}
    \item[] Question: For each theoretical result, does the paper provide the full set of assumptions and a complete (and correct) proof?
    \item[] Answer: \answerYes{} %
    \item[] Justification: Each theoretical result is stated with all necessary assumptions and is accompanied by complete (and correct) proofs. %

    \item {\bf Experimental Result Reproducibility}
    \item[] Question: Does the paper fully disclose all the information needed to reproduce the main experimental results of the paper to the extent that it affects the main claims and/or conclusions of the paper (regardless of whether the code and data are provided or not)?
    \item[] Answer: \answerNA{} %
    \item[] Justification:  The paper does not include experiments. %

\item {\bf Open access to data and code}
    \item[] Question: Does the paper provide open access to the data and code, with sufficient instructions to faithfully reproduce the main experimental results, as described in supplemental material?
    \item[] Answer: \answerNA{} %
    \item[] Justification: The paper does not include experiments requiring code. %

\item {\bf Experimental Setting/Details}
    \item[] Question: Does the paper specify all the training and test details (e.g., data splits, hyperparameters, how they were chosen, type of optimizer, etc.) necessary to understand the results?
    \item[] Answer: \answerNA{} %
    \item[] Justification: The paper does not include experiments. %

\item {\bf Experiment Statistical Significance}
    \item[] Question: Does the paper report error bars suitably and correctly defined or other appropriate information about the statistical significance of the experiments?
    \item[] Answer: \answerNA{} %
    \item[] Justification: The paper does not include experiments. %

\item {\bf Experiments Compute Resources}
    \item[] Question: For each experiment, does the paper provide sufficient information on the computer resources (type of compute workers, memory, time of execution) needed to reproduce the experiments?
    \item[] Answer: \answerNA{} %
    \item[] Justification: The paper does not include experiments.%
    
\item {\bf Code Of Ethics}
    \item[] Question: Does the research conducted in the paper conform, in every respect, with the NeurIPS Code of Ethics \url{https://neurips.cc/public/EthicsGuidelines}?
    \item[] Answer: \answerYes{} %
    \item[] Justification: The research conducted in the paper conforms with the NeurIPS Code of Etichs. %

\item {\bf Broader Impacts}
    \item[] Question: Does the paper discuss both potential positive societal impacts and negative societal impacts of the work performed?
    \item[] Answer: \answerNA{} %
    \item[] Justification: This is a primarily theoretical work.
    
\item {\bf Safeguards}
    \item[] Question: Does the paper describe safeguards that have been put in place for responsible release of data or models that have a high risk for misuse (e.g., pretrained language models, image generators, or scraped datasets)?
    \item[] Answer: \answerNA{} %
    \item[] Justification: This is a purely theoretical work, and as such poses no such risks. %

\item {\bf Licenses for existing assets}
    \item[] Question: Are the creators or original owners of assets (e.g., code, data, models), used in the paper, properly credited and are the license and terms of use explicitly mentioned and properly respected?
    \item[] Answer: \answerNA{} %
    \item[] Justification: The paper does not use existing assets. %

\item {\bf New Assets}
    \item[] Question: Are new assets introduced in the paper well documented and is the documentation provided alongside the assets?
    \item[] Answer: \answerNA{} %
    \item[] Justification: The paper does not release new assets. %

\item {\bf Crowdsourcing and Research with Human Subjects}
    \item[] Question: For crowdsourcing experiments and research with human subjects, does the paper include the full text of instructions given to participants and screenshots, if applicable, as well as details about compensation (if any)? 
    \item[] Answer: \answerNA{} %
    \item[] Justification: The paper does not involve crowdsouring nor research with human subjects. %

\item {\bf Institutional Review Board (IRB) Approvals or Equivalent for Research with Human Subjects}
    \item[] Question: Does the paper describe potential risks incurred by study participants, whether such risks were disclosed to the subjects, and whether Institutional Review Board (IRB) approvals (or an equivalent approval/review based on the requirements of your country or institution) were obtained?
    \item[] Answer: \answerNA{} %
    \item[] Justification: The paper does not involve crowdsourcing nor research with human subjects. %

\end{enumerate}

}

\end{document}